\documentclass[runningheads]{llncs}

\usepackage[utf8]{inputenc}
\inputencoding{utf8}

\usepackage{booktabs}
\usepackage{graphicx}
\usepackage[hidelinks]{hyperref}
\usepackage{amsmath,amssymb}
\usepackage{acronym}
\usepackage[capitalise]{cleveref}
\usepackage{csquotes}
\usepackage{tikz}
\usetikzlibrary{arrows,arrows.meta,automata,backgrounds,calc,patterns,positioning,shapes,shadows}
\usepackage{subcaption}
\usepackage{algorithm}
\makeatletter
\let\If\@undefined
\let\Else\@undefined
\let\ElsIf\@undefined
\let\EndIf\@undefined
\let\For\@undefined
\let\ForAll\@undefined
\let\ForEach\@undefined
\let\EndFor\@undefined
\let\While\@undefined
\let\EndWhile\@undefined
\let\Repeat\@undefined
\let\Until\@undefined
\let\State\@undefined
\let\Return\@undefined
\let\Continue\@undefined
\let\Break\@undefined
\makeatother
\usepackage[noend]{algpseudocode}
\algnewcommand{\algorithmicbreak}{\textbf{break}}
\algnewcommand{\algorithmiccontinue}{\textbf{continue}}
\algnewcommand{\algorithmicforeach}{\textbf{for each}}
\algnewcommand\Break{\algorithmicbreak}
\algnewcommand\Continue{\algorithmiccontinue}
\algnewcommand{\LeftComment}[1]{\Statex \(\triangleright\) #1}
\algdef{SE}[FOR]{ForEach}{EndForEach}[1]
	{\algorithmicforeach\ #1\ \algorithmicdo}
	{\algorithmicend\ \algorithmicforeach}
\algtext*{EndForEach} 
\algdef{SE}[DOWHILE]{DoWhile}{EndDoWhile}{\algorithmicdo}[1]{\algorithmicwhile\ #1}

\usepackage{amsthm}
\usepackage{thmtools}
\usepackage{thm-restate}
\acrodef{acp}[ACP]{advanced colour passing}
\acrodef{crv}[CRV]{counting randvar}
\acrodef{decor}[DECOR]{detection of commutative factors}
\acrodef{fg}[FG]{factor graph}
\acrodef{lve}[LVE]{lifted variable elimination}
\acrodef{pfg}[PFG]{parametric factor graph}
\acrodef{rv}[randvar]{random variable}

\crefname{algorithm}{Alg.}{Algs.}
\Crefname{algorithm}{Algorithm}{Algorithms}
\crefname{appendix}{App.}{Apps.}
\Crefname{appendix}{Appendix}{Appendices}
\crefname{corollary}{Cor.}{Cors.}
\Crefname{corollary}{Corollary}{Corollaries}
\crefname{definition}{Def.}{Defs.}
\Crefname{definition}{Definition}{Definitions}
\crefname{equation}{Eq.}{Eqs.}
\Crefname{equation}{Equation}{Equations}
\crefname{example}{Ex.}{Exs.}
\Crefname{example}{Example}{Examples}
\crefname{proposition}{Prop.}{Props.}
\Crefname{proposition}{Proposition}{Propositions}
\crefname{section}{Sec.}{Secs.}
\Crefname{section}{Section}{Sections}
\crefname{theorem}{Thm.}{Thms.}
\Crefname{theorem}{Theorem}{Theorems}

\newcommand{\abs}[1]{\lvert #1 \rvert}

\newcommand{\range}[1]{\ensuremath{\mathrm{range}(#1)}}

\newcommand{\low}{\ensuremath{\mathrm{low}}}
\newcommand{\medium}{\ensuremath{\mathrm{medium}}}
\newcommand{\high}{\ensuremath{\mathrm{high}}}

\definecolor{myyellow}{RGB}{247,192,26}
\definecolor{myblue}{RGB}{37,122,164}
\definecolor{mygreen}{RGB}{78,155,133}
\definecolor{mypurple}{RGB}{86,51,94}

\definecolor{newblue}{RGB}{50,113,173}
\definecolor{newred}{RGB}{222,32,36}
\definecolor{newgreen}{RGB}{70,165,69}
\definecolor{newpurple}{RGB}{140,69,152}

\definecolor{cborange}{RGB}{230,159,0}
\definecolor{cbblue}{RGB}{30,136,229}
\definecolor{cbbluedark}{RGB}{46,37,133}
\definecolor{cbpurple}{RGB}{170,68,153}
\definecolor{cbgreen}{RGB}{0,77,64}
\definecolor{cbgreenlight}{RGB}{93,168,153}
\definecolor{cbbrown}{RGB}{126,41,84}

\pgfdeclarelayer{bg}
\pgfsetlayers{bg,main}

\tikzset{
	rv/.style={draw, ellipse},
	pf/.style={draw, rectangle, fill = gray!30},
	arc/.style = {->, >={[round,sep]Stealth}},
}

\newcommand\factor[6]{
	\node[pf, #1=#3 of #2, label={#4:{#5}}](#6) {};
}

\allowdisplaybreaks

\begin{document}

\title{Lifted Model Construction under Approximate Commutativity}

\author{
	Malte Luttermann\inst{1}\textsuperscript{$*$} \and
	Jan Speller\inst{2}\textsuperscript{$*$} \and
	Tanya Braun\inst{2} \and
	Marcel Gehrke\inst{1} \and
	Ralf Möller\inst{1}
}
\authorrunning{M. Luttermann, J. Speller, T. Braun, M. Gehrke, R. Möller}
\institute{
	Institute for Humanities-Centered Artificial Intelligence, University of Hamburg \\
	\email{\{malte.luttermann,marcel.gehrke,ralf.moeller\}@uni-hamburg.de} \and
	Data Science Group, University of Münster \\
	\email{\{jan.speller,tanya.braun\}@uni-muenster.de}
}

\maketitle

\renewcommand{\thefootnote}{$*$}
\footnotetext{These authors contributed equally to this work.}
\renewcommand{\thefootnote}{\arabic{footnote}}

\begin{abstract}
	Lifted inference algorithms enable scalable probabilistic inference even for large object domains by leveraging the indistinguishability of objects in a probability distribution.
    An essential prerequisite for constructing a lifted representation is to identify commutative factors, i.e., functions whose output values are invariant under permutations of a subset of their input values, in a potential-based factorisation.
    In practice, however, parameters learned from data inevitably deviate even if associated objects are indistinguishable, causing their corresponding factors to be only approximately commutative instead of being exactly commutative.
	We address this problem by introducing the concept of $\varepsilon$-commutativity, a relaxation of commutativity where output values are only approximately invariant under permutations of input values.
	Specifically, we show how $\varepsilon$-commutativity can be exploited for lifted model construction, downstream probabilistic inference, and prove strict bounds on the induced approximation error, thereby ensuring the practical applicability of lifted model construction while maintaining highly accurate query results. These theoretical guarantees are confirmed empirically, demonstrating comparable query accuracy at lower runtime.
\end{abstract}


\section{Introduction}
Uncertain reasoning is a fundamental task in artificial intelligence, where uncertainty is commonly represented through probability distributions over random variables.
Probabilistic graphical models such as \acp{fg} provide a compact way to encode joint distributions over a set of \acp{rv} by decomposing a distribution into products of local functions, called factors.
Even with a factorised representation, the problem of probabilistic inference (i.e., computing marginal or conditional distributions of \acp{rv}) is intractable in general as the computational cost grows exponentially with the number of \acp{rv}~\cite{Cooper1990a}.
To allow for tractable probabilistic inference whose complexity grows polynomially in the number of objects in the underlying domains, lifted inference algorithms exploit the indistinguishability of objects by using a representative for computations~\cite{Niepert2014a}.
Performing lifted inference, however, requires a lifted representation in which indistinguishable objects are grouped together.
The construction of such a lifted representation includes identifying factors that exhibit symmetries and can therefore be grouped together.
One important class are \emph{commutative factors} (that is, factors whose output values are invariant under permutations of a subset of their input values).
Whether a factor is commutative depends on the precise output values it maps its input values to, and herein lies a subtle but pervasive practical issue: When numbers are estimated from observational data, two values that should ideally coincide hardly ever do so up to the last digit.
Even in simple cases such as counted data, learning can always yield a distribution of values, which 
makes an otherwise commutative factor fail the strict definition of commutativity, even though the commutativity is clearly present from a modelling perspective.
Current state-of-the-art algorithms for detecting commutative factors have no means to tolerate such discrepancies.
In this paper, we address this problem by softening the underlying equality requirement, thereby enabling the detection of approximate commutativity and allowing for further compression of \acp{fg}.

\paragraph{Previous work.}
When dealing with noisy distributions, typical approximate techniques control the deviation on the level of the full (global) probability distribution based on classes of prior distributions~\cite{Berger1990a}, as done, e.g., by distortion models such as the constant odds ratio model, which bounds the multiplicative deviation from a nominal reference distribution~\cite{Miranda2019a}.
In contrast, we consider a factorised distribution and bound individual potentials directly via a local constant-odds-ratio-type neighbourhood on the level of a single factor.
Structurally, however, the overall global approximation follows the same guarantees as distortion models.
There has also been research on \emph{robustness} in graph learning against false measurements and outliers~\cite{Li2025a,Pfeifer2016a}, whereas our approach relaxes exact structural symmetry in factors for lifted inference.
The line of work on lifted probabilistic inference goes back to Poole~\cite{Poole2003a}, who introduced \acp{pfg} as a representation that interleaves first-order logic with probabilistic modelling and \ac{lve} as an inference algorithm operating on \acp{pfg}.
A substantial body of subsequent research has refined and extended \ac{lve} in various directions~\cite{Taghipour2013c}.
One of these refinements includes the introduction of so-called \acp{crv}~\cite{Milch2008a}, which count over the occurrences of range values in the input of a factor instead of explicitly listing possible combinations of input values, thereby exploiting the commutativity of factors.
The current state-of-the-art to construct a lifted representation such as a \ac{pfg} is the \ac{acp} algorithm~\cite{Luttermann2024a}.
A core ingredient of \ac{acp} is to determine, for every factor, a maximum sized subset of commutative arguments such that factors can be compressed by using \acp{crv} to count over the commutative arguments.
To determine maximum sized subsets of commutative arguments, the \ac{decor} algorithm~\cite{Luttermann2024f,Luttermann2026a} is the state-of-the-art.
Such lifted model construction techniques have so far focused on approximate factors rather than approximate commutativity~\cite{Luttermann2025c,Speller2025a}.

\paragraph{Our contributions.}
We introduce \emph{$\varepsilon$-commutativity}, an approximate variant of commutativity in which output values are allowed to deviate depending on a hyperparameter $\varepsilon$.
On the basis of $\varepsilon$-commutativity, we show how approximately commutative factors can be compressed by replacing groups of similar values with representative values, thereby simplifying the underlying distribution by giving events of a symmetry set the same value.
We further establish a strict bound on the approximation error induced by such a compression, guaranteeing that subsequent inference on the resulting lifted representation remains accurate.
Finally, we empirically demonstrate that the approximation error remains well below the theoretical bounds in practice, while the proposed algorithm substantially reduces runtime across a range of parameter settings.

\paragraph{Structure of this paper.}
The remainder of the paper is organised as follows.
We first introduce the necessary background on \acp{fg} and commutative factors therein.
We then formalise the concept of $\varepsilon$-commutativity, develop a compression scheme based on $\varepsilon$-commutative representatives, and prove important properties leading to a strict bound on the resulting approximation error.
Afterwards, we embed the compression of $\varepsilon$-commutative factors into the \ac{acp} algorithm for lifted model construction and empirically show that downstream probabilistic inference still produces highly accurate results.
Due to space constraints, detailed proofs are relegated to \cref{appendix:commutative_approx_proofs} and further experimental results to \cref{appendix:commutative_approx_further_results}.

\section{Background}
We briefly introduce the notion of an \ac{fg} and then formally define what a commutative factor is.
An \ac{fg} compactly encodes a joint probability distribution over \acp{rv} as a product of factors~\cite{Frey1997a,Kschischang2001a}.
In the following, we write $\range{R}$ to denote the values a \ac{rv} $R$ can take (called range) and $\mathcal{X}_{\boldsymbol{R}} :=\times_{R \in \boldsymbol{R}} \range{R}$ to denote the Cartesian product over the ranges of a set of \acp{rv} $\boldsymbol{R}$.

\begin{definition}[Factor Graph]
	An \emph{\ac{fg}} is a tuple $M = (\boldsymbol R, \allowbreak \boldsymbol F, \allowbreak \boldsymbol E, \allowbreak \boldsymbol \Phi)$ where $\boldsymbol R = \{R_1, \ldots, \allowbreak R_p\}$ is a set of \acp{rv}, $\boldsymbol F = \{f_1, \ldots, \allowbreak f_m\}$ is a set of factor nodes, and $\boldsymbol \Phi = \{\phi_1, \ldots, \allowbreak \phi_m\}$ is a set of function definitions (called factors).
	Each $\phi_j(\boldsymbol{R}_{\phi_j}) \in \boldsymbol \Phi$ defines a function $\phi_j \colon \mathcal{X}_{\boldsymbol{R}_{\phi_j}}\to \mathbb{R}_{\geq 0}$ over the Cartesian product of the ranges of its argument sequence of \acp{rv} $\boldsymbol{R}_{\phi_j} \subseteq \boldsymbol{R}$.
    The image of $\phi_j$ consists of non-negative real numbers (potentials), and each factor is non-trivial in the sense that it attains at least one positive value.
	For each pair of variable node $R_i \in \boldsymbol R$ and factor node $f_j \in \boldsymbol F$, there is an edge $\{ R_i, \allowbreak f_j \} \in \boldsymbol E \subseteq \{ \{R, f\} \mid R \in \boldsymbol R \land f \in \boldsymbol F \}$ if $R_i \in \boldsymbol{R}_{\phi_j}$.
	The full joint probability distribution encoded by $M$ for an assignment $(R_1 = r_1, \allowbreak \ldots, \allowbreak R_p = r_p)$ to the \acp{rv} in $\boldsymbol R$, abbreviated as $\boldsymbol R = \boldsymbol r$, is the normalised product over all factors in $M$:
	\begin{align*}
		P_M(\boldsymbol R = \boldsymbol r)
		&= \frac{1}{Z} \psi(\boldsymbol r)=\frac{1}{Z} \prod_{j = 1}^{m} \phi_j(\boldsymbol{R}_{\phi_j} = \boldsymbol r_j),\\[-0.75cm]
	\end{align*}
	where $\boldsymbol r_j$ is a projection of $\boldsymbol r$ to the argument list of $\phi_j$ and $Z$ is the normalisation constant, defined as $Z = \sum_{\boldsymbol r \in \mathcal{X}_{\boldsymbol{R}}
    } \prod_{j=1}^{m} \phi_j(\boldsymbol{R}_{\phi_j}= \boldsymbol r_j)$.
\end{definition}

\begin{example} \label{ex:commutative_approx_example_fg}
	\cref{fig:commutative_approx_example_fg_graph}~shows an \ac{fg} $M$ modelling the interplay between the competences $ComA$ and $ComB$ of two employees $Alice$ and $Bob$, respectively, and the revenue $Rev$ of the company they work for.
	The ranges of the \acp{rv} are $\range{ComA} = \range{ComB} = \range{Rev} = \{\high, \allowbreak \low\}$ and \cref{fig:commutative_approx_example_potential_table} shows the function definition of the factor $\phi_3(ComA, ComB, Rev)$, illustrated as a potential table, i.e., $\phi_3(ComA = \high, \allowbreak ComB = \high, \allowbreak Rev = \high) = \varphi_1$, and so on, such that $\varphi_1, \allowbreak \ldots, \allowbreak \varphi_6 \in \mathbb{R}_{\geq 0}$ are non-negative real numbers with $\varphi_i \neq \varphi_j$ for all $i \neq j$.
	We omit the potential tables of $\phi_1(ComA)$ and $\phi_2(ComB)$ for brevity.
\end{example}

\begin{figure}
	\centering
	\sbox0{\begin{tikzpicture}
	\node (t3) {
		\begin{tabular}{lllc}
			\toprule
			$ComA$ & $ComB$ & $Rev$ & $\phi_3(ComA, ComB, Rev)$ \\ \midrule
			\high  & \high  & \high & $\varphi_{1}$             \\
			\high  & \high  & \low  & $\varphi_{2}$             \\
			\high  & \low   & \high & $\varphi_{3}$             \\
			\high  & \low   & \low  & $\varphi_{4}$             \\
			\low   & \high  & \high & $\varphi_{3}$             \\
			\low   & \high  & \low  & $\varphi_{4}$             \\
			\low   & \low   & \high & $\varphi_{5}$             \\
			\low   & \low   & \low  & $\varphi_{6}$             \\ \bottomrule
		\end{tabular}
	};
\end{tikzpicture}}
	\begin{subfigure}[b]{0.4\linewidth}
		\centering
		\begin{minipage}[c][\dimexpr\ht0+\dp0\relax][c]{\linewidth}
			\centering
			\begin{tikzpicture}[rv/.append style={minimum height=2.2em, minimum width=5.2em}]
	\node[rv, draw] (rev) {$Rev$};
	\node[rv, draw, above left = 2.5em and 0.3em of rev] (ca) {$ComA$};
	\node[rv, draw, above right = 2.5em and 0.3em of rev] (cb) {$ComB$};

	\factor{above}{ca}{0.5em}{180}{$f_1$}{f1}
	\factor{above}{cb}{0.5em}{0}{$f_2$}{f2}
	\factor{above}{rev}{0.75em}{90}{$f_3$}{f3}

	\draw (f1) -- (ca);
	\draw (f2) -- (cb);
	\draw (ca) -- (f3);
	\draw (cb) -- (f3);
	\draw (f3) -- (rev);
\end{tikzpicture}
		\end{minipage}
		\caption{}
		\label{fig:commutative_approx_example_fg_graph}
	\end{subfigure}
	\begin{subfigure}[b]{0.55\linewidth}
		\centering
		\usebox0
		\caption{}
		\label{fig:commutative_approx_example_potential_table}
	\end{subfigure}
	\caption{(a) An \ac{fg} modelling the interplay between the competences of two employees and the revenue of the company they work for. (b) The function definition of $\phi_3$, where $\varphi_1, \allowbreak \ldots, \allowbreak \varphi_6 \in \mathbb{R}_{\geq 0}$ with $\varphi_i \neq \varphi_j$ for all $i \neq j$.}
	\label{fig:commutative_approx_example_fg}
\end{figure}
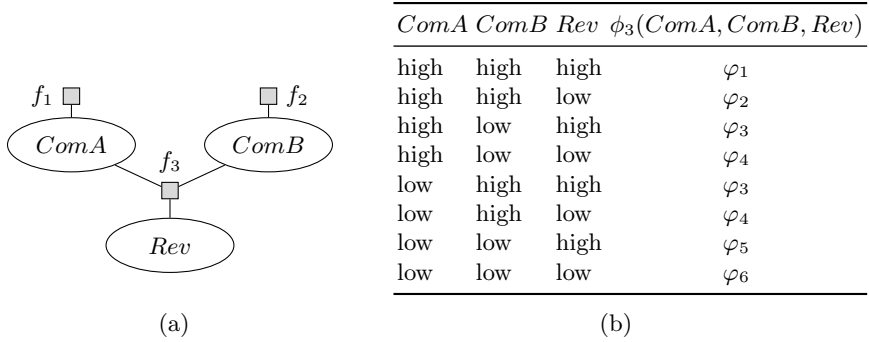

\cref{ex:commutative_approx_example_fg} illustrates that the output of $\phi_3(ComA, \allowbreak ComB, \allowbreak Rev)$ depends only on the number of highly competent employees, not on their identities.
For instance, $\phi_3(ComA = \high, \allowbreak ComB = \low, \allowbreak Rev = \high) = \phi_3(ComA = \low, \allowbreak ComB = \high, \allowbreak Rev = \high) = \varphi_3$.
Thus, permuting $ComA$ and $ComB$ leaves the output unchanged and only the count of competences set to $\high$ matters.
This invariance under permutations is a key property exploited by lifted inference algorithms and motivates the following formal definition of commutativity.


\begin{definition}[Commutative Factor~\cite{Luttermann2024a}] \label{def:commutative_approx_commutative_factor}
	Let $\phi(R_1, \ldots, R_n)$ denote a factor and let $\boldsymbol C_{\phi} \subseteq \boldsymbol{R}_{\phi} := \{ R_1, \ldots, R_n \}$ be a subset of $\phi$'s arguments with $\abs{\boldsymbol C_{\phi}} \geq 2$.
	Then, $\phi$ is \emph{commutative} with respect to $\boldsymbol C_{\phi}$ if for all assignments $\boldsymbol{r}=(r_1, \ldots, \allowbreak r_n) \in \mathcal{X}_{\boldsymbol{R}_{\phi}}$ it holds that $\phi(r_1, \ldots, r_n) = \phi(r_{\pi(1)}, \ldots, r_{\pi(n)})$ for all permutations $\pi$ of $\{ 1, \ldots, n \}$ with $\pi(i) = i$ for all $R_i \notin \boldsymbol C_{\phi}$.
\end{definition}

We refer to the arguments in $\boldsymbol C_{\phi}$ as commutative arguments.
Note that arguments can only be commutative if they have the same range because only arguments with identical ranges can be permuted.
Thus, when searching for commutative arguments, only subsets with the same range need to be considered.
Consider again 
$\phi_3(ComA, ComB, Rev)$ from \cref{fig:commutative_approx_example_potential_table}.
According to \cref{def:commutative_approx_commutative_factor}, $\phi_3$ is commutative with respect to $\{ComA, \allowbreak ComB\}$.
The above definition relies on exact equality of potentials.
In practice, however, small estimation errors often break this equality even when the underlying symmetry is still present.
The next section therefore introduces approximate commutativity.

\section{Approximate Commutativity}
So far, the definition of a commutative factor requires strict equality among potentials, which we relax here, introducing the notion of $\varepsilon$-commutativity based on the idea of $\varepsilon$-equivalence \cite{Luttermann2025c}, establish its fundamental properties, and derive the bounds needed for its practical application in lifted model construction.

We begin to introduce the more general notion of \emph{$\varepsilon$-equivalence} between two potentials $\varphi_1, \varphi_2 \in \mathbb{R}_{\geq 0}$, written as $\varphi_1 =_{\varepsilon} \varphi_2$, which holds if
\begin{align} \label{eq:commutative_approx_eps_equivalent}
    \varphi_1 \in [\varphi_2 \cdot (1 - \varepsilon), \varphi_2 \cdot (1 + \varepsilon)] \text{ and }\varphi_2 \in [\varphi_1 \cdot (1 - \varepsilon), \varphi_1 \cdot (1 + \varepsilon)]. 
    \\[-0.7cm] \nonumber
\end{align} 
Two factors are \emph{$\varepsilon$-equivalent} if there exists a permutation of arguments such that \cref{eq:commutative_approx_eps_equivalent} holds for all pairwise row comparisons of potentials.
Combining $\varepsilon$-equivalence and commutativity naturally leads to the notion of $\varepsilon$-commutativity.
\begin{definition}[$\varepsilon$-Commutative Factor] \label{def:commutative_approx_eps_commutative_factor}
	Let $\phi(R_1, \ldots, R_n)$ denote a factor, let $\boldsymbol C_{\phi} \subseteq \boldsymbol{R}_{\phi}$ be a subset of its arguments with $\abs{\boldsymbol C_{\phi}} \geq 2$, and let $\varepsilon \in [0, 1]$.
	 Then, $\phi$ is \emph{$\varepsilon$-commutative} with respect to $\boldsymbol C_{\phi}$ if for all assignments $(r_1, \ldots, \allowbreak r_n) \in \mathcal{X}_{\boldsymbol{R}_{\phi}}$ and all permutations
 \begin{align*}
  \pi\in \Pi_{\boldsymbol C_{\phi}} := \left\{ \pi \in S_n \mid \pi(i)=i \text{ for all } R_i \notin \boldsymbol C_{\phi} \right\},\\[-0.75cm]
 \end{align*}
 where $S_n$ is the symmetric group over $\{ 1, \ldots, n \}$, the potentials $\phi(r_1, \ldots, r_n)$ and $\phi(r_{\pi(1)}, \ldots, r_{\pi(n)}) $ are $\varepsilon$-equivalent, i.e.,
 \begin{alignat*}{3}
        \phi(r_1, \ldots, r_n) &\in [\phi(r_{\pi(1)}, \ldots, r_{\pi(n)}) &&\cdot (1 - \varepsilon), \phi(r_{\pi(1)}, \ldots, r_{\pi(n)}) &&\cdot (1 + \varepsilon)] \text{ and} \\
        \phi(r_{\pi(1)}, \ldots, r_{\pi(n)}) &\in [\phi(r_1, \ldots, r_n) &&\cdot (1 - \varepsilon), \phi(r_1, \ldots, r_n) &&\cdot (1 + \varepsilon)].
 \end{alignat*}
\end{definition}

We refer to the arguments in $\boldsymbol C_{\phi}$ as $\varepsilon$-commutative arguments.
Note that, for the choice of $\varepsilon = 0$, $\varepsilon$-commutativity reduces to strict commutativity.
 
As illustrated in \cref{fig:commutative_approx_example_potential_table_crv}, the commutativity of a factor can be exploited to obtain a more compact representation of the factor by introducing a \ac{crv} that groups the commutative arguments such that the original potentials are replaced by aggregated potentials.
Here, $ComA$ and $ComB$ are not listed individually but have been replaced by a \ac{crv} $\#_{E}[Com(E)]$ that counts the number of employees with a certain competence (i.e., $[2,0]$ stands for two employees with a high competence and none with a low competence, $[1,1]$ stands for one employee with a high competence and one with a low competence, and $[0,2]$ stands for two employees with a low competence and none with a high competence).
Given a commutative factor such as $\phi_3$ from \cref{fig:commutative_approx_example_potential_table}, all assignments that differ only by a permutation of the range values of the commutative arguments share the same potential, so choosing $\varphi^*_i = \varphi_i$ preserves the semantics exactly.
However, for an $\varepsilon$-commutative factor such as $\phi$ from \cref{fig:commutative_approx_example_potential_table_approx}, potentials corresponding to assignments only differ by a permutation on the commutative arguments may deviate by a factor of up to $(1 \pm \varepsilon)$, and hence collapsing them into a single potential $\varphi^*_i$ inevitably discards information (e.g., $\varphi^*_3$ must be chosen to represent both $\varphi_3$ and $\varphi'_3$ at the same time).
We therefore seek to choose $\varphi^*_i$ such that its deviation from the original potentials is minimised.

\begin{table}
	\centering
	\begin{subtable}[t]{0.49\linewidth}
		\centering
		\resizebox{\linewidth}{!}{\begin{tikzpicture}
	\node (t3) {
		\begin{tabular}{lllc}
			\toprule
			$ComA$ & $ComB$ & $Rev$ & $\phi(ComA, ComB, Rev)$ \\ \midrule
			\high  & \high  & \high & $\varphi_{1}$           \\
			\high  & \high  & \low  & $\varphi_{2}$           \\
			\high  & \low   & \high & $\varphi_{3}$           \\
			\high  & \low   & \low  & $\varphi_{4}$           \\
			\low   & \high  & \high & $\varphi'_{3}$          \\
			\low   & \high  & \low  & $\varphi'_{4}$          \\
			\low   & \low   & \high & $\varphi_{5}$           \\
			\low   & \low   & \low  & $\varphi_{6}$           \\ \bottomrule
		\end{tabular}
	};
\end{tikzpicture}}
        \caption{}
		\label{fig:commutative_approx_example_potential_table_approx}
	\end{subtable}
	\begin{subtable}[t]{0.49\linewidth}
		\centering
		\resizebox{\linewidth}{!}{\begin{tikzpicture}
	\node (t3) {
		\begin{tabular}{clc}
			\toprule
			$\#_E[Com(E)]$ & $Rev$ & $\phi(\#_E[Com(E)], Rev)$ \\ \midrule
			$[2,0]$        & \high & $\varphi^*_{1}$           \\
			$[2,0]$        & \low  & $\varphi^*_{2}$           \\
			$[1,1]$        & \high & $\varphi^*_{3}$           \\
			$[1,1]$        & \low  & $\varphi^*_{4}$           \\
			$[0,2]$        & \high & $\varphi^*_{5}$           \\
			$[0,2]$        & \low  & $\varphi^*_{6}$           \\ \bottomrule
		\end{tabular}
	};
\end{tikzpicture}}
		\caption{}
		\label{fig:commutative_approx_example_potential_table_crv}
	\end{subtable}
	\caption{(a) An exemplary factor $\phi(ComA, ComB, Rev)$ with $\varepsilon$-commutative arguments $\{ComA, \allowbreak ComB\}$ provided that $\varphi_3$ and $\varphi'_3$ as well as $\varphi_4$ and $\varphi'_4$ are $\varepsilon$-equivalent. (b) A compressed representation of $\phi(ComA, ComB, Rev)$ from (a) using a \ac{crv} to represent the $\varepsilon$-commutative arguments $\{ComA, \allowbreak ComB\}$.}
	\label{fig:commutative_approx_example_potential_table_approx_crv}
\end{table}

\begin{example}
	Consider the factor $\phi(ComA, ComB, Rev)$ from \cref{fig:commutative_approx_example_potential_table_approx}.
	Given that $\varphi_3$ and $\varphi'_3$ as well as $\varphi_4$ and $\varphi'_4$ are $\varepsilon$-equivalent for some $\varepsilon \in [0, 1]$, $\phi$ is $\varepsilon$-commutative with respect to $\{ComA, \allowbreak ComB\}$.
\end{example}

In the following, we aim to find an optimal representative factor to compress an $\varepsilon$-commutative factor such that the underlying semantics is preserved as much as possible.
To do so, we need a few theoretical properties induced by the notion of $\varepsilon$-commutativity.
The following lemma, which has been proven only across different factors \cite[Lemma~6]{Luttermann2025c}, establishes a useful consequence of the symmetric definition of $\varepsilon$-commutativity and serves as a tool to derive subsequent bounds.

\begin{restatable}{lemma}{symmetriclemma}\label{lemma:symmetry_varepsi_commut}
Let $\phi$ denote a factor that is $\varepsilon$-commutative with respect to a subset of its arguments $\boldsymbol C_{\phi}$. Then, for all permutations $\pi\in\Pi_{\boldsymbol C_{\phi}}$, it holds that
     \begin{alignat*}{3}
        \phi(r_1, \ldots, r_n) &\in [\phi(r_{\pi(1)}, \ldots, r_{\pi(n)}) &&\cdot \frac{1}{1 + \varepsilon}, \phi(r_{\pi(1)}, \ldots, r_{\pi(n)}) &&\cdot (1 + \varepsilon)] \text{ and} \\
        \phi(r_{\pi(1)}, \ldots, r_{\pi(n)}) &\in [\phi(r_1, \ldots, r_n) &&\cdot \frac{1}{1 + \varepsilon}, \phi(r_1, \ldots, r_n) &&\cdot (1 + \varepsilon)].
    \end{alignat*}
\end{restatable}

The proof follows the same structural argument as \cite[Lemma~6]{Luttermann2025c}, using the symmetric definition of $\varepsilon$-commutativity.
Details are given in~\cref{appendix:commutative_approx_proofs}.

Before we turn our attention to the compression of an $\varepsilon$-commutative factor, we briefly state two important structural properties of $\varepsilon$-commutativity, which are algorithmically relevant when it comes to detecting $\varepsilon$-commutativity of a factor (which needs to be done before the compression can take place).
A natural question is whether $\varepsilon$-commutativity with respect to a set is inherited by its subsets (as it is the case for the exact version of commutativity \cite[Prop.~10]{Luttermann2026a}).
\begin{proposition}\label{prop:eps_comm_subset}
    Let $\phi$ be $\varepsilon$-commutative with respect to $\boldsymbol C_{\phi}$.
    Then, $\phi$ is $\varepsilon$-commutative with respect to every subset $\boldsymbol C_{\phi}'\subseteq \boldsymbol C_{\phi}$ satisfying $\abs{\boldsymbol C_{\phi}'}\geq 2$.
\end{proposition}
\begin{proof}
Since, $\boldsymbol C_{\phi}'\subseteq \boldsymbol C_{\phi}$, we have $\Pi_{\boldsymbol C'_{\phi}}\subseteq \Pi_{\boldsymbol C_{\phi}}$.
Hence, by \cref{def:commutative_approx_eps_commutative_factor}, for all $\pi\in\Pi_{\boldsymbol{C}_\phi}$
and assignments $\boldsymbol{r}=(r_1, \ldots, r_n)$, it holds that
\begin{alignat*}{3}
        \phi(\boldsymbol{r}) &\in [\phi(r_{\pi(1)}, \ldots, r_{\pi(n)}) &&\cdot (1 - \varepsilon), \phi(r_{\pi(1)}, \ldots, r_{\pi(n)}) &&\cdot (1 + \varepsilon)]\;\text{and} \\
        \phi(r_{\pi(1)}, \ldots, r_{\pi(n)}) &\in [\phi(\boldsymbol{r}) &&\cdot (1 - \varepsilon), \phi(\boldsymbol{r}) &&\cdot (1 + \varepsilon)].
\end{alignat*}
Since $\Pi_{\boldsymbol C'_{\phi}} \subseteq \Pi_{\boldsymbol C_{\phi}}$, the same relations hold for all $\pi' \in \Pi_{\boldsymbol C'_{\phi}}$. \qedhere
\end{proof}

\cref{prop:eps_comm_subset} shows that $\varepsilon$-commutativity is preserved when moving from a set of \acp{rv} to its subsets.
Hence, a factor that is commutative with respect to $\boldsymbol C_{\phi}'$ may be $\varepsilon$-commutative with respect to a strict superset $\boldsymbol C_{\phi}\supsetneq\boldsymbol C_{\phi}'$, thereby enabling further compression of the original factor.
An important follow-up question is whether local $\varepsilon$-commutativity of smaller subsets implies this also for their union, which is indeed the case for exact commutativity~\cite[Theorem~11]{Luttermann2026a} and gives rise to an efficient algorithm to compute subsets of commutative arguments~\cite[Algorithm~2]{Luttermann2026a}.
The following theorem shows that the corresponding property does not hold for $\varepsilon$-commutativity, implying that computing subsets of $\varepsilon$-commutative arguments is more sophisticated than computing subsets of strictly commutative arguments.

\begin{theorem} \label{th:commutative_approx_union_not_closed}
     $\varepsilon$-commutativity is not closed under unions of sets with respect to which a factor $\phi$ is $\varepsilon$-commutative.
\end{theorem}
\begin{proof}
   Let $(r_1,r_2,r_3) \in \times_{i=1}^3\{\high, \allowbreak \medium, \allowbreak \low\}$ denote any assignment to the \acp{rv} $R_1,R_2,R_3$.
   Let $\varepsilon$ be $0.1$.
   The potentials of the factor $\phi$ are given by
    \[
    \begin{array}{ccc|c}
    R_1 & R_2 & R_3 & \phi(R_1,R_2,R_3)\\
    \hline
    \high & \medium & \low & 1\\
    \high & \low & \medium & 1.1\\
    \medium & \high & \low & 1.1\\
    \medium & \low & \high & 1.21\\
    \low & \high & \medium & 1.21\\
    \low & \medium & \high & 1.1\\
    \end{array}
    \]
All non-listed assignments are mapped to the potential $1$.
$\phi$ is $\varepsilon$-commutative with respect to all individual pairs $\{R_1,\allowbreak R_2\}$, $\{R_1,\allowbreak R_3\}$, and $\{R_2,\allowbreak R_3\}$, but not with respect to the set $\{R_1,\allowbreak R_2,\allowbreak R_3\}$.
Consider the assignment $(r_1,\allowbreak r_2,\allowbreak r_3)=(\high,\allowbreak \medium,\allowbreak \low)$ and the permutation $\pi(1)=3,\allowbreak \pi(2)=1,\allowbreak \pi(3)=2$, (cycle of two transpositions), which maps $(\high,\allowbreak \medium,\allowbreak \low)$ to $(r_{\pi(1)},\allowbreak r_{\pi(2)},\allowbreak r_{\pi(3)})=(r_{3},\allowbreak r_{1},\allowbreak r_{2})=(\low,\allowbreak \high,\allowbreak \medium)$.
However, $\phi(\high,\allowbreak \medium,\allowbreak \low)=1$ is not an element of the interval $[\phi(\low,\allowbreak \high,\allowbreak \medium)\cdot(1-0.1),\phi(\low,\allowbreak \high,\allowbreak \medium)\cdot(1+0.1) ]=[1.089,\allowbreak 1.331]$.
Thus, $\phi$ is not $\varepsilon$-commutative with respect to $\{R_1,\allowbreak R_2,\allowbreak R_3\}$.\qedhere
\end{proof}

\Cref{th:commutative_approx_union_not_closed} implies that the algorithm presented in \cite[Algorithm~2]{Luttermann2026a} to efficiently compute a maximum sized subset of commutative arguments cannot directly be transferred to compute a subset of $\varepsilon$-commutative arguments instead.
We next turn our attention to the compression of an $\varepsilon$-commutative factor, which is performed by replacing it with a smaller representative factor.

\subsection{Replacement of Potentials by Representative}
The construction developed in this subsection formalises the aggregation of
potentials over symmetry sets induced by permutations of $\varepsilon$-commutative arguments.
We first introduce symmetrised representative factors as weighted
representatives over symmetry sets and then study the algebraic structure
induced by the permutation group $\Pi_{\boldsymbol C_\phi}$, which
partitions the assignment space into equivalence classes of pairwise
$\varepsilon$-equivalent potentials.

\begin{definition}[Symmetrised Representative Factor]
\label{def:symmetrised_representative_factor} Let $\phi(R_1,\ldots,R_n)$ be a factor, which is $\varepsilon$-commutative with respect to $\boldsymbol C_{\phi} \subseteq \boldsymbol R_{\phi}
 = \{R_1,\ldots,R_n\}$.
For an assignment
$\boldsymbol{r}=(r_1,\ldots,r_n)\in\mathcal X_{\boldsymbol R_{\phi}}$,
the \emph{symmetrised representative factor}
$\phi^*$ associated with $\phi$ is defined by
\begin{align*}
\phi^*(\boldsymbol{r})
:=
\sum_{\boldsymbol{r}' \in \mathfrak{S}_{\phi}(\boldsymbol{r})} \omega_{\boldsymbol{r}'}\cdot
\phi(\boldsymbol{r}'),
\\[-0.75cm] \nonumber
\end{align*}
where $\mathfrak{S}_{\phi}(\boldsymbol{r})
:=
\left\{
(r_{\pi(1)},\ldots,r_{\pi(n)})
\;\middle|\;
\pi \in \Pi_{\boldsymbol C_{\phi}}
\right\}$
denotes the \textit{symmetry set} of $\boldsymbol{r}$ under permutations of the $\varepsilon$-commutative arguments, and where the weights satisfy $\omega_{\boldsymbol{r}'}\geq 0$ and $\sum_{\boldsymbol{r}' \in \mathfrak{S}_{\phi}(\boldsymbol{r})} \omega_{\boldsymbol{r}'}=1$.
\end{definition}

To better understand why such a construction is possible, we first investigate the structure underlying these symmetry classes.
This leads us to examine $\Pi_{\boldsymbol C_\phi}$, which is shown to be closed under composition and inversion in the following.

\begin{lemma}\label{lemma:subgroup}
For a set $\boldsymbol{C}_{\phi}\subseteq \boldsymbol{R}_{\phi} =\{R_1,\ldots,R_n\}$, $\Pi_{\boldsymbol{C}_{\phi}}$ is a subgroup of $S_n$.
\end{lemma}

\begin{proof}
By definition $\pi_{\text{id}}\in\Pi_{\boldsymbol{C}_{\phi}}$.
If $\pi_{1},\pi_2\in \Pi_{\boldsymbol{C}_{\phi}}$, then $(\pi_1\circ \pi_2) (i) = \pi_1(\pi_2(i))=\pi_1(i)=i$ for all $R_i\notin \boldsymbol{C}_{\phi}$, resulting in closed group operations $\pi_1\circ \pi_2\in\Pi_{\boldsymbol{C}_{\phi}}$.
If $\pi\in\Pi_{\boldsymbol{C}_{\phi}}$, then $\pi\in S_n$ due to $\Pi_{\boldsymbol{C}_{\phi}}\subseteq S_n$.
$S_n$ is a group and therefore $\pi^{-1}\in S_n$.
This means that for all $i$ with $R_i\notin \boldsymbol{C}_{\phi}$ it holds that $i=\pi_{\text{id}}(i)=(\pi^{-1} \circ \pi)(i) =\pi^{-1}(\pi(i)) =\pi^{-1}(i)$, which means that the inverse $\pi^{-1}$ is also in $\Pi_{\boldsymbol{C}_{\phi}}$.
\end{proof}
This subgroup naturally induces a group operation of $\Pi_{\boldsymbol C_\phi}$ on the assignment space $\mathcal X_{\boldsymbol R_\phi}$ via $(r_1,\ldots,r_n)
\;\mapsto\;
(r_{\pi(1)},\ldots,r_{\pi(n)})
 \;\text{for } \pi\in\Pi_{\boldsymbol C_\phi}$, whose properties are inherited by the corresponding symmetry sets $\mathfrak S_\phi(\boldsymbol{r})$.

Hence, symmetry sets are precisely the orbits of this group operation.
\begin{corollary}\label{corollary:subgroup_overlap}
Let $\phi(R_1,\ldots,R_n)$ be a factor, which is $\varepsilon$-commutative with respect to $\boldsymbol C_{\phi} \subseteq \boldsymbol R_{\phi}  = \{R_1,\ldots,R_n\}$. 
 If $\boldsymbol{r}'\in \mathfrak{S}_{\phi}(\boldsymbol{r})$, then $\mathfrak{S}_{\phi}(\boldsymbol{r}')=\mathfrak{S}_{\phi}(\boldsymbol{r})$.
\end{corollary}
\begin{proof}
As $\boldsymbol{r}'\in\mathfrak S_\phi(\boldsymbol{r})$, there exists $ \pi_1\in\Pi_{\boldsymbol C_\phi}$ such that $\boldsymbol{r}'=(r_{\pi_1(1)},\ldots,r_{\pi_1(n)})$.
Let $\boldsymbol{r}''\in\mathfrak S_\phi(\boldsymbol{r}')$.
Then, there is $\pi_2\in\Pi_{\boldsymbol C_\phi}$ such that $\boldsymbol{r}''=(r'_{\pi_2(1)},\ldots,r'_{\pi_2(n)})= (r_{\pi_1(\pi_2(1))},\ldots,r_{\pi_1(\pi_2(n))})$.
Since $\Pi_{\boldsymbol C_\phi}$ is a subgroup of $S_n$, it is closed under composition, hence $\pi_1\circ\pi_2\in\Pi_{\boldsymbol C_\phi}$.
Therefore $\boldsymbol{r}'' \in\mathfrak S_\phi(\boldsymbol{r})$, which shows $\mathfrak S_\phi(\boldsymbol{r}')\subseteq\mathfrak S_\phi(\boldsymbol{r})$.
The reverse inclusion follows analogously for $\pi_1^{-1},\pi_2^{-1}\in\Pi_{\boldsymbol C_\phi}$.
\end{proof} 
While the previous result describes how $\varepsilon$-commutativity behaves on the level of argument subsets, the following lemma establishes the induced equivalence on the level of potentials via $\varepsilon$-equivalence.

\begin{lemma}\label{lemma:pairwise_eps_equivalent}
Let $\phi(R_1,\ldots,R_n)$ be $\varepsilon$-commutative with respect to $\boldsymbol C_\phi$.
Then, all potentials induced by assignments within the same symmetry set are pairwise $\varepsilon$-equivalent, i.e., if $\boldsymbol{r}',\boldsymbol{r}'' \in \mathfrak S_\phi(\boldsymbol{r})$, then $\phi(\boldsymbol{r}')=_{\varepsilon}\phi(\boldsymbol{r}'')$.
\end{lemma}
\begin{proof}
Let $\boldsymbol{r}',\boldsymbol{r}'' \in \mathfrak S_\phi(\boldsymbol{r})$.
Then there exist permutations $\pi_1,\pi_2\in\Pi_{\boldsymbol C_\phi}$ such that
$\boldsymbol{r}'=(r_{\pi_1(1)},\ldots,r_{\pi_1(n)})$ and $\boldsymbol{r}''=(r_{\pi_2(1)},\ldots,r_{\pi_2(n)})$.
Since $\Pi_{\boldsymbol C_\phi}$ is a subgroup of $S_n$ (\cref{lemma:subgroup}), it is closed under composition and taking inverses.
Hence, $\tau:=\pi_2\circ\pi_1^{-1} \in \Pi_{\boldsymbol C_\phi}$.
Therefore, $\boldsymbol{r}''= (r'_{\tau(1)},\ldots,r'_{\tau(n)})$.
Because $\phi $ is $\varepsilon$-commutative with respect to $\boldsymbol C_\phi$, it follows from \cref{def:commutative_approx_eps_commutative_factor} that $\phi(\boldsymbol{r}') \in [\phi(\boldsymbol{r}'' )\cdot(1-\varepsilon),
\phi(\boldsymbol{r}'')\cdot(1+\varepsilon)]$ and $\phi( \boldsymbol{r}'')\in [\phi( \boldsymbol{r}')\cdot(1-\varepsilon),\phi( \boldsymbol{r}')\cdot(1+\varepsilon)].$
Hence, it holds that $\phi(r_1',\ldots,r_n') =_{\varepsilon} \phi(r_1'',\ldots,r_n'')$.
\end{proof}
Thus, we can also view $\mathfrak S_\phi(\boldsymbol{r})=:[\boldsymbol{r}]_\phi$ as the equivalence class of the equivalence relation $\sim$ on $\mathcal X_{\boldsymbol R_\phi}$ induced by the group operation of $\Pi_{\boldsymbol C_\phi}$ via
\begin{align*}
\boldsymbol{r} \sim \boldsymbol{r}'
\quad :\Longleftrightarrow \quad
\exists \pi\in\Pi_{\boldsymbol C_\phi}\colon
(r_1',\ldots,r_n')=(r_{\pi(1)},\ldots,r_{\pi(n)}).
\end{align*}
Consequently, the quotient space $\mathcal X_{\boldsymbol R_\phi}/\sim $ is well-defined and can be identified with the set of all symmetry sets.
Hence, there exists a representative set $\mathcal X_{\mathcal R_\phi} \subset \mathcal X_{\boldsymbol R_\phi}$ containing exactly one element per equivalence class, inducing a bijection
\begin{align*}
\mathcal X_{\boldsymbol R_\phi}/\sim
\;\cong\;
\{[\boldsymbol{r}]_\phi \mid \boldsymbol{r} \in \mathcal X_{\mathcal R_\phi}\}=\{\mathfrak S_\phi(\boldsymbol{r}) \mid \boldsymbol{r} \in \mathcal X_{\mathcal R_\phi}\}.
\end{align*}
Different symmetry sets (equivalence classes) of the same factor may induce different weighting schemes.
Within each equivalence class, however, the weights must be chosen consistently with the quotient structure $\mathcal X_{\boldsymbol R_\phi}/\sim$ so that every class contributes exactly once to the representative potential (for the size of the symmetry set, see \Cref{app:additional_results}).
This raises the question whether replacing potentials by symmetrised representatives preserves  $\varepsilon$-equivalence of factors.
\begin{theorem}\label{th:equivalentfactors}
Let $\phi$ be $\varepsilon$-commutative with respect to $\boldsymbol C_{\phi}$ and let $\phi^*$ be its symmetrised representative factor for $\varepsilon\in[0,1]$.
Then, $\phi$ and $\phi^*$ are $\varepsilon$-equivalent.
\end{theorem}

\begin{proof}
By \cref{lemma:pairwise_eps_equivalent}, potentials $\phi(\boldsymbol{r}')$ with $\boldsymbol{r}'\in \mathfrak{S}_{\phi}(\boldsymbol{r})$ are pairwise $\varepsilon$-equivalent.
Hence, by \cite[Theorem~15]{Speller2025b}, any weighted mean $\sum_{\boldsymbol{r}' \in \mathfrak{S}_{\phi}(\boldsymbol{r})} \omega_{\boldsymbol{r}'}\cdot
\phi(\boldsymbol{r}')$ is also $\varepsilon$-equivalent to $\phi(\boldsymbol{r}')$.
Applied row-wise to all assignments and potentials relative to their equivalence classes, this yields $\varepsilon$-equivalent factors.
\end{proof}
Choosing the arithmetic mean $\overline{\phi}$ as the symmetrised representative factor arises as a trivial special case of \cref{th:equivalentfactors} by assigning uniform weights $\omega_{\boldsymbol{r}'} :=\frac{1}{\abs{\mathfrak{S}_{\phi}(\boldsymbol{r})}}$ on $\mathfrak{S}_{\phi}(\boldsymbol{r})$ for which $\sum_{\boldsymbol{r}' \in \mathfrak{S}_{\phi}(\boldsymbol{r})} \omega_{\boldsymbol{r}'} =1$ holds, yielding $\phi^*(\boldsymbol{r})=\overline{\phi}(\boldsymbol{r})$.

\begin{corollary}
Let $\phi$ be $\varepsilon$-commutative with respect to $\boldsymbol C_{\phi}$ and let the mean
\begin{align}
\phi^*(\boldsymbol{r}):=\overline{\phi}(\boldsymbol{r})
:=
\frac{1}{|\mathfrak{S}_{\phi}(\boldsymbol{r})|}
\sum_{\boldsymbol{r}' \in \mathfrak{S}_{\phi}(\boldsymbol{r})} 
\phi(\boldsymbol{r}'),\label{eq:mean_sym}
\\[-0.75cm] \nonumber
\end{align}
be defined with weights $\omega_{\boldsymbol{r}'} :=\frac{1}{\abs{\mathfrak{S}_{\phi}(\boldsymbol{r})}}$ for $\boldsymbol r'\in \mathfrak{S}_{\phi}(\boldsymbol{r})$ for a given $\varepsilon\in[0,1]$.
Then $\phi^*$ is a symmetrised representative factor of $\phi$ and is $\varepsilon$-equivalent to $\phi$.
\end{corollary}
The mean is particularly attractive as it treats all symmetric assignments uniformly avoiding bias towards specific permutations and minimises the squared deviation~\cite{Luttermann2025c}.
It is also optimal with respect to the squared deviation in the sense that, when all values $\phi(\boldsymbol r')$ are replaced by the same representative within the group $\boldsymbol r \in \mathfrak{S}_{\phi}(\boldsymbol r)$, the resulting deviation from the original factor is minimised.
These properties make it a natural choice for the symmetrised representative factor in practical implementations (see \cref{sec:commutative_approx_eval} for experiments).
However, when prior knowledge indicates a bias toward specific arguments in the underlying data, or when domain expertise is available, the weights assigned to individual arguments of the original factor can always be adapted accordingly.

\subsection{Asymptotic Bounds}
To quantify the approximation error induced by symmetrisation, we next analyse the resulting deviation between the original model and the model obtained by symmetrising $\varepsilon$-commutative factors.
We make use of the following symmetric distance measure, which allows us to bound the change in query results~\cite{Chan2005a,Luttermann2025c}.
\begin{definition}[\cite{Chan2005a}]
   The symmetric distance $D_{CD}(P_M, P_{M'})$ between two distributions $P_M$ and $P_{M'}$ introduced in \cite{Chan2005a} (see \cite{Luttermann2025c} for details) is defined as
\begin{align}
	D_{CD}(P_M, P_{M'}) :=& \ln \max_{\boldsymbol r} \frac{P_{M'}(\boldsymbol r)}{P_M(\boldsymbol r)} - \ln \min_{\boldsymbol r} \frac{P_{M'}(\boldsymbol r)}{P_M(\boldsymbol r)}\\
    =& \ln \max_{\boldsymbol r} \frac{\psi'(\boldsymbol r)}{\psi(\boldsymbol r)} - \ln \min_{\boldsymbol r} \frac{\psi'(\boldsymbol r)}{\psi(\boldsymbol r)}. \label{eq:eacp_distance_measure}
    \\[-0.75cm] \nonumber
\end{align}
\end{definition}
With this measure at hand, we first derive a general bound for arbitrary weighting schemes.
\begin{theorem}
    Let $M$ be an \ac{fg} with factors $\phi_1,\ldots,\phi_m $, of which $\phi_1,\ldots,\phi_k $ with $0\leq k\leq m$ are $\varepsilon$-commutative with respect to $\boldsymbol{C}_{\phi_1},\ldots, \boldsymbol{C}_{\phi_k}$, respectively.
    Let $M'$ be the \ac{fg} with factors $\phi_1^*,\ldots,\phi_k^*,\phi_{k+1},\ldots,\phi_m$, where $\varepsilon$-commutative factors of $M$ are replaced by their symmetrised representative factors.
    Further, let $P_M$ and $P_{M'}$ denote the underlying full joint probability distributions encoded by $M$ and $M'$, respectively.
    Then, it holds that
	\begin{align*}
		D_{CD}(P_M, P_{M'})   &\leq      
        \ln \big( 1 + \varepsilon \big)^{2k}.
        \\[-0.75cm] \nonumber
	\end{align*}
\end{theorem}
\begin{proof}
	According to \cref{def:commutative_approx_eps_commutative_factor}, every updated potential $\phi_i^*(\boldsymbol{r}_i)$ in $M'$ differs from its original potential $\phi_i(\boldsymbol{r}_i)$ in $M$ by at most  a factor $(1 \pm \varepsilon)$ for $i=1,\ldots,k$, while all remaining factors $i>k$ are left unchanged.
	Since $\psi(\boldsymbol r) = \prod_{j=1}^{m} \phi_j(\boldsymbol r_j)$ for any assignment $\boldsymbol r$, using \cref{lemma:symmetry_varepsi_commut} we obtain
	\begin{align*}
		\psi'(\boldsymbol r) &\geq \prod_{j=1}^{k} \phi_j(\boldsymbol r_j) \cdot \frac{1}{1+\varepsilon}
        \cdot \prod_{j=k+1}^{m} \phi_j(\boldsymbol r_j) =\frac{1}{(1+\varepsilon)^k}\cdot \prod_{j=1}^{m} \phi_j(\boldsymbol r_j),\text{ and }\\
        \psi'(\boldsymbol r) &\leq \prod_{j=1}^{k} \phi_j(\boldsymbol r_j) \cdot (1 + \varepsilon)\cdot \prod_{j=k+1}^{m} \phi_j(\boldsymbol r_j)=(1+\varepsilon)^k\cdot \prod_{j=1}^{m} \phi_j(\boldsymbol r_j).
	\end{align*}
	Consequently,
	\begin{align*}
	\min_{\boldsymbol r} \frac{\psi'(\boldsymbol r)}{\psi(\boldsymbol r)} &\geq \frac{\frac{1}{(1+\varepsilon)^k}
        \cdot \prod\limits_{j = 1}^m \phi_j(\boldsymbol r_j) }{\prod\limits_{j = 1}^m \phi_j(\boldsymbol r_j)}=\frac{1}{(1+\varepsilon)^k},~\text{and} \\
		\max_{\boldsymbol r} \frac{\psi'(\boldsymbol r)}{\psi(\boldsymbol r)} &\leq \frac{ (1 + \varepsilon)^k\cdot \prod\limits_{j = 1}^m  \phi_j(\boldsymbol r_j)  }{\prod\limits_{j = 1}^m \phi_j(\boldsymbol r_j)} =(1 + \varepsilon)^k,
	\end{align*}
	where $\boldsymbol r_j\in\mathcal{X}_{\boldsymbol{R}_{\phi_j}}$ denotes the projection of $\boldsymbol r\in\mathcal{X}_{\boldsymbol{R}}$.
	Substituting these bounds into \cref{eq:eacp_distance_measure} yields
    $D_{CD}(P_M, P_{M'})
        \leq \ln (1 + \varepsilon)^k
		- \ln 1 \mathbin{/} (1+\varepsilon)^k 
	= \ln (1 + \varepsilon)^{2k}$.
\end{proof}

For the special case of arithmetic averaging, we can sharpen this bound.
\begin{restatable}{lemma}{lemmabound}\label{lemma:bound}
Let $\phi$ be an $\varepsilon$-commutative factor with respect to $\boldsymbol{C}_{\phi}$, let $\mathcal{X}_{\mathcal{R}_{\phi}} = \dot \cup_{j=1}^l \{\boldsymbol{r}^{j}\}$ be a set of disjoint representatives $\boldsymbol{r}^{j}$ for all disjoint equivalence classes $\dot\cup_{j=1}^l [\boldsymbol{r}^{j}]_{\phi} = \dot\cup_{j=1}^l \mathfrak{S}_{\phi}(\boldsymbol{r}^{j})=\mathcal{X}_{\boldsymbol{R}_{\phi}}$ ordered by size such that $m_1:=\abs{[\boldsymbol{r}^{1}]_{\phi}}\geq\ldots \geq m_l:=\abs{[\boldsymbol{r}^{l}]_{\phi}}\geq 1$,
and let $\phi^*$ be the symmetrised representative factor via the mean as in \cref{eq:mean_sym}.
Then,
\begin{align*}
       1\;\leq\;\frac{\max_{\boldsymbol r} \frac{\phi^*(\boldsymbol r)}{\phi(\boldsymbol r)}}{\min_{\boldsymbol r} \frac{\phi^*(\boldsymbol r)}{\phi(\boldsymbol r)}}\;\leq\; \frac{\left(1+\frac{m_2-1}{m_2}\varepsilon\right)(1+\varepsilon)}{1+\frac{1}{m_1}\varepsilon}.
\end{align*}
\end{restatable}
\begin{proof}[Condensed Proof (Longer Version in \cref{appendix:commutative_approx_proofs})]
The mean $\phi^*(\boldsymbol r)$ is uniquely determined within each equivalence class $[\boldsymbol{r}^j]_{\phi}$.
First, observe that for every equivalence class $[\boldsymbol{r}^j]_{\phi}$ and all $\boldsymbol{r}',\boldsymbol{r}'' \in [\boldsymbol{r}^j]_{\phi}$, \cref{lemma:pairwise_eps_equivalent} implies pairwise $\varepsilon$-equivalence, so the potentials $\phi( \boldsymbol{r}'')$ within one equivalence class $[\boldsymbol{r}^j]_{\phi}$ can be ordered 
\begin{align*}
0 \leq \frac{1}{1+\varepsilon}\phi(\boldsymbol{r}^{j,\max})\leq \phi(\boldsymbol{r}^{j,\min}) \leq \phi(\boldsymbol{r}'')\leq \phi(\boldsymbol{r}^{j,\max})\leq \phi(\boldsymbol{r}^{j,\min})(1+\varepsilon),
\end{align*}
where $\phi(\boldsymbol{r}^{j,\max}):= \max_{\boldsymbol{r}' \in [\boldsymbol{r}^j]_{\phi}} \phi(\boldsymbol{r}')$ and $\phi(\boldsymbol{r}^{j,\min}):= \min_{\boldsymbol{r}' \in [\boldsymbol{r}^j]_{\phi}} \phi(\boldsymbol{r}')$.
When the maximum and minimum originate from two different equivalence classes $[\boldsymbol{r}^i]_{\phi}$ and $[\boldsymbol{r}^j]_{\phi}$ with $i\neq j$, the worst-case assignments are independent. 
For the maximum choose $ \phi(\boldsymbol r)=\phi(\boldsymbol{r}^{i,\min})$, while all remaining potentials in the same equivalence class attain their maximal admissible value $\phi(\boldsymbol r'):=\phi(\boldsymbol{r}^{i,\max}) = (1+\varepsilon) \phi(\boldsymbol{r}^{i,\min})$ for all $\boldsymbol{r}' \in [\boldsymbol{r}^i]_{\phi}\setminus \{\boldsymbol{r} \}$.
This yields
\begin{align*}
   \max_{\boldsymbol r\in [\boldsymbol{r}^i]_{\phi}} \frac{\phi^*(\boldsymbol r)}{\phi(\boldsymbol r)}&= \frac{(\phi(\boldsymbol{r}^{i,\min})+(\abs{[\boldsymbol{r}^i]_{\phi}}-1) (1+\varepsilon) \phi(\boldsymbol{r}^{i,\min}) ) \mathbin{/} \abs{[\boldsymbol{r}^i]_{\phi}}}{\phi(\boldsymbol{r}^{i,\min})}\\
   &=1+ \frac{(m_i-1)\varepsilon}{m_i}.
\end{align*}
For the minimum choose $\phi(\boldsymbol r)=\phi(\boldsymbol{r}^{j,\max})$, while all remaining values in the same equivalence class attain their minimal admissible value $\phi(\boldsymbol r'):=\phi(\boldsymbol{r}^{j,\min}) =\frac{1}{1+\varepsilon} \phi(\boldsymbol{r}^{i,\max})$ for all $\boldsymbol{r}' \in [\boldsymbol{r}^j]_{\phi}\setminus \{\boldsymbol{r} \}$.
Consequently,
\begin{align*}
   \min_{\boldsymbol r\in [\boldsymbol{r}^j]_{\phi}} \frac{\phi^*(\boldsymbol r)}{\phi(\boldsymbol r)}&= \frac{((\abs{[\boldsymbol{r}^j]_{\phi}}-1)\frac{1}{1+\varepsilon}\phi(\boldsymbol{r}^{j,\max})+ \phi(\boldsymbol{r}^{j,\max}) ) \mathbin{/} \abs{[\boldsymbol{r}^j]_{\phi}}}{\phi(\boldsymbol{r}^{j,\max})}  = \frac{1+\frac{1}{m_j}\varepsilon}{1+\varepsilon}.
\end{align*}
Observe, that $1+ \frac{(m_i-1)\varepsilon}{m_i}$ is monotonically increasing in $m_i$, whereas $\frac{1+\frac{1}{m_j}\varepsilon}{1+\varepsilon}$ is monotonically decreasing in $m_j$.
Hence, the largest equivalence classes by size produce the extremal deviations.
To determine which class size appear in the numerator and denominator, consider the monotonically decreasing $g(x)=\left(1+ \frac{x-1}{x}\;\varepsilon\right)\left(1+\frac{1}{x}\varepsilon\right)$ for $x\geq 2$.
Therefore, for $x_1,x_2\in\mathbb{N}_{\geq 2}$ it follows
\begin{align*}  
   g(x_1) &\geq g(x_2)
   \;\;\Leftrightarrow \;\; \frac{1+ \frac{x_1-1}{x_1}\;\varepsilon}{1+\frac{1}{x_2}\varepsilon}\geq \frac{1+ \frac{x_2-1}{x_2}\;\varepsilon}{1+\frac{1}{x_1}\varepsilon}.
\end{align*}
Consequently, the sharpest bound is obtained by using $x_1=m_2$ and $x_2=m_1$.
When both extrema originate from the same equivalence class, $\phi^*(\boldsymbol r)=\phi^*(\boldsymbol r')$ for all $\boldsymbol{r},\boldsymbol{r}'\in [\boldsymbol{r}^j]_{\phi}$, the quotient is maximised via $\phi(\boldsymbol r)=\phi(\boldsymbol{r}^{j,\min}) =\frac{1}{1+\varepsilon} \phi(\boldsymbol{r}^{j,\max})$ and $\phi(\boldsymbol r')=\phi(\boldsymbol{r}^{j,\max}) =(1+\varepsilon) \phi(\boldsymbol{r}^{j,\min})$,
leading to smaller values
\begin{align*}
    &\frac{\max_{\boldsymbol r\in [\boldsymbol{r}^j]_{\phi}} \frac{\phi^*(\boldsymbol r)}{\phi(\boldsymbol r)}}{\min_{\boldsymbol r'\in [\boldsymbol{r}^j]_{\phi}} \frac{\phi^*(\boldsymbol r')}{\phi(\boldsymbol r')}}\;= \frac{\max_{\boldsymbol r\in [\boldsymbol{r}^j]_{\phi}} \frac{1}{\phi(\boldsymbol r)}}{\min_{\boldsymbol r'\in [\boldsymbol{r}^j]_{\phi}} \frac{1}{\phi(\boldsymbol r')}}\leq \frac{\frac{1}{\phi(\boldsymbol{r}^{j,\min})}}{\frac{1}{(1+\varepsilon) \phi(\boldsymbol{r}^{j,\min})}}=1+\varepsilon. 
    \tag*\qedhere
\end{align*}
\end{proof}

Combining the class-wise extremal bounds yields the final global statement for the representative mean-construction, which we show next.
\begin{theorem}\label{theorem:genral_bound}
 Let $M$ be an \ac{fg} with factors $\phi_1,\ldots,\phi_m $, of which $\phi_1,\ldots,\phi_k $ with $0\leq k\leq m$ are $\varepsilon$-commutative with respect to $\boldsymbol{C}_{\phi_1},\ldots, \boldsymbol{C}_{\phi_k}$, respectively. 
 Let $\mathcal{X}_{\mathcal{R}_{\phi_i}} = \dot \cup_{j=1}^{l_i} \{\boldsymbol{r}_i^{j}\}$ be the set of disjoint representatives $\boldsymbol{r}^{j}_i$ for all equivalence classes $\dot\cup_{j=1}^{l_i} [\boldsymbol{r}^{j}_i]_{\phi_i} = \mathcal{X}_{\boldsymbol{R}_{\phi_i}}$ ordered by size with $m_{1_i}:=\abs{[\boldsymbol{r}^{1}_i]_{\phi_i}}\geq\ldots \geq m_{l_i}:=\abs{[\boldsymbol{r}^{l}_i]_{\phi_i}}\geq 1$.
 Let $M'$ be the \ac{fg} obtained by $M$ by replacing each factor $\phi_i$ by its mean-based symmetrised representative factor $\phi_i^*$ as in \cref{eq:mean_sym}.
 Then, for the induced distributions $P_M$ and $P_{M'}$ and $\varepsilon>0$, it holds that
	\begin{align} 
		D_{CD}(P_M, P_{M'})   &\leq 
        \ln \;\prod_{i=1}^k\frac{\left( 1+\frac{m_{2_i}-1}{m_{2_i}}\varepsilon\right)( 1+\varepsilon)}{1+\frac{1}{m_{1_i}}\varepsilon} \\ 
        &<      \ln \;\big( 1 + \varepsilon \big)^{2k}. \nonumber
	\end{align}
\end{theorem}
\begin{proof}
The statement follows by factorisation of the joint potential, the determination of the extrema on a projected factor-level instead of a global assignment-level, and the class-wise bound from \cref{lemma:bound} for the monotonic logarithm.
\begin{align*}
		D_{CD}&(P_M, P_{M'}) \overset{\text{\cref{eq:eacp_distance_measure}}}{=}   \ln \;\max_{\boldsymbol r} \frac{\psi'(\boldsymbol r)}{\psi(\boldsymbol r)} - \ln\; \min_{\boldsymbol r} \frac{\psi'(\boldsymbol r)}{\psi(\boldsymbol r)}\\
        &= \ln \;\max_{\boldsymbol r}\frac{\prod_{i=1}^k \phi_i^*(\boldsymbol{r}_i)\prod_{i=k+1}^m \phi_i(\boldsymbol{r}_i)}{\prod_{i=1}^m \phi_i(\boldsymbol{r}_i)} - \ln \;\min_{\boldsymbol r} \frac{\prod_{i=1}^k \phi_i^*(\boldsymbol{r}_i)\prod_{i=k+1}^m \phi_i(\boldsymbol{r}_i)}{\prod_{i=1}^m \phi_i(\boldsymbol{r}_i)}\\
        &= \ln \;\max_{\boldsymbol r}\prod_{i=1}^k\frac{ \phi_i^*(\boldsymbol{r}_i)}{\phi_i(\boldsymbol{r}_i)} - \ln\; \min_{\boldsymbol r}\prod_{i=1}^k \frac{\phi_i^*(\boldsymbol{r}_i)}{\phi_i(\boldsymbol{r}_i)}\\
        &\leq \ln\; \prod_{i=1}^k \max_{\boldsymbol r_i} \frac{\phi_i^*(\boldsymbol{r}_i)}{\phi_i(\boldsymbol{r}_i)} - \ln \;\prod_{i=1}^k \min_{\boldsymbol r_i} \frac{\phi_i^*(\boldsymbol{r}_i)}{\phi_i(\boldsymbol{r}_i)}\\
		&= \ln \;\prod_{i=1}^k \frac{\max_{\boldsymbol r_i} \frac{\phi_i^*(\boldsymbol{r}_i)}{\phi_i(\boldsymbol{r}_i)}}{\min_{\boldsymbol r_i} \frac{\phi_i^*(\boldsymbol{r}_i)}{\phi_i(\boldsymbol{r}_i)}}\\
        &\overset{\text{\cref{lemma:bound}}}{\leq}\ln \;\prod_{i=1}^k \frac{\left( 1+\frac{m_{2_i}-1}{m_{2_i}}\varepsilon\right)( 1+\varepsilon)}{1+\frac{1}{m_{1_i}}\varepsilon}.  \tag*\qedhere
	\end{align*}
\end{proof}
Even more precisely, the bound is sharp in the general non-trivial case.
\begin{restatable}{theorem}{boundtheorem}\label{th:sharp_optimal}
The bound given in \cref{theorem:genral_bound} is optimal.
\end{restatable}
\begin{proof}[Proof Sketch]
\cref{ex:sharp} in \cref{appendix:commutative_approx_proofs}
uses the construction from the proof of \cref{lemma:bound} per factor, hits the bound, and is therefore sharp.  
\end{proof}
Note that the choice of a multiplicative relaxation in \cref{def:commutative_approx_eps_commutative_factor} is what makes this bound attainable in the first place, as $D_{CD}$ is based on quotients of potentials and only relies on relative sizes.
An additive relaxation, requiring permuted potentials differ by at most a constant, would lead to arbitrary large deviations.

\section{Experiments} \label{sec:commutative_approx_eval}
We complement our theoretical results with an empirical evaluation assessing to what extent $\varepsilon$-commutativity translates into practical benefits for lifted model construction and downstream inference.
In particular, we answer the question of how the trade-off between higher compression and accuracy of query results behaves when compressing $\varepsilon$-commutative factors.
Concretely, we compare the run time of \ac{lve} on the compressed model returned by the exact \ac{acp} algorithm (which is only able to compress strictly commutative factors) to the run time of \ac{lve} on the compressed model returned by its $\varepsilon$-relaxed variant using the mean as the symmetrised representative factor (denoted as \enquote{\acs{acp} $\pm \varepsilon$} in the following) and measure the deviation of the query results, quantified by the per-query quotient $p' \mathbin{/} p$ between the query result (i.e., marginal probability) $p'$ obtained on the \acs{acp} $\pm \varepsilon$-compressed model and the query result $p$ obtained on the exact \ac{acp}-compressed model.
For our experiments, we use the same input \acp{fg} as in the original \ac{acp} paper~\cite{Luttermann2024a} and add noise to the potentials of the (strictly) commutative factors to obtain $\varepsilon$-commutative factors.
We only manipulate commutative factors and leave the remaining factors unchanged to investigate the effect of $\varepsilon$-commutativity on the compression-accuracy trade-off in isolation (the effect of grouping different factors that are approximately equal has been investigated in \cite{Luttermann2025c} and hence, we only consider indistinguishability \emph{within} factors, characterised by $\varepsilon$-commutativity, and no indistinguishability \emph{between} factors).
Specifically, each input \ac{fg} contains between $2d + 1$ and $d \cdot \lfloor \log_2(d) \rfloor + 2d + 2$ \acp{rv} with Boolean range and between $2d + 1$ and $d \cdot \lfloor \log_2(d) \rfloor + d + 2$ factors, where $d \in \{2, \allowbreak 4, \allowbreak 8, \allowbreak 12, \allowbreak 16, \allowbreak 20\}$ controls the size of the \acp{fg}.
Each input \ac{fg} contains $k \in \{1, \allowbreak 3, \allowbreak 7\}$ $\varepsilon$-commutative factors, where $\varepsilon \in \{0.001, \allowbreak 0.01, \allowbreak 0.1\}$.
The $\varepsilon$-commutative factors are obtained by multiplying the potentials of the (strictly) commutative factors by a uniform sample from $[1, 1 + \varepsilon]$.

\begin{figure}[t]
	\centering
	\begin{subfigure}[t]{0.49\linewidth}
		\centering
		\resizebox{\linewidth}{!}{
\begin{tikzpicture}[x=1pt,y=1pt]
\definecolor{fillColor}{RGB}{255,255,255}
\path[use as bounding box,fill=fillColor,fill opacity=0.00] (0,0) rectangle (209.58,115.63);
\begin{scope}
\path[clip] (  0.00,  0.00) rectangle (209.58,115.63);
\definecolor{drawColor}{RGB}{255,255,255}
\definecolor{fillColor}{RGB}{255,255,255}

\path[draw=drawColor,line width= 0.5pt,line join=round,line cap=round,fill=fillColor] (  0.00,  0.00) rectangle (209.58,115.63);
\end{scope}
\begin{scope}
\path[clip] ( 28.25, 22.31) rectangle (203.58,113.63);
\definecolor{fillColor}{RGB}{255,255,255}

\path[fill=fillColor] ( 28.25, 22.31) rectangle (203.58,113.63);
\definecolor{drawColor}{RGB}{247,192,26}

\path[draw=drawColor,line width= 0.5pt,line join=round] ( 36.22, 26.46) --
	( 53.93, 35.50) --
	( 89.35, 49.16) --
	(124.77, 64.82) --
	(160.19, 81.16) --
	(195.61,109.48);
\definecolor{drawColor}{RGB}{78,155,133}

\path[draw=drawColor,line width= 0.5pt,dash pattern=on 4pt off 4pt ,line join=round] ( 36.22, 30.03) --
	( 53.93, 33.05) --
	( 89.35, 33.81) --
	(124.77, 35.35) --
	(160.19, 34.67) --
	(195.61, 36.11);
\definecolor{drawColor}{RGB}{247,192,26}
\definecolor{fillColor}{RGB}{247,192,26}

\path[draw=drawColor,line width= 0.4pt,line join=round,line cap=round,fill=fillColor] ( 36.22, 26.46) circle (  1.53);
\definecolor{fillColor}{RGB}{78,155,133}

\path[fill=fillColor] ( 34.68, 28.49) --
	( 37.75, 28.49) --
	( 37.75, 31.56) --
	( 34.68, 31.56) --
	cycle;
\definecolor{fillColor}{RGB}{247,192,26}

\path[draw=drawColor,line width= 0.4pt,line join=round,line cap=round,fill=fillColor] ( 53.93, 35.50) circle (  1.53);
\definecolor{fillColor}{RGB}{78,155,133}

\path[fill=fillColor] ( 52.39, 31.52) --
	( 55.46, 31.52) --
	( 55.46, 34.59) --
	( 52.39, 34.59) --
	cycle;
\definecolor{fillColor}{RGB}{247,192,26}

\path[draw=drawColor,line width= 0.4pt,line join=round,line cap=round,fill=fillColor] ( 89.35, 49.16) circle (  1.53);
\definecolor{fillColor}{RGB}{78,155,133}

\path[fill=fillColor] ( 87.81, 32.27) --
	( 90.88, 32.27) --
	( 90.88, 35.34) --
	( 87.81, 35.34) --
	cycle;
\definecolor{fillColor}{RGB}{247,192,26}

\path[draw=drawColor,line width= 0.4pt,line join=round,line cap=round,fill=fillColor] (124.77, 64.82) circle (  1.53);
\definecolor{fillColor}{RGB}{78,155,133}

\path[fill=fillColor] (123.24, 33.82) --
	(126.30, 33.82) --
	(126.30, 36.89) --
	(123.24, 36.89) --
	cycle;
\definecolor{fillColor}{RGB}{247,192,26}

\path[draw=drawColor,line width= 0.4pt,line join=round,line cap=round,fill=fillColor] (160.19, 81.16) circle (  1.53);
\definecolor{fillColor}{RGB}{78,155,133}

\path[fill=fillColor] (158.66, 33.14) --
	(161.73, 33.14) --
	(161.73, 36.21) --
	(158.66, 36.21) --
	cycle;
\definecolor{fillColor}{RGB}{247,192,26}

\path[draw=drawColor,line width= 0.4pt,line join=round,line cap=round,fill=fillColor] (195.61,109.48) circle (  1.53);
\definecolor{fillColor}{RGB}{78,155,133}

\path[fill=fillColor] (194.08, 34.58) --
	(197.15, 34.58) --
	(197.15, 37.65) --
	(194.08, 37.65) --
	cycle;
\end{scope}
\begin{scope}
\path[clip] (  0.00,  0.00) rectangle (209.58,115.63);
\definecolor{drawColor}{RGB}{0,0,0}

\path[draw=drawColor,line width= 0.5pt,line join=round] ( 28.25, 22.31) --
	( 28.25,113.63);

\path[draw=drawColor,line width= 0.5pt,line join=round] ( 29.38,111.66) --
	( 28.25,113.63) --
	( 27.11,111.66);
\end{scope}
\begin{scope}
\path[clip] (  0.00,  0.00) rectangle (209.58,115.63);
\definecolor{drawColor}{gray}{0.30}

\node[text=drawColor,anchor=base east,inner sep=0pt, outer sep=0pt, scale=  0.80] at ( 24.20, 25.89) {10};

\node[text=drawColor,anchor=base east,inner sep=0pt, outer sep=0pt, scale=  0.80] at ( 24.20, 47.60) {30};

\node[text=drawColor,anchor=base east,inner sep=0pt, outer sep=0pt, scale=  0.80] at ( 24.20, 71.39) {100};

\node[text=drawColor,anchor=base east,inner sep=0pt, outer sep=0pt, scale=  0.80] at ( 24.20, 93.10) {300};
\end{scope}
\begin{scope}
\path[clip] (  0.00,  0.00) rectangle (209.58,115.63);
\definecolor{drawColor}{gray}{0.20}

\path[draw=drawColor,line width= 0.5pt,line join=round] ( 26.00, 28.64) --
	( 28.25, 28.64);

\path[draw=drawColor,line width= 0.5pt,line join=round] ( 26.00, 50.35) --
	( 28.25, 50.35);

\path[draw=drawColor,line width= 0.5pt,line join=round] ( 26.00, 74.14) --
	( 28.25, 74.14);

\path[draw=drawColor,line width= 0.5pt,line join=round] ( 26.00, 95.85) --
	( 28.25, 95.85);
\end{scope}
\begin{scope}
\path[clip] (  0.00,  0.00) rectangle (209.58,115.63);
\definecolor{drawColor}{RGB}{0,0,0}

\path[draw=drawColor,line width= 0.5pt,line join=round] ( 28.25, 22.31) --
	(203.58, 22.31);

\path[draw=drawColor,line width= 0.5pt,line join=round] (201.61, 21.17) --
	(203.58, 22.31) --
	(201.61, 23.45);
\end{scope}
\begin{scope}
\path[clip] (  0.00,  0.00) rectangle (209.58,115.63);
\definecolor{drawColor}{gray}{0.20}

\path[draw=drawColor,line width= 0.5pt,line join=round] ( 36.22, 20.06) --
	( 36.22, 22.31);

\path[draw=drawColor,line width= 0.5pt,line join=round] ( 53.93, 20.06) --
	( 53.93, 22.31);

\path[draw=drawColor,line width= 0.5pt,line join=round] ( 89.35, 20.06) --
	( 89.35, 22.31);

\path[draw=drawColor,line width= 0.5pt,line join=round] (124.77, 20.06) --
	(124.77, 22.31);

\path[draw=drawColor,line width= 0.5pt,line join=round] (160.19, 20.06) --
	(160.19, 22.31);

\path[draw=drawColor,line width= 0.5pt,line join=round] (195.61, 20.06) --
	(195.61, 22.31);
\end{scope}
\begin{scope}
\path[clip] (  0.00,  0.00) rectangle (209.58,115.63);
\definecolor{drawColor}{gray}{0.30}

\node[text=drawColor,anchor=base,inner sep=0pt, outer sep=0pt, scale=  0.80] at ( 36.22, 12.75) {2};

\node[text=drawColor,anchor=base,inner sep=0pt, outer sep=0pt, scale=  0.80] at ( 53.93, 12.75) {4};

\node[text=drawColor,anchor=base,inner sep=0pt, outer sep=0pt, scale=  0.80] at ( 89.35, 12.75) {8};

\node[text=drawColor,anchor=base,inner sep=0pt, outer sep=0pt, scale=  0.80] at (124.77, 12.75) {12};

\node[text=drawColor,anchor=base,inner sep=0pt, outer sep=0pt, scale=  0.80] at (160.19, 12.75) {16};

\node[text=drawColor,anchor=base,inner sep=0pt, outer sep=0pt, scale=  0.80] at (195.61, 12.75) {20};
\end{scope}
\begin{scope}
\path[clip] (  0.00,  0.00) rectangle (209.58,115.63);
\definecolor{drawColor}{RGB}{0,0,0}

\node[text=drawColor,anchor=base,inner sep=0pt, outer sep=0pt, scale=  0.90] at (115.91,  2.75) {domain size $d$};
\end{scope}
\begin{scope}
\path[clip] (  0.00,  0.00) rectangle (209.58,115.63);
\definecolor{drawColor}{RGB}{0,0,0}

\node[text=drawColor,rotate= 90.00,anchor=base,inner sep=0pt, outer sep=0pt, scale=  0.90] at (  8.20, 67.97) {time (ms)};
\end{scope}
\begin{scope}
\path[clip] (  0.00,  0.00) rectangle (209.58,115.63);

\path[] ( 32.50, 97.23) rectangle (101.14,113.60);
\end{scope}
\begin{scope}
\path[clip] (  0.00,  0.00) rectangle (209.58,115.63);
\definecolor{drawColor}{RGB}{247,192,26}

\path[draw=drawColor,line width= 0.5pt,line join=round] ( 33.59,106.50) -- ( 42.26,106.50);
\end{scope}
\begin{scope}
\path[clip] (  0.00,  0.00) rectangle (209.58,115.63);
\definecolor{drawColor}{RGB}{247,192,26}
\definecolor{fillColor}{RGB}{247,192,26}

\path[draw=drawColor,line width= 0.4pt,line join=round,line cap=round,fill=fillColor] ( 37.92,106.50) circle (  1.53);
\end{scope}
\begin{scope}
\path[clip] (  0.00,  0.00) rectangle (209.58,115.63);
\definecolor{drawColor}{RGB}{78,155,133}

\path[draw=drawColor,line width= 0.5pt,dash pattern=on 4pt off 4pt ,line join=round] ( 33.59,100.32) -- ( 42.26,100.32);
\end{scope}
\begin{scope}
\path[clip] (  0.00,  0.00) rectangle (209.58,115.63);
\definecolor{fillColor}{RGB}{78,155,133}

\path[fill=fillColor] ( 36.39, 98.79) --
	( 39.46, 98.79) --
	( 39.46,101.86) --
	( 36.39,101.86) --
	cycle;
\end{scope}
\begin{scope}
\path[clip] (  0.00,  0.00) rectangle (209.58,115.63);
\definecolor{drawColor}{RGB}{0,0,0}

\node[text=drawColor,anchor=base west,inner sep=0pt, outer sep=0pt, scale=  0.70] at ( 47.84,104.09) {LVE (ACP)};
\end{scope}
\begin{scope}
\path[clip] (  0.00,  0.00) rectangle (209.58,115.63);
\definecolor{drawColor}{RGB}{0,0,0}

\node[text=drawColor,anchor=base west,inner sep=0pt, outer sep=0pt, scale=  0.70] at ( 47.84, 97.91) {LVE (ACP $\pm \varepsilon$)};
\end{scope}
\end{tikzpicture}}
		\caption{\acs{lve} run time (ms).}
		\label{fig:commutative_approx_plot_times_avg}
	\end{subfigure}
	\begin{subfigure}[t]{0.49\linewidth}
		\centering
		\resizebox{\linewidth}{!}{\input{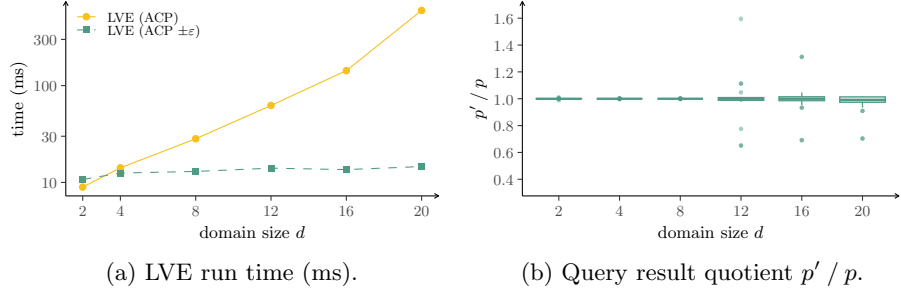}}
		\caption{Query result quotient $p' \mathbin{/} p$.}
		\label{fig:commutative_approx_plot_quots_avg}
	\end{subfigure}
	\caption{Downstream lifted inference performance, averaged over all generated instances and queries. (a) Run time of \ac{lve} on the compressed model returned by \ac{acp} and its $\varepsilon$-relaxed variant (\acs{acp} $\pm \varepsilon$). (b) Distribution of the per-query quotient $p' \mathbin{/} p$ between the approximate result $p'$ and the exact result $p$.}
	\label{fig:commutative_approx_plot_avg}
\end{figure}
\cref{fig:commutative_approx_plot_avg} reports the results averaged over all generated instances and queries.
\cref{fig:commutative_approx_plot_times_avg} shows that the compression of $\varepsilon$-commutative factors substantially accelerates downstream probabilistic inference, where the speedup increases exponentially with the domain size.
At the same time, \cref{fig:commutative_approx_plot_quots_avg} confirms that the query results (i.e., marginal probabilities) remain extremely close to those obtained with exact \ac{acp}: The per-query quotient $p' \mathbin{/} p$ concentrates tightly around the optimal value one across all evaluated configurations (except for very few outliers).
In other words, \acs{acp} $\pm \varepsilon$ unlocks the speedup of lifted inference at almost no loss of accuracy.
To give insights into the influence of $k$ and $\varepsilon$, we provide further experimental results for individual choices in~\cref{appendix:commutative_approx_further_results}.


\section{Conclusion}
We have introduced $\varepsilon$-commutativity, a principled relaxation of commutativity that tolerates the minute numerical deviations that inevitably arise in practice, e.g., when factor potentials are estimated from data.
On this basis, we have developed a representative-based compression scheme that collapses $\varepsilon$-equivalent potentials into a shared value, integrated the relaxation into the state-of-the-art \ac{acp} algorithm, and proved a sharp (optimal) bound on the resulting approximation error of probabilistic queries.
Our empirical evaluation confirms that the bound is loose in practice: Compressing $\varepsilon$-commutative factors unlocks substantial lifted inference speedups while the deviation of query results stays well within the theoretical guarantee.
As factors are merely non-negative functions, all proposed concepts directly apply to the detection and compact storage of approximately symmetric functions in more general settings.

An interesting direction for future work is to investigate approximate indistinguishability for \acp{fg} with continuous variables.
Currently, the usage of a \ac{crv} restricts the present work to the discrete case and a continuous counterpart of a \ac{crv} is required to handle $\varepsilon$-commutativity in models with continuous variables.

\section*{Acknowledgments}
This work was partially funded by the Ministry of Culture and Science of the German State of North Rhine-Westphalia.
This preprint has not undergone peer review or any post-submission improvements or corrections.
The Version of Record of this contribution is published in \emph{Lecture Notes in Computer Science, Volume 17188}, and is available online at \url{https://doi.org/10.1007/978-3-032-43514-9_14}.

\section*{Declaration on Generative AI}
During the preparation of this work, the authors used Claude Code (Anthropic) to discuss ideas and to polish text passages.
After using this tool, the authors reviewed and edited the content as needed and take full responsibility for the publication's content.

\bibliographystyle{splncs04}
\bibliography{references}

\clearpage
\appendix
\crefalias{section}{appendix}

\section{Detailed Proofs} \label{appendix:commutative_approx_proofs}
\symmetriclemma*
\begin{proof}
The proof follows the same structural argument as \cite[Lemma~6]{Luttermann2025c}.
For any permutation $\pi\in\Pi_{\boldsymbol C_{\phi}}$ of a given assignment $\boldsymbol{r} = (r_1,\ldots,r_n)$, by definition, the symmetric properties $\phi(\boldsymbol{r}) \leq \phi(r_{\pi(1)}, \ldots, r_{\pi(n)}) \cdot (1+\varepsilon)$ and $\phi(r_{\pi(1)}, \ldots, r_{\pi(n)}) \leq \phi(\boldsymbol{r}) \cdot (1+\varepsilon)$ hold, and therefore also the rearranged inequalities $\phi(\boldsymbol{r}) \cdot \frac{1}{1+\varepsilon} \leq \phi(r_{\pi(1)}, \ldots, r_{\pi(n)}) $ and $\phi(r_{\pi(1)}, \ldots, r_{\pi(n)})\cdot \frac{1}{1+\varepsilon}\leq \phi(\boldsymbol{r}) $.
\end{proof}

\lemmabound*
\begin{proof}
 A direct worst-case analysis requires some care.
Although the extrema $\max_{\boldsymbol r} \frac{\phi^*(\boldsymbol r)}{\phi(\boldsymbol r)}$ and $\min_{\boldsymbol r} \frac{\phi^*(\boldsymbol r)}{\phi(\boldsymbol r)}$ are taken independently over the entire assignment space, the mean value $\phi^*(\boldsymbol r)$ is uniquely determined within each equivalence class $[\boldsymbol{r}^j]_{\phi}$.
Consequently, the extremal configurations cannot always be chosen independently, and we therefore distinguish several cases.

First observe that for every equivalence class $[\boldsymbol{r}^j]_{\phi}$ and all $\boldsymbol{r}',\boldsymbol{r}'' \in [\boldsymbol{r}^j]_{\phi}$, \cref{lemma:pairwise_eps_equivalent} guarantees pairwise $\varepsilon$-equivalence and \cref{lemma:symmetry_varepsi_commut} implies $\phi(\boldsymbol{r}') \in [\phi(\boldsymbol{r}'') \cdot \frac{1}{1 + \varepsilon}, \phi(\boldsymbol{r}'') \cdot (1 + \varepsilon)] \text{ and }     \phi(\boldsymbol{r}'') \in [\phi(\boldsymbol{r}') \cdot \frac{1}{1 + \varepsilon}, \phi(\boldsymbol{r}') \cdot (1 + \varepsilon)].$
Hence, the potentials within one equivalence class can be ordered as
\begin{align*}
0 < \frac{1}{1+\varepsilon}\cdot \phi(\boldsymbol{r}^{j,\max})\leq \phi(\boldsymbol{r}^{j,\min}) \leq \phi(\boldsymbol{r}'')\leq \phi(\boldsymbol{r}^{j,\max})\leq \phi(\boldsymbol{r}^{j,\min})\cdot(1+\varepsilon),
\end{align*}
where $\phi(\boldsymbol{r}^{j,\max}):= \max_{\boldsymbol{r}' \in [\boldsymbol{r}^j]_{\phi}} \phi(\boldsymbol{r}')$ and $\phi(\boldsymbol{r}^{j,\min}):= \min_{\boldsymbol{r}' \in [\boldsymbol{r}^j]_{\phi}} \phi(\boldsymbol{r}')$. 
This directly affects the class-wise mean
\begin{align*}
\phi^*(\boldsymbol{r}'')=\overline{\phi}(\boldsymbol{r}'')
=
\frac{1}{\abs{[\boldsymbol{r}^j]_{\phi}}}
\sum_{\boldsymbol{r}' \in [\boldsymbol{r}^j]_{\phi}}
\phi(\boldsymbol{r}'),
\end{align*}
which is constrained to vary only within a restricted range of values.
Consider first the case where the assignments attaining the maximum and minimum originate from two different equivalence classes $[\boldsymbol{r}^i]_{\phi}$ and $[\boldsymbol{r}^j]_{\phi}$ with  $i\neq j$.
In this situation, the worst-case deviations can be chosen independently.

For the maximum, the denominator should be as small as possible.
Hence, choose $ \phi(\boldsymbol r)=\phi(\boldsymbol{r}^{i,\min})$ for the denominator, while all remaining potentials in the same equivalence class attain their maximal admissible value $\phi(\boldsymbol r'):=\phi(\boldsymbol{r}^{i,\max}) = (1+\varepsilon) \phi(\boldsymbol{r}^{i,\min})$ for all $\boldsymbol{r}' \in [\boldsymbol{r}^i]_{\phi}\setminus \{\boldsymbol{r} \}$ to maximise the numerator.
This yields
\begin{align*}
   \max_{\boldsymbol r\in [\boldsymbol{r}^i]_{\phi}} \frac{\phi^*(\boldsymbol r)}{\phi(\boldsymbol r)}&= \frac{(\phi(\boldsymbol{r}^{i,\min})+(\abs{[\boldsymbol{r}^i]_{\phi}}-1)\cdot (1+\varepsilon)\cdot \phi(\boldsymbol{r}^{i,\min}) ) \mathbin{/} \abs{[\boldsymbol{r}^i]_{\phi}}}{\phi(\boldsymbol{r}^{i,\min})}\\
   &= \frac{1+ (\abs{[\boldsymbol{r}^i]_{\phi}}-1)(1+\varepsilon)}{\abs{[\boldsymbol{r}^i]_{\phi}}}=1+ \frac{m_i-1}{m_i}\;\varepsilon.
\end{align*}
For the minimum, the denominator should instead be as large as possible. Thus, choose $\phi(\boldsymbol r)=\phi(\boldsymbol{r}^{j,\max})$, while all remaining values in the same equivalence class attain their minimal admissible value $\phi(\boldsymbol r'):=\phi(\boldsymbol{r}^{j,\min}) =\frac{1}{1+\varepsilon} \phi(\boldsymbol{r}^{j,\max})$ for all $\boldsymbol{r}' \in [\boldsymbol{r}^j]_{\phi}\setminus \{\boldsymbol{r} \}$.
Consequently,
\begin{align*}
   \min_{\boldsymbol r\in [\boldsymbol{r}^j]_{\phi}} \frac{\phi^*(\boldsymbol r)}{\phi(\boldsymbol r)}&= \frac{((\abs{[\boldsymbol{r}^j]_{\phi}}-1)\cdot\frac{1}{1+\varepsilon}\cdot \phi(\boldsymbol{r}^{j,\max})+ \phi(\boldsymbol{r}^{j,\max}) ) \mathbin{/} \abs{[\boldsymbol{r}^j]_{\phi}}}{\phi(\boldsymbol{r}^{j,\max})}\\
   &= \frac{(\abs{[\boldsymbol{r}^j]_{\phi}}-1)\cdot\frac{1}{1+\varepsilon} +1}{\abs{[\boldsymbol{r}^j]_{\phi}}}
   = \frac{\abs{[\boldsymbol{r}^j]_{\phi}}+\varepsilon}{(1+\varepsilon) \cdot \abs{[\boldsymbol{r}^j]_{\phi}}} = \frac{1+\frac{1}{m_j}\varepsilon}{1+\varepsilon}.
\end{align*}
Observe, that $1+ \frac{m_i-1}{m_i}\;\varepsilon$ is monotonically increasing in $m_i$, whereas $\frac{1+\frac{1}{m_j}\varepsilon}{1+\varepsilon}$ is monotonically decreasing in $m_j$.
Hence, the largest equivalence classes by group size produce the extremal deviations.
To determine which class size appear in the outer quotient as the numerator and denominator, respectively, consider $g(x):=\left(1+ \frac{x-1}{x}\;\varepsilon\right)\left(1+\frac{1}{x}\varepsilon\right)$.
Its derivative equals $g'(x)=\frac{\varepsilon^2(2-x)}{x^3}$, which is non-positive for all $x\geq 2$.
Therefore, for $x_1,x_2\in\mathbb{N}_{\geq 2}$ it follows
\begin{align*}
   \left(1+ \frac{x_1-1}{x_1}\;\varepsilon\right)\left(1+\frac{1}{x_1}\varepsilon\right)=  g(x_1) &\geq g(x_2) =\left(1+ \frac{x_2-1}{x_2}\;\varepsilon\right)\left(1+\frac{1}{x_2}\varepsilon\right)\\
   \Leftrightarrow \;\;\;\;\; \frac{1+ \frac{x_1-1}{x_1}\;\varepsilon}{1+\frac{1}{x_2}\varepsilon}&\geq \frac{1+ \frac{x_2-1}{x_2}\;\varepsilon}{1+\frac{1}{x_1}\varepsilon}. 
\end{align*}
Consequently, the sharpest bound is obtained by using $m_2$ for the maximum term and $m_1$ for the minimum term.
If several equivalence classes with maximal group size exist, then simply $m_1=m_2$.

It remains to analyse the case where both extrema originate from the same equivalence class.
Since $\phi^*(\boldsymbol r)=\phi^*(\boldsymbol r')$ for all $\boldsymbol{r},\boldsymbol{r}'\in [\boldsymbol{r}^j]_{\phi}$, we obtain
\begin{align*}
    &\frac{\max_{\boldsymbol r\in [\boldsymbol{r}^j]_{\phi}} \frac{\phi^*(\boldsymbol r)}{\phi(\boldsymbol r)}}{\min_{\boldsymbol r'\in [\boldsymbol{r}^j]_{\phi}} \frac{\phi^*(\boldsymbol r')}{\phi(\boldsymbol r')}}\;= \frac{\max_{\boldsymbol r\in [\boldsymbol{r}^j]_{\phi}} \frac{1}{\phi(\boldsymbol r)}}{\min_{\boldsymbol r'\in [\boldsymbol{r}^j]_{\phi}} \frac{1}{\phi(\boldsymbol r')}}.
\end{align*}
If the equivalence class contains only a single assignment, the quotient equals $1$. Otherwise, $m_1\geq m_2\geq 2$ holds, and the quotient is maximised by choosing $\phi(\boldsymbol r)=\phi(\boldsymbol{r}^{j,\min}) =\frac{1}{1+\varepsilon} \phi(\boldsymbol{r}^{j,\max})$ and $\phi(\boldsymbol r')=\phi(\boldsymbol{r}^{j,\max}) =(1+\varepsilon) \phi(\boldsymbol{r}^{j,\min})$.
Hence, for $m_2\leq m_1$ we get
\begin{align*}
    \frac{\max_{\boldsymbol r\in [\boldsymbol{r}^j]_{\phi}} \frac{1}{\phi(\boldsymbol r)}}{\min_{\boldsymbol r'\in [\boldsymbol{r}^j]_{\phi}} \frac{1}{\phi(\boldsymbol r')}}&\leq \frac{\frac{1}{\phi(\boldsymbol{r}^{j,\min})}}{\frac{1}{(1+\varepsilon) \phi(\boldsymbol{r}^{j,\min})}}=1+\varepsilon \leq \frac{\left(1+\frac{m_2-1}{m_2}\varepsilon\right)(1+\varepsilon)}{1+\frac{1}{m_1}\varepsilon},
\end{align*}
which establishes the bound in all cases.   
\end{proof}

\boundtheorem*
For the proof, consider the following example.
\begin{example}\label{ex:sharp}
Let $M$ be the \ac{fg} with factors $\phi_1,\ldots,\phi_m$, of which $\phi_1,\ldots,\phi_k$ with $0\leq k \leq m$ are $\varepsilon$-commutative with respect to disjoint argument sets $\boldsymbol{C}_{\phi_1},\ldots, \allowbreak \boldsymbol{C}_{\phi_k}$ with finite domains, respectively, each satisfying $\abs{\boldsymbol{C}_{\phi_i}}\geq 2$.
Let $\mathcal{X}_{\mathcal{R}_{\phi_i}} = \dot \cup_{j=1}^{l_i} \{\boldsymbol{r}_i^{j}\}$ be the set of disjoint representatives $\boldsymbol{r}^{j}_i$ for all disjoint equivalence classes $[\boldsymbol{r}^{j}_i]_{\phi_i}$ united to $\dot\cup_{j=1}^{l_i} [\boldsymbol{r}^{j}_i]_{\phi_i} = \mathcal{X}_{\boldsymbol{R}_{\phi_i}}$ ordered by group size with $m_{1_i}:=\abs{[\boldsymbol{r}^{1}_i]_{\phi_i}}\geq\ldots \geq m_{l_i}:=\abs{[\boldsymbol{r}^{l}_i]_{\phi_i}}\geq 1$ for every $\varepsilon$-commutative factor $\phi_i$ with $i=1,\ldots,k$.
 Let $\mathcal{X}_{R_j} =\range{R_j}$ have at least two deviating elements and within a group $\boldsymbol{C}_{\phi_i}$ the same ranges.
For $i=1,\ldots,k$ define
\begin{align*}
\phi_i(\boldsymbol{r}_i)
:=
\begin{cases}
\;(1+\varepsilon)\cdot 1 & \text{if } \boldsymbol{r}_i = \boldsymbol{r}_i^1,
\\
\;1 & \text{if } \boldsymbol{r}_i \in [\boldsymbol{r}_i^1]_{\phi_i} \setminus \{\boldsymbol{r}_i^1\},\\
\;2 & \text{if } \boldsymbol{r}_i=\boldsymbol{r}_i^2,\\ 
\;(1+\varepsilon)\cdot 2 & \text{if } \boldsymbol{r}_i \in [\boldsymbol{r}_i^2]_{\phi_i} \setminus \{\boldsymbol{r}_i^2\},\\
\;j & \text{for } \boldsymbol{r}_i\in [\boldsymbol{r}_i^j]_{\phi_i} \text{ and }2<j \leq l_i.
\end{cases}
\end{align*}
Thus, the first equivalence class contains exactly one maximal value and otherwise minimal values, whereas the second class contains exactly one minimal value and otherwise maximal ones, 
following the same construction as in the proof of \cref{lemma:bound}.
All remaining parts are constant and therefore contribute no approximation error.
\end{example}

\begin{proof}[Proof of \cref{th:sharp_optimal}]
We show that the construction from \cref{ex:sharp} attains the bound.
Let $M'$ denote the model obtained from $M$ by replacing every $\varepsilon$-commutative factor $\phi_i$ by the arithmetic mean as in \cref{eq:mean_sym}.
Notice that $\phi_i^*(\boldsymbol{r}_i) = \phi_i(\boldsymbol{r}_i)$ for $i=k+1,\ldots,m$ and all $\boldsymbol{r}_i\in\mathcal{X}_{\boldsymbol{R}_{\phi_i}}$, but also $\phi_i^*(\boldsymbol{r}_i) = \phi_i(\boldsymbol{r}_i)$ for $i=1,\ldots,k$ and $\boldsymbol{r}_i\in [\boldsymbol{r}_i^j]_{\phi_i}$ and $2<j \leq l_i$ implying $\frac{ \phi_i^*(\boldsymbol{r}_i)}{\phi_i(\boldsymbol{r}_i)}=1$.
Consequently, all extremal deviation are attained exclusively within the first two equivalence classes. 
For $\boldsymbol{r}_i \in [\boldsymbol{r}_i^1]_{\phi_i}$, the definition results in
\begin{align*}
    \phi_i^*(\boldsymbol{r}_i)= \frac{(1+\varepsilon)+(m_{1_i}-1)}{m_{1_i}}=1+\frac{1}{m_{1_i}}\varepsilon,
\end{align*}
and for $\boldsymbol{r}'_i \in [\boldsymbol{r}_i^2]_{\phi_i}$ in
\begin{align*}
\phi_i^*(\boldsymbol{r}'_i)= \frac{(1+\varepsilon)\cdot 2\cdot (m_{2_i}-1) + 2}{m_{2_i}}=2\cdot\left(1+\frac{(m_{2_i}-1)}{m_{2_i}}\varepsilon\right).
\end{align*}
For $\boldsymbol{r}_i \in [\boldsymbol{r}_i^1]_{\phi_i}$ the maximum becomes
\begin{align*}
    \max_{\boldsymbol{r}_i \in [\boldsymbol{r}_i^1]_{\phi_i}}\frac{\phi_i^*(\boldsymbol{r}_i)}{\phi_i(\boldsymbol{r}_i)} = \frac{1+\frac{1}{m_{1_i}}\varepsilon}{1}=1+\frac{1}{m_{1_i}}\varepsilon>1,
\end{align*}
and for $\boldsymbol{r}'_i \in [\boldsymbol{r}_i^2]_{\phi_i}$ it becomes
\begin{align*}
\max_{\boldsymbol{r}'_i \in [\boldsymbol{r}_i^2]_{\phi_i}} \frac{\phi_i^*(\boldsymbol{r}'_i)}{\phi_i(\boldsymbol{r}'_i)}= \frac{2\cdot\left(1+\frac{(m_{2_i}-1)}{m_{2_i}}\varepsilon\right)}{2} =1+\frac{(m_{2_i}-1)}{m_{2_i}}\varepsilon>1,
\end{align*}
which is larger than the first one due to $m_{2_i}\geq 2$.

Analogously, for $\boldsymbol{r}_i \in [\boldsymbol{r}_i^1]_{\phi_i}$ the minimal quotient inside this class reaches
\begin{align*}
 \min_{\boldsymbol{r}_i \in [\boldsymbol{r}_i^1]_{\phi_i}}\frac{\phi_i^*(\boldsymbol{r}_i)}{\phi_i(\boldsymbol{r}_i)} = \frac{1+\frac{1}{m_{1_i}}\varepsilon}{1+\varepsilon}<1,    
\end{align*}
which is as small as the following one for $\boldsymbol{r}'_i \in [\boldsymbol{r}_i^2]_{\phi_i}$ due to $m_{2_i}\geq 2$
\begin{align*}
  \min_{\boldsymbol{r}'_i \in [\boldsymbol{r}_i^2]_{\phi_i}} \frac{\phi_i^*(\boldsymbol{r}'_i)}{\phi_i(\boldsymbol{r}'_i)}= \frac{2\cdot\left(1+\frac{(m_{2_i}-1)}{m_{2_i}}\varepsilon\right)}{2\cdot(1+\varepsilon)}= \frac{1+\frac{(m_{2_i}-1)}{m_{2_i}}\varepsilon}{1+\varepsilon}<1. 
\end{align*} 
Using \cref{eq:eacp_distance_measure} from the definition of the Chan-Darwiche distance leads to
\begin{align*}
    D&_{CD}(P_M, P_{M'})=   \ln \;\max_{\boldsymbol r} \frac{\psi'(\boldsymbol r)}{\psi(\boldsymbol r)} - \ln\; \min_{\boldsymbol r} \frac{\psi'(\boldsymbol r)}{\psi(\boldsymbol r)}\\
        &= \ln \;\max_{\boldsymbol r}\frac{\prod_{i=1}^k \phi_i^*(\boldsymbol{r}_i)\prod_{i=k+1}^m \phi_i(\boldsymbol{r}_i)}{\prod_{i=1}^m \phi_i(\boldsymbol{r}_i)} - \ln \;\min_{\boldsymbol r} \frac{\prod_{i=1}^k \phi_i^*(\boldsymbol{r}_i)\prod_{i=k+1}^m \phi_i(\boldsymbol{r}_i)}{\prod_{i=1}^m \phi_i(\boldsymbol{r}_i)}\\
        &= \ln \;\max_{\boldsymbol r}\prod_{i=1}^k\frac{ \phi_i^*(\boldsymbol{r}_i)}{\phi_i(\boldsymbol{r}_i)} - \ln\; \min_{\boldsymbol r}\prod_{i=1}^k \frac{\phi_i^*(\boldsymbol{r}_i)}{\phi_i(\boldsymbol{r}_i)}.
\end{align*}
Since the argument sets are pairwise disjoint, the extrema can be attained independently for every factor.
Hence,
\begin{align*} 
		D&_{CD}(P_M, P_{M'}) =\ln\; \prod_{i=1}^k \max_{\boldsymbol r_i} \frac{\phi_i^*(\boldsymbol{r}_i)}{\phi_i(\boldsymbol{r}_i)} - \ln \;\prod_{i=1}^k \min_{\boldsymbol r_i} \frac{\phi_i^*(\boldsymbol{r}_i)}{\phi_i(\boldsymbol{r}_i)}\\
&= \ln\; \prod_{i=1}^k \max_{\boldsymbol{r}'_i \in [\boldsymbol{r}_i^2]_{\phi_i}} \frac{\phi_i^*(\boldsymbol{r}'_i)}{\phi_i(\boldsymbol{r}'_i)} - \ln \;\prod_{i=1}^k \min_{\boldsymbol{r}_i \in [\boldsymbol{r}_i^1]_{\phi_i}} \frac{\phi_i^*(\boldsymbol{r}_i)}{\phi_i(\boldsymbol{r}_i)}\\
        &= \ln\; \prod_{i=1}^k \left(1+\frac{(m_{2_i}-1)}{m_{2_i}}\varepsilon\right) -\ln\; \prod_{i=1}^k \frac{1+\frac{1}{m_{1_i}}\varepsilon}{1+\varepsilon}\\
        &= \ln \;\prod_{i=1}^k \frac{\left( 1+\frac{m_{2_i}-1}{m_{2_i}}\varepsilon\right)( 1+\varepsilon)}{1+\frac{1}{m_{1_i}}\varepsilon}.
\end{align*} 
Thus, the bound of \cref{theorem:genral_bound} is attained exactly and is therefore optimal.
\qedhere
\end{proof}

\subsection{Additional Result Symmetry Set} \label{app:additional_results}

\begin{proposition}[Cardinality of the Symmetry Set]
\label{prop:orbit_cardinality} 
Let $\boldsymbol{r}=(r_1,\ldots,r_n)\in\mathcal X_{\boldsymbol R_{\phi}}$
and $ I_{\boldsymbol C_{\phi}}:=\{\, i \mid R_i \in \boldsymbol C_{\phi} \,\} $ denote the indices of the commutative arguments.
Further, let $\{a_1,\ldots,a_k\}=\{\, r_i \mid i\in I_{\boldsymbol C_{\phi}} \,\}$
be the set of distinct values occurring among the commutative arguments,
and let $ n_j:=\bigl|\{\, i\in I_{\boldsymbol C_{\phi}} \mid r_i=a_j \,\}\bigr| $
denote the multiplicity of value $a_j$ for a fixed assignment $\boldsymbol{r}$.
Then the cardinality of the symmetry set is 
\begin{align*}
|\mathfrak{S}_{\phi}(\boldsymbol{r})|
=\frac{|\boldsymbol C_{\phi}|!}
 {\prod_{j=1}^{k} n_j!} 
\end{align*}
with $ 1 \leq |\mathfrak{S}_{\phi}(\boldsymbol{r})| \leq |\boldsymbol C_{\phi}|!$, where the upper bound is attained if and only if all commutative arguments have pairwise non-$\varepsilon$-equivalent values.
\end{proposition}

\begin{proof}
Since permutations in
$\Pi_{\boldsymbol C_{\phi}}$
act only on the positions corresponding to $\varepsilon$-commutative arguments, it holds that $|\Pi_{\boldsymbol C_{\phi}}|=|\boldsymbol C_{\phi}|!$.

If all values among the commutative arguments are pairwise distinct,
every permutation produces a distinct tuple, and thus
$|\mathfrak{S}_{\phi}(\boldsymbol{r})| =|\boldsymbol C_{\phi}|!$.

Now assume that some values coincide.
For each distinct value $a_j$, there are exactly $m_j$ positions among
the commutative arguments whose assigned value equals $a_j$.
Permuting these $m_j$ positions among themselves does not change the
resulting tuple.
Hence, every distinct tuple in the symmetry set is generated exactly $ \prod_{j=1}^{k} n_j!$ times by permutations in
$\Pi_{\boldsymbol C_{\phi}}$.
Since there are $|\boldsymbol C_{\phi}|!$ permutations in total, the number of distinct tuples is 
\begin{align*}
|\mathfrak{S}_{\phi}(\boldsymbol{r})| =\frac{|\boldsymbol C_{\phi}|!}
{\prod_{j=1}^{k} n_j!}.\tag*\qedhere
\end{align*} 
\end{proof}
\section{Additional Experimental Results} \label{appendix:commutative_approx_further_results}
In addition to the experimental results provided in \cref{sec:commutative_approx_eval}, we provide further experimental results for individual scenarios in this section.

\begin{figure}[t]
	\centering
	\begin{subfigure}[t]{0.49\linewidth}
		\centering
		\resizebox{\linewidth}{!}{
\begin{tikzpicture}[x=1pt,y=1pt]
\definecolor{fillColor}{RGB}{255,255,255}
\path[use as bounding box,fill=fillColor,fill opacity=0.00] (0,0) rectangle (209.58,115.63);
\begin{scope}
\path[clip] (  0.00,  0.00) rectangle (209.58,115.63);
\definecolor{drawColor}{RGB}{255,255,255}
\definecolor{fillColor}{RGB}{255,255,255}

\path[draw=drawColor,line width= 0.5pt,line join=round,line cap=round,fill=fillColor] (  0.00,  0.00) rectangle (209.58,115.63);
\end{scope}
\begin{scope}
\path[clip] ( 28.25, 22.31) rectangle (203.58,113.63);
\definecolor{fillColor}{RGB}{255,255,255}

\path[fill=fillColor] ( 28.25, 22.31) rectangle (203.58,113.63);
\definecolor{drawColor}{RGB}{247,192,26}

\path[draw=drawColor,line width= 0.5pt,line join=round] ( 36.22, 26.46) --
	( 53.93, 36.26) --
	( 89.35, 47.61) --
	(124.77, 60.53) --
	(160.19, 78.97) --
	(195.61,109.48);
\definecolor{drawColor}{RGB}{78,155,133}

\path[draw=drawColor,line width= 0.5pt,dash pattern=on 4pt off 4pt ,line join=round] ( 36.22, 31.16) --
	( 53.93, 35.44) --
	( 89.35, 35.80) --
	(124.77, 35.60) --
	(160.19, 36.37) --
	(195.61, 37.78);
\definecolor{drawColor}{RGB}{247,192,26}
\definecolor{fillColor}{RGB}{247,192,26}

\path[draw=drawColor,line width= 0.4pt,line join=round,line cap=round,fill=fillColor] ( 36.22, 26.46) circle (  1.53);
\definecolor{fillColor}{RGB}{78,155,133}

\path[fill=fillColor] ( 34.68, 29.62) --
	( 37.75, 29.62) --
	( 37.75, 32.69) --
	( 34.68, 32.69) --
	cycle;
\definecolor{fillColor}{RGB}{247,192,26}

\path[draw=drawColor,line width= 0.4pt,line join=round,line cap=round,fill=fillColor] ( 53.93, 36.26) circle (  1.53);
\definecolor{fillColor}{RGB}{78,155,133}

\path[fill=fillColor] ( 52.39, 33.91) --
	( 55.46, 33.91) --
	( 55.46, 36.98) --
	( 52.39, 36.98) --
	cycle;
\definecolor{fillColor}{RGB}{247,192,26}

\path[draw=drawColor,line width= 0.4pt,line join=round,line cap=round,fill=fillColor] ( 89.35, 47.61) circle (  1.53);
\definecolor{fillColor}{RGB}{78,155,133}

\path[fill=fillColor] ( 87.81, 34.27) --
	( 90.88, 34.27) --
	( 90.88, 37.33) --
	( 87.81, 37.33) --
	cycle;
\definecolor{fillColor}{RGB}{247,192,26}

\path[draw=drawColor,line width= 0.4pt,line join=round,line cap=round,fill=fillColor] (124.77, 60.53) circle (  1.53);
\definecolor{fillColor}{RGB}{78,155,133}

\path[fill=fillColor] (123.24, 34.06) --
	(126.30, 34.06) --
	(126.30, 37.13) --
	(123.24, 37.13) --
	cycle;
\definecolor{fillColor}{RGB}{247,192,26}

\path[draw=drawColor,line width= 0.4pt,line join=round,line cap=round,fill=fillColor] (160.19, 78.97) circle (  1.53);
\definecolor{fillColor}{RGB}{78,155,133}

\path[fill=fillColor] (158.66, 34.84) --
	(161.73, 34.84) --
	(161.73, 37.90) --
	(158.66, 37.90) --
	cycle;
\definecolor{fillColor}{RGB}{247,192,26}

\path[draw=drawColor,line width= 0.4pt,line join=round,line cap=round,fill=fillColor] (195.61,109.48) circle (  1.53);
\definecolor{fillColor}{RGB}{78,155,133}

\path[fill=fillColor] (194.08, 36.25) --
	(197.15, 36.25) --
	(197.15, 39.31) --
	(194.08, 39.31) --
	cycle;
\end{scope}
\begin{scope}
\path[clip] (  0.00,  0.00) rectangle (209.58,115.63);
\definecolor{drawColor}{RGB}{0,0,0}

\path[draw=drawColor,line width= 0.5pt,line join=round] ( 28.25, 22.31) --
	( 28.25,113.63);

\path[draw=drawColor,line width= 0.5pt,line join=round] ( 29.38,111.66) --
	( 28.25,113.63) --
	( 27.11,111.66);
\end{scope}
\begin{scope}
\path[clip] (  0.00,  0.00) rectangle (209.58,115.63);
\definecolor{drawColor}{gray}{0.30}

\node[text=drawColor,anchor=base east,inner sep=0pt, outer sep=0pt, scale=  0.80] at ( 24.20, 32.74) {10};

\node[text=drawColor,anchor=base east,inner sep=0pt, outer sep=0pt, scale=  0.80] at ( 24.20, 56.31) {30};

\node[text=drawColor,anchor=base east,inner sep=0pt, outer sep=0pt, scale=  0.80] at ( 24.20, 82.13) {100};

\node[text=drawColor,anchor=base east,inner sep=0pt, outer sep=0pt, scale=  0.80] at ( 24.20,105.70) {300};
\end{scope}
\begin{scope}
\path[clip] (  0.00,  0.00) rectangle (209.58,115.63);
\definecolor{drawColor}{gray}{0.20}

\path[draw=drawColor,line width= 0.5pt,line join=round] ( 26.00, 35.49) --
	( 28.25, 35.49);

\path[draw=drawColor,line width= 0.5pt,line join=round] ( 26.00, 59.06) --
	( 28.25, 59.06);

\path[draw=drawColor,line width= 0.5pt,line join=round] ( 26.00, 84.89) --
	( 28.25, 84.89);

\path[draw=drawColor,line width= 0.5pt,line join=round] ( 26.00,108.45) --
	( 28.25,108.45);
\end{scope}
\begin{scope}
\path[clip] (  0.00,  0.00) rectangle (209.58,115.63);
\definecolor{drawColor}{RGB}{0,0,0}

\path[draw=drawColor,line width= 0.5pt,line join=round] ( 28.25, 22.31) --
	(203.58, 22.31);

\path[draw=drawColor,line width= 0.5pt,line join=round] (201.61, 21.17) --
	(203.58, 22.31) --
	(201.61, 23.45);
\end{scope}
\begin{scope}
\path[clip] (  0.00,  0.00) rectangle (209.58,115.63);
\definecolor{drawColor}{gray}{0.20}

\path[draw=drawColor,line width= 0.5pt,line join=round] ( 36.22, 20.06) --
	( 36.22, 22.31);

\path[draw=drawColor,line width= 0.5pt,line join=round] ( 53.93, 20.06) --
	( 53.93, 22.31);

\path[draw=drawColor,line width= 0.5pt,line join=round] ( 89.35, 20.06) --
	( 89.35, 22.31);

\path[draw=drawColor,line width= 0.5pt,line join=round] (124.77, 20.06) --
	(124.77, 22.31);

\path[draw=drawColor,line width= 0.5pt,line join=round] (160.19, 20.06) --
	(160.19, 22.31);

\path[draw=drawColor,line width= 0.5pt,line join=round] (195.61, 20.06) --
	(195.61, 22.31);
\end{scope}
\begin{scope}
\path[clip] (  0.00,  0.00) rectangle (209.58,115.63);
\definecolor{drawColor}{gray}{0.30}

\node[text=drawColor,anchor=base,inner sep=0pt, outer sep=0pt, scale=  0.80] at ( 36.22, 12.75) {2};

\node[text=drawColor,anchor=base,inner sep=0pt, outer sep=0pt, scale=  0.80] at ( 53.93, 12.75) {4};

\node[text=drawColor,anchor=base,inner sep=0pt, outer sep=0pt, scale=  0.80] at ( 89.35, 12.75) {8};

\node[text=drawColor,anchor=base,inner sep=0pt, outer sep=0pt, scale=  0.80] at (124.77, 12.75) {12};

\node[text=drawColor,anchor=base,inner sep=0pt, outer sep=0pt, scale=  0.80] at (160.19, 12.75) {16};

\node[text=drawColor,anchor=base,inner sep=0pt, outer sep=0pt, scale=  0.80] at (195.61, 12.75) {20};
\end{scope}
\begin{scope}
\path[clip] (  0.00,  0.00) rectangle (209.58,115.63);
\definecolor{drawColor}{RGB}{0,0,0}

\node[text=drawColor,anchor=base,inner sep=0pt, outer sep=0pt, scale=  0.90] at (115.91,  2.75) {domain size $d$};
\end{scope}
\begin{scope}
\path[clip] (  0.00,  0.00) rectangle (209.58,115.63);
\definecolor{drawColor}{RGB}{0,0,0}

\node[text=drawColor,rotate= 90.00,anchor=base,inner sep=0pt, outer sep=0pt, scale=  0.90] at (  8.20, 67.97) {time (ms)};
\end{scope}
\begin{scope}
\path[clip] (  0.00,  0.00) rectangle (209.58,115.63);

\path[] ( 32.50, 97.23) rectangle (101.14,113.60);
\end{scope}
\begin{scope}
\path[clip] (  0.00,  0.00) rectangle (209.58,115.63);
\definecolor{drawColor}{RGB}{247,192,26}

\path[draw=drawColor,line width= 0.5pt,line join=round] ( 33.59,106.50) -- ( 42.26,106.50);
\end{scope}
\begin{scope}
\path[clip] (  0.00,  0.00) rectangle (209.58,115.63);
\definecolor{drawColor}{RGB}{247,192,26}
\definecolor{fillColor}{RGB}{247,192,26}

\path[draw=drawColor,line width= 0.4pt,line join=round,line cap=round,fill=fillColor] ( 37.92,106.50) circle (  1.53);
\end{scope}
\begin{scope}
\path[clip] (  0.00,  0.00) rectangle (209.58,115.63);
\definecolor{drawColor}{RGB}{78,155,133}

\path[draw=drawColor,line width= 0.5pt,dash pattern=on 4pt off 4pt ,line join=round] ( 33.59,100.32) -- ( 42.26,100.32);
\end{scope}
\begin{scope}
\path[clip] (  0.00,  0.00) rectangle (209.58,115.63);
\definecolor{fillColor}{RGB}{78,155,133}

\path[fill=fillColor] ( 36.39, 98.79) --
	( 39.46, 98.79) --
	( 39.46,101.86) --
	( 36.39,101.86) --
	cycle;
\end{scope}
\begin{scope}
\path[clip] (  0.00,  0.00) rectangle (209.58,115.63);
\definecolor{drawColor}{RGB}{0,0,0}

\node[text=drawColor,anchor=base west,inner sep=0pt, outer sep=0pt, scale=  0.70] at ( 47.84,104.09) {LVE (ACP)};
\end{scope}
\begin{scope}
\path[clip] (  0.00,  0.00) rectangle (209.58,115.63);
\definecolor{drawColor}{RGB}{0,0,0}

\node[text=drawColor,anchor=base west,inner sep=0pt, outer sep=0pt, scale=  0.70] at ( 47.84, 97.91) {LVE (ACP $\pm \varepsilon$)};
\end{scope}
\end{tikzpicture}}
		\caption{\acs{lve} run time (ms).}
		\label{fig:commutative_approx_plot_times_k=1_eps=0.001}
	\end{subfigure}
	\begin{subfigure}[t]{0.49\linewidth}
		\centering
		\resizebox{\linewidth}{!}{
\begin{tikzpicture}[x=1pt,y=1pt]
\definecolor{fillColor}{RGB}{255,255,255}
\path[use as bounding box,fill=fillColor,fill opacity=0.00] (0,0) rectangle (209.58,115.63);
\begin{scope}
\path[clip] (  0.00,  0.00) rectangle (209.58,115.63);
\definecolor{drawColor}{RGB}{255,255,255}
\definecolor{fillColor}{RGB}{255,255,255}

\path[draw=drawColor,line width= 0.5pt,line join=round,line cap=round,fill=fillColor] (  0.00,  0.00) rectangle (209.58,115.63);
\end{scope}
\begin{scope}
\path[clip] ( 26.47, 22.31) rectangle (203.58,113.63);
\definecolor{fillColor}{RGB}{255,255,255}

\path[fill=fillColor] ( 26.47, 22.31) rectangle (203.58,113.63);
\definecolor{drawColor}{RGB}{78,155,133}
\definecolor{fillColor}{RGB}{78,155,133}

\path[draw=drawColor,draw opacity=0.20,line width= 0.4pt,line join=round,line cap=round,fill=fillColor,fill opacity=0.20] ( 43.61, 67.96) circle (  0.78);

\path[draw=drawColor,draw opacity=0.20,line width= 0.4pt,line join=round,line cap=round,fill=fillColor,fill opacity=0.20] ( 43.61, 67.98) circle (  0.78);
\definecolor{drawColor}{RGB}{78,155,133}

\path[draw=drawColor,line width= 0.6pt,line join=round] ( 43.61, 67.97) -- ( 43.61, 67.97);

\path[draw=drawColor,line width= 0.6pt,line join=round] ( 43.61, 67.97) -- ( 43.61, 67.97);

\path[draw=drawColor,line width= 0.6pt,fill=fillColor,fill opacity=0.20] ( 32.90, 67.97) --
	( 32.90, 67.97) --
	( 54.32, 67.97) --
	( 54.32, 67.97) --
	( 32.90, 67.97) --
	cycle;

\path[draw=drawColor,line width= 1.1pt] ( 32.90, 67.97) -- ( 54.32, 67.97);
\definecolor{drawColor}{RGB}{78,155,133}

\path[draw=drawColor,draw opacity=0.20,line width= 0.4pt,line join=round,line cap=round,fill=fillColor,fill opacity=0.20] ( 72.17, 67.97) circle (  0.78);

\path[draw=drawColor,draw opacity=0.20,line width= 0.4pt,line join=round,line cap=round,fill=fillColor,fill opacity=0.20] ( 72.17, 67.97) circle (  0.78);
\definecolor{drawColor}{RGB}{78,155,133}

\path[draw=drawColor,line width= 0.6pt,line join=round] ( 72.17, 67.97) -- ( 72.17, 67.97);

\path[draw=drawColor,line width= 0.6pt,line join=round] ( 72.17, 67.97) -- ( 72.17, 67.97);

\path[draw=drawColor,line width= 0.6pt,fill=fillColor,fill opacity=0.20] ( 61.46, 67.97) --
	( 61.46, 67.97) --
	( 82.89, 67.97) --
	( 82.89, 67.97) --
	( 61.46, 67.97) --
	cycle;

\path[draw=drawColor,line width= 1.1pt] ( 61.46, 67.97) -- ( 82.89, 67.97);
\definecolor{drawColor}{RGB}{78,155,133}

\path[draw=drawColor,draw opacity=0.20,line width= 0.4pt,line join=round,line cap=round,fill=fillColor,fill opacity=0.20] (100.74, 67.97) circle (  0.78);

\path[draw=drawColor,draw opacity=0.20,line width= 0.4pt,line join=round,line cap=round,fill=fillColor,fill opacity=0.20] (100.74, 67.97) circle (  0.78);
\definecolor{drawColor}{RGB}{78,155,133}

\path[draw=drawColor,line width= 0.6pt,line join=round] (100.74, 67.97) -- (100.74, 67.97);

\path[draw=drawColor,line width= 0.6pt,line join=round] (100.74, 67.97) -- (100.74, 67.97);

\path[draw=drawColor,line width= 0.6pt,fill=fillColor,fill opacity=0.20] ( 90.03, 67.97) --
	( 90.03, 67.97) --
	(111.45, 67.97) --
	(111.45, 67.97) --
	( 90.03, 67.97) --
	cycle;

\path[draw=drawColor,line width= 1.1pt] ( 90.03, 67.97) -- (111.45, 67.97);
\definecolor{drawColor}{RGB}{78,155,133}

\path[draw=drawColor,draw opacity=0.20,line width= 0.4pt,line join=round,line cap=round,fill=fillColor,fill opacity=0.20] (129.31,105.71) circle (  0.78);

\path[draw=drawColor,draw opacity=0.20,line width= 0.4pt,line join=round,line cap=round,fill=fillColor,fill opacity=0.20] (129.31, 53.82) circle (  0.78);
\definecolor{drawColor}{RGB}{78,155,133}

\path[draw=drawColor,line width= 0.6pt,line join=round] (129.31, 69.14) -- (129.31, 70.94);

\path[draw=drawColor,line width= 0.6pt,line join=round] (129.31, 67.67) -- (129.31, 67.22);

\path[draw=drawColor,line width= 0.6pt,fill=fillColor,fill opacity=0.20] (118.60, 69.14) --
	(118.60, 67.67) --
	(140.02, 67.67) --
	(140.02, 69.14) --
	(118.60, 69.14) --
	cycle;

\path[draw=drawColor,line width= 1.1pt] (118.60, 67.97) -- (140.02, 67.97);

\path[draw=drawColor,line width= 0.6pt,line join=round] (157.88, 68.92) -- (157.88, 70.93);

\path[draw=drawColor,line width= 0.6pt,line join=round] (157.88, 67.14) -- (157.88, 65.05);

\path[draw=drawColor,line width= 0.6pt,fill=fillColor,fill opacity=0.20] (147.16, 68.92) --
	(147.16, 67.14) --
	(168.59, 67.14) --
	(168.59, 68.92) --
	(147.16, 68.92) --
	cycle;

\path[draw=drawColor,line width= 1.1pt] (147.16, 68.26) -- (168.59, 68.26);

\path[draw=drawColor,line width= 0.6pt,line join=round] (186.44, 68.58) -- (186.44, 69.10);

\path[draw=drawColor,line width= 0.6pt,line join=round] (186.44, 66.32) -- (186.44, 63.93);

\path[draw=drawColor,line width= 0.6pt,fill=fillColor,fill opacity=0.20] (175.73, 68.58) --
	(175.73, 66.32) --
	(197.16, 66.32) --
	(197.16, 68.58) --
	(175.73, 68.58) --
	cycle;

\path[draw=drawColor,line width= 1.1pt] (175.73, 67.41) -- (197.16, 67.41);
\end{scope}
\begin{scope}
\path[clip] (  0.00,  0.00) rectangle (209.58,115.63);
\definecolor{drawColor}{RGB}{0,0,0}

\path[draw=drawColor,line width= 0.5pt,line join=round] ( 26.47, 22.31) --
	( 26.47,113.63);

\path[draw=drawColor,line width= 0.5pt,line join=round] ( 27.61,111.66) --
	( 26.47,113.63) --
	( 25.33,111.66);
\end{scope}
\begin{scope}
\path[clip] (  0.00,  0.00) rectangle (209.58,115.63);
\definecolor{drawColor}{gray}{0.30}

\node[text=drawColor,anchor=base east,inner sep=0pt, outer sep=0pt, scale=  0.80] at ( 22.42, 27.37) {0.4};

\node[text=drawColor,anchor=base east,inner sep=0pt, outer sep=0pt, scale=  0.80] at ( 22.42, 39.99) {0.6};

\node[text=drawColor,anchor=base east,inner sep=0pt, outer sep=0pt, scale=  0.80] at ( 22.42, 52.60) {0.8};

\node[text=drawColor,anchor=base east,inner sep=0pt, outer sep=0pt, scale=  0.80] at ( 22.42, 65.22) {1.0};

\node[text=drawColor,anchor=base east,inner sep=0pt, outer sep=0pt, scale=  0.80] at ( 22.42, 77.83) {1.2};

\node[text=drawColor,anchor=base east,inner sep=0pt, outer sep=0pt, scale=  0.80] at ( 22.42, 90.45) {1.4};

\node[text=drawColor,anchor=base east,inner sep=0pt, outer sep=0pt, scale=  0.80] at ( 22.42,103.06) {1.6};
\end{scope}
\begin{scope}
\path[clip] (  0.00,  0.00) rectangle (209.58,115.63);
\definecolor{drawColor}{gray}{0.20}

\path[draw=drawColor,line width= 0.5pt,line join=round] ( 24.22, 30.13) --
	( 26.47, 30.13);

\path[draw=drawColor,line width= 0.5pt,line join=round] ( 24.22, 42.74) --
	( 26.47, 42.74);

\path[draw=drawColor,line width= 0.5pt,line join=round] ( 24.22, 55.36) --
	( 26.47, 55.36);

\path[draw=drawColor,line width= 0.5pt,line join=round] ( 24.22, 67.97) --
	( 26.47, 67.97);

\path[draw=drawColor,line width= 0.5pt,line join=round] ( 24.22, 80.59) --
	( 26.47, 80.59);

\path[draw=drawColor,line width= 0.5pt,line join=round] ( 24.22, 93.20) --
	( 26.47, 93.20);

\path[draw=drawColor,line width= 0.5pt,line join=round] ( 24.22,105.82) --
	( 26.47,105.82);
\end{scope}
\begin{scope}
\path[clip] (  0.00,  0.00) rectangle (209.58,115.63);
\definecolor{drawColor}{RGB}{0,0,0}

\path[draw=drawColor,line width= 0.5pt,line join=round] ( 26.47, 22.31) --
	(203.58, 22.31);

\path[draw=drawColor,line width= 0.5pt,line join=round] (201.61, 21.17) --
	(203.58, 22.31) --
	(201.61, 23.45);
\end{scope}
\begin{scope}
\path[clip] (  0.00,  0.00) rectangle (209.58,115.63);
\definecolor{drawColor}{gray}{0.20}

\path[draw=drawColor,line width= 0.5pt,line join=round] ( 43.61, 20.06) --
	( 43.61, 22.31);

\path[draw=drawColor,line width= 0.5pt,line join=round] ( 72.17, 20.06) --
	( 72.17, 22.31);

\path[draw=drawColor,line width= 0.5pt,line join=round] (100.74, 20.06) --
	(100.74, 22.31);

\path[draw=drawColor,line width= 0.5pt,line join=round] (129.31, 20.06) --
	(129.31, 22.31);

\path[draw=drawColor,line width= 0.5pt,line join=round] (157.88, 20.06) --
	(157.88, 22.31);

\path[draw=drawColor,line width= 0.5pt,line join=round] (186.44, 20.06) --
	(186.44, 22.31);
\end{scope}
\begin{scope}
\path[clip] (  0.00,  0.00) rectangle (209.58,115.63);
\definecolor{drawColor}{gray}{0.30}

\node[text=drawColor,anchor=base,inner sep=0pt, outer sep=0pt, scale=  0.80] at ( 43.61, 12.75) {2};

\node[text=drawColor,anchor=base,inner sep=0pt, outer sep=0pt, scale=  0.80] at ( 72.17, 12.75) {4};

\node[text=drawColor,anchor=base,inner sep=0pt, outer sep=0pt, scale=  0.80] at (100.74, 12.75) {8};

\node[text=drawColor,anchor=base,inner sep=0pt, outer sep=0pt, scale=  0.80] at (129.31, 12.75) {12};

\node[text=drawColor,anchor=base,inner sep=0pt, outer sep=0pt, scale=  0.80] at (157.88, 12.75) {16};

\node[text=drawColor,anchor=base,inner sep=0pt, outer sep=0pt, scale=  0.80] at (186.44, 12.75) {20};
\end{scope}
\begin{scope}
\path[clip] (  0.00,  0.00) rectangle (209.58,115.63);
\definecolor{drawColor}{RGB}{0,0,0}

\node[text=drawColor,anchor=base,inner sep=0pt, outer sep=0pt, scale=  0.90] at (115.03,  2.75) {domain size $d$};
\end{scope}
\begin{scope}
\path[clip] (  0.00,  0.00) rectangle (209.58,115.63);
\definecolor{drawColor}{RGB}{0,0,0}

\node[text=drawColor,rotate= 90.00,anchor=base,inner sep=0pt, outer sep=0pt, scale=  0.90] at (  8.20, 67.97) {$p' \mathbin{/} p$};
\end{scope}
\end{tikzpicture}}
		\caption{Query result quotient $p' \mathbin{/} p$.}
		\label{fig:commutative_approx_plot_quots_k=1_eps=0.001}
	\end{subfigure}
	\caption{(a) Run time of \ac{lve} on the compressed model returned by \ac{acp} and its $\varepsilon$-relaxed variant (\acs{acp} $\pm \varepsilon$), and (b) distribution of the per-query quotient $p' \mathbin{/} p$ between the approximate result $p'$ and the exact result $p$ for input \acp{fg} containing $k = 1$ $\varepsilon$-commutative factors with $\varepsilon = 0.001$.}
	\label{fig:commutative_approx_plot_te_k=1_eps=0.001}
\end{figure}
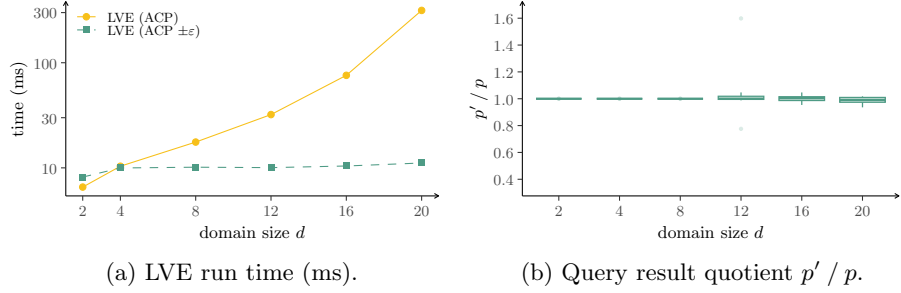

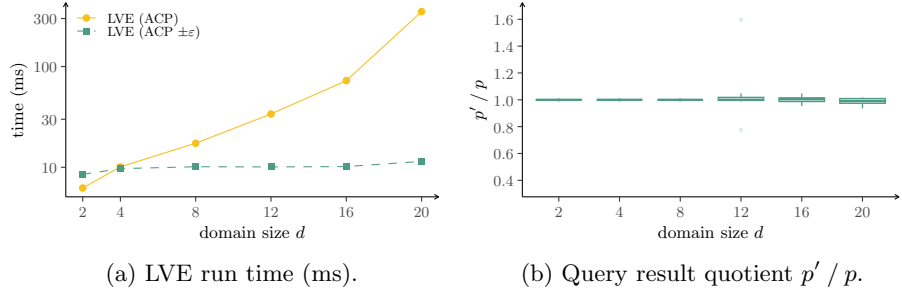
\begin{figure}[t]
	\centering
	\begin{subfigure}[t]{0.49\linewidth}
		\centering
		\resizebox{\linewidth}{!}{
\begin{tikzpicture}[x=1pt,y=1pt]
\definecolor{fillColor}{RGB}{255,255,255}
\path[use as bounding box,fill=fillColor,fill opacity=0.00] (0,0) rectangle (209.58,115.63);
\begin{scope}
\path[clip] (  0.00,  0.00) rectangle (209.58,115.63);
\definecolor{drawColor}{RGB}{255,255,255}
\definecolor{fillColor}{RGB}{255,255,255}

\path[draw=drawColor,line width= 0.5pt,line join=round,line cap=round,fill=fillColor] (  0.00,  0.00) rectangle (209.58,115.63);
\end{scope}
\begin{scope}
\path[clip] ( 28.25, 22.31) rectangle (203.58,113.63);
\definecolor{fillColor}{RGB}{255,255,255}

\path[fill=fillColor] ( 28.25, 22.31) rectangle (203.58,113.63);
\definecolor{drawColor}{RGB}{247,192,26}

\path[draw=drawColor,line width= 0.5pt,line join=round] ( 36.22, 26.46) --
	( 53.93, 36.44) --
	( 89.35, 47.53) --
	(124.77, 61.31) --
	(160.19, 76.99) --
	(195.61,109.48);
\definecolor{drawColor}{RGB}{78,155,133}

\path[draw=drawColor,line width= 0.5pt,dash pattern=on 4pt off 4pt ,line join=round] ( 36.22, 32.89) --
	( 53.93, 35.58) --
	( 89.35, 36.47) --
	(124.77, 36.38) --
	(160.19, 36.54) --
	(195.61, 39.01);
\definecolor{drawColor}{RGB}{247,192,26}
\definecolor{fillColor}{RGB}{247,192,26}

\path[draw=drawColor,line width= 0.4pt,line join=round,line cap=round,fill=fillColor] ( 36.22, 26.46) circle (  1.53);
\definecolor{fillColor}{RGB}{78,155,133}

\path[fill=fillColor] ( 34.68, 31.35) --
	( 37.75, 31.35) --
	( 37.75, 34.42) --
	( 34.68, 34.42) --
	cycle;
\definecolor{fillColor}{RGB}{247,192,26}

\path[draw=drawColor,line width= 0.4pt,line join=round,line cap=round,fill=fillColor] ( 53.93, 36.44) circle (  1.53);
\definecolor{fillColor}{RGB}{78,155,133}

\path[fill=fillColor] ( 52.39, 34.05) --
	( 55.46, 34.05) --
	( 55.46, 37.11) --
	( 52.39, 37.11) --
	cycle;
\definecolor{fillColor}{RGB}{247,192,26}

\path[draw=drawColor,line width= 0.4pt,line join=round,line cap=round,fill=fillColor] ( 89.35, 47.53) circle (  1.53);
\definecolor{fillColor}{RGB}{78,155,133}

\path[fill=fillColor] ( 87.81, 34.94) --
	( 90.88, 34.94) --
	( 90.88, 38.01) --
	( 87.81, 38.01) --
	cycle;
\definecolor{fillColor}{RGB}{247,192,26}

\path[draw=drawColor,line width= 0.4pt,line join=round,line cap=round,fill=fillColor] (124.77, 61.31) circle (  1.53);
\definecolor{fillColor}{RGB}{78,155,133}

\path[fill=fillColor] (123.24, 34.85) --
	(126.30, 34.85) --
	(126.30, 37.91) --
	(123.24, 37.91) --
	cycle;
\definecolor{fillColor}{RGB}{247,192,26}

\path[draw=drawColor,line width= 0.4pt,line join=round,line cap=round,fill=fillColor] (160.19, 76.99) circle (  1.53);
\definecolor{fillColor}{RGB}{78,155,133}

\path[fill=fillColor] (158.66, 35.01) --
	(161.73, 35.01) --
	(161.73, 38.08) --
	(158.66, 38.08) --
	cycle;
\definecolor{fillColor}{RGB}{247,192,26}

\path[draw=drawColor,line width= 0.4pt,line join=round,line cap=round,fill=fillColor] (195.61,109.48) circle (  1.53);
\definecolor{fillColor}{RGB}{78,155,133}

\path[fill=fillColor] (194.08, 37.48) --
	(197.15, 37.48) --
	(197.15, 40.55) --
	(194.08, 40.55) --
	cycle;
\end{scope}
\begin{scope}
\path[clip] (  0.00,  0.00) rectangle (209.58,115.63);
\definecolor{drawColor}{RGB}{0,0,0}

\path[draw=drawColor,line width= 0.5pt,line join=round] ( 28.25, 22.31) --
	( 28.25,113.63);

\path[draw=drawColor,line width= 0.5pt,line join=round] ( 29.38,111.66) --
	( 28.25,113.63) --
	( 27.11,111.66);
\end{scope}
\begin{scope}
\path[clip] (  0.00,  0.00) rectangle (209.58,115.63);
\definecolor{drawColor}{gray}{0.30}

\node[text=drawColor,anchor=base east,inner sep=0pt, outer sep=0pt, scale=  0.80] at ( 24.20, 33.47) {10};

\node[text=drawColor,anchor=base east,inner sep=0pt, outer sep=0pt, scale=  0.80] at ( 24.20, 56.08) {30};

\node[text=drawColor,anchor=base east,inner sep=0pt, outer sep=0pt, scale=  0.80] at ( 24.20, 80.85) {100};

\node[text=drawColor,anchor=base east,inner sep=0pt, outer sep=0pt, scale=  0.80] at ( 24.20,103.46) {300};
\end{scope}
\begin{scope}
\path[clip] (  0.00,  0.00) rectangle (209.58,115.63);
\definecolor{drawColor}{gray}{0.20}

\path[draw=drawColor,line width= 0.5pt,line join=round] ( 26.00, 36.22) --
	( 28.25, 36.22);

\path[draw=drawColor,line width= 0.5pt,line join=round] ( 26.00, 58.83) --
	( 28.25, 58.83);

\path[draw=drawColor,line width= 0.5pt,line join=round] ( 26.00, 83.60) --
	( 28.25, 83.60);

\path[draw=drawColor,line width= 0.5pt,line join=round] ( 26.00,106.21) --
	( 28.25,106.21);
\end{scope}
\begin{scope}
\path[clip] (  0.00,  0.00) rectangle (209.58,115.63);
\definecolor{drawColor}{RGB}{0,0,0}

\path[draw=drawColor,line width= 0.5pt,line join=round] ( 28.25, 22.31) --
	(203.58, 22.31);

\path[draw=drawColor,line width= 0.5pt,line join=round] (201.61, 21.17) --
	(203.58, 22.31) --
	(201.61, 23.45);
\end{scope}
\begin{scope}
\path[clip] (  0.00,  0.00) rectangle (209.58,115.63);
\definecolor{drawColor}{gray}{0.20}

\path[draw=drawColor,line width= 0.5pt,line join=round] ( 36.22, 20.06) --
	( 36.22, 22.31);

\path[draw=drawColor,line width= 0.5pt,line join=round] ( 53.93, 20.06) --
	( 53.93, 22.31);

\path[draw=drawColor,line width= 0.5pt,line join=round] ( 89.35, 20.06) --
	( 89.35, 22.31);

\path[draw=drawColor,line width= 0.5pt,line join=round] (124.77, 20.06) --
	(124.77, 22.31);

\path[draw=drawColor,line width= 0.5pt,line join=round] (160.19, 20.06) --
	(160.19, 22.31);

\path[draw=drawColor,line width= 0.5pt,line join=round] (195.61, 20.06) --
	(195.61, 22.31);
\end{scope}
\begin{scope}
\path[clip] (  0.00,  0.00) rectangle (209.58,115.63);
\definecolor{drawColor}{gray}{0.30}

\node[text=drawColor,anchor=base,inner sep=0pt, outer sep=0pt, scale=  0.80] at ( 36.22, 12.75) {2};

\node[text=drawColor,anchor=base,inner sep=0pt, outer sep=0pt, scale=  0.80] at ( 53.93, 12.75) {4};

\node[text=drawColor,anchor=base,inner sep=0pt, outer sep=0pt, scale=  0.80] at ( 89.35, 12.75) {8};

\node[text=drawColor,anchor=base,inner sep=0pt, outer sep=0pt, scale=  0.80] at (124.77, 12.75) {12};

\node[text=drawColor,anchor=base,inner sep=0pt, outer sep=0pt, scale=  0.80] at (160.19, 12.75) {16};

\node[text=drawColor,anchor=base,inner sep=0pt, outer sep=0pt, scale=  0.80] at (195.61, 12.75) {20};
\end{scope}
\begin{scope}
\path[clip] (  0.00,  0.00) rectangle (209.58,115.63);
\definecolor{drawColor}{RGB}{0,0,0}

\node[text=drawColor,anchor=base,inner sep=0pt, outer sep=0pt, scale=  0.90] at (115.91,  2.75) {domain size $d$};
\end{scope}
\begin{scope}
\path[clip] (  0.00,  0.00) rectangle (209.58,115.63);
\definecolor{drawColor}{RGB}{0,0,0}

\node[text=drawColor,rotate= 90.00,anchor=base,inner sep=0pt, outer sep=0pt, scale=  0.90] at (  8.20, 67.97) {time (ms)};
\end{scope}
\begin{scope}
\path[clip] (  0.00,  0.00) rectangle (209.58,115.63);

\path[] ( 32.50, 97.23) rectangle (101.14,113.60);
\end{scope}
\begin{scope}
\path[clip] (  0.00,  0.00) rectangle (209.58,115.63);
\definecolor{drawColor}{RGB}{247,192,26}

\path[draw=drawColor,line width= 0.5pt,line join=round] ( 33.59,106.50) -- ( 42.26,106.50);
\end{scope}
\begin{scope}
\path[clip] (  0.00,  0.00) rectangle (209.58,115.63);
\definecolor{drawColor}{RGB}{247,192,26}
\definecolor{fillColor}{RGB}{247,192,26}

\path[draw=drawColor,line width= 0.4pt,line join=round,line cap=round,fill=fillColor] ( 37.92,106.50) circle (  1.53);
\end{scope}
\begin{scope}
\path[clip] (  0.00,  0.00) rectangle (209.58,115.63);
\definecolor{drawColor}{RGB}{78,155,133}

\path[draw=drawColor,line width= 0.5pt,dash pattern=on 4pt off 4pt ,line join=round] ( 33.59,100.32) -- ( 42.26,100.32);
\end{scope}
\begin{scope}
\path[clip] (  0.00,  0.00) rectangle (209.58,115.63);
\definecolor{fillColor}{RGB}{78,155,133}

\path[fill=fillColor] ( 36.39, 98.79) --
	( 39.46, 98.79) --
	( 39.46,101.86) --
	( 36.39,101.86) --
	cycle;
\end{scope}
\begin{scope}
\path[clip] (  0.00,  0.00) rectangle (209.58,115.63);
\definecolor{drawColor}{RGB}{0,0,0}

\node[text=drawColor,anchor=base west,inner sep=0pt, outer sep=0pt, scale=  0.70] at ( 47.84,104.09) {LVE (ACP)};
\end{scope}
\begin{scope}
\path[clip] (  0.00,  0.00) rectangle (209.58,115.63);
\definecolor{drawColor}{RGB}{0,0,0}

\node[text=drawColor,anchor=base west,inner sep=0pt, outer sep=0pt, scale=  0.70] at ( 47.84, 97.91) {LVE (ACP $\pm \varepsilon$)};
\end{scope}
\end{tikzpicture}}
		\caption{\acs{lve} run time (ms).}
		\label{fig:commutative_approx_plot_times_k=1_eps=0.01}
	\end{subfigure}
	\begin{subfigure}[t]{0.49\linewidth}
		\centering
		\resizebox{\linewidth}{!}{
\begin{tikzpicture}[x=1pt,y=1pt]
\definecolor{fillColor}{RGB}{255,255,255}
\path[use as bounding box,fill=fillColor,fill opacity=0.00] (0,0) rectangle (209.58,115.63);
\begin{scope}
\path[clip] (  0.00,  0.00) rectangle (209.58,115.63);
\definecolor{drawColor}{RGB}{255,255,255}
\definecolor{fillColor}{RGB}{255,255,255}

\path[draw=drawColor,line width= 0.5pt,line join=round,line cap=round,fill=fillColor] (  0.00,  0.00) rectangle (209.58,115.63);
\end{scope}
\begin{scope}
\path[clip] ( 26.47, 22.31) rectangle (203.58,113.63);
\definecolor{fillColor}{RGB}{255,255,255}

\path[fill=fillColor] ( 26.47, 22.31) rectangle (203.58,113.63);
\definecolor{drawColor}{RGB}{78,155,133}
\definecolor{fillColor}{RGB}{78,155,133}

\path[draw=drawColor,draw opacity=0.20,line width= 0.4pt,line join=round,line cap=round,fill=fillColor,fill opacity=0.20] ( 43.61, 67.86) circle (  0.78);

\path[draw=drawColor,draw opacity=0.20,line width= 0.4pt,line join=round,line cap=round,fill=fillColor,fill opacity=0.20] ( 43.61, 68.02) circle (  0.78);
\definecolor{drawColor}{RGB}{78,155,133}

\path[draw=drawColor,line width= 0.6pt,line join=round] ( 43.61, 67.98) -- ( 43.61, 67.98);

\path[draw=drawColor,line width= 0.6pt,line join=round] ( 43.61, 67.97) -- ( 43.61, 67.97);

\path[draw=drawColor,line width= 0.6pt,fill=fillColor,fill opacity=0.20] ( 32.90, 67.98) --
	( 32.90, 67.97) --
	( 54.32, 67.97) --
	( 54.32, 67.98) --
	( 32.90, 67.98) --
	cycle;

\path[draw=drawColor,line width= 1.1pt] ( 32.90, 67.97) -- ( 54.32, 67.97);
\definecolor{drawColor}{RGB}{78,155,133}

\path[draw=drawColor,draw opacity=0.20,line width= 0.4pt,line join=round,line cap=round,fill=fillColor,fill opacity=0.20] ( 72.17, 67.92) circle (  0.78);

\path[draw=drawColor,draw opacity=0.20,line width= 0.4pt,line join=round,line cap=round,fill=fillColor,fill opacity=0.20] ( 72.17, 67.99) circle (  0.78);
\definecolor{drawColor}{RGB}{78,155,133}

\path[draw=drawColor,line width= 0.6pt,line join=round] ( 72.17, 67.97) -- ( 72.17, 67.98);

\path[draw=drawColor,line width= 0.6pt,line join=round] ( 72.17, 67.97) -- ( 72.17, 67.97);

\path[draw=drawColor,line width= 0.6pt,fill=fillColor,fill opacity=0.20] ( 61.46, 67.97) --
	( 61.46, 67.97) --
	( 82.89, 67.97) --
	( 82.89, 67.97) --
	( 61.46, 67.97) --
	cycle;

\path[draw=drawColor,line width= 1.1pt] ( 61.46, 67.97) -- ( 82.89, 67.97);
\definecolor{drawColor}{RGB}{78,155,133}

\path[draw=drawColor,draw opacity=0.20,line width= 0.4pt,line join=round,line cap=round,fill=fillColor,fill opacity=0.20] (100.74, 67.94) circle (  0.78);

\path[draw=drawColor,draw opacity=0.20,line width= 0.4pt,line join=round,line cap=round,fill=fillColor,fill opacity=0.20] (100.74, 67.98) circle (  0.78);
\definecolor{drawColor}{RGB}{78,155,133}

\path[draw=drawColor,line width= 0.6pt,line join=round] (100.74, 67.97) -- (100.74, 67.97);

\path[draw=drawColor,line width= 0.6pt,line join=round] (100.74, 67.97) -- (100.74, 67.97);

\path[draw=drawColor,line width= 0.6pt,fill=fillColor,fill opacity=0.20] ( 90.03, 67.97) --
	( 90.03, 67.97) --
	(111.45, 67.97) --
	(111.45, 67.97) --
	( 90.03, 67.97) --
	cycle;

\path[draw=drawColor,line width= 1.1pt] ( 90.03, 67.97) -- (111.45, 67.97);
\definecolor{drawColor}{RGB}{78,155,133}

\path[draw=drawColor,draw opacity=0.20,line width= 0.4pt,line join=round,line cap=round,fill=fillColor,fill opacity=0.20] (129.31,105.71) circle (  0.78);

\path[draw=drawColor,draw opacity=0.20,line width= 0.4pt,line join=round,line cap=round,fill=fillColor,fill opacity=0.20] (129.31, 53.79) circle (  0.78);
\definecolor{drawColor}{RGB}{78,155,133}

\path[draw=drawColor,line width= 0.6pt,line join=round] (129.31, 69.14) -- (129.31, 70.94);

\path[draw=drawColor,line width= 0.6pt,line join=round] (129.31, 67.67) -- (129.31, 67.22);

\path[draw=drawColor,line width= 0.6pt,fill=fillColor,fill opacity=0.20] (118.60, 69.14) --
	(118.60, 67.67) --
	(140.02, 67.67) --
	(140.02, 69.14) --
	(118.60, 69.14) --
	cycle;

\path[draw=drawColor,line width= 1.1pt] (118.60, 67.97) -- (140.02, 67.97);

\path[draw=drawColor,line width= 0.6pt,line join=round] (157.88, 68.92) -- (157.88, 70.94);

\path[draw=drawColor,line width= 0.6pt,line join=round] (157.88, 67.14) -- (157.88, 65.04);

\path[draw=drawColor,line width= 0.6pt,fill=fillColor,fill opacity=0.20] (147.16, 68.92) --
	(147.16, 67.14) --
	(168.59, 67.14) --
	(168.59, 68.92) --
	(147.16, 68.92) --
	cycle;

\path[draw=drawColor,line width= 1.1pt] (147.16, 68.26) -- (168.59, 68.26);

\path[draw=drawColor,line width= 0.6pt,line join=round] (186.44, 68.58) -- (186.44, 69.10);

\path[draw=drawColor,line width= 0.6pt,line join=round] (186.44, 66.32) -- (186.44, 63.91);

\path[draw=drawColor,line width= 0.6pt,fill=fillColor,fill opacity=0.20] (175.73, 68.58) --
	(175.73, 66.32) --
	(197.16, 66.32) --
	(197.16, 68.58) --
	(175.73, 68.58) --
	cycle;

\path[draw=drawColor,line width= 1.1pt] (175.73, 67.40) -- (197.16, 67.40);
\end{scope}
\begin{scope}
\path[clip] (  0.00,  0.00) rectangle (209.58,115.63);
\definecolor{drawColor}{RGB}{0,0,0}

\path[draw=drawColor,line width= 0.5pt,line join=round] ( 26.47, 22.31) --
	( 26.47,113.63);

\path[draw=drawColor,line width= 0.5pt,line join=round] ( 27.61,111.66) --
	( 26.47,113.63) --
	( 25.33,111.66);
\end{scope}
\begin{scope}
\path[clip] (  0.00,  0.00) rectangle (209.58,115.63);
\definecolor{drawColor}{gray}{0.30}

\node[text=drawColor,anchor=base east,inner sep=0pt, outer sep=0pt, scale=  0.80] at ( 22.42, 27.32) {0.4};

\node[text=drawColor,anchor=base east,inner sep=0pt, outer sep=0pt, scale=  0.80] at ( 22.42, 39.95) {0.6};

\node[text=drawColor,anchor=base east,inner sep=0pt, outer sep=0pt, scale=  0.80] at ( 22.42, 52.58) {0.8};

\node[text=drawColor,anchor=base east,inner sep=0pt, outer sep=0pt, scale=  0.80] at ( 22.42, 65.22) {1.0};

\node[text=drawColor,anchor=base east,inner sep=0pt, outer sep=0pt, scale=  0.80] at ( 22.42, 77.85) {1.2};

\node[text=drawColor,anchor=base east,inner sep=0pt, outer sep=0pt, scale=  0.80] at ( 22.42, 90.48) {1.4};

\node[text=drawColor,anchor=base east,inner sep=0pt, outer sep=0pt, scale=  0.80] at ( 22.42,103.12) {1.6};
\end{scope}
\begin{scope}
\path[clip] (  0.00,  0.00) rectangle (209.58,115.63);
\definecolor{drawColor}{gray}{0.20}

\path[draw=drawColor,line width= 0.5pt,line join=round] ( 24.22, 30.07) --
	( 26.47, 30.07);

\path[draw=drawColor,line width= 0.5pt,line join=round] ( 24.22, 42.71) --
	( 26.47, 42.71);

\path[draw=drawColor,line width= 0.5pt,line join=round] ( 24.22, 55.34) --
	( 26.47, 55.34);

\path[draw=drawColor,line width= 0.5pt,line join=round] ( 24.22, 67.97) --
	( 26.47, 67.97);

\path[draw=drawColor,line width= 0.5pt,line join=round] ( 24.22, 80.61) --
	( 26.47, 80.61);

\path[draw=drawColor,line width= 0.5pt,line join=round] ( 24.22, 93.24) --
	( 26.47, 93.24);

\path[draw=drawColor,line width= 0.5pt,line join=round] ( 24.22,105.87) --
	( 26.47,105.87);
\end{scope}
\begin{scope}
\path[clip] (  0.00,  0.00) rectangle (209.58,115.63);
\definecolor{drawColor}{RGB}{0,0,0}

\path[draw=drawColor,line width= 0.5pt,line join=round] ( 26.47, 22.31) --
	(203.58, 22.31);

\path[draw=drawColor,line width= 0.5pt,line join=round] (201.61, 21.17) --
	(203.58, 22.31) --
	(201.61, 23.45);
\end{scope}
\begin{scope}
\path[clip] (  0.00,  0.00) rectangle (209.58,115.63);
\definecolor{drawColor}{gray}{0.20}

\path[draw=drawColor,line width= 0.5pt,line join=round] ( 43.61, 20.06) --
	( 43.61, 22.31);

\path[draw=drawColor,line width= 0.5pt,line join=round] ( 72.17, 20.06) --
	( 72.17, 22.31);

\path[draw=drawColor,line width= 0.5pt,line join=round] (100.74, 20.06) --
	(100.74, 22.31);

\path[draw=drawColor,line width= 0.5pt,line join=round] (129.31, 20.06) --
	(129.31, 22.31);

\path[draw=drawColor,line width= 0.5pt,line join=round] (157.88, 20.06) --
	(157.88, 22.31);

\path[draw=drawColor,line width= 0.5pt,line join=round] (186.44, 20.06) --
	(186.44, 22.31);
\end{scope}
\begin{scope}
\path[clip] (  0.00,  0.00) rectangle (209.58,115.63);
\definecolor{drawColor}{gray}{0.30}

\node[text=drawColor,anchor=base,inner sep=0pt, outer sep=0pt, scale=  0.80] at ( 43.61, 12.75) {2};

\node[text=drawColor,anchor=base,inner sep=0pt, outer sep=0pt, scale=  0.80] at ( 72.17, 12.75) {4};

\node[text=drawColor,anchor=base,inner sep=0pt, outer sep=0pt, scale=  0.80] at (100.74, 12.75) {8};

\node[text=drawColor,anchor=base,inner sep=0pt, outer sep=0pt, scale=  0.80] at (129.31, 12.75) {12};

\node[text=drawColor,anchor=base,inner sep=0pt, outer sep=0pt, scale=  0.80] at (157.88, 12.75) {16};

\node[text=drawColor,anchor=base,inner sep=0pt, outer sep=0pt, scale=  0.80] at (186.44, 12.75) {20};
\end{scope}
\begin{scope}
\path[clip] (  0.00,  0.00) rectangle (209.58,115.63);
\definecolor{drawColor}{RGB}{0,0,0}

\node[text=drawColor,anchor=base,inner sep=0pt, outer sep=0pt, scale=  0.90] at (115.03,  2.75) {domain size $d$};
\end{scope}
\begin{scope}
\path[clip] (  0.00,  0.00) rectangle (209.58,115.63);
\definecolor{drawColor}{RGB}{0,0,0}

\node[text=drawColor,rotate= 90.00,anchor=base,inner sep=0pt, outer sep=0pt, scale=  0.90] at (  8.20, 67.97) {$p' \mathbin{/} p$};
\end{scope}
\end{tikzpicture}}
		\caption{Query result quotient $p' \mathbin{/} p$.}
		\label{fig:commutative_approx_plot_quots_k=1_eps=0.01}
	\end{subfigure}
	\caption{(a) Run time of \ac{lve} on the compressed model returned by \ac{acp} and its $\varepsilon$-relaxed variant (\acs{acp} $\pm \varepsilon$), and (b) distribution of the per-query quotient $p' \mathbin{/} p$ between the approximate result $p'$ and the exact result $p$ for input \acp{fg} containing $k = 1$ $\varepsilon$-commutative factors with $\varepsilon = 0.01$.}
	\label{fig:commutative_approx_plot_te_k=1_eps=0.01}
\end{figure}

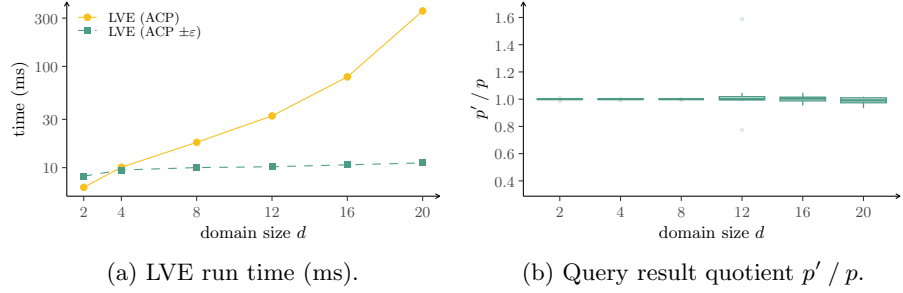
\begin{figure}[t]
	\centering
	\begin{subfigure}[t]{0.49\linewidth}
		\centering
		\resizebox{\linewidth}{!}{
\begin{tikzpicture}[x=1pt,y=1pt]
\definecolor{fillColor}{RGB}{255,255,255}
\path[use as bounding box,fill=fillColor,fill opacity=0.00] (0,0) rectangle (209.58,115.63);
\begin{scope}
\path[clip] (  0.00,  0.00) rectangle (209.58,115.63);
\definecolor{drawColor}{RGB}{255,255,255}
\definecolor{fillColor}{RGB}{255,255,255}

\path[draw=drawColor,line width= 0.5pt,line join=round,line cap=round,fill=fillColor] (  0.00,  0.00) rectangle (209.58,115.63);
\end{scope}
\begin{scope}
\path[clip] ( 28.25, 22.31) rectangle (203.58,113.63);
\definecolor{fillColor}{RGB}{255,255,255}

\path[fill=fillColor] ( 28.25, 22.31) rectangle (203.58,113.63);
\definecolor{drawColor}{RGB}{247,192,26}

\path[draw=drawColor,line width= 0.5pt,line join=round] ( 36.22, 26.46) --
	( 53.93, 35.94) --
	( 89.35, 47.66) --
	(124.77, 60.10) --
	(160.19, 78.45) --
	(195.61,109.48);
\definecolor{drawColor}{RGB}{78,155,133}

\path[draw=drawColor,line width= 0.5pt,dash pattern=on 4pt off 4pt ,line join=round] ( 36.22, 31.73) --
	( 53.93, 34.63) --
	( 89.35, 35.77) --
	(124.77, 36.16) --
	(160.19, 37.08) --
	(195.61, 37.96);
\definecolor{drawColor}{RGB}{247,192,26}
\definecolor{fillColor}{RGB}{247,192,26}

\path[draw=drawColor,line width= 0.4pt,line join=round,line cap=round,fill=fillColor] ( 36.22, 26.46) circle (  1.53);
\definecolor{fillColor}{RGB}{78,155,133}

\path[fill=fillColor] ( 34.68, 30.20) --
	( 37.75, 30.20) --
	( 37.75, 33.26) --
	( 34.68, 33.26) --
	cycle;
\definecolor{fillColor}{RGB}{247,192,26}

\path[draw=drawColor,line width= 0.4pt,line join=round,line cap=round,fill=fillColor] ( 53.93, 35.94) circle (  1.53);
\definecolor{fillColor}{RGB}{78,155,133}

\path[fill=fillColor] ( 52.39, 33.10) --
	( 55.46, 33.10) --
	( 55.46, 36.16) --
	( 52.39, 36.16) --
	cycle;
\definecolor{fillColor}{RGB}{247,192,26}

\path[draw=drawColor,line width= 0.4pt,line join=round,line cap=round,fill=fillColor] ( 89.35, 47.66) circle (  1.53);
\definecolor{fillColor}{RGB}{78,155,133}

\path[fill=fillColor] ( 87.81, 34.24) --
	( 90.88, 34.24) --
	( 90.88, 37.30) --
	( 87.81, 37.30) --
	cycle;
\definecolor{fillColor}{RGB}{247,192,26}

\path[draw=drawColor,line width= 0.4pt,line join=round,line cap=round,fill=fillColor] (124.77, 60.10) circle (  1.53);
\definecolor{fillColor}{RGB}{78,155,133}

\path[fill=fillColor] (123.24, 34.62) --
	(126.30, 34.62) --
	(126.30, 37.69) --
	(123.24, 37.69) --
	cycle;
\definecolor{fillColor}{RGB}{247,192,26}

\path[draw=drawColor,line width= 0.4pt,line join=round,line cap=round,fill=fillColor] (160.19, 78.45) circle (  1.53);
\definecolor{fillColor}{RGB}{78,155,133}

\path[fill=fillColor] (158.66, 35.55) --
	(161.73, 35.55) --
	(161.73, 38.62) --
	(158.66, 38.62) --
	cycle;
\definecolor{fillColor}{RGB}{247,192,26}

\path[draw=drawColor,line width= 0.4pt,line join=round,line cap=round,fill=fillColor] (195.61,109.48) circle (  1.53);
\definecolor{fillColor}{RGB}{78,155,133}

\path[fill=fillColor] (194.08, 36.43) --
	(197.15, 36.43) --
	(197.15, 39.49) --
	(194.08, 39.49) --
	cycle;
\end{scope}
\begin{scope}
\path[clip] (  0.00,  0.00) rectangle (209.58,115.63);
\definecolor{drawColor}{RGB}{0,0,0}

\path[draw=drawColor,line width= 0.5pt,line join=round] ( 28.25, 22.31) --
	( 28.25,113.63);

\path[draw=drawColor,line width= 0.5pt,line join=round] ( 29.38,111.66) --
	( 28.25,113.63) --
	( 27.11,111.66);
\end{scope}
\begin{scope}
\path[clip] (  0.00,  0.00) rectangle (209.58,115.63);
\definecolor{drawColor}{gray}{0.30}

\node[text=drawColor,anchor=base east,inner sep=0pt, outer sep=0pt, scale=  0.80] at ( 24.20, 32.99) {10};

\node[text=drawColor,anchor=base east,inner sep=0pt, outer sep=0pt, scale=  0.80] at ( 24.20, 55.72) {30};

\node[text=drawColor,anchor=base east,inner sep=0pt, outer sep=0pt, scale=  0.80] at ( 24.20, 80.63) {100};

\node[text=drawColor,anchor=base east,inner sep=0pt, outer sep=0pt, scale=  0.80] at ( 24.20,103.36) {300};
\end{scope}
\begin{scope}
\path[clip] (  0.00,  0.00) rectangle (209.58,115.63);
\definecolor{drawColor}{gray}{0.20}

\path[draw=drawColor,line width= 0.5pt,line join=round] ( 26.00, 35.74) --
	( 28.25, 35.74);

\path[draw=drawColor,line width= 0.5pt,line join=round] ( 26.00, 58.47) --
	( 28.25, 58.47);

\path[draw=drawColor,line width= 0.5pt,line join=round] ( 26.00, 83.39) --
	( 28.25, 83.39);

\path[draw=drawColor,line width= 0.5pt,line join=round] ( 26.00,106.12) --
	( 28.25,106.12);
\end{scope}
\begin{scope}
\path[clip] (  0.00,  0.00) rectangle (209.58,115.63);
\definecolor{drawColor}{RGB}{0,0,0}

\path[draw=drawColor,line width= 0.5pt,line join=round] ( 28.25, 22.31) --
	(203.58, 22.31);

\path[draw=drawColor,line width= 0.5pt,line join=round] (201.61, 21.17) --
	(203.58, 22.31) --
	(201.61, 23.45);
\end{scope}
\begin{scope}
\path[clip] (  0.00,  0.00) rectangle (209.58,115.63);
\definecolor{drawColor}{gray}{0.20}

\path[draw=drawColor,line width= 0.5pt,line join=round] ( 36.22, 20.06) --
	( 36.22, 22.31);

\path[draw=drawColor,line width= 0.5pt,line join=round] ( 53.93, 20.06) --
	( 53.93, 22.31);

\path[draw=drawColor,line width= 0.5pt,line join=round] ( 89.35, 20.06) --
	( 89.35, 22.31);

\path[draw=drawColor,line width= 0.5pt,line join=round] (124.77, 20.06) --
	(124.77, 22.31);

\path[draw=drawColor,line width= 0.5pt,line join=round] (160.19, 20.06) --
	(160.19, 22.31);

\path[draw=drawColor,line width= 0.5pt,line join=round] (195.61, 20.06) --
	(195.61, 22.31);
\end{scope}
\begin{scope}
\path[clip] (  0.00,  0.00) rectangle (209.58,115.63);
\definecolor{drawColor}{gray}{0.30}

\node[text=drawColor,anchor=base,inner sep=0pt, outer sep=0pt, scale=  0.80] at ( 36.22, 12.75) {2};

\node[text=drawColor,anchor=base,inner sep=0pt, outer sep=0pt, scale=  0.80] at ( 53.93, 12.75) {4};

\node[text=drawColor,anchor=base,inner sep=0pt, outer sep=0pt, scale=  0.80] at ( 89.35, 12.75) {8};

\node[text=drawColor,anchor=base,inner sep=0pt, outer sep=0pt, scale=  0.80] at (124.77, 12.75) {12};

\node[text=drawColor,anchor=base,inner sep=0pt, outer sep=0pt, scale=  0.80] at (160.19, 12.75) {16};

\node[text=drawColor,anchor=base,inner sep=0pt, outer sep=0pt, scale=  0.80] at (195.61, 12.75) {20};
\end{scope}
\begin{scope}
\path[clip] (  0.00,  0.00) rectangle (209.58,115.63);
\definecolor{drawColor}{RGB}{0,0,0}

\node[text=drawColor,anchor=base,inner sep=0pt, outer sep=0pt, scale=  0.90] at (115.91,  2.75) {domain size $d$};
\end{scope}
\begin{scope}
\path[clip] (  0.00,  0.00) rectangle (209.58,115.63);
\definecolor{drawColor}{RGB}{0,0,0}

\node[text=drawColor,rotate= 90.00,anchor=base,inner sep=0pt, outer sep=0pt, scale=  0.90] at (  8.20, 67.97) {time (ms)};
\end{scope}
\begin{scope}
\path[clip] (  0.00,  0.00) rectangle (209.58,115.63);

\path[] ( 32.50, 97.23) rectangle (101.14,113.60);
\end{scope}
\begin{scope}
\path[clip] (  0.00,  0.00) rectangle (209.58,115.63);
\definecolor{drawColor}{RGB}{247,192,26}

\path[draw=drawColor,line width= 0.5pt,line join=round] ( 33.59,106.50) -- ( 42.26,106.50);
\end{scope}
\begin{scope}
\path[clip] (  0.00,  0.00) rectangle (209.58,115.63);
\definecolor{drawColor}{RGB}{247,192,26}
\definecolor{fillColor}{RGB}{247,192,26}

\path[draw=drawColor,line width= 0.4pt,line join=round,line cap=round,fill=fillColor] ( 37.92,106.50) circle (  1.53);
\end{scope}
\begin{scope}
\path[clip] (  0.00,  0.00) rectangle (209.58,115.63);
\definecolor{drawColor}{RGB}{78,155,133}

\path[draw=drawColor,line width= 0.5pt,dash pattern=on 4pt off 4pt ,line join=round] ( 33.59,100.32) -- ( 42.26,100.32);
\end{scope}
\begin{scope}
\path[clip] (  0.00,  0.00) rectangle (209.58,115.63);
\definecolor{fillColor}{RGB}{78,155,133}

\path[fill=fillColor] ( 36.39, 98.79) --
	( 39.46, 98.79) --
	( 39.46,101.86) --
	( 36.39,101.86) --
	cycle;
\end{scope}
\begin{scope}
\path[clip] (  0.00,  0.00) rectangle (209.58,115.63);
\definecolor{drawColor}{RGB}{0,0,0}

\node[text=drawColor,anchor=base west,inner sep=0pt, outer sep=0pt, scale=  0.70] at ( 47.84,104.09) {LVE (ACP)};
\end{scope}
\begin{scope}
\path[clip] (  0.00,  0.00) rectangle (209.58,115.63);
\definecolor{drawColor}{RGB}{0,0,0}

\node[text=drawColor,anchor=base west,inner sep=0pt, outer sep=0pt, scale=  0.70] at ( 47.84, 97.91) {LVE (ACP $\pm \varepsilon$)};
\end{scope}
\end{tikzpicture}}
		\caption{\acs{lve} run time (ms).}
		\label{fig:commutative_approx_plot_times_k=1_eps=0.1}
	\end{subfigure}
	\begin{subfigure}[t]{0.49\linewidth}
		\centering
		\resizebox{\linewidth}{!}{
\begin{tikzpicture}[x=1pt,y=1pt]
\definecolor{fillColor}{RGB}{255,255,255}
\path[use as bounding box,fill=fillColor,fill opacity=0.00] (0,0) rectangle (209.58,115.63);
\begin{scope}
\path[clip] (  0.00,  0.00) rectangle (209.58,115.63);
\definecolor{drawColor}{RGB}{255,255,255}
\definecolor{fillColor}{RGB}{255,255,255}

\path[draw=drawColor,line width= 0.5pt,line join=round,line cap=round,fill=fillColor] (  0.00,  0.00) rectangle (209.58,115.63);
\end{scope}
\begin{scope}
\path[clip] ( 26.47, 22.31) rectangle (203.58,113.63);
\definecolor{fillColor}{RGB}{255,255,255}

\path[fill=fillColor] ( 26.47, 22.31) rectangle (203.58,113.63);
\definecolor{drawColor}{RGB}{78,155,133}
\definecolor{fillColor}{RGB}{78,155,133}

\path[draw=drawColor,draw opacity=0.20,line width= 0.4pt,line join=round,line cap=round,fill=fillColor,fill opacity=0.20] ( 43.61, 66.93) circle (  0.78);

\path[draw=drawColor,draw opacity=0.20,line width= 0.4pt,line join=round,line cap=round,fill=fillColor,fill opacity=0.20] ( 43.61, 68.48) circle (  0.78);
\definecolor{drawColor}{RGB}{78,155,133}

\path[draw=drawColor,line width= 0.6pt,line join=round] ( 43.61, 68.01) -- ( 43.61, 68.01);

\path[draw=drawColor,line width= 0.6pt,line join=round] ( 43.61, 67.94) -- ( 43.61, 67.94);

\path[draw=drawColor,line width= 0.6pt,fill=fillColor,fill opacity=0.20] ( 32.90, 68.01) --
	( 32.90, 67.94) --
	( 54.32, 67.94) --
	( 54.32, 68.01) --
	( 32.90, 68.01) --
	cycle;

\path[draw=drawColor,line width= 1.1pt] ( 32.90, 67.97) -- ( 54.32, 67.97);
\definecolor{drawColor}{RGB}{78,155,133}

\path[draw=drawColor,draw opacity=0.20,line width= 0.4pt,line join=round,line cap=round,fill=fillColor,fill opacity=0.20] ( 72.17, 67.46) circle (  0.78);

\path[draw=drawColor,draw opacity=0.20,line width= 0.4pt,line join=round,line cap=round,fill=fillColor,fill opacity=0.20] ( 72.17, 68.14) circle (  0.78);
\definecolor{drawColor}{RGB}{78,155,133}

\path[draw=drawColor,line width= 0.6pt,line join=round] ( 72.17, 67.98) -- ( 72.17, 68.00);

\path[draw=drawColor,line width= 0.6pt,line join=round] ( 72.17, 67.96) -- ( 72.17, 67.96);

\path[draw=drawColor,line width= 0.6pt,fill=fillColor,fill opacity=0.20] ( 61.46, 67.98) --
	( 61.46, 67.96) --
	( 82.89, 67.96) --
	( 82.89, 67.98) --
	( 61.46, 67.98) --
	cycle;

\path[draw=drawColor,line width= 1.1pt] ( 61.46, 67.97) -- ( 82.89, 67.97);
\definecolor{drawColor}{RGB}{78,155,133}

\path[draw=drawColor,draw opacity=0.20,line width= 0.4pt,line join=round,line cap=round,fill=fillColor,fill opacity=0.20] (100.74, 67.64) circle (  0.78);

\path[draw=drawColor,draw opacity=0.20,line width= 0.4pt,line join=round,line cap=round,fill=fillColor,fill opacity=0.20] (100.74, 68.05) circle (  0.78);
\definecolor{drawColor}{RGB}{78,155,133}

\path[draw=drawColor,line width= 0.6pt,line join=round] (100.74, 67.97) -- (100.74, 67.97);

\path[draw=drawColor,line width= 0.6pt,line join=round] (100.74, 67.97) -- (100.74, 67.97);

\path[draw=drawColor,line width= 0.6pt,fill=fillColor,fill opacity=0.20] ( 90.03, 67.97) --
	( 90.03, 67.97) --
	(111.45, 67.97) --
	(111.45, 67.97) --
	( 90.03, 67.97) --
	cycle;

\path[draw=drawColor,line width= 1.1pt] ( 90.03, 67.97) -- (111.45, 67.97);
\definecolor{drawColor}{RGB}{78,155,133}

\path[draw=drawColor,draw opacity=0.20,line width= 0.4pt,line join=round,line cap=round,fill=fillColor,fill opacity=0.20] (129.31,105.71) circle (  0.78);

\path[draw=drawColor,draw opacity=0.20,line width= 0.4pt,line join=round,line cap=round,fill=fillColor,fill opacity=0.20] (129.31, 53.54) circle (  0.78);
\definecolor{drawColor}{RGB}{78,155,133}

\path[draw=drawColor,line width= 0.6pt,line join=round] (129.31, 69.16) -- (129.31, 71.00);

\path[draw=drawColor,line width= 0.6pt,line join=round] (129.31, 67.67) -- (129.31, 67.21);

\path[draw=drawColor,line width= 0.6pt,fill=fillColor,fill opacity=0.20] (118.60, 69.16) --
	(118.60, 67.67) --
	(140.02, 67.67) --
	(140.02, 69.16) --
	(118.60, 69.16) --
	cycle;

\path[draw=drawColor,line width= 1.1pt] (118.60, 67.97) -- (140.02, 67.97);

\path[draw=drawColor,line width= 0.6pt,line join=round] (157.88, 68.91) -- (157.88, 70.97);

\path[draw=drawColor,line width= 0.6pt,line join=round] (157.88, 67.15) -- (157.88, 64.94);

\path[draw=drawColor,line width= 0.6pt,fill=fillColor,fill opacity=0.20] (147.16, 68.91) --
	(147.16, 67.15) --
	(168.59, 67.15) --
	(168.59, 68.91) --
	(147.16, 68.91) --
	cycle;

\path[draw=drawColor,line width= 1.1pt] (147.16, 68.27) -- (168.59, 68.27);

\path[draw=drawColor,line width= 0.6pt,line join=round] (186.44, 68.60) -- (186.44, 69.13);

\path[draw=drawColor,line width= 0.6pt,line join=round] (186.44, 66.27) -- (186.44, 63.71);

\path[draw=drawColor,line width= 0.6pt,fill=fillColor,fill opacity=0.20] (175.73, 68.60) --
	(175.73, 66.27) --
	(197.16, 66.27) --
	(197.16, 68.60) --
	(175.73, 68.60) --
	cycle;

\path[draw=drawColor,line width= 1.1pt] (175.73, 67.39) -- (197.16, 67.39);
\end{scope}
\begin{scope}
\path[clip] (  0.00,  0.00) rectangle (209.58,115.63);
\definecolor{drawColor}{RGB}{0,0,0}

\path[draw=drawColor,line width= 0.5pt,line join=round] ( 26.47, 22.31) --
	( 26.47,113.63);

\path[draw=drawColor,line width= 0.5pt,line join=round] ( 27.61,111.66) --
	( 26.47,113.63) --
	( 25.33,111.66);
\end{scope}
\begin{scope}
\path[clip] (  0.00,  0.00) rectangle (209.58,115.63);
\definecolor{drawColor}{gray}{0.30}

\node[text=drawColor,anchor=base east,inner sep=0pt, outer sep=0pt, scale=  0.80] at ( 22.42, 26.76) {0.4};

\node[text=drawColor,anchor=base east,inner sep=0pt, outer sep=0pt, scale=  0.80] at ( 22.42, 39.58) {0.6};

\node[text=drawColor,anchor=base east,inner sep=0pt, outer sep=0pt, scale=  0.80] at ( 22.42, 52.40) {0.8};

\node[text=drawColor,anchor=base east,inner sep=0pt, outer sep=0pt, scale=  0.80] at ( 22.42, 65.22) {1.0};

\node[text=drawColor,anchor=base east,inner sep=0pt, outer sep=0pt, scale=  0.80] at ( 22.42, 78.04) {1.2};

\node[text=drawColor,anchor=base east,inner sep=0pt, outer sep=0pt, scale=  0.80] at ( 22.42, 90.86) {1.4};

\node[text=drawColor,anchor=base east,inner sep=0pt, outer sep=0pt, scale=  0.80] at ( 22.42,103.68) {1.6};
\end{scope}
\begin{scope}
\path[clip] (  0.00,  0.00) rectangle (209.58,115.63);
\definecolor{drawColor}{gray}{0.20}

\path[draw=drawColor,line width= 0.5pt,line join=round] ( 24.22, 29.51) --
	( 26.47, 29.51);

\path[draw=drawColor,line width= 0.5pt,line join=round] ( 24.22, 42.33) --
	( 26.47, 42.33);

\path[draw=drawColor,line width= 0.5pt,line join=round] ( 24.22, 55.15) --
	( 26.47, 55.15);

\path[draw=drawColor,line width= 0.5pt,line join=round] ( 24.22, 67.97) --
	( 26.47, 67.97);

\path[draw=drawColor,line width= 0.5pt,line join=round] ( 24.22, 80.79) --
	( 26.47, 80.79);

\path[draw=drawColor,line width= 0.5pt,line join=round] ( 24.22, 93.61) --
	( 26.47, 93.61);

\path[draw=drawColor,line width= 0.5pt,line join=round] ( 24.22,106.43) --
	( 26.47,106.43);
\end{scope}
\begin{scope}
\path[clip] (  0.00,  0.00) rectangle (209.58,115.63);
\definecolor{drawColor}{RGB}{0,0,0}

\path[draw=drawColor,line width= 0.5pt,line join=round] ( 26.47, 22.31) --
	(203.58, 22.31);

\path[draw=drawColor,line width= 0.5pt,line join=round] (201.61, 21.17) --
	(203.58, 22.31) --
	(201.61, 23.45);
\end{scope}
\begin{scope}
\path[clip] (  0.00,  0.00) rectangle (209.58,115.63);
\definecolor{drawColor}{gray}{0.20}

\path[draw=drawColor,line width= 0.5pt,line join=round] ( 43.61, 20.06) --
	( 43.61, 22.31);

\path[draw=drawColor,line width= 0.5pt,line join=round] ( 72.17, 20.06) --
	( 72.17, 22.31);

\path[draw=drawColor,line width= 0.5pt,line join=round] (100.74, 20.06) --
	(100.74, 22.31);

\path[draw=drawColor,line width= 0.5pt,line join=round] (129.31, 20.06) --
	(129.31, 22.31);

\path[draw=drawColor,line width= 0.5pt,line join=round] (157.88, 20.06) --
	(157.88, 22.31);

\path[draw=drawColor,line width= 0.5pt,line join=round] (186.44, 20.06) --
	(186.44, 22.31);
\end{scope}
\begin{scope}
\path[clip] (  0.00,  0.00) rectangle (209.58,115.63);
\definecolor{drawColor}{gray}{0.30}

\node[text=drawColor,anchor=base,inner sep=0pt, outer sep=0pt, scale=  0.80] at ( 43.61, 12.75) {2};

\node[text=drawColor,anchor=base,inner sep=0pt, outer sep=0pt, scale=  0.80] at ( 72.17, 12.75) {4};

\node[text=drawColor,anchor=base,inner sep=0pt, outer sep=0pt, scale=  0.80] at (100.74, 12.75) {8};

\node[text=drawColor,anchor=base,inner sep=0pt, outer sep=0pt, scale=  0.80] at (129.31, 12.75) {12};

\node[text=drawColor,anchor=base,inner sep=0pt, outer sep=0pt, scale=  0.80] at (157.88, 12.75) {16};

\node[text=drawColor,anchor=base,inner sep=0pt, outer sep=0pt, scale=  0.80] at (186.44, 12.75) {20};
\end{scope}
\begin{scope}
\path[clip] (  0.00,  0.00) rectangle (209.58,115.63);
\definecolor{drawColor}{RGB}{0,0,0}

\node[text=drawColor,anchor=base,inner sep=0pt, outer sep=0pt, scale=  0.90] at (115.03,  2.75) {domain size $d$};
\end{scope}
\begin{scope}
\path[clip] (  0.00,  0.00) rectangle (209.58,115.63);
\definecolor{drawColor}{RGB}{0,0,0}

\node[text=drawColor,rotate= 90.00,anchor=base,inner sep=0pt, outer sep=0pt, scale=  0.90] at (  8.20, 67.97) {$p' \mathbin{/} p$};
\end{scope}
\end{tikzpicture}}
		\caption{Query result quotient $p' \mathbin{/} p$.}
		\label{fig:commutative_approx_plot_quots_k=1_eps=0.1}
	\end{subfigure}
	\caption{(a) Run time of \ac{lve} on the compressed model returned by \ac{acp} and its $\varepsilon$-relaxed variant (\acs{acp} $\pm \varepsilon$), and (b) distribution of the per-query quotient $p' \mathbin{/} p$ between the approximate result $p'$ and the exact result $p$ for input \acp{fg} containing $k = 1$ $\varepsilon$-commutative factors with $\varepsilon = 0.1$.}
	\label{fig:commutative_approx_plot_te_k=1_eps=0.1}
\end{figure}

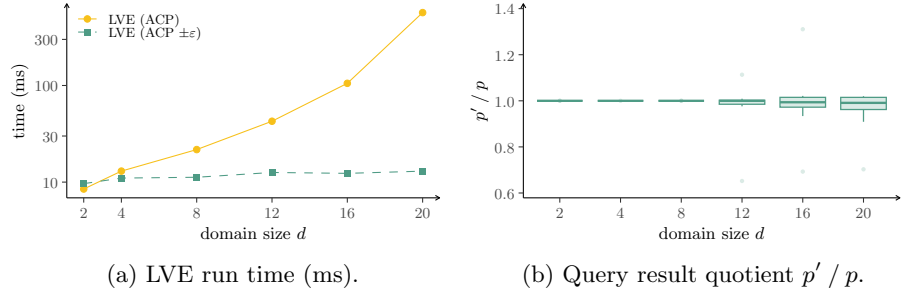
\begin{figure}[t]
	\centering
	\begin{subfigure}[t]{0.49\linewidth}
		\centering
		\resizebox{\linewidth}{!}{
\begin{tikzpicture}[x=1pt,y=1pt]
\definecolor{fillColor}{RGB}{255,255,255}
\path[use as bounding box,fill=fillColor,fill opacity=0.00] (0,0) rectangle (209.58,115.63);
\begin{scope}
\path[clip] (  0.00,  0.00) rectangle (209.58,115.63);
\definecolor{drawColor}{RGB}{255,255,255}
\definecolor{fillColor}{RGB}{255,255,255}

\path[draw=drawColor,line width= 0.5pt,line join=round,line cap=round,fill=fillColor] (  0.00,  0.00) rectangle (209.58,115.63);
\end{scope}
\begin{scope}
\path[clip] ( 28.25, 22.31) rectangle (203.58,113.63);
\definecolor{fillColor}{RGB}{255,255,255}

\path[fill=fillColor] ( 28.25, 22.31) rectangle (203.58,113.63);
\definecolor{drawColor}{RGB}{247,192,26}

\path[draw=drawColor,line width= 0.5pt,line join=round] ( 36.22, 26.46) --
	( 53.93, 34.89) --
	( 89.35, 45.02) --
	(124.77, 58.36) --
	(160.19, 76.20) --
	(195.61,109.48);
\definecolor{drawColor}{RGB}{78,155,133}

\path[draw=drawColor,line width= 0.5pt,dash pattern=on 4pt off 4pt ,line join=round] ( 36.22, 28.99) --
	( 53.93, 31.61) --
	( 89.35, 31.99) --
	(124.77, 34.25) --
	(160.19, 33.83) --
	(195.61, 34.87);
\definecolor{drawColor}{RGB}{247,192,26}
\definecolor{fillColor}{RGB}{247,192,26}

\path[draw=drawColor,line width= 0.4pt,line join=round,line cap=round,fill=fillColor] ( 36.22, 26.46) circle (  1.53);
\definecolor{fillColor}{RGB}{78,155,133}

\path[fill=fillColor] ( 34.68, 27.45) --
	( 37.75, 27.45) --
	( 37.75, 30.52) --
	( 34.68, 30.52) --
	cycle;
\definecolor{fillColor}{RGB}{247,192,26}

\path[draw=drawColor,line width= 0.4pt,line join=round,line cap=round,fill=fillColor] ( 53.93, 34.89) circle (  1.53);
\definecolor{fillColor}{RGB}{78,155,133}

\path[fill=fillColor] ( 52.39, 30.07) --
	( 55.46, 30.07) --
	( 55.46, 33.14) --
	( 52.39, 33.14) --
	cycle;
\definecolor{fillColor}{RGB}{247,192,26}

\path[draw=drawColor,line width= 0.4pt,line join=round,line cap=round,fill=fillColor] ( 89.35, 45.02) circle (  1.53);
\definecolor{fillColor}{RGB}{78,155,133}

\path[fill=fillColor] ( 87.81, 30.46) --
	( 90.88, 30.46) --
	( 90.88, 33.52) --
	( 87.81, 33.52) --
	cycle;
\definecolor{fillColor}{RGB}{247,192,26}

\path[draw=drawColor,line width= 0.4pt,line join=round,line cap=round,fill=fillColor] (124.77, 58.36) circle (  1.53);
\definecolor{fillColor}{RGB}{78,155,133}

\path[fill=fillColor] (123.24, 32.72) --
	(126.30, 32.72) --
	(126.30, 35.79) --
	(123.24, 35.79) --
	cycle;
\definecolor{fillColor}{RGB}{247,192,26}

\path[draw=drawColor,line width= 0.4pt,line join=round,line cap=round,fill=fillColor] (160.19, 76.20) circle (  1.53);
\definecolor{fillColor}{RGB}{78,155,133}

\path[fill=fillColor] (158.66, 32.30) --
	(161.73, 32.30) --
	(161.73, 35.36) --
	(158.66, 35.36) --
	cycle;
\definecolor{fillColor}{RGB}{247,192,26}

\path[draw=drawColor,line width= 0.4pt,line join=round,line cap=round,fill=fillColor] (195.61,109.48) circle (  1.53);
\definecolor{fillColor}{RGB}{78,155,133}

\path[fill=fillColor] (194.08, 33.34) --
	(197.15, 33.34) --
	(197.15, 36.41) --
	(194.08, 36.41) --
	cycle;
\end{scope}
\begin{scope}
\path[clip] (  0.00,  0.00) rectangle (209.58,115.63);
\definecolor{drawColor}{RGB}{0,0,0}

\path[draw=drawColor,line width= 0.5pt,line join=round] ( 28.25, 22.31) --
	( 28.25,113.63);

\path[draw=drawColor,line width= 0.5pt,line join=round] ( 29.38,111.66) --
	( 28.25,113.63) --
	( 27.11,111.66);
\end{scope}
\begin{scope}
\path[clip] (  0.00,  0.00) rectangle (209.58,115.63);
\definecolor{drawColor}{gray}{0.30}

\node[text=drawColor,anchor=base east,inner sep=0pt, outer sep=0pt, scale=  0.80] at ( 24.20, 26.96) {10};

\node[text=drawColor,anchor=base east,inner sep=0pt, outer sep=0pt, scale=  0.80] at ( 24.20, 48.65) {30};

\node[text=drawColor,anchor=base east,inner sep=0pt, outer sep=0pt, scale=  0.80] at ( 24.20, 72.42) {100};

\node[text=drawColor,anchor=base east,inner sep=0pt, outer sep=0pt, scale=  0.80] at ( 24.20, 94.12) {300};
\end{scope}
\begin{scope}
\path[clip] (  0.00,  0.00) rectangle (209.58,115.63);
\definecolor{drawColor}{gray}{0.20}

\path[draw=drawColor,line width= 0.5pt,line join=round] ( 26.00, 29.71) --
	( 28.25, 29.71);

\path[draw=drawColor,line width= 0.5pt,line join=round] ( 26.00, 51.40) --
	( 28.25, 51.40);

\path[draw=drawColor,line width= 0.5pt,line join=round] ( 26.00, 75.18) --
	( 28.25, 75.18);

\path[draw=drawColor,line width= 0.5pt,line join=round] ( 26.00, 96.87) --
	( 28.25, 96.87);
\end{scope}
\begin{scope}
\path[clip] (  0.00,  0.00) rectangle (209.58,115.63);
\definecolor{drawColor}{RGB}{0,0,0}

\path[draw=drawColor,line width= 0.5pt,line join=round] ( 28.25, 22.31) --
	(203.58, 22.31);

\path[draw=drawColor,line width= 0.5pt,line join=round] (201.61, 21.17) --
	(203.58, 22.31) --
	(201.61, 23.45);
\end{scope}
\begin{scope}
\path[clip] (  0.00,  0.00) rectangle (209.58,115.63);
\definecolor{drawColor}{gray}{0.20}

\path[draw=drawColor,line width= 0.5pt,line join=round] ( 36.22, 20.06) --
	( 36.22, 22.31);

\path[draw=drawColor,line width= 0.5pt,line join=round] ( 53.93, 20.06) --
	( 53.93, 22.31);

\path[draw=drawColor,line width= 0.5pt,line join=round] ( 89.35, 20.06) --
	( 89.35, 22.31);

\path[draw=drawColor,line width= 0.5pt,line join=round] (124.77, 20.06) --
	(124.77, 22.31);

\path[draw=drawColor,line width= 0.5pt,line join=round] (160.19, 20.06) --
	(160.19, 22.31);

\path[draw=drawColor,line width= 0.5pt,line join=round] (195.61, 20.06) --
	(195.61, 22.31);
\end{scope}
\begin{scope}
\path[clip] (  0.00,  0.00) rectangle (209.58,115.63);
\definecolor{drawColor}{gray}{0.30}

\node[text=drawColor,anchor=base,inner sep=0pt, outer sep=0pt, scale=  0.80] at ( 36.22, 12.75) {2};

\node[text=drawColor,anchor=base,inner sep=0pt, outer sep=0pt, scale=  0.80] at ( 53.93, 12.75) {4};

\node[text=drawColor,anchor=base,inner sep=0pt, outer sep=0pt, scale=  0.80] at ( 89.35, 12.75) {8};

\node[text=drawColor,anchor=base,inner sep=0pt, outer sep=0pt, scale=  0.80] at (124.77, 12.75) {12};

\node[text=drawColor,anchor=base,inner sep=0pt, outer sep=0pt, scale=  0.80] at (160.19, 12.75) {16};

\node[text=drawColor,anchor=base,inner sep=0pt, outer sep=0pt, scale=  0.80] at (195.61, 12.75) {20};
\end{scope}
\begin{scope}
\path[clip] (  0.00,  0.00) rectangle (209.58,115.63);
\definecolor{drawColor}{RGB}{0,0,0}

\node[text=drawColor,anchor=base,inner sep=0pt, outer sep=0pt, scale=  0.90] at (115.91,  2.75) {domain size $d$};
\end{scope}
\begin{scope}
\path[clip] (  0.00,  0.00) rectangle (209.58,115.63);
\definecolor{drawColor}{RGB}{0,0,0}

\node[text=drawColor,rotate= 90.00,anchor=base,inner sep=0pt, outer sep=0pt, scale=  0.90] at (  8.20, 67.97) {time (ms)};
\end{scope}
\begin{scope}
\path[clip] (  0.00,  0.00) rectangle (209.58,115.63);

\path[] ( 32.50, 97.23) rectangle (101.14,113.60);
\end{scope}
\begin{scope}
\path[clip] (  0.00,  0.00) rectangle (209.58,115.63);
\definecolor{drawColor}{RGB}{247,192,26}

\path[draw=drawColor,line width= 0.5pt,line join=round] ( 33.59,106.50) -- ( 42.26,106.50);
\end{scope}
\begin{scope}
\path[clip] (  0.00,  0.00) rectangle (209.58,115.63);
\definecolor{drawColor}{RGB}{247,192,26}
\definecolor{fillColor}{RGB}{247,192,26}

\path[draw=drawColor,line width= 0.4pt,line join=round,line cap=round,fill=fillColor] ( 37.92,106.50) circle (  1.53);
\end{scope}
\begin{scope}
\path[clip] (  0.00,  0.00) rectangle (209.58,115.63);
\definecolor{drawColor}{RGB}{78,155,133}

\path[draw=drawColor,line width= 0.5pt,dash pattern=on 4pt off 4pt ,line join=round] ( 33.59,100.32) -- ( 42.26,100.32);
\end{scope}
\begin{scope}
\path[clip] (  0.00,  0.00) rectangle (209.58,115.63);
\definecolor{fillColor}{RGB}{78,155,133}

\path[fill=fillColor] ( 36.39, 98.79) --
	( 39.46, 98.79) --
	( 39.46,101.86) --
	( 36.39,101.86) --
	cycle;
\end{scope}
\begin{scope}
\path[clip] (  0.00,  0.00) rectangle (209.58,115.63);
\definecolor{drawColor}{RGB}{0,0,0}

\node[text=drawColor,anchor=base west,inner sep=0pt, outer sep=0pt, scale=  0.70] at ( 47.84,104.09) {LVE (ACP)};
\end{scope}
\begin{scope}
\path[clip] (  0.00,  0.00) rectangle (209.58,115.63);
\definecolor{drawColor}{RGB}{0,0,0}

\node[text=drawColor,anchor=base west,inner sep=0pt, outer sep=0pt, scale=  0.70] at ( 47.84, 97.91) {LVE (ACP $\pm \varepsilon$)};
\end{scope}
\end{tikzpicture}}
		\caption{\acs{lve} run time (ms).}
		\label{fig:commutative_approx_plot_times_k=3_eps=0.001}
	\end{subfigure}
	\begin{subfigure}[t]{0.49\linewidth}
		\centering
		\resizebox{\linewidth}{!}{
\begin{tikzpicture}[x=1pt,y=1pt]
\definecolor{fillColor}{RGB}{255,255,255}
\path[use as bounding box,fill=fillColor,fill opacity=0.00] (0,0) rectangle (209.58,115.63);
\begin{scope}
\path[clip] (  0.00,  0.00) rectangle (209.58,115.63);
\definecolor{drawColor}{RGB}{255,255,255}
\definecolor{fillColor}{RGB}{255,255,255}

\path[draw=drawColor,line width= 0.5pt,line join=round,line cap=round,fill=fillColor] (  0.00,  0.00) rectangle (209.58,115.63);
\end{scope}
\begin{scope}
\path[clip] ( 26.47, 22.31) rectangle (203.58,113.63);
\definecolor{fillColor}{RGB}{255,255,255}

\path[fill=fillColor] ( 26.47, 22.31) rectangle (203.58,113.63);
\definecolor{drawColor}{RGB}{78,155,133}
\definecolor{fillColor}{RGB}{78,155,133}

\path[draw=drawColor,draw opacity=0.20,line width= 0.4pt,line join=round,line cap=round,fill=fillColor,fill opacity=0.20] ( 43.61, 67.96) circle (  0.78);

\path[draw=drawColor,draw opacity=0.20,line width= 0.4pt,line join=round,line cap=round,fill=fillColor,fill opacity=0.20] ( 43.61, 67.99) circle (  0.78);
\definecolor{drawColor}{RGB}{78,155,133}

\path[draw=drawColor,line width= 0.6pt,line join=round] ( 43.61, 67.97) -- ( 43.61, 67.97);

\path[draw=drawColor,line width= 0.6pt,line join=round] ( 43.61, 67.97) -- ( 43.61, 67.97);

\path[draw=drawColor,line width= 0.6pt,fill=fillColor,fill opacity=0.20] ( 32.90, 67.97) --
	( 32.90, 67.97) --
	( 54.32, 67.97) --
	( 54.32, 67.97) --
	( 32.90, 67.97) --
	cycle;

\path[draw=drawColor,line width= 1.1pt] ( 32.90, 67.97) -- ( 54.32, 67.97);
\definecolor{drawColor}{RGB}{78,155,133}

\path[draw=drawColor,draw opacity=0.20,line width= 0.4pt,line join=round,line cap=round,fill=fillColor,fill opacity=0.20] ( 72.17, 67.97) circle (  0.78);
\definecolor{drawColor}{RGB}{78,155,133}

\path[draw=drawColor,line width= 0.6pt,line join=round] ( 72.17, 67.97) -- ( 72.17, 67.97);

\path[draw=drawColor,line width= 0.6pt,line join=round] ( 72.17, 67.97) -- ( 72.17, 67.97);

\path[draw=drawColor,line width= 0.6pt,fill=fillColor,fill opacity=0.20] ( 61.46, 67.97) --
	( 61.46, 67.97) --
	( 82.89, 67.97) --
	( 82.89, 67.97) --
	( 61.46, 67.97) --
	cycle;

\path[draw=drawColor,line width= 1.1pt] ( 61.46, 67.97) -- ( 82.89, 67.97);
\definecolor{drawColor}{RGB}{78,155,133}

\path[draw=drawColor,draw opacity=0.20,line width= 0.4pt,line join=round,line cap=round,fill=fillColor,fill opacity=0.20] (100.74, 67.97) circle (  0.78);

\path[draw=drawColor,draw opacity=0.20,line width= 0.4pt,line join=round,line cap=round,fill=fillColor,fill opacity=0.20] (100.74, 67.97) circle (  0.78);
\definecolor{drawColor}{RGB}{78,155,133}

\path[draw=drawColor,line width= 0.6pt,line join=round] (100.74, 67.97) -- (100.74, 67.97);

\path[draw=drawColor,line width= 0.6pt,line join=round] (100.74, 67.97) -- (100.74, 67.97);

\path[draw=drawColor,line width= 0.6pt,fill=fillColor,fill opacity=0.20] ( 90.03, 67.97) --
	( 90.03, 67.97) --
	(111.45, 67.97) --
	(111.45, 67.97) --
	( 90.03, 67.97) --
	cycle;

\path[draw=drawColor,line width= 1.1pt] ( 90.03, 67.97) -- (111.45, 67.97);
\definecolor{drawColor}{RGB}{78,155,133}

\path[draw=drawColor,draw opacity=0.20,line width= 0.4pt,line join=round,line cap=round,fill=fillColor,fill opacity=0.20] (129.31, 30.24) circle (  0.78);

\path[draw=drawColor,draw opacity=0.20,line width= 0.4pt,line join=round,line cap=round,fill=fillColor,fill opacity=0.20] (129.31, 80.25) circle (  0.78);
\definecolor{drawColor}{RGB}{78,155,133}

\path[draw=drawColor,line width= 0.6pt,line join=round] (129.31, 68.24) -- (129.31, 68.95);

\path[draw=drawColor,line width= 0.6pt,line join=round] (129.31, 66.35) -- (129.31, 65.37);

\path[draw=drawColor,line width= 0.6pt,fill=fillColor,fill opacity=0.20] (118.60, 68.24) --
	(118.60, 66.35) --
	(140.02, 66.35) --
	(140.02, 68.24) --
	(118.60, 68.24) --
	cycle;

\path[draw=drawColor,line width= 1.1pt] (118.60, 67.97) -- (140.02, 67.97);
\definecolor{drawColor}{RGB}{78,155,133}

\path[draw=drawColor,draw opacity=0.20,line width= 0.4pt,line join=round,line cap=round,fill=fillColor,fill opacity=0.20] (157.88, 34.66) circle (  0.78);

\path[draw=drawColor,draw opacity=0.20,line width= 0.4pt,line join=round,line cap=round,fill=fillColor,fill opacity=0.20] (157.88,101.65) circle (  0.78);
\definecolor{drawColor}{RGB}{78,155,133}

\path[draw=drawColor,line width= 0.6pt,line join=round] (157.88, 69.60) -- (157.88, 70.28);

\path[draw=drawColor,line width= 0.6pt,line join=round] (157.88, 64.97) -- (157.88, 60.79);

\path[draw=drawColor,line width= 0.6pt,fill=fillColor,fill opacity=0.20] (147.16, 69.60) --
	(147.16, 64.97) --
	(168.59, 64.97) --
	(168.59, 69.60) --
	(147.16, 69.60) --
	cycle;

\path[draw=drawColor,line width= 1.1pt] (147.16, 67.29) -- (168.59, 67.29);
\definecolor{drawColor}{RGB}{78,155,133}

\path[draw=drawColor,draw opacity=0.20,line width= 0.4pt,line join=round,line cap=round,fill=fillColor,fill opacity=0.20] (186.44, 35.73) circle (  0.78);
\definecolor{drawColor}{RGB}{78,155,133}

\path[draw=drawColor,line width= 0.6pt,line join=round] (186.44, 69.59) -- (186.44, 70.14);

\path[draw=drawColor,line width= 0.6pt,line join=round] (186.44, 63.87) -- (186.44, 58.10);

\path[draw=drawColor,line width= 0.6pt,fill=fillColor,fill opacity=0.20] (175.73, 69.59) --
	(175.73, 63.87) --
	(197.16, 63.87) --
	(197.16, 69.59) --
	(175.73, 69.59) --
	cycle;

\path[draw=drawColor,line width= 1.1pt] (175.73, 67.00) -- (197.16, 67.00);
\end{scope}
\begin{scope}
\path[clip] (  0.00,  0.00) rectangle (209.58,115.63);
\definecolor{drawColor}{RGB}{0,0,0}

\path[draw=drawColor,line width= 0.5pt,line join=round] ( 26.47, 22.31) --
	( 26.47,113.63);

\path[draw=drawColor,line width= 0.5pt,line join=round] ( 27.61,111.66) --
	( 26.47,113.63) --
	( 25.33,111.66);
\end{scope}
\begin{scope}
\path[clip] (  0.00,  0.00) rectangle (209.58,115.63);
\definecolor{drawColor}{gray}{0.30}

\node[text=drawColor,anchor=base east,inner sep=0pt, outer sep=0pt, scale=  0.80] at ( 22.42, 21.84) {0.6};

\node[text=drawColor,anchor=base east,inner sep=0pt, outer sep=0pt, scale=  0.80] at ( 22.42, 43.53) {0.8};

\node[text=drawColor,anchor=base east,inner sep=0pt, outer sep=0pt, scale=  0.80] at ( 22.42, 65.22) {1.0};

\node[text=drawColor,anchor=base east,inner sep=0pt, outer sep=0pt, scale=  0.80] at ( 22.42, 86.91) {1.2};

\node[text=drawColor,anchor=base east,inner sep=0pt, outer sep=0pt, scale=  0.80] at ( 22.42,108.60) {1.4};
\end{scope}
\begin{scope}
\path[clip] (  0.00,  0.00) rectangle (209.58,115.63);
\definecolor{drawColor}{gray}{0.20}

\path[draw=drawColor,line width= 0.5pt,line join=round] ( 24.22, 24.59) --
	( 26.47, 24.59);

\path[draw=drawColor,line width= 0.5pt,line join=round] ( 24.22, 46.28) --
	( 26.47, 46.28);

\path[draw=drawColor,line width= 0.5pt,line join=round] ( 24.22, 67.97) --
	( 26.47, 67.97);

\path[draw=drawColor,line width= 0.5pt,line join=round] ( 24.22, 89.66) --
	( 26.47, 89.66);

\path[draw=drawColor,line width= 0.5pt,line join=round] ( 24.22,111.35) --
	( 26.47,111.35);
\end{scope}
\begin{scope}
\path[clip] (  0.00,  0.00) rectangle (209.58,115.63);
\definecolor{drawColor}{RGB}{0,0,0}

\path[draw=drawColor,line width= 0.5pt,line join=round] ( 26.47, 22.31) --
	(203.58, 22.31);

\path[draw=drawColor,line width= 0.5pt,line join=round] (201.61, 21.17) --
	(203.58, 22.31) --
	(201.61, 23.45);
\end{scope}
\begin{scope}
\path[clip] (  0.00,  0.00) rectangle (209.58,115.63);
\definecolor{drawColor}{gray}{0.20}

\path[draw=drawColor,line width= 0.5pt,line join=round] ( 43.61, 20.06) --
	( 43.61, 22.31);

\path[draw=drawColor,line width= 0.5pt,line join=round] ( 72.17, 20.06) --
	( 72.17, 22.31);

\path[draw=drawColor,line width= 0.5pt,line join=round] (100.74, 20.06) --
	(100.74, 22.31);

\path[draw=drawColor,line width= 0.5pt,line join=round] (129.31, 20.06) --
	(129.31, 22.31);

\path[draw=drawColor,line width= 0.5pt,line join=round] (157.88, 20.06) --
	(157.88, 22.31);

\path[draw=drawColor,line width= 0.5pt,line join=round] (186.44, 20.06) --
	(186.44, 22.31);
\end{scope}
\begin{scope}
\path[clip] (  0.00,  0.00) rectangle (209.58,115.63);
\definecolor{drawColor}{gray}{0.30}

\node[text=drawColor,anchor=base,inner sep=0pt, outer sep=0pt, scale=  0.80] at ( 43.61, 12.75) {2};

\node[text=drawColor,anchor=base,inner sep=0pt, outer sep=0pt, scale=  0.80] at ( 72.17, 12.75) {4};

\node[text=drawColor,anchor=base,inner sep=0pt, outer sep=0pt, scale=  0.80] at (100.74, 12.75) {8};

\node[text=drawColor,anchor=base,inner sep=0pt, outer sep=0pt, scale=  0.80] at (129.31, 12.75) {12};

\node[text=drawColor,anchor=base,inner sep=0pt, outer sep=0pt, scale=  0.80] at (157.88, 12.75) {16};

\node[text=drawColor,anchor=base,inner sep=0pt, outer sep=0pt, scale=  0.80] at (186.44, 12.75) {20};
\end{scope}
\begin{scope}
\path[clip] (  0.00,  0.00) rectangle (209.58,115.63);
\definecolor{drawColor}{RGB}{0,0,0}

\node[text=drawColor,anchor=base,inner sep=0pt, outer sep=0pt, scale=  0.90] at (115.03,  2.75) {domain size $d$};
\end{scope}
\begin{scope}
\path[clip] (  0.00,  0.00) rectangle (209.58,115.63);
\definecolor{drawColor}{RGB}{0,0,0}

\node[text=drawColor,rotate= 90.00,anchor=base,inner sep=0pt, outer sep=0pt, scale=  0.90] at (  8.20, 67.97) {$p' \mathbin{/} p$};
\end{scope}
\end{tikzpicture}}
		\caption{Query result quotient $p' \mathbin{/} p$.}
		\label{fig:commutative_approx_plot_quots_k=3_eps=0.001}
	\end{subfigure}
	\caption{(a) Run time of \ac{lve} on the compressed model returned by \ac{acp} and its $\varepsilon$-relaxed variant (\acs{acp} $\pm \varepsilon$), and (b) distribution of the per-query quotient $p' \mathbin{/} p$ between the approximate result $p'$ and the exact result $p$ for input \acp{fg} containing $k = 3$ $\varepsilon$-commutative factors with $\varepsilon = 0.001$.}
	\label{fig:commutative_approx_plot_te_k=3_eps=0.001}
\end{figure}

\begin{figure}[t]
	\centering
	\begin{subfigure}[t]{0.49\linewidth}
		\centering
		\resizebox{\linewidth}{!}{
\begin{tikzpicture}[x=1pt,y=1pt]
\definecolor{fillColor}{RGB}{255,255,255}
\path[use as bounding box,fill=fillColor,fill opacity=0.00] (0,0) rectangle (209.58,115.63);
\begin{scope}
\path[clip] (  0.00,  0.00) rectangle (209.58,115.63);
\definecolor{drawColor}{RGB}{255,255,255}
\definecolor{fillColor}{RGB}{255,255,255}

\path[draw=drawColor,line width= 0.5pt,line join=round,line cap=round,fill=fillColor] (  0.00,  0.00) rectangle (209.58,115.63);
\end{scope}
\begin{scope}
\path[clip] ( 28.25, 22.31) rectangle (203.58,113.63);
\definecolor{fillColor}{RGB}{255,255,255}

\path[fill=fillColor] ( 28.25, 22.31) rectangle (203.58,113.63);
\definecolor{drawColor}{RGB}{247,192,26}

\path[draw=drawColor,line width= 0.5pt,line join=round] ( 36.22, 26.46) --
	( 53.93, 35.72) --
	( 89.35, 46.79) --
	(124.77, 61.90) --
	(160.19, 80.50) --
	(195.61,109.48);
\definecolor{drawColor}{RGB}{78,155,133}

\path[draw=drawColor,line width= 0.5pt,dash pattern=on 4pt off 4pt ,line join=round] ( 36.22, 29.87) --
	( 53.93, 33.72) --
	( 89.35, 33.63) --
	(124.77, 35.63) --
	(160.19, 36.01) --
	(195.61, 36.72);
\definecolor{drawColor}{RGB}{247,192,26}
\definecolor{fillColor}{RGB}{247,192,26}

\path[draw=drawColor,line width= 0.4pt,line join=round,line cap=round,fill=fillColor] ( 36.22, 26.46) circle (  1.53);
\definecolor{fillColor}{RGB}{78,155,133}

\path[fill=fillColor] ( 34.68, 28.33) --
	( 37.75, 28.33) --
	( 37.75, 31.40) --
	( 34.68, 31.40) --
	cycle;
\definecolor{fillColor}{RGB}{247,192,26}

\path[draw=drawColor,line width= 0.4pt,line join=round,line cap=round,fill=fillColor] ( 53.93, 35.72) circle (  1.53);
\definecolor{fillColor}{RGB}{78,155,133}

\path[fill=fillColor] ( 52.39, 32.19) --
	( 55.46, 32.19) --
	( 55.46, 35.26) --
	( 52.39, 35.26) --
	cycle;
\definecolor{fillColor}{RGB}{247,192,26}

\path[draw=drawColor,line width= 0.4pt,line join=round,line cap=round,fill=fillColor] ( 89.35, 46.79) circle (  1.53);
\definecolor{fillColor}{RGB}{78,155,133}

\path[fill=fillColor] ( 87.81, 32.09) --
	( 90.88, 32.09) --
	( 90.88, 35.16) --
	( 87.81, 35.16) --
	cycle;
\definecolor{fillColor}{RGB}{247,192,26}

\path[draw=drawColor,line width= 0.4pt,line join=round,line cap=round,fill=fillColor] (124.77, 61.90) circle (  1.53);
\definecolor{fillColor}{RGB}{78,155,133}

\path[fill=fillColor] (123.24, 34.10) --
	(126.30, 34.10) --
	(126.30, 37.16) --
	(123.24, 37.16) --
	cycle;
\definecolor{fillColor}{RGB}{247,192,26}

\path[draw=drawColor,line width= 0.4pt,line join=round,line cap=round,fill=fillColor] (160.19, 80.50) circle (  1.53);
\definecolor{fillColor}{RGB}{78,155,133}

\path[fill=fillColor] (158.66, 34.48) --
	(161.73, 34.48) --
	(161.73, 37.54) --
	(158.66, 37.54) --
	cycle;
\definecolor{fillColor}{RGB}{247,192,26}

\path[draw=drawColor,line width= 0.4pt,line join=round,line cap=round,fill=fillColor] (195.61,109.48) circle (  1.53);
\definecolor{fillColor}{RGB}{78,155,133}

\path[fill=fillColor] (194.08, 35.19) --
	(197.15, 35.19) --
	(197.15, 38.26) --
	(194.08, 38.26) --
	cycle;
\end{scope}
\begin{scope}
\path[clip] (  0.00,  0.00) rectangle (209.58,115.63);
\definecolor{drawColor}{RGB}{0,0,0}

\path[draw=drawColor,line width= 0.5pt,line join=round] ( 28.25, 22.31) --
	( 28.25,113.63);

\path[draw=drawColor,line width= 0.5pt,line join=round] ( 29.38,111.66) --
	( 28.25,113.63) --
	( 27.11,111.66);
\end{scope}
\begin{scope}
\path[clip] (  0.00,  0.00) rectangle (209.58,115.63);
\definecolor{drawColor}{gray}{0.30}

\node[text=drawColor,anchor=base east,inner sep=0pt, outer sep=0pt, scale=  0.80] at ( 24.20, 27.90) {10};

\node[text=drawColor,anchor=base east,inner sep=0pt, outer sep=0pt, scale=  0.80] at ( 24.20, 50.51) {30};

\node[text=drawColor,anchor=base east,inner sep=0pt, outer sep=0pt, scale=  0.80] at ( 24.20, 75.30) {100};

\node[text=drawColor,anchor=base east,inner sep=0pt, outer sep=0pt, scale=  0.80] at ( 24.20, 97.91) {300};
\end{scope}
\begin{scope}
\path[clip] (  0.00,  0.00) rectangle (209.58,115.63);
\definecolor{drawColor}{gray}{0.20}

\path[draw=drawColor,line width= 0.5pt,line join=round] ( 26.00, 30.66) --
	( 28.25, 30.66);

\path[draw=drawColor,line width= 0.5pt,line join=round] ( 26.00, 53.27) --
	( 28.25, 53.27);

\path[draw=drawColor,line width= 0.5pt,line join=round] ( 26.00, 78.05) --
	( 28.25, 78.05);

\path[draw=drawColor,line width= 0.5pt,line join=round] ( 26.00,100.66) --
	( 28.25,100.66);
\end{scope}
\begin{scope}
\path[clip] (  0.00,  0.00) rectangle (209.58,115.63);
\definecolor{drawColor}{RGB}{0,0,0}

\path[draw=drawColor,line width= 0.5pt,line join=round] ( 28.25, 22.31) --
	(203.58, 22.31);

\path[draw=drawColor,line width= 0.5pt,line join=round] (201.61, 21.17) --
	(203.58, 22.31) --
	(201.61, 23.45);
\end{scope}
\begin{scope}
\path[clip] (  0.00,  0.00) rectangle (209.58,115.63);
\definecolor{drawColor}{gray}{0.20}

\path[draw=drawColor,line width= 0.5pt,line join=round] ( 36.22, 20.06) --
	( 36.22, 22.31);

\path[draw=drawColor,line width= 0.5pt,line join=round] ( 53.93, 20.06) --
	( 53.93, 22.31);

\path[draw=drawColor,line width= 0.5pt,line join=round] ( 89.35, 20.06) --
	( 89.35, 22.31);

\path[draw=drawColor,line width= 0.5pt,line join=round] (124.77, 20.06) --
	(124.77, 22.31);

\path[draw=drawColor,line width= 0.5pt,line join=round] (160.19, 20.06) --
	(160.19, 22.31);

\path[draw=drawColor,line width= 0.5pt,line join=round] (195.61, 20.06) --
	(195.61, 22.31);
\end{scope}
\begin{scope}
\path[clip] (  0.00,  0.00) rectangle (209.58,115.63);
\definecolor{drawColor}{gray}{0.30}

\node[text=drawColor,anchor=base,inner sep=0pt, outer sep=0pt, scale=  0.80] at ( 36.22, 12.75) {2};

\node[text=drawColor,anchor=base,inner sep=0pt, outer sep=0pt, scale=  0.80] at ( 53.93, 12.75) {4};

\node[text=drawColor,anchor=base,inner sep=0pt, outer sep=0pt, scale=  0.80] at ( 89.35, 12.75) {8};

\node[text=drawColor,anchor=base,inner sep=0pt, outer sep=0pt, scale=  0.80] at (124.77, 12.75) {12};

\node[text=drawColor,anchor=base,inner sep=0pt, outer sep=0pt, scale=  0.80] at (160.19, 12.75) {16};

\node[text=drawColor,anchor=base,inner sep=0pt, outer sep=0pt, scale=  0.80] at (195.61, 12.75) {20};
\end{scope}
\begin{scope}
\path[clip] (  0.00,  0.00) rectangle (209.58,115.63);
\definecolor{drawColor}{RGB}{0,0,0}

\node[text=drawColor,anchor=base,inner sep=0pt, outer sep=0pt, scale=  0.90] at (115.91,  2.75) {domain size $d$};
\end{scope}
\begin{scope}
\path[clip] (  0.00,  0.00) rectangle (209.58,115.63);
\definecolor{drawColor}{RGB}{0,0,0}

\node[text=drawColor,rotate= 90.00,anchor=base,inner sep=0pt, outer sep=0pt, scale=  0.90] at (  8.20, 67.97) {time (ms)};
\end{scope}
\begin{scope}
\path[clip] (  0.00,  0.00) rectangle (209.58,115.63);

\path[] ( 32.50, 97.23) rectangle (101.14,113.60);
\end{scope}
\begin{scope}
\path[clip] (  0.00,  0.00) rectangle (209.58,115.63);
\definecolor{drawColor}{RGB}{247,192,26}

\path[draw=drawColor,line width= 0.5pt,line join=round] ( 33.59,106.50) -- ( 42.26,106.50);
\end{scope}
\begin{scope}
\path[clip] (  0.00,  0.00) rectangle (209.58,115.63);
\definecolor{drawColor}{RGB}{247,192,26}
\definecolor{fillColor}{RGB}{247,192,26}

\path[draw=drawColor,line width= 0.4pt,line join=round,line cap=round,fill=fillColor] ( 37.92,106.50) circle (  1.53);
\end{scope}
\begin{scope}
\path[clip] (  0.00,  0.00) rectangle (209.58,115.63);
\definecolor{drawColor}{RGB}{78,155,133}

\path[draw=drawColor,line width= 0.5pt,dash pattern=on 4pt off 4pt ,line join=round] ( 33.59,100.32) -- ( 42.26,100.32);
\end{scope}
\begin{scope}
\path[clip] (  0.00,  0.00) rectangle (209.58,115.63);
\definecolor{fillColor}{RGB}{78,155,133}

\path[fill=fillColor] ( 36.39, 98.79) --
	( 39.46, 98.79) --
	( 39.46,101.86) --
	( 36.39,101.86) --
	cycle;
\end{scope}
\begin{scope}
\path[clip] (  0.00,  0.00) rectangle (209.58,115.63);
\definecolor{drawColor}{RGB}{0,0,0}

\node[text=drawColor,anchor=base west,inner sep=0pt, outer sep=0pt, scale=  0.70] at ( 47.84,104.09) {LVE (ACP)};
\end{scope}
\begin{scope}
\path[clip] (  0.00,  0.00) rectangle (209.58,115.63);
\definecolor{drawColor}{RGB}{0,0,0}

\node[text=drawColor,anchor=base west,inner sep=0pt, outer sep=0pt, scale=  0.70] at ( 47.84, 97.91) {LVE (ACP $\pm \varepsilon$)};
\end{scope}
\end{tikzpicture}}
		\caption{\acs{lve} run time (ms).}
		\label{fig:commutative_approx_plot_times_k=3_eps=0.01}
	\end{subfigure}
	\begin{subfigure}[t]{0.49\linewidth}
		\centering
		\resizebox{\linewidth}{!}{
\begin{tikzpicture}[x=1pt,y=1pt]
\definecolor{fillColor}{RGB}{255,255,255}
\path[use as bounding box,fill=fillColor,fill opacity=0.00] (0,0) rectangle (209.58,115.63);
\begin{scope}
\path[clip] (  0.00,  0.00) rectangle (209.58,115.63);
\definecolor{drawColor}{RGB}{255,255,255}
\definecolor{fillColor}{RGB}{255,255,255}

\path[draw=drawColor,line width= 0.5pt,line join=round,line cap=round,fill=fillColor] (  0.00,  0.00) rectangle (209.58,115.63);
\end{scope}
\begin{scope}
\path[clip] ( 26.47, 22.31) rectangle (203.58,113.63);
\definecolor{fillColor}{RGB}{255,255,255}

\path[fill=fillColor] ( 26.47, 22.31) rectangle (203.58,113.63);
\definecolor{drawColor}{RGB}{78,155,133}
\definecolor{fillColor}{RGB}{78,155,133}

\path[draw=drawColor,draw opacity=0.20,line width= 0.4pt,line join=round,line cap=round,fill=fillColor,fill opacity=0.20] ( 43.61, 67.88) circle (  0.78);

\path[draw=drawColor,draw opacity=0.20,line width= 0.4pt,line join=round,line cap=round,fill=fillColor,fill opacity=0.20] ( 43.61, 68.12) circle (  0.78);
\definecolor{drawColor}{RGB}{78,155,133}

\path[draw=drawColor,line width= 0.6pt,line join=round] ( 43.61, 67.98) -- ( 43.61, 67.98);

\path[draw=drawColor,line width= 0.6pt,line join=round] ( 43.61, 67.97) -- ( 43.61, 67.97);

\path[draw=drawColor,line width= 0.6pt,fill=fillColor,fill opacity=0.20] ( 32.90, 67.98) --
	( 32.90, 67.97) --
	( 54.32, 67.97) --
	( 54.32, 67.98) --
	( 32.90, 67.98) --
	cycle;

\path[draw=drawColor,line width= 1.1pt] ( 32.90, 67.97) -- ( 54.32, 67.97);
\definecolor{drawColor}{RGB}{78,155,133}

\path[draw=drawColor,draw opacity=0.20,line width= 0.4pt,line join=round,line cap=round,fill=fillColor,fill opacity=0.20] ( 72.17, 67.98) circle (  0.78);
\definecolor{drawColor}{RGB}{78,155,133}

\path[draw=drawColor,line width= 0.6pt,line join=round] ( 72.17, 67.97) -- ( 72.17, 67.98);

\path[draw=drawColor,line width= 0.6pt,line join=round] ( 72.17, 67.97) -- ( 72.17, 67.97);

\path[draw=drawColor,line width= 0.6pt,fill=fillColor,fill opacity=0.20] ( 61.46, 67.97) --
	( 61.46, 67.97) --
	( 82.89, 67.97) --
	( 82.89, 67.97) --
	( 61.46, 67.97) --
	cycle;

\path[draw=drawColor,line width= 1.1pt] ( 61.46, 67.97) -- ( 82.89, 67.97);
\definecolor{drawColor}{RGB}{78,155,133}

\path[draw=drawColor,draw opacity=0.20,line width= 0.4pt,line join=round,line cap=round,fill=fillColor,fill opacity=0.20] (100.74, 67.98) circle (  0.78);

\path[draw=drawColor,draw opacity=0.20,line width= 0.4pt,line join=round,line cap=round,fill=fillColor,fill opacity=0.20] (100.74, 67.95) circle (  0.78);
\definecolor{drawColor}{RGB}{78,155,133}

\path[draw=drawColor,line width= 0.6pt,line join=round] (100.74, 67.97) -- (100.74, 67.97);

\path[draw=drawColor,line width= 0.6pt,line join=round] (100.74, 67.97) -- (100.74, 67.97);

\path[draw=drawColor,line width= 0.6pt,fill=fillColor,fill opacity=0.20] ( 90.03, 67.97) --
	( 90.03, 67.97) --
	(111.45, 67.97) --
	(111.45, 67.97) --
	( 90.03, 67.97) --
	cycle;

\path[draw=drawColor,line width= 1.1pt] ( 90.03, 67.97) -- (111.45, 67.97);
\definecolor{drawColor}{RGB}{78,155,133}

\path[draw=drawColor,draw opacity=0.20,line width= 0.4pt,line join=round,line cap=round,fill=fillColor,fill opacity=0.20] (129.31, 30.24) circle (  0.78);

\path[draw=drawColor,draw opacity=0.20,line width= 0.4pt,line join=round,line cap=round,fill=fillColor,fill opacity=0.20] (129.31, 80.23) circle (  0.78);
\definecolor{drawColor}{RGB}{78,155,133}

\path[draw=drawColor,line width= 0.6pt,line join=round] (129.31, 68.24) -- (129.31, 68.95);

\path[draw=drawColor,line width= 0.6pt,line join=round] (129.31, 66.35) -- (129.31, 65.37);

\path[draw=drawColor,line width= 0.6pt,fill=fillColor,fill opacity=0.20] (118.60, 68.24) --
	(118.60, 66.35) --
	(140.02, 66.35) --
	(140.02, 68.24) --
	(118.60, 68.24) --
	cycle;

\path[draw=drawColor,line width= 1.1pt] (118.60, 67.97) -- (140.02, 67.97);
\definecolor{drawColor}{RGB}{78,155,133}

\path[draw=drawColor,draw opacity=0.20,line width= 0.4pt,line join=round,line cap=round,fill=fillColor,fill opacity=0.20] (157.88, 34.64) circle (  0.78);

\path[draw=drawColor,draw opacity=0.20,line width= 0.4pt,line join=round,line cap=round,fill=fillColor,fill opacity=0.20] (157.88,101.67) circle (  0.78);
\definecolor{drawColor}{RGB}{78,155,133}

\path[draw=drawColor,line width= 0.6pt,line join=round] (157.88, 69.59) -- (157.88, 70.28);

\path[draw=drawColor,line width= 0.6pt,line join=round] (157.88, 64.97) -- (157.88, 60.79);

\path[draw=drawColor,line width= 0.6pt,fill=fillColor,fill opacity=0.20] (147.16, 69.59) --
	(147.16, 64.97) --
	(168.59, 64.97) --
	(168.59, 69.59) --
	(147.16, 69.59) --
	cycle;

\path[draw=drawColor,line width= 1.1pt] (147.16, 67.29) -- (168.59, 67.29);
\definecolor{drawColor}{RGB}{78,155,133}

\path[draw=drawColor,draw opacity=0.20,line width= 0.4pt,line join=round,line cap=round,fill=fillColor,fill opacity=0.20] (186.44, 35.78) circle (  0.78);
\definecolor{drawColor}{RGB}{78,155,133}

\path[draw=drawColor,line width= 0.6pt,line join=round] (186.44, 69.60) -- (186.44, 70.13);

\path[draw=drawColor,line width= 0.6pt,line join=round] (186.44, 63.87) -- (186.44, 58.12);

\path[draw=drawColor,line width= 0.6pt,fill=fillColor,fill opacity=0.20] (175.73, 69.60) --
	(175.73, 63.87) --
	(197.16, 63.87) --
	(197.16, 69.60) --
	(175.73, 69.60) --
	cycle;

\path[draw=drawColor,line width= 1.1pt] (175.73, 67.00) -- (197.16, 67.00);
\end{scope}
\begin{scope}
\path[clip] (  0.00,  0.00) rectangle (209.58,115.63);
\definecolor{drawColor}{RGB}{0,0,0}

\path[draw=drawColor,line width= 0.5pt,line join=round] ( 26.47, 22.31) --
	( 26.47,113.63);

\path[draw=drawColor,line width= 0.5pt,line join=round] ( 27.61,111.66) --
	( 26.47,113.63) --
	( 25.33,111.66);
\end{scope}
\begin{scope}
\path[clip] (  0.00,  0.00) rectangle (209.58,115.63);
\definecolor{drawColor}{gray}{0.30}

\node[text=drawColor,anchor=base east,inner sep=0pt, outer sep=0pt, scale=  0.80] at ( 22.42, 21.87) {0.6};

\node[text=drawColor,anchor=base east,inner sep=0pt, outer sep=0pt, scale=  0.80] at ( 22.42, 43.54) {0.8};

\node[text=drawColor,anchor=base east,inner sep=0pt, outer sep=0pt, scale=  0.80] at ( 22.42, 65.22) {1.0};

\node[text=drawColor,anchor=base east,inner sep=0pt, outer sep=0pt, scale=  0.80] at ( 22.42, 86.89) {1.2};

\node[text=drawColor,anchor=base east,inner sep=0pt, outer sep=0pt, scale=  0.80] at ( 22.42,108.57) {1.4};
\end{scope}
\begin{scope}
\path[clip] (  0.00,  0.00) rectangle (209.58,115.63);
\definecolor{drawColor}{gray}{0.20}

\path[draw=drawColor,line width= 0.5pt,line join=round] ( 24.22, 24.62) --
	( 26.47, 24.62);

\path[draw=drawColor,line width= 0.5pt,line join=round] ( 24.22, 46.30) --
	( 26.47, 46.30);

\path[draw=drawColor,line width= 0.5pt,line join=round] ( 24.22, 67.97) --
	( 26.47, 67.97);

\path[draw=drawColor,line width= 0.5pt,line join=round] ( 24.22, 89.65) --
	( 26.47, 89.65);

\path[draw=drawColor,line width= 0.5pt,line join=round] ( 24.22,111.32) --
	( 26.47,111.32);
\end{scope}
\begin{scope}
\path[clip] (  0.00,  0.00) rectangle (209.58,115.63);
\definecolor{drawColor}{RGB}{0,0,0}

\path[draw=drawColor,line width= 0.5pt,line join=round] ( 26.47, 22.31) --
	(203.58, 22.31);

\path[draw=drawColor,line width= 0.5pt,line join=round] (201.61, 21.17) --
	(203.58, 22.31) --
	(201.61, 23.45);
\end{scope}
\begin{scope}
\path[clip] (  0.00,  0.00) rectangle (209.58,115.63);
\definecolor{drawColor}{gray}{0.20}

\path[draw=drawColor,line width= 0.5pt,line join=round] ( 43.61, 20.06) --
	( 43.61, 22.31);

\path[draw=drawColor,line width= 0.5pt,line join=round] ( 72.17, 20.06) --
	( 72.17, 22.31);

\path[draw=drawColor,line width= 0.5pt,line join=round] (100.74, 20.06) --
	(100.74, 22.31);

\path[draw=drawColor,line width= 0.5pt,line join=round] (129.31, 20.06) --
	(129.31, 22.31);

\path[draw=drawColor,line width= 0.5pt,line join=round] (157.88, 20.06) --
	(157.88, 22.31);

\path[draw=drawColor,line width= 0.5pt,line join=round] (186.44, 20.06) --
	(186.44, 22.31);
\end{scope}
\begin{scope}
\path[clip] (  0.00,  0.00) rectangle (209.58,115.63);
\definecolor{drawColor}{gray}{0.30}

\node[text=drawColor,anchor=base,inner sep=0pt, outer sep=0pt, scale=  0.80] at ( 43.61, 12.75) {2};

\node[text=drawColor,anchor=base,inner sep=0pt, outer sep=0pt, scale=  0.80] at ( 72.17, 12.75) {4};

\node[text=drawColor,anchor=base,inner sep=0pt, outer sep=0pt, scale=  0.80] at (100.74, 12.75) {8};

\node[text=drawColor,anchor=base,inner sep=0pt, outer sep=0pt, scale=  0.80] at (129.31, 12.75) {12};

\node[text=drawColor,anchor=base,inner sep=0pt, outer sep=0pt, scale=  0.80] at (157.88, 12.75) {16};

\node[text=drawColor,anchor=base,inner sep=0pt, outer sep=0pt, scale=  0.80] at (186.44, 12.75) {20};
\end{scope}
\begin{scope}
\path[clip] (  0.00,  0.00) rectangle (209.58,115.63);
\definecolor{drawColor}{RGB}{0,0,0}

\node[text=drawColor,anchor=base,inner sep=0pt, outer sep=0pt, scale=  0.90] at (115.03,  2.75) {domain size $d$};
\end{scope}
\begin{scope}
\path[clip] (  0.00,  0.00) rectangle (209.58,115.63);
\definecolor{drawColor}{RGB}{0,0,0}

\node[text=drawColor,rotate= 90.00,anchor=base,inner sep=0pt, outer sep=0pt, scale=  0.90] at (  8.20, 67.97) {$p' \mathbin{/} p$};
\end{scope}
\end{tikzpicture}}
		\caption{Query result quotient $p' \mathbin{/} p$.}
		\label{fig:commutative_approx_plot_quots_k=3_eps=0.01}
	\end{subfigure}
	\caption{(a) Run time of \ac{lve} on the compressed model returned by \ac{acp} and its $\varepsilon$-relaxed variant (\acs{acp} $\pm \varepsilon$), and (b) distribution of the per-query quotient $p' \mathbin{/} p$ between the approximate result $p'$ and the exact result $p$ for input \acp{fg} containing $k = 3$ $\varepsilon$-commutative factors with $\varepsilon = 0.01$.}
	\label{fig:commutative_approx_plot_te_k=3_eps=0.01}
\end{figure}
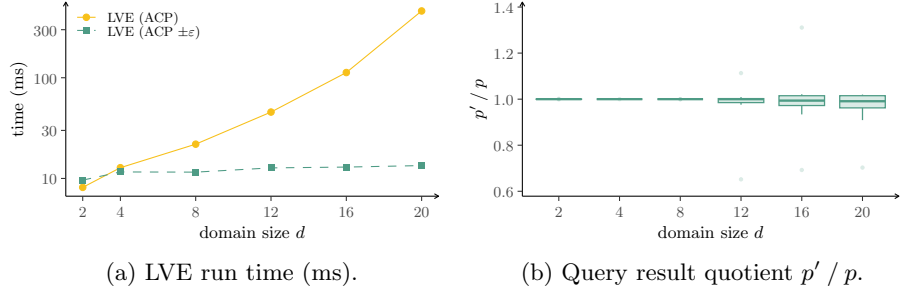

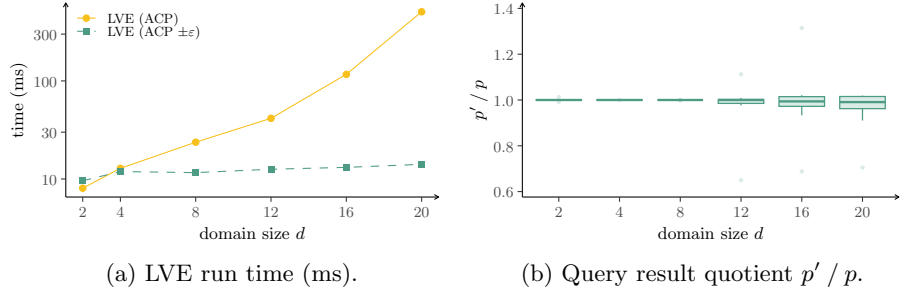
\begin{figure}[t]
	\centering
	\begin{subfigure}[t]{0.49\linewidth}
		\centering
		\resizebox{\linewidth}{!}{
\begin{tikzpicture}[x=1pt,y=1pt]
\definecolor{fillColor}{RGB}{255,255,255}
\path[use as bounding box,fill=fillColor,fill opacity=0.00] (0,0) rectangle (209.58,115.63);
\begin{scope}
\path[clip] (  0.00,  0.00) rectangle (209.58,115.63);
\definecolor{drawColor}{RGB}{255,255,255}
\definecolor{fillColor}{RGB}{255,255,255}

\path[draw=drawColor,line width= 0.5pt,line join=round,line cap=round,fill=fillColor] (  0.00,  0.00) rectangle (209.58,115.63);
\end{scope}
\begin{scope}
\path[clip] ( 28.25, 22.31) rectangle (203.58,113.63);
\definecolor{fillColor}{RGB}{255,255,255}

\path[fill=fillColor] ( 28.25, 22.31) rectangle (203.58,113.63);
\definecolor{drawColor}{RGB}{247,192,26}

\path[draw=drawColor,line width= 0.5pt,line join=round] ( 36.22, 26.46) --
	( 53.93, 35.82) --
	( 89.35, 48.11) --
	(124.77, 59.36) --
	(160.19, 80.04) --
	(195.61,109.48);
\definecolor{drawColor}{RGB}{78,155,133}

\path[draw=drawColor,line width= 0.5pt,dash pattern=on 4pt off 4pt ,line join=round] ( 36.22, 29.98) --
	( 53.93, 34.29) --
	( 89.35, 33.79) --
	(124.77, 35.39) --
	(160.19, 36.27) --
	(195.61, 37.69);
\definecolor{drawColor}{RGB}{247,192,26}
\definecolor{fillColor}{RGB}{247,192,26}

\path[draw=drawColor,line width= 0.4pt,line join=round,line cap=round,fill=fillColor] ( 36.22, 26.46) circle (  1.53);
\definecolor{fillColor}{RGB}{78,155,133}

\path[fill=fillColor] ( 34.68, 28.45) --
	( 37.75, 28.45) --
	( 37.75, 31.51) --
	( 34.68, 31.51) --
	cycle;
\definecolor{fillColor}{RGB}{247,192,26}

\path[draw=drawColor,line width= 0.4pt,line join=round,line cap=round,fill=fillColor] ( 53.93, 35.82) circle (  1.53);
\definecolor{fillColor}{RGB}{78,155,133}

\path[fill=fillColor] ( 52.39, 32.76) --
	( 55.46, 32.76) --
	( 55.46, 35.82) --
	( 52.39, 35.82) --
	cycle;
\definecolor{fillColor}{RGB}{247,192,26}

\path[draw=drawColor,line width= 0.4pt,line join=round,line cap=round,fill=fillColor] ( 89.35, 48.11) circle (  1.53);
\definecolor{fillColor}{RGB}{78,155,133}

\path[fill=fillColor] ( 87.81, 32.25) --
	( 90.88, 32.25) --
	( 90.88, 35.32) --
	( 87.81, 35.32) --
	cycle;
\definecolor{fillColor}{RGB}{247,192,26}

\path[draw=drawColor,line width= 0.4pt,line join=round,line cap=round,fill=fillColor] (124.77, 59.36) circle (  1.53);
\definecolor{fillColor}{RGB}{78,155,133}

\path[fill=fillColor] (123.24, 33.86) --
	(126.30, 33.86) --
	(126.30, 36.93) --
	(123.24, 36.93) --
	cycle;
\definecolor{fillColor}{RGB}{247,192,26}

\path[draw=drawColor,line width= 0.4pt,line join=round,line cap=round,fill=fillColor] (160.19, 80.04) circle (  1.53);
\definecolor{fillColor}{RGB}{78,155,133}

\path[fill=fillColor] (158.66, 34.74) --
	(161.73, 34.74) --
	(161.73, 37.81) --
	(158.66, 37.81) --
	cycle;
\definecolor{fillColor}{RGB}{247,192,26}

\path[draw=drawColor,line width= 0.4pt,line join=round,line cap=round,fill=fillColor] (195.61,109.48) circle (  1.53);
\definecolor{fillColor}{RGB}{78,155,133}

\path[fill=fillColor] (194.08, 36.16) --
	(197.15, 36.16) --
	(197.15, 39.23) --
	(194.08, 39.23) --
	cycle;
\end{scope}
\begin{scope}
\path[clip] (  0.00,  0.00) rectangle (209.58,115.63);
\definecolor{drawColor}{RGB}{0,0,0}

\path[draw=drawColor,line width= 0.5pt,line join=round] ( 28.25, 22.31) --
	( 28.25,113.63);

\path[draw=drawColor,line width= 0.5pt,line join=round] ( 29.38,111.66) --
	( 28.25,113.63) --
	( 27.11,111.66);
\end{scope}
\begin{scope}
\path[clip] (  0.00,  0.00) rectangle (209.58,115.63);
\definecolor{drawColor}{gray}{0.30}

\node[text=drawColor,anchor=base east,inner sep=0pt, outer sep=0pt, scale=  0.80] at ( 24.20, 28.04) {10};

\node[text=drawColor,anchor=base east,inner sep=0pt, outer sep=0pt, scale=  0.80] at ( 24.20, 50.06) {30};

\node[text=drawColor,anchor=base east,inner sep=0pt, outer sep=0pt, scale=  0.80] at ( 24.20, 74.19) {100};

\node[text=drawColor,anchor=base east,inner sep=0pt, outer sep=0pt, scale=  0.80] at ( 24.20, 96.21) {300};
\end{scope}
\begin{scope}
\path[clip] (  0.00,  0.00) rectangle (209.58,115.63);
\definecolor{drawColor}{gray}{0.20}

\path[draw=drawColor,line width= 0.5pt,line join=round] ( 26.00, 30.80) --
	( 28.25, 30.80);

\path[draw=drawColor,line width= 0.5pt,line join=round] ( 26.00, 52.82) --
	( 28.25, 52.82);

\path[draw=drawColor,line width= 0.5pt,line join=round] ( 26.00, 76.95) --
	( 28.25, 76.95);

\path[draw=drawColor,line width= 0.5pt,line join=round] ( 26.00, 98.97) --
	( 28.25, 98.97);
\end{scope}
\begin{scope}
\path[clip] (  0.00,  0.00) rectangle (209.58,115.63);
\definecolor{drawColor}{RGB}{0,0,0}

\path[draw=drawColor,line width= 0.5pt,line join=round] ( 28.25, 22.31) --
	(203.58, 22.31);

\path[draw=drawColor,line width= 0.5pt,line join=round] (201.61, 21.17) --
	(203.58, 22.31) --
	(201.61, 23.45);
\end{scope}
\begin{scope}
\path[clip] (  0.00,  0.00) rectangle (209.58,115.63);
\definecolor{drawColor}{gray}{0.20}

\path[draw=drawColor,line width= 0.5pt,line join=round] ( 36.22, 20.06) --
	( 36.22, 22.31);

\path[draw=drawColor,line width= 0.5pt,line join=round] ( 53.93, 20.06) --
	( 53.93, 22.31);

\path[draw=drawColor,line width= 0.5pt,line join=round] ( 89.35, 20.06) --
	( 89.35, 22.31);

\path[draw=drawColor,line width= 0.5pt,line join=round] (124.77, 20.06) --
	(124.77, 22.31);

\path[draw=drawColor,line width= 0.5pt,line join=round] (160.19, 20.06) --
	(160.19, 22.31);

\path[draw=drawColor,line width= 0.5pt,line join=round] (195.61, 20.06) --
	(195.61, 22.31);
\end{scope}
\begin{scope}
\path[clip] (  0.00,  0.00) rectangle (209.58,115.63);
\definecolor{drawColor}{gray}{0.30}

\node[text=drawColor,anchor=base,inner sep=0pt, outer sep=0pt, scale=  0.80] at ( 36.22, 12.75) {2};

\node[text=drawColor,anchor=base,inner sep=0pt, outer sep=0pt, scale=  0.80] at ( 53.93, 12.75) {4};

\node[text=drawColor,anchor=base,inner sep=0pt, outer sep=0pt, scale=  0.80] at ( 89.35, 12.75) {8};

\node[text=drawColor,anchor=base,inner sep=0pt, outer sep=0pt, scale=  0.80] at (124.77, 12.75) {12};

\node[text=drawColor,anchor=base,inner sep=0pt, outer sep=0pt, scale=  0.80] at (160.19, 12.75) {16};

\node[text=drawColor,anchor=base,inner sep=0pt, outer sep=0pt, scale=  0.80] at (195.61, 12.75) {20};
\end{scope}
\begin{scope}
\path[clip] (  0.00,  0.00) rectangle (209.58,115.63);
\definecolor{drawColor}{RGB}{0,0,0}

\node[text=drawColor,anchor=base,inner sep=0pt, outer sep=0pt, scale=  0.90] at (115.91,  2.75) {domain size $d$};
\end{scope}
\begin{scope}
\path[clip] (  0.00,  0.00) rectangle (209.58,115.63);
\definecolor{drawColor}{RGB}{0,0,0}

\node[text=drawColor,rotate= 90.00,anchor=base,inner sep=0pt, outer sep=0pt, scale=  0.90] at (  8.20, 67.97) {time (ms)};
\end{scope}
\begin{scope}
\path[clip] (  0.00,  0.00) rectangle (209.58,115.63);

\path[] ( 32.50, 97.23) rectangle (101.14,113.60);
\end{scope}
\begin{scope}
\path[clip] (  0.00,  0.00) rectangle (209.58,115.63);
\definecolor{drawColor}{RGB}{247,192,26}

\path[draw=drawColor,line width= 0.5pt,line join=round] ( 33.59,106.50) -- ( 42.26,106.50);
\end{scope}
\begin{scope}
\path[clip] (  0.00,  0.00) rectangle (209.58,115.63);
\definecolor{drawColor}{RGB}{247,192,26}
\definecolor{fillColor}{RGB}{247,192,26}

\path[draw=drawColor,line width= 0.4pt,line join=round,line cap=round,fill=fillColor] ( 37.92,106.50) circle (  1.53);
\end{scope}
\begin{scope}
\path[clip] (  0.00,  0.00) rectangle (209.58,115.63);
\definecolor{drawColor}{RGB}{78,155,133}

\path[draw=drawColor,line width= 0.5pt,dash pattern=on 4pt off 4pt ,line join=round] ( 33.59,100.32) -- ( 42.26,100.32);
\end{scope}
\begin{scope}
\path[clip] (  0.00,  0.00) rectangle (209.58,115.63);
\definecolor{fillColor}{RGB}{78,155,133}

\path[fill=fillColor] ( 36.39, 98.79) --
	( 39.46, 98.79) --
	( 39.46,101.86) --
	( 36.39,101.86) --
	cycle;
\end{scope}
\begin{scope}
\path[clip] (  0.00,  0.00) rectangle (209.58,115.63);
\definecolor{drawColor}{RGB}{0,0,0}

\node[text=drawColor,anchor=base west,inner sep=0pt, outer sep=0pt, scale=  0.70] at ( 47.84,104.09) {LVE (ACP)};
\end{scope}
\begin{scope}
\path[clip] (  0.00,  0.00) rectangle (209.58,115.63);
\definecolor{drawColor}{RGB}{0,0,0}

\node[text=drawColor,anchor=base west,inner sep=0pt, outer sep=0pt, scale=  0.70] at ( 47.84, 97.91) {LVE (ACP $\pm \varepsilon$)};
\end{scope}
\end{tikzpicture}}
		\caption{\acs{lve} run time (ms).}
		\label{fig:commutative_approx_plot_times_k=3_eps=0.1}
	\end{subfigure}
	\begin{subfigure}[t]{0.49\linewidth}
		\centering
		\resizebox{\linewidth}{!}{
\begin{tikzpicture}[x=1pt,y=1pt]
\definecolor{fillColor}{RGB}{255,255,255}
\path[use as bounding box,fill=fillColor,fill opacity=0.00] (0,0) rectangle (209.58,115.63);
\begin{scope}
\path[clip] (  0.00,  0.00) rectangle (209.58,115.63);
\definecolor{drawColor}{RGB}{255,255,255}
\definecolor{fillColor}{RGB}{255,255,255}

\path[draw=drawColor,line width= 0.5pt,line join=round,line cap=round,fill=fillColor] (  0.00,  0.00) rectangle (209.58,115.63);
\end{scope}
\begin{scope}
\path[clip] ( 26.47, 22.31) rectangle (203.58,113.63);
\definecolor{fillColor}{RGB}{255,255,255}

\path[fill=fillColor] ( 26.47, 22.31) rectangle (203.58,113.63);
\definecolor{drawColor}{RGB}{78,155,133}
\definecolor{fillColor}{RGB}{78,155,133}

\path[draw=drawColor,draw opacity=0.20,line width= 0.4pt,line join=round,line cap=round,fill=fillColor,fill opacity=0.20] ( 43.61, 67.09) circle (  0.78);

\path[draw=drawColor,draw opacity=0.20,line width= 0.4pt,line join=round,line cap=round,fill=fillColor,fill opacity=0.20] ( 43.61, 69.34) circle (  0.78);
\definecolor{drawColor}{RGB}{78,155,133}

\path[draw=drawColor,line width= 0.6pt,line join=round] ( 43.61, 68.03) -- ( 43.61, 68.03);

\path[draw=drawColor,line width= 0.6pt,line join=round] ( 43.61, 67.93) -- ( 43.61, 67.92);

\path[draw=drawColor,line width= 0.6pt,fill=fillColor,fill opacity=0.20] ( 32.90, 68.03) --
	( 32.90, 67.93) --
	( 54.32, 67.93) --
	( 54.32, 68.03) --
	( 32.90, 68.03) --
	cycle;

\path[draw=drawColor,line width= 1.1pt] ( 32.90, 67.97) -- ( 54.32, 67.97);
\definecolor{drawColor}{RGB}{78,155,133}

\path[draw=drawColor,draw opacity=0.20,line width= 0.4pt,line join=round,line cap=round,fill=fillColor,fill opacity=0.20] ( 72.17, 68.05) circle (  0.78);
\definecolor{drawColor}{RGB}{78,155,133}

\path[draw=drawColor,line width= 0.6pt,line join=round] ( 72.17, 67.99) -- ( 72.17, 68.02);

\path[draw=drawColor,line width= 0.6pt,line join=round] ( 72.17, 67.96) -- ( 72.17, 67.94);

\path[draw=drawColor,line width= 0.6pt,fill=fillColor,fill opacity=0.20] ( 61.46, 67.99) --
	( 61.46, 67.96) --
	( 82.89, 67.96) --
	( 82.89, 67.99) --
	( 61.46, 67.99) --
	cycle;

\path[draw=drawColor,line width= 1.1pt] ( 61.46, 67.97) -- ( 82.89, 67.97);
\definecolor{drawColor}{RGB}{78,155,133}

\path[draw=drawColor,draw opacity=0.20,line width= 0.4pt,line join=round,line cap=round,fill=fillColor,fill opacity=0.20] (100.74, 68.02) circle (  0.78);

\path[draw=drawColor,draw opacity=0.20,line width= 0.4pt,line join=round,line cap=round,fill=fillColor,fill opacity=0.20] (100.74, 67.77) circle (  0.78);
\definecolor{drawColor}{RGB}{78,155,133}

\path[draw=drawColor,line width= 0.6pt,line join=round] (100.74, 67.97) -- (100.74, 67.97);

\path[draw=drawColor,line width= 0.6pt,line join=round] (100.74, 67.97) -- (100.74, 67.97);

\path[draw=drawColor,line width= 0.6pt,fill=fillColor,fill opacity=0.20] ( 90.03, 67.97) --
	( 90.03, 67.97) --
	(111.45, 67.97) --
	(111.45, 67.97) --
	( 90.03, 67.97) --
	cycle;

\path[draw=drawColor,line width= 1.1pt] ( 90.03, 67.97) -- (111.45, 67.97);
\definecolor{drawColor}{RGB}{78,155,133}

\path[draw=drawColor,draw opacity=0.20,line width= 0.4pt,line join=round,line cap=round,fill=fillColor,fill opacity=0.20] (129.31, 30.24) circle (  0.78);

\path[draw=drawColor,draw opacity=0.20,line width= 0.4pt,line join=round,line cap=round,fill=fillColor,fill opacity=0.20] (129.31, 80.06) circle (  0.78);
\definecolor{drawColor}{RGB}{78,155,133}

\path[draw=drawColor,line width= 0.6pt,line join=round] (129.31, 68.24) -- (129.31, 68.94);

\path[draw=drawColor,line width= 0.6pt,line join=round] (129.31, 66.37) -- (129.31, 65.41);

\path[draw=drawColor,line width= 0.6pt,fill=fillColor,fill opacity=0.20] (118.60, 68.24) --
	(118.60, 66.37) --
	(140.02, 66.37) --
	(140.02, 68.24) --
	(118.60, 68.24) --
	cycle;

\path[draw=drawColor,line width= 1.1pt] (118.60, 67.97) -- (140.02, 67.97);
\definecolor{drawColor}{RGB}{78,155,133}

\path[draw=drawColor,draw opacity=0.20,line width= 0.4pt,line join=round,line cap=round,fill=fillColor,fill opacity=0.20] (157.88, 34.38) circle (  0.78);

\path[draw=drawColor,draw opacity=0.20,line width= 0.4pt,line join=round,line cap=round,fill=fillColor,fill opacity=0.20] (157.88,101.85) circle (  0.78);
\definecolor{drawColor}{RGB}{78,155,133}

\path[draw=drawColor,line width= 0.6pt,line join=round] (157.88, 69.55) -- (157.88, 70.24);

\path[draw=drawColor,line width= 0.6pt,line join=round] (157.88, 64.98) -- (157.88, 60.77);

\path[draw=drawColor,line width= 0.6pt,fill=fillColor,fill opacity=0.20] (147.16, 69.55) --
	(147.16, 64.98) --
	(168.59, 64.98) --
	(168.59, 69.55) --
	(147.16, 69.55) --
	cycle;

\path[draw=drawColor,line width= 1.1pt] (147.16, 67.31) -- (168.59, 67.31);
\definecolor{drawColor}{RGB}{78,155,133}

\path[draw=drawColor,draw opacity=0.20,line width= 0.4pt,line join=round,line cap=round,fill=fillColor,fill opacity=0.20] (186.44, 36.20) circle (  0.78);
\definecolor{drawColor}{RGB}{78,155,133}

\path[draw=drawColor,line width= 0.6pt,line join=round] (186.44, 69.62) -- (186.44, 70.08);

\path[draw=drawColor,line width= 0.6pt,line join=round] (186.44, 63.90) -- (186.44, 58.32);

\path[draw=drawColor,line width= 0.6pt,fill=fillColor,fill opacity=0.20] (175.73, 69.62) --
	(175.73, 63.90) --
	(197.16, 63.90) --
	(197.16, 69.62) --
	(175.73, 69.62) --
	cycle;

\path[draw=drawColor,line width= 1.1pt] (175.73, 66.99) -- (197.16, 66.99);
\end{scope}
\begin{scope}
\path[clip] (  0.00,  0.00) rectangle (209.58,115.63);
\definecolor{drawColor}{RGB}{0,0,0}

\path[draw=drawColor,line width= 0.5pt,line join=round] ( 26.47, 22.31) --
	( 26.47,113.63);

\path[draw=drawColor,line width= 0.5pt,line join=round] ( 27.61,111.66) --
	( 26.47,113.63) --
	( 25.33,111.66);
\end{scope}
\begin{scope}
\path[clip] (  0.00,  0.00) rectangle (209.58,115.63);
\definecolor{drawColor}{gray}{0.30}

\node[text=drawColor,anchor=base east,inner sep=0pt, outer sep=0pt, scale=  0.80] at ( 22.42, 22.11) {0.6};

\node[text=drawColor,anchor=base east,inner sep=0pt, outer sep=0pt, scale=  0.80] at ( 22.42, 43.67) {0.8};

\node[text=drawColor,anchor=base east,inner sep=0pt, outer sep=0pt, scale=  0.80] at ( 22.42, 65.22) {1.0};

\node[text=drawColor,anchor=base east,inner sep=0pt, outer sep=0pt, scale=  0.80] at ( 22.42, 86.77) {1.2};

\node[text=drawColor,anchor=base east,inner sep=0pt, outer sep=0pt, scale=  0.80] at ( 22.42,108.32) {1.4};
\end{scope}
\begin{scope}
\path[clip] (  0.00,  0.00) rectangle (209.58,115.63);
\definecolor{drawColor}{gray}{0.20}

\path[draw=drawColor,line width= 0.5pt,line join=round] ( 24.22, 24.87) --
	( 26.47, 24.87);

\path[draw=drawColor,line width= 0.5pt,line join=round] ( 24.22, 46.42) --
	( 26.47, 46.42);

\path[draw=drawColor,line width= 0.5pt,line join=round] ( 24.22, 67.97) --
	( 26.47, 67.97);

\path[draw=drawColor,line width= 0.5pt,line join=round] ( 24.22, 89.52) --
	( 26.47, 89.52);

\path[draw=drawColor,line width= 0.5pt,line join=round] ( 24.22,111.08) --
	( 26.47,111.08);
\end{scope}
\begin{scope}
\path[clip] (  0.00,  0.00) rectangle (209.58,115.63);
\definecolor{drawColor}{RGB}{0,0,0}

\path[draw=drawColor,line width= 0.5pt,line join=round] ( 26.47, 22.31) --
	(203.58, 22.31);

\path[draw=drawColor,line width= 0.5pt,line join=round] (201.61, 21.17) --
	(203.58, 22.31) --
	(201.61, 23.45);
\end{scope}
\begin{scope}
\path[clip] (  0.00,  0.00) rectangle (209.58,115.63);
\definecolor{drawColor}{gray}{0.20}

\path[draw=drawColor,line width= 0.5pt,line join=round] ( 43.61, 20.06) --
	( 43.61, 22.31);

\path[draw=drawColor,line width= 0.5pt,line join=round] ( 72.17, 20.06) --
	( 72.17, 22.31);

\path[draw=drawColor,line width= 0.5pt,line join=round] (100.74, 20.06) --
	(100.74, 22.31);

\path[draw=drawColor,line width= 0.5pt,line join=round] (129.31, 20.06) --
	(129.31, 22.31);

\path[draw=drawColor,line width= 0.5pt,line join=round] (157.88, 20.06) --
	(157.88, 22.31);

\path[draw=drawColor,line width= 0.5pt,line join=round] (186.44, 20.06) --
	(186.44, 22.31);
\end{scope}
\begin{scope}
\path[clip] (  0.00,  0.00) rectangle (209.58,115.63);
\definecolor{drawColor}{gray}{0.30}

\node[text=drawColor,anchor=base,inner sep=0pt, outer sep=0pt, scale=  0.80] at ( 43.61, 12.75) {2};

\node[text=drawColor,anchor=base,inner sep=0pt, outer sep=0pt, scale=  0.80] at ( 72.17, 12.75) {4};

\node[text=drawColor,anchor=base,inner sep=0pt, outer sep=0pt, scale=  0.80] at (100.74, 12.75) {8};

\node[text=drawColor,anchor=base,inner sep=0pt, outer sep=0pt, scale=  0.80] at (129.31, 12.75) {12};

\node[text=drawColor,anchor=base,inner sep=0pt, outer sep=0pt, scale=  0.80] at (157.88, 12.75) {16};

\node[text=drawColor,anchor=base,inner sep=0pt, outer sep=0pt, scale=  0.80] at (186.44, 12.75) {20};
\end{scope}
\begin{scope}
\path[clip] (  0.00,  0.00) rectangle (209.58,115.63);
\definecolor{drawColor}{RGB}{0,0,0}

\node[text=drawColor,anchor=base,inner sep=0pt, outer sep=0pt, scale=  0.90] at (115.03,  2.75) {domain size $d$};
\end{scope}
\begin{scope}
\path[clip] (  0.00,  0.00) rectangle (209.58,115.63);
\definecolor{drawColor}{RGB}{0,0,0}

\node[text=drawColor,rotate= 90.00,anchor=base,inner sep=0pt, outer sep=0pt, scale=  0.90] at (  8.20, 67.97) {$p' \mathbin{/} p$};
\end{scope}
\end{tikzpicture}}
		\caption{Query result quotient $p' \mathbin{/} p$.}
		\label{fig:commutative_approx_plot_quots_k=3_eps=0.1}
	\end{subfigure}
	\caption{(a) Run time of \ac{lve} on the compressed model returned by \ac{acp} and its $\varepsilon$-relaxed variant (\acs{acp} $\pm \varepsilon$), and (b) distribution of the per-query quotient $p' \mathbin{/} p$ between the approximate result $p'$ and the exact result $p$ for input \acp{fg} containing $k = 3$ $\varepsilon$-commutative factors with $\varepsilon = 0.1$.}
	\label{fig:commutative_approx_plot_te_k=3_eps=0.1}
\end{figure}

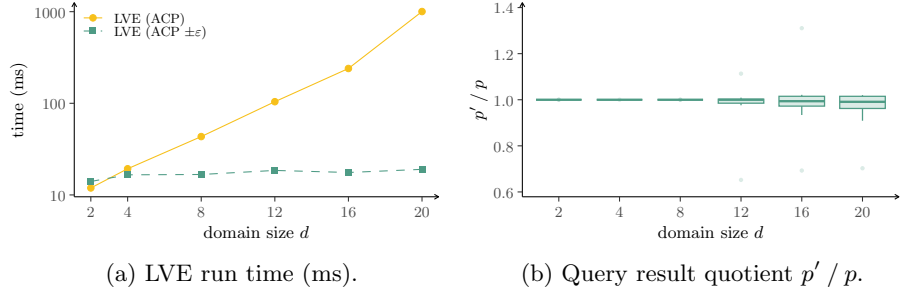
\begin{figure}[t]
	\centering
	\begin{subfigure}[t]{0.49\linewidth}
		\centering
		\resizebox{\linewidth}{!}{
\begin{tikzpicture}[x=1pt,y=1pt]
\definecolor{fillColor}{RGB}{255,255,255}
\path[use as bounding box,fill=fillColor,fill opacity=0.00] (0,0) rectangle (209.58,115.63);
\begin{scope}
\path[clip] (  0.00,  0.00) rectangle (209.58,115.63);
\definecolor{drawColor}{RGB}{255,255,255}
\definecolor{fillColor}{RGB}{255,255,255}

\path[draw=drawColor,line width= 0.5pt,line join=round,line cap=round,fill=fillColor] (  0.00,  0.00) rectangle (209.58,115.63);
\end{scope}
\begin{scope}
\path[clip] ( 32.24, 22.31) rectangle (203.58,113.63);
\definecolor{fillColor}{RGB}{255,255,255}

\path[fill=fillColor] ( 32.24, 22.31) rectangle (203.58,113.63);
\definecolor{drawColor}{RGB}{247,192,26}

\path[draw=drawColor,line width= 0.5pt,line join=round] ( 40.03, 26.46) --
	( 57.34, 35.49) --
	( 91.95, 50.64) --
	(126.57, 67.05) --
	(161.18, 82.68) --
	(195.79,109.48);
\definecolor{drawColor}{RGB}{78,155,133}

\path[draw=drawColor,line width= 0.5pt,dash pattern=on 4pt off 4pt ,line join=round] ( 40.03, 29.52) --
	( 57.34, 32.64) --
	( 91.95, 32.80) --
	(126.57, 34.71) --
	(161.18, 33.74) --
	(195.79, 35.22);
\definecolor{drawColor}{RGB}{247,192,26}
\definecolor{fillColor}{RGB}{247,192,26}

\path[draw=drawColor,line width= 0.4pt,line join=round,line cap=round,fill=fillColor] ( 40.03, 26.46) circle (  1.53);
\definecolor{fillColor}{RGB}{78,155,133}

\path[fill=fillColor] ( 38.50, 27.99) --
	( 41.57, 27.99) --
	( 41.57, 31.06) --
	( 38.50, 31.06) --
	cycle;
\definecolor{fillColor}{RGB}{247,192,26}

\path[draw=drawColor,line width= 0.4pt,line join=round,line cap=round,fill=fillColor] ( 57.34, 35.49) circle (  1.53);
\definecolor{fillColor}{RGB}{78,155,133}

\path[fill=fillColor] ( 55.81, 31.11) --
	( 58.87, 31.11) --
	( 58.87, 34.18) --
	( 55.81, 34.18) --
	cycle;
\definecolor{fillColor}{RGB}{247,192,26}

\path[draw=drawColor,line width= 0.4pt,line join=round,line cap=round,fill=fillColor] ( 91.95, 50.64) circle (  1.53);
\definecolor{fillColor}{RGB}{78,155,133}

\path[fill=fillColor] ( 90.42, 31.26) --
	( 93.49, 31.26) --
	( 93.49, 34.33) --
	( 90.42, 34.33) --
	cycle;
\definecolor{fillColor}{RGB}{247,192,26}

\path[draw=drawColor,line width= 0.4pt,line join=round,line cap=round,fill=fillColor] (126.57, 67.05) circle (  1.53);
\definecolor{fillColor}{RGB}{78,155,133}

\path[fill=fillColor] (125.03, 33.18) --
	(128.10, 33.18) --
	(128.10, 36.25) --
	(125.03, 36.25) --
	cycle;
\definecolor{fillColor}{RGB}{247,192,26}

\path[draw=drawColor,line width= 0.4pt,line join=round,line cap=round,fill=fillColor] (161.18, 82.68) circle (  1.53);
\definecolor{fillColor}{RGB}{78,155,133}

\path[fill=fillColor] (159.65, 32.20) --
	(162.71, 32.20) --
	(162.71, 35.27) --
	(159.65, 35.27) --
	cycle;
\definecolor{fillColor}{RGB}{247,192,26}

\path[draw=drawColor,line width= 0.4pt,line join=round,line cap=round,fill=fillColor] (195.79,109.48) circle (  1.53);
\definecolor{fillColor}{RGB}{78,155,133}

\path[fill=fillColor] (194.26, 33.69) --
	(197.33, 33.69) --
	(197.33, 36.76) --
	(194.26, 36.76) --
	cycle;
\end{scope}
\begin{scope}
\path[clip] (  0.00,  0.00) rectangle (209.58,115.63);
\definecolor{drawColor}{RGB}{0,0,0}

\path[draw=drawColor,line width= 0.5pt,line join=round] ( 32.24, 22.31) --
	( 32.24,113.63);

\path[draw=drawColor,line width= 0.5pt,line join=round] ( 33.38,111.66) --
	( 32.24,113.63) --
	( 31.11,111.66);
\end{scope}
\begin{scope}
\path[clip] (  0.00,  0.00) rectangle (209.58,115.63);
\definecolor{drawColor}{gray}{0.30}

\node[text=drawColor,anchor=base east,inner sep=0pt, outer sep=0pt, scale=  0.80] at ( 28.19, 20.43) {10};

\node[text=drawColor,anchor=base east,inner sep=0pt, outer sep=0pt, scale=  0.80] at ( 28.19, 63.55) {100};

\node[text=drawColor,anchor=base east,inner sep=0pt, outer sep=0pt, scale=  0.80] at ( 28.19,106.67) {1000};
\end{scope}
\begin{scope}
\path[clip] (  0.00,  0.00) rectangle (209.58,115.63);
\definecolor{drawColor}{gray}{0.20}

\path[draw=drawColor,line width= 0.5pt,line join=round] ( 29.99, 23.18) --
	( 32.24, 23.18);

\path[draw=drawColor,line width= 0.5pt,line join=round] ( 29.99, 66.30) --
	( 32.24, 66.30);

\path[draw=drawColor,line width= 0.5pt,line join=round] ( 29.99,109.42) --
	( 32.24,109.42);
\end{scope}
\begin{scope}
\path[clip] (  0.00,  0.00) rectangle (209.58,115.63);
\definecolor{drawColor}{RGB}{0,0,0}

\path[draw=drawColor,line width= 0.5pt,line join=round] ( 32.24, 22.31) --
	(203.58, 22.31);

\path[draw=drawColor,line width= 0.5pt,line join=round] (201.61, 21.17) --
	(203.58, 22.31) --
	(201.61, 23.45);
\end{scope}
\begin{scope}
\path[clip] (  0.00,  0.00) rectangle (209.58,115.63);
\definecolor{drawColor}{gray}{0.20}

\path[draw=drawColor,line width= 0.5pt,line join=round] ( 40.03, 20.06) --
	( 40.03, 22.31);

\path[draw=drawColor,line width= 0.5pt,line join=round] ( 57.34, 20.06) --
	( 57.34, 22.31);

\path[draw=drawColor,line width= 0.5pt,line join=round] ( 91.95, 20.06) --
	( 91.95, 22.31);

\path[draw=drawColor,line width= 0.5pt,line join=round] (126.57, 20.06) --
	(126.57, 22.31);

\path[draw=drawColor,line width= 0.5pt,line join=round] (161.18, 20.06) --
	(161.18, 22.31);

\path[draw=drawColor,line width= 0.5pt,line join=round] (195.79, 20.06) --
	(195.79, 22.31);
\end{scope}
\begin{scope}
\path[clip] (  0.00,  0.00) rectangle (209.58,115.63);
\definecolor{drawColor}{gray}{0.30}

\node[text=drawColor,anchor=base,inner sep=0pt, outer sep=0pt, scale=  0.80] at ( 40.03, 12.75) {2};

\node[text=drawColor,anchor=base,inner sep=0pt, outer sep=0pt, scale=  0.80] at ( 57.34, 12.75) {4};

\node[text=drawColor,anchor=base,inner sep=0pt, outer sep=0pt, scale=  0.80] at ( 91.95, 12.75) {8};

\node[text=drawColor,anchor=base,inner sep=0pt, outer sep=0pt, scale=  0.80] at (126.57, 12.75) {12};

\node[text=drawColor,anchor=base,inner sep=0pt, outer sep=0pt, scale=  0.80] at (161.18, 12.75) {16};

\node[text=drawColor,anchor=base,inner sep=0pt, outer sep=0pt, scale=  0.80] at (195.79, 12.75) {20};
\end{scope}
\begin{scope}
\path[clip] (  0.00,  0.00) rectangle (209.58,115.63);
\definecolor{drawColor}{RGB}{0,0,0}

\node[text=drawColor,anchor=base,inner sep=0pt, outer sep=0pt, scale=  0.90] at (117.91,  2.75) {domain size $d$};
\end{scope}
\begin{scope}
\path[clip] (  0.00,  0.00) rectangle (209.58,115.63);
\definecolor{drawColor}{RGB}{0,0,0}

\node[text=drawColor,rotate= 90.00,anchor=base,inner sep=0pt, outer sep=0pt, scale=  0.90] at (  8.20, 67.97) {time (ms)};
\end{scope}
\begin{scope}
\path[clip] (  0.00,  0.00) rectangle (209.58,115.63);

\path[] ( 35.62, 97.23) rectangle (104.26,113.60);
\end{scope}
\begin{scope}
\path[clip] (  0.00,  0.00) rectangle (209.58,115.63);
\definecolor{drawColor}{RGB}{247,192,26}

\path[draw=drawColor,line width= 0.5pt,line join=round] ( 36.70,106.50) -- ( 45.38,106.50);
\end{scope}
\begin{scope}
\path[clip] (  0.00,  0.00) rectangle (209.58,115.63);
\definecolor{drawColor}{RGB}{247,192,26}
\definecolor{fillColor}{RGB}{247,192,26}

\path[draw=drawColor,line width= 0.4pt,line join=round,line cap=round,fill=fillColor] ( 41.04,106.50) circle (  1.53);
\end{scope}
\begin{scope}
\path[clip] (  0.00,  0.00) rectangle (209.58,115.63);
\definecolor{drawColor}{RGB}{78,155,133}

\path[draw=drawColor,line width= 0.5pt,dash pattern=on 4pt off 4pt ,line join=round] ( 36.70,100.32) -- ( 45.38,100.32);
\end{scope}
\begin{scope}
\path[clip] (  0.00,  0.00) rectangle (209.58,115.63);
\definecolor{fillColor}{RGB}{78,155,133}

\path[fill=fillColor] ( 39.51, 98.79) --
	( 42.57, 98.79) --
	( 42.57,101.86) --
	( 39.51,101.86) --
	cycle;
\end{scope}
\begin{scope}
\path[clip] (  0.00,  0.00) rectangle (209.58,115.63);
\definecolor{drawColor}{RGB}{0,0,0}

\node[text=drawColor,anchor=base west,inner sep=0pt, outer sep=0pt, scale=  0.70] at ( 50.96,104.09) {LVE (ACP)};
\end{scope}
\begin{scope}
\path[clip] (  0.00,  0.00) rectangle (209.58,115.63);
\definecolor{drawColor}{RGB}{0,0,0}

\node[text=drawColor,anchor=base west,inner sep=0pt, outer sep=0pt, scale=  0.70] at ( 50.96, 97.91) {LVE (ACP $\pm \varepsilon$)};
\end{scope}
\end{tikzpicture}}
		\caption{\acs{lve} run time (ms).}
		\label{fig:commutative_approx_plot_times_k=7_eps=0.001}
	\end{subfigure}
	\begin{subfigure}[t]{0.49\linewidth}
		\centering
		\resizebox{\linewidth}{!}{
\begin{tikzpicture}[x=1pt,y=1pt]
\definecolor{fillColor}{RGB}{255,255,255}
\path[use as bounding box,fill=fillColor,fill opacity=0.00] (0,0) rectangle (209.58,115.63);
\begin{scope}
\path[clip] (  0.00,  0.00) rectangle (209.58,115.63);
\definecolor{drawColor}{RGB}{255,255,255}
\definecolor{fillColor}{RGB}{255,255,255}

\path[draw=drawColor,line width= 0.5pt,line join=round,line cap=round,fill=fillColor] (  0.00,  0.00) rectangle (209.58,115.63);
\end{scope}
\begin{scope}
\path[clip] ( 26.47, 22.31) rectangle (203.58,113.63);
\definecolor{fillColor}{RGB}{255,255,255}

\path[fill=fillColor] ( 26.47, 22.31) rectangle (203.58,113.63);
\definecolor{drawColor}{RGB}{78,155,133}
\definecolor{fillColor}{RGB}{78,155,133}

\path[draw=drawColor,draw opacity=0.20,line width= 0.4pt,line join=round,line cap=round,fill=fillColor,fill opacity=0.20] ( 43.61, 67.96) circle (  0.78);

\path[draw=drawColor,draw opacity=0.20,line width= 0.4pt,line join=round,line cap=round,fill=fillColor,fill opacity=0.20] ( 43.61, 67.99) circle (  0.78);
\definecolor{drawColor}{RGB}{78,155,133}

\path[draw=drawColor,line width= 0.6pt,line join=round] ( 43.61, 67.97) -- ( 43.61, 67.97);

\path[draw=drawColor,line width= 0.6pt,line join=round] ( 43.61, 67.97) -- ( 43.61, 67.97);

\path[draw=drawColor,line width= 0.6pt,fill=fillColor,fill opacity=0.20] ( 32.90, 67.97) --
	( 32.90, 67.97) --
	( 54.32, 67.97) --
	( 54.32, 67.97) --
	( 32.90, 67.97) --
	cycle;

\path[draw=drawColor,line width= 1.1pt] ( 32.90, 67.97) -- ( 54.32, 67.97);
\definecolor{drawColor}{RGB}{78,155,133}

\path[draw=drawColor,draw opacity=0.20,line width= 0.4pt,line join=round,line cap=round,fill=fillColor,fill opacity=0.20] ( 72.17, 67.97) circle (  0.78);
\definecolor{drawColor}{RGB}{78,155,133}

\path[draw=drawColor,line width= 0.6pt,line join=round] ( 72.17, 67.97) -- ( 72.17, 67.97);

\path[draw=drawColor,line width= 0.6pt,line join=round] ( 72.17, 67.97) -- ( 72.17, 67.97);

\path[draw=drawColor,line width= 0.6pt,fill=fillColor,fill opacity=0.20] ( 61.46, 67.97) --
	( 61.46, 67.97) --
	( 82.89, 67.97) --
	( 82.89, 67.97) --
	( 61.46, 67.97) --
	cycle;

\path[draw=drawColor,line width= 1.1pt] ( 61.46, 67.97) -- ( 82.89, 67.97);
\definecolor{drawColor}{RGB}{78,155,133}

\path[draw=drawColor,draw opacity=0.20,line width= 0.4pt,line join=round,line cap=round,fill=fillColor,fill opacity=0.20] (100.74, 67.97) circle (  0.78);

\path[draw=drawColor,draw opacity=0.20,line width= 0.4pt,line join=round,line cap=round,fill=fillColor,fill opacity=0.20] (100.74, 67.97) circle (  0.78);
\definecolor{drawColor}{RGB}{78,155,133}

\path[draw=drawColor,line width= 0.6pt,line join=round] (100.74, 67.97) -- (100.74, 67.97);

\path[draw=drawColor,line width= 0.6pt,line join=round] (100.74, 67.97) -- (100.74, 67.97);

\path[draw=drawColor,line width= 0.6pt,fill=fillColor,fill opacity=0.20] ( 90.03, 67.97) --
	( 90.03, 67.97) --
	(111.45, 67.97) --
	(111.45, 67.97) --
	( 90.03, 67.97) --
	cycle;

\path[draw=drawColor,line width= 1.1pt] ( 90.03, 67.97) -- (111.45, 67.97);
\definecolor{drawColor}{RGB}{78,155,133}

\path[draw=drawColor,draw opacity=0.20,line width= 0.4pt,line join=round,line cap=round,fill=fillColor,fill opacity=0.20] (129.31, 30.24) circle (  0.78);

\path[draw=drawColor,draw opacity=0.20,line width= 0.4pt,line join=round,line cap=round,fill=fillColor,fill opacity=0.20] (129.31, 80.25) circle (  0.78);
\definecolor{drawColor}{RGB}{78,155,133}

\path[draw=drawColor,line width= 0.6pt,line join=round] (129.31, 68.24) -- (129.31, 68.95);

\path[draw=drawColor,line width= 0.6pt,line join=round] (129.31, 66.35) -- (129.31, 65.37);

\path[draw=drawColor,line width= 0.6pt,fill=fillColor,fill opacity=0.20] (118.60, 68.24) --
	(118.60, 66.35) --
	(140.02, 66.35) --
	(140.02, 68.24) --
	(118.60, 68.24) --
	cycle;

\path[draw=drawColor,line width= 1.1pt] (118.60, 67.97) -- (140.02, 67.97);
\definecolor{drawColor}{RGB}{78,155,133}

\path[draw=drawColor,draw opacity=0.20,line width= 0.4pt,line join=round,line cap=round,fill=fillColor,fill opacity=0.20] (157.88, 34.66) circle (  0.78);

\path[draw=drawColor,draw opacity=0.20,line width= 0.4pt,line join=round,line cap=round,fill=fillColor,fill opacity=0.20] (157.88,101.65) circle (  0.78);
\definecolor{drawColor}{RGB}{78,155,133}

\path[draw=drawColor,line width= 0.6pt,line join=round] (157.88, 69.60) -- (157.88, 70.28);

\path[draw=drawColor,line width= 0.6pt,line join=round] (157.88, 64.97) -- (157.88, 60.79);

\path[draw=drawColor,line width= 0.6pt,fill=fillColor,fill opacity=0.20] (147.16, 69.60) --
	(147.16, 64.97) --
	(168.59, 64.97) --
	(168.59, 69.60) --
	(147.16, 69.60) --
	cycle;

\path[draw=drawColor,line width= 1.1pt] (147.16, 67.29) -- (168.59, 67.29);
\definecolor{drawColor}{RGB}{78,155,133}

\path[draw=drawColor,draw opacity=0.20,line width= 0.4pt,line join=round,line cap=round,fill=fillColor,fill opacity=0.20] (186.44, 35.73) circle (  0.78);
\definecolor{drawColor}{RGB}{78,155,133}

\path[draw=drawColor,line width= 0.6pt,line join=round] (186.44, 69.59) -- (186.44, 70.14);

\path[draw=drawColor,line width= 0.6pt,line join=round] (186.44, 63.87) -- (186.44, 58.10);

\path[draw=drawColor,line width= 0.6pt,fill=fillColor,fill opacity=0.20] (175.73, 69.59) --
	(175.73, 63.87) --
	(197.16, 63.87) --
	(197.16, 69.59) --
	(175.73, 69.59) --
	cycle;

\path[draw=drawColor,line width= 1.1pt] (175.73, 67.00) -- (197.16, 67.00);
\end{scope}
\begin{scope}
\path[clip] (  0.00,  0.00) rectangle (209.58,115.63);
\definecolor{drawColor}{RGB}{0,0,0}

\path[draw=drawColor,line width= 0.5pt,line join=round] ( 26.47, 22.31) --
	( 26.47,113.63);

\path[draw=drawColor,line width= 0.5pt,line join=round] ( 27.61,111.66) --
	( 26.47,113.63) --
	( 25.33,111.66);
\end{scope}
\begin{scope}
\path[clip] (  0.00,  0.00) rectangle (209.58,115.63);
\definecolor{drawColor}{gray}{0.30}

\node[text=drawColor,anchor=base east,inner sep=0pt, outer sep=0pt, scale=  0.80] at ( 22.42, 21.84) {0.6};

\node[text=drawColor,anchor=base east,inner sep=0pt, outer sep=0pt, scale=  0.80] at ( 22.42, 43.53) {0.8};

\node[text=drawColor,anchor=base east,inner sep=0pt, outer sep=0pt, scale=  0.80] at ( 22.42, 65.22) {1.0};

\node[text=drawColor,anchor=base east,inner sep=0pt, outer sep=0pt, scale=  0.80] at ( 22.42, 86.91) {1.2};

\node[text=drawColor,anchor=base east,inner sep=0pt, outer sep=0pt, scale=  0.80] at ( 22.42,108.60) {1.4};
\end{scope}
\begin{scope}
\path[clip] (  0.00,  0.00) rectangle (209.58,115.63);
\definecolor{drawColor}{gray}{0.20}

\path[draw=drawColor,line width= 0.5pt,line join=round] ( 24.22, 24.59) --
	( 26.47, 24.59);

\path[draw=drawColor,line width= 0.5pt,line join=round] ( 24.22, 46.28) --
	( 26.47, 46.28);

\path[draw=drawColor,line width= 0.5pt,line join=round] ( 24.22, 67.97) --
	( 26.47, 67.97);

\path[draw=drawColor,line width= 0.5pt,line join=round] ( 24.22, 89.66) --
	( 26.47, 89.66);

\path[draw=drawColor,line width= 0.5pt,line join=round] ( 24.22,111.35) --
	( 26.47,111.35);
\end{scope}
\begin{scope}
\path[clip] (  0.00,  0.00) rectangle (209.58,115.63);
\definecolor{drawColor}{RGB}{0,0,0}

\path[draw=drawColor,line width= 0.5pt,line join=round] ( 26.47, 22.31) --
	(203.58, 22.31);

\path[draw=drawColor,line width= 0.5pt,line join=round] (201.61, 21.17) --
	(203.58, 22.31) --
	(201.61, 23.45);
\end{scope}
\begin{scope}
\path[clip] (  0.00,  0.00) rectangle (209.58,115.63);
\definecolor{drawColor}{gray}{0.20}

\path[draw=drawColor,line width= 0.5pt,line join=round] ( 43.61, 20.06) --
	( 43.61, 22.31);

\path[draw=drawColor,line width= 0.5pt,line join=round] ( 72.17, 20.06) --
	( 72.17, 22.31);

\path[draw=drawColor,line width= 0.5pt,line join=round] (100.74, 20.06) --
	(100.74, 22.31);

\path[draw=drawColor,line width= 0.5pt,line join=round] (129.31, 20.06) --
	(129.31, 22.31);

\path[draw=drawColor,line width= 0.5pt,line join=round] (157.88, 20.06) --
	(157.88, 22.31);

\path[draw=drawColor,line width= 0.5pt,line join=round] (186.44, 20.06) --
	(186.44, 22.31);
\end{scope}
\begin{scope}
\path[clip] (  0.00,  0.00) rectangle (209.58,115.63);
\definecolor{drawColor}{gray}{0.30}

\node[text=drawColor,anchor=base,inner sep=0pt, outer sep=0pt, scale=  0.80] at ( 43.61, 12.75) {2};

\node[text=drawColor,anchor=base,inner sep=0pt, outer sep=0pt, scale=  0.80] at ( 72.17, 12.75) {4};

\node[text=drawColor,anchor=base,inner sep=0pt, outer sep=0pt, scale=  0.80] at (100.74, 12.75) {8};

\node[text=drawColor,anchor=base,inner sep=0pt, outer sep=0pt, scale=  0.80] at (129.31, 12.75) {12};

\node[text=drawColor,anchor=base,inner sep=0pt, outer sep=0pt, scale=  0.80] at (157.88, 12.75) {16};

\node[text=drawColor,anchor=base,inner sep=0pt, outer sep=0pt, scale=  0.80] at (186.44, 12.75) {20};
\end{scope}
\begin{scope}
\path[clip] (  0.00,  0.00) rectangle (209.58,115.63);
\definecolor{drawColor}{RGB}{0,0,0}

\node[text=drawColor,anchor=base,inner sep=0pt, outer sep=0pt, scale=  0.90] at (115.03,  2.75) {domain size $d$};
\end{scope}
\begin{scope}
\path[clip] (  0.00,  0.00) rectangle (209.58,115.63);
\definecolor{drawColor}{RGB}{0,0,0}

\node[text=drawColor,rotate= 90.00,anchor=base,inner sep=0pt, outer sep=0pt, scale=  0.90] at (  8.20, 67.97) {$p' \mathbin{/} p$};
\end{scope}
\end{tikzpicture}}
		\caption{Query result quotient $p' \mathbin{/} p$.}
		\label{fig:commutative_approx_plot_quots_k=7_eps=0.001}
	\end{subfigure}
	\caption{(a) Run time of \ac{lve} on the compressed model returned by \ac{acp} and its $\varepsilon$-relaxed variant (\acs{acp} $\pm \varepsilon$), and (b) distribution of the per-query quotient $p' \mathbin{/} p$ between the approximate result $p'$ and the exact result $p$ for input \acp{fg} containing $k = 7$ $\varepsilon$-commutative factors with $\varepsilon = 0.001$.}
	\label{fig:commutative_approx_plot_te_k=7_eps=0.001}
\end{figure}

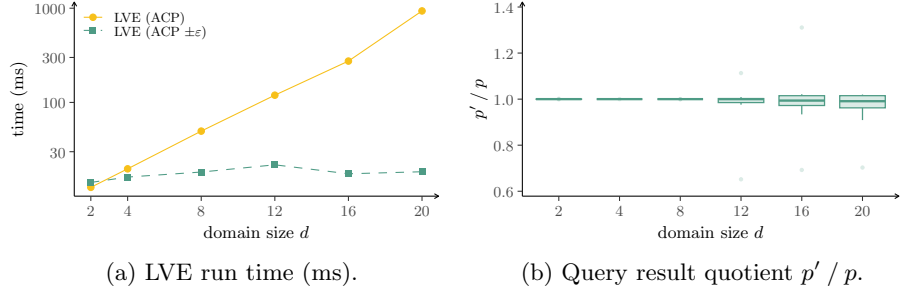
\begin{figure}[t]
	\centering
	\begin{subfigure}[t]{0.49\linewidth}
		\centering
		\resizebox{\linewidth}{!}{
\begin{tikzpicture}[x=1pt,y=1pt]
\definecolor{fillColor}{RGB}{255,255,255}
\path[use as bounding box,fill=fillColor,fill opacity=0.00] (0,0) rectangle (209.58,115.63);
\begin{scope}
\path[clip] (  0.00,  0.00) rectangle (209.58,115.63);
\definecolor{drawColor}{RGB}{255,255,255}
\definecolor{fillColor}{RGB}{255,255,255}

\path[draw=drawColor,line width= 0.5pt,line join=round,line cap=round,fill=fillColor] (  0.00,  0.00) rectangle (209.58,115.63);
\end{scope}
\begin{scope}
\path[clip] ( 32.24, 22.31) rectangle (203.58,113.63);
\definecolor{fillColor}{RGB}{255,255,255}

\path[fill=fillColor] ( 32.24, 22.31) rectangle (203.58,113.63);
\definecolor{drawColor}{RGB}{247,192,26}

\path[draw=drawColor,line width= 0.5pt,line join=round] ( 40.03, 26.46) --
	( 57.34, 35.17) --
	( 91.95, 52.92) --
	(126.57, 69.83) --
	(161.18, 85.88) --
	(195.79,109.48);
\definecolor{drawColor}{RGB}{78,155,133}

\path[draw=drawColor,line width= 0.5pt,dash pattern=on 4pt off 4pt ,line join=round] ( 40.03, 28.77) --
	( 57.34, 31.30) --
	( 91.95, 33.66) --
	(126.57, 37.06) --
	(161.18, 32.84) --
	(195.79, 33.80);
\definecolor{drawColor}{RGB}{247,192,26}
\definecolor{fillColor}{RGB}{247,192,26}

\path[draw=drawColor,line width= 0.4pt,line join=round,line cap=round,fill=fillColor] ( 40.03, 26.46) circle (  1.53);
\definecolor{fillColor}{RGB}{78,155,133}

\path[fill=fillColor] ( 38.50, 27.23) --
	( 41.57, 27.23) --
	( 41.57, 30.30) --
	( 38.50, 30.30) --
	cycle;
\definecolor{fillColor}{RGB}{247,192,26}

\path[draw=drawColor,line width= 0.4pt,line join=round,line cap=round,fill=fillColor] ( 57.34, 35.17) circle (  1.53);
\definecolor{fillColor}{RGB}{78,155,133}

\path[fill=fillColor] ( 55.81, 29.77) --
	( 58.87, 29.77) --
	( 58.87, 32.84) --
	( 55.81, 32.84) --
	cycle;
\definecolor{fillColor}{RGB}{247,192,26}

\path[draw=drawColor,line width= 0.4pt,line join=round,line cap=round,fill=fillColor] ( 91.95, 52.92) circle (  1.53);
\definecolor{fillColor}{RGB}{78,155,133}

\path[fill=fillColor] ( 90.42, 32.13) --
	( 93.49, 32.13) --
	( 93.49, 35.20) --
	( 90.42, 35.20) --
	cycle;
\definecolor{fillColor}{RGB}{247,192,26}

\path[draw=drawColor,line width= 0.4pt,line join=round,line cap=round,fill=fillColor] (126.57, 69.83) circle (  1.53);
\definecolor{fillColor}{RGB}{78,155,133}

\path[fill=fillColor] (125.03, 35.52) --
	(128.10, 35.52) --
	(128.10, 38.59) --
	(125.03, 38.59) --
	cycle;
\definecolor{fillColor}{RGB}{247,192,26}

\path[draw=drawColor,line width= 0.4pt,line join=round,line cap=round,fill=fillColor] (161.18, 85.88) circle (  1.53);
\definecolor{fillColor}{RGB}{78,155,133}

\path[fill=fillColor] (159.65, 31.31) --
	(162.71, 31.31) --
	(162.71, 34.38) --
	(159.65, 34.38) --
	cycle;
\definecolor{fillColor}{RGB}{247,192,26}

\path[draw=drawColor,line width= 0.4pt,line join=round,line cap=round,fill=fillColor] (195.79,109.48) circle (  1.53);
\definecolor{fillColor}{RGB}{78,155,133}

\path[fill=fillColor] (194.26, 32.27) --
	(197.33, 32.27) --
	(197.33, 35.34) --
	(194.26, 35.34) --
	cycle;
\end{scope}
\begin{scope}
\path[clip] (  0.00,  0.00) rectangle (209.58,115.63);
\definecolor{drawColor}{RGB}{0,0,0}

\path[draw=drawColor,line width= 0.5pt,line join=round] ( 32.24, 22.31) --
	( 32.24,113.63);

\path[draw=drawColor,line width= 0.5pt,line join=round] ( 33.38,111.66) --
	( 32.24,113.63) --
	( 31.11,111.66);
\end{scope}
\begin{scope}
\path[clip] (  0.00,  0.00) rectangle (209.58,115.63);
\definecolor{drawColor}{gray}{0.30}

\node[text=drawColor,anchor=base east,inner sep=0pt, outer sep=0pt, scale=  0.80] at ( 28.19, 40.47) {30};

\node[text=drawColor,anchor=base east,inner sep=0pt, outer sep=0pt, scale=  0.80] at ( 28.19, 63.66) {100};

\node[text=drawColor,anchor=base east,inner sep=0pt, outer sep=0pt, scale=  0.80] at ( 28.19, 84.82) {300};

\node[text=drawColor,anchor=base east,inner sep=0pt, outer sep=0pt, scale=  0.80] at ( 28.19,108.01) {1000};
\end{scope}
\begin{scope}
\path[clip] (  0.00,  0.00) rectangle (209.58,115.63);
\definecolor{drawColor}{gray}{0.20}

\path[draw=drawColor,line width= 0.5pt,line join=round] ( 29.99, 43.23) --
	( 32.24, 43.23);

\path[draw=drawColor,line width= 0.5pt,line join=round] ( 29.99, 66.42) --
	( 32.24, 66.42);

\path[draw=drawColor,line width= 0.5pt,line join=round] ( 29.99, 87.58) --
	( 32.24, 87.58);

\path[draw=drawColor,line width= 0.5pt,line join=round] ( 29.99,110.76) --
	( 32.24,110.76);
\end{scope}
\begin{scope}
\path[clip] (  0.00,  0.00) rectangle (209.58,115.63);
\definecolor{drawColor}{RGB}{0,0,0}

\path[draw=drawColor,line width= 0.5pt,line join=round] ( 32.24, 22.31) --
	(203.58, 22.31);

\path[draw=drawColor,line width= 0.5pt,line join=round] (201.61, 21.17) --
	(203.58, 22.31) --
	(201.61, 23.45);
\end{scope}
\begin{scope}
\path[clip] (  0.00,  0.00) rectangle (209.58,115.63);
\definecolor{drawColor}{gray}{0.20}

\path[draw=drawColor,line width= 0.5pt,line join=round] ( 40.03, 20.06) --
	( 40.03, 22.31);

\path[draw=drawColor,line width= 0.5pt,line join=round] ( 57.34, 20.06) --
	( 57.34, 22.31);

\path[draw=drawColor,line width= 0.5pt,line join=round] ( 91.95, 20.06) --
	( 91.95, 22.31);

\path[draw=drawColor,line width= 0.5pt,line join=round] (126.57, 20.06) --
	(126.57, 22.31);

\path[draw=drawColor,line width= 0.5pt,line join=round] (161.18, 20.06) --
	(161.18, 22.31);

\path[draw=drawColor,line width= 0.5pt,line join=round] (195.79, 20.06) --
	(195.79, 22.31);
\end{scope}
\begin{scope}
\path[clip] (  0.00,  0.00) rectangle (209.58,115.63);
\definecolor{drawColor}{gray}{0.30}

\node[text=drawColor,anchor=base,inner sep=0pt, outer sep=0pt, scale=  0.80] at ( 40.03, 12.75) {2};

\node[text=drawColor,anchor=base,inner sep=0pt, outer sep=0pt, scale=  0.80] at ( 57.34, 12.75) {4};

\node[text=drawColor,anchor=base,inner sep=0pt, outer sep=0pt, scale=  0.80] at ( 91.95, 12.75) {8};

\node[text=drawColor,anchor=base,inner sep=0pt, outer sep=0pt, scale=  0.80] at (126.57, 12.75) {12};

\node[text=drawColor,anchor=base,inner sep=0pt, outer sep=0pt, scale=  0.80] at (161.18, 12.75) {16};

\node[text=drawColor,anchor=base,inner sep=0pt, outer sep=0pt, scale=  0.80] at (195.79, 12.75) {20};
\end{scope}
\begin{scope}
\path[clip] (  0.00,  0.00) rectangle (209.58,115.63);
\definecolor{drawColor}{RGB}{0,0,0}

\node[text=drawColor,anchor=base,inner sep=0pt, outer sep=0pt, scale=  0.90] at (117.91,  2.75) {domain size $d$};
\end{scope}
\begin{scope}
\path[clip] (  0.00,  0.00) rectangle (209.58,115.63);
\definecolor{drawColor}{RGB}{0,0,0}

\node[text=drawColor,rotate= 90.00,anchor=base,inner sep=0pt, outer sep=0pt, scale=  0.90] at (  8.20, 67.97) {time (ms)};
\end{scope}
\begin{scope}
\path[clip] (  0.00,  0.00) rectangle (209.58,115.63);

\path[] ( 35.62, 97.23) rectangle (104.26,113.60);
\end{scope}
\begin{scope}
\path[clip] (  0.00,  0.00) rectangle (209.58,115.63);
\definecolor{drawColor}{RGB}{247,192,26}

\path[draw=drawColor,line width= 0.5pt,line join=round] ( 36.70,106.50) -- ( 45.38,106.50);
\end{scope}
\begin{scope}
\path[clip] (  0.00,  0.00) rectangle (209.58,115.63);
\definecolor{drawColor}{RGB}{247,192,26}
\definecolor{fillColor}{RGB}{247,192,26}

\path[draw=drawColor,line width= 0.4pt,line join=round,line cap=round,fill=fillColor] ( 41.04,106.50) circle (  1.53);
\end{scope}
\begin{scope}
\path[clip] (  0.00,  0.00) rectangle (209.58,115.63);
\definecolor{drawColor}{RGB}{78,155,133}

\path[draw=drawColor,line width= 0.5pt,dash pattern=on 4pt off 4pt ,line join=round] ( 36.70,100.32) -- ( 45.38,100.32);
\end{scope}
\begin{scope}
\path[clip] (  0.00,  0.00) rectangle (209.58,115.63);
\definecolor{fillColor}{RGB}{78,155,133}

\path[fill=fillColor] ( 39.51, 98.79) --
	( 42.57, 98.79) --
	( 42.57,101.86) --
	( 39.51,101.86) --
	cycle;
\end{scope}
\begin{scope}
\path[clip] (  0.00,  0.00) rectangle (209.58,115.63);
\definecolor{drawColor}{RGB}{0,0,0}

\node[text=drawColor,anchor=base west,inner sep=0pt, outer sep=0pt, scale=  0.70] at ( 50.96,104.09) {LVE (ACP)};
\end{scope}
\begin{scope}
\path[clip] (  0.00,  0.00) rectangle (209.58,115.63);
\definecolor{drawColor}{RGB}{0,0,0}

\node[text=drawColor,anchor=base west,inner sep=0pt, outer sep=0pt, scale=  0.70] at ( 50.96, 97.91) {LVE (ACP $\pm \varepsilon$)};
\end{scope}
\end{tikzpicture}}
		\caption{\acs{lve} run time (ms).}
		\label{fig:commutative_approx_plot_times_k=7_eps=0.01}
	\end{subfigure}
	\begin{subfigure}[t]{0.49\linewidth}
		\centering
		\resizebox{\linewidth}{!}{
\begin{tikzpicture}[x=1pt,y=1pt]
\definecolor{fillColor}{RGB}{255,255,255}
\path[use as bounding box,fill=fillColor,fill opacity=0.00] (0,0) rectangle (209.58,115.63);
\begin{scope}
\path[clip] (  0.00,  0.00) rectangle (209.58,115.63);
\definecolor{drawColor}{RGB}{255,255,255}
\definecolor{fillColor}{RGB}{255,255,255}

\path[draw=drawColor,line width= 0.5pt,line join=round,line cap=round,fill=fillColor] (  0.00,  0.00) rectangle (209.58,115.63);
\end{scope}
\begin{scope}
\path[clip] ( 26.47, 22.31) rectangle (203.58,113.63);
\definecolor{fillColor}{RGB}{255,255,255}

\path[fill=fillColor] ( 26.47, 22.31) rectangle (203.58,113.63);
\definecolor{drawColor}{RGB}{78,155,133}
\definecolor{fillColor}{RGB}{78,155,133}

\path[draw=drawColor,draw opacity=0.20,line width= 0.4pt,line join=round,line cap=round,fill=fillColor,fill opacity=0.20] ( 43.61, 67.88) circle (  0.78);

\path[draw=drawColor,draw opacity=0.20,line width= 0.4pt,line join=round,line cap=round,fill=fillColor,fill opacity=0.20] ( 43.61, 68.12) circle (  0.78);
\definecolor{drawColor}{RGB}{78,155,133}

\path[draw=drawColor,line width= 0.6pt,line join=round] ( 43.61, 67.98) -- ( 43.61, 67.98);

\path[draw=drawColor,line width= 0.6pt,line join=round] ( 43.61, 67.97) -- ( 43.61, 67.97);

\path[draw=drawColor,line width= 0.6pt,fill=fillColor,fill opacity=0.20] ( 32.90, 67.98) --
	( 32.90, 67.97) --
	( 54.32, 67.97) --
	( 54.32, 67.98) --
	( 32.90, 67.98) --
	cycle;

\path[draw=drawColor,line width= 1.1pt] ( 32.90, 67.97) -- ( 54.32, 67.97);
\definecolor{drawColor}{RGB}{78,155,133}

\path[draw=drawColor,draw opacity=0.20,line width= 0.4pt,line join=round,line cap=round,fill=fillColor,fill opacity=0.20] ( 72.17, 67.98) circle (  0.78);
\definecolor{drawColor}{RGB}{78,155,133}

\path[draw=drawColor,line width= 0.6pt,line join=round] ( 72.17, 67.97) -- ( 72.17, 67.98);

\path[draw=drawColor,line width= 0.6pt,line join=round] ( 72.17, 67.97) -- ( 72.17, 67.97);

\path[draw=drawColor,line width= 0.6pt,fill=fillColor,fill opacity=0.20] ( 61.46, 67.97) --
	( 61.46, 67.97) --
	( 82.89, 67.97) --
	( 82.89, 67.97) --
	( 61.46, 67.97) --
	cycle;

\path[draw=drawColor,line width= 1.1pt] ( 61.46, 67.97) -- ( 82.89, 67.97);
\definecolor{drawColor}{RGB}{78,155,133}

\path[draw=drawColor,draw opacity=0.20,line width= 0.4pt,line join=round,line cap=round,fill=fillColor,fill opacity=0.20] (100.74, 67.98) circle (  0.78);

\path[draw=drawColor,draw opacity=0.20,line width= 0.4pt,line join=round,line cap=round,fill=fillColor,fill opacity=0.20] (100.74, 67.95) circle (  0.78);
\definecolor{drawColor}{RGB}{78,155,133}

\path[draw=drawColor,line width= 0.6pt,line join=round] (100.74, 67.97) -- (100.74, 67.97);

\path[draw=drawColor,line width= 0.6pt,line join=round] (100.74, 67.97) -- (100.74, 67.97);

\path[draw=drawColor,line width= 0.6pt,fill=fillColor,fill opacity=0.20] ( 90.03, 67.97) --
	( 90.03, 67.97) --
	(111.45, 67.97) --
	(111.45, 67.97) --
	( 90.03, 67.97) --
	cycle;

\path[draw=drawColor,line width= 1.1pt] ( 90.03, 67.97) -- (111.45, 67.97);
\definecolor{drawColor}{RGB}{78,155,133}

\path[draw=drawColor,draw opacity=0.20,line width= 0.4pt,line join=round,line cap=round,fill=fillColor,fill opacity=0.20] (129.31, 30.24) circle (  0.78);

\path[draw=drawColor,draw opacity=0.20,line width= 0.4pt,line join=round,line cap=round,fill=fillColor,fill opacity=0.20] (129.31, 80.23) circle (  0.78);
\definecolor{drawColor}{RGB}{78,155,133}

\path[draw=drawColor,line width= 0.6pt,line join=round] (129.31, 68.24) -- (129.31, 68.95);

\path[draw=drawColor,line width= 0.6pt,line join=round] (129.31, 66.35) -- (129.31, 65.37);

\path[draw=drawColor,line width= 0.6pt,fill=fillColor,fill opacity=0.20] (118.60, 68.24) --
	(118.60, 66.35) --
	(140.02, 66.35) --
	(140.02, 68.24) --
	(118.60, 68.24) --
	cycle;

\path[draw=drawColor,line width= 1.1pt] (118.60, 67.97) -- (140.02, 67.97);
\definecolor{drawColor}{RGB}{78,155,133}

\path[draw=drawColor,draw opacity=0.20,line width= 0.4pt,line join=round,line cap=round,fill=fillColor,fill opacity=0.20] (157.88, 34.64) circle (  0.78);

\path[draw=drawColor,draw opacity=0.20,line width= 0.4pt,line join=round,line cap=round,fill=fillColor,fill opacity=0.20] (157.88,101.67) circle (  0.78);
\definecolor{drawColor}{RGB}{78,155,133}

\path[draw=drawColor,line width= 0.6pt,line join=round] (157.88, 69.59) -- (157.88, 70.28);

\path[draw=drawColor,line width= 0.6pt,line join=round] (157.88, 64.97) -- (157.88, 60.79);

\path[draw=drawColor,line width= 0.6pt,fill=fillColor,fill opacity=0.20] (147.16, 69.59) --
	(147.16, 64.97) --
	(168.59, 64.97) --
	(168.59, 69.59) --
	(147.16, 69.59) --
	cycle;

\path[draw=drawColor,line width= 1.1pt] (147.16, 67.29) -- (168.59, 67.29);
\definecolor{drawColor}{RGB}{78,155,133}

\path[draw=drawColor,draw opacity=0.20,line width= 0.4pt,line join=round,line cap=round,fill=fillColor,fill opacity=0.20] (186.44, 35.78) circle (  0.78);
\definecolor{drawColor}{RGB}{78,155,133}

\path[draw=drawColor,line width= 0.6pt,line join=round] (186.44, 69.60) -- (186.44, 70.13);

\path[draw=drawColor,line width= 0.6pt,line join=round] (186.44, 63.87) -- (186.44, 58.12);

\path[draw=drawColor,line width= 0.6pt,fill=fillColor,fill opacity=0.20] (175.73, 69.60) --
	(175.73, 63.87) --
	(197.16, 63.87) --
	(197.16, 69.60) --
	(175.73, 69.60) --
	cycle;

\path[draw=drawColor,line width= 1.1pt] (175.73, 67.00) -- (197.16, 67.00);
\end{scope}
\begin{scope}
\path[clip] (  0.00,  0.00) rectangle (209.58,115.63);
\definecolor{drawColor}{RGB}{0,0,0}

\path[draw=drawColor,line width= 0.5pt,line join=round] ( 26.47, 22.31) --
	( 26.47,113.63);

\path[draw=drawColor,line width= 0.5pt,line join=round] ( 27.61,111.66) --
	( 26.47,113.63) --
	( 25.33,111.66);
\end{scope}
\begin{scope}
\path[clip] (  0.00,  0.00) rectangle (209.58,115.63);
\definecolor{drawColor}{gray}{0.30}

\node[text=drawColor,anchor=base east,inner sep=0pt, outer sep=0pt, scale=  0.80] at ( 22.42, 21.87) {0.6};

\node[text=drawColor,anchor=base east,inner sep=0pt, outer sep=0pt, scale=  0.80] at ( 22.42, 43.54) {0.8};

\node[text=drawColor,anchor=base east,inner sep=0pt, outer sep=0pt, scale=  0.80] at ( 22.42, 65.22) {1.0};

\node[text=drawColor,anchor=base east,inner sep=0pt, outer sep=0pt, scale=  0.80] at ( 22.42, 86.89) {1.2};

\node[text=drawColor,anchor=base east,inner sep=0pt, outer sep=0pt, scale=  0.80] at ( 22.42,108.57) {1.4};
\end{scope}
\begin{scope}
\path[clip] (  0.00,  0.00) rectangle (209.58,115.63);
\definecolor{drawColor}{gray}{0.20}

\path[draw=drawColor,line width= 0.5pt,line join=round] ( 24.22, 24.62) --
	( 26.47, 24.62);

\path[draw=drawColor,line width= 0.5pt,line join=round] ( 24.22, 46.30) --
	( 26.47, 46.30);

\path[draw=drawColor,line width= 0.5pt,line join=round] ( 24.22, 67.97) --
	( 26.47, 67.97);

\path[draw=drawColor,line width= 0.5pt,line join=round] ( 24.22, 89.65) --
	( 26.47, 89.65);

\path[draw=drawColor,line width= 0.5pt,line join=round] ( 24.22,111.32) --
	( 26.47,111.32);
\end{scope}
\begin{scope}
\path[clip] (  0.00,  0.00) rectangle (209.58,115.63);
\definecolor{drawColor}{RGB}{0,0,0}

\path[draw=drawColor,line width= 0.5pt,line join=round] ( 26.47, 22.31) --
	(203.58, 22.31);

\path[draw=drawColor,line width= 0.5pt,line join=round] (201.61, 21.17) --
	(203.58, 22.31) --
	(201.61, 23.45);
\end{scope}
\begin{scope}
\path[clip] (  0.00,  0.00) rectangle (209.58,115.63);
\definecolor{drawColor}{gray}{0.20}

\path[draw=drawColor,line width= 0.5pt,line join=round] ( 43.61, 20.06) --
	( 43.61, 22.31);

\path[draw=drawColor,line width= 0.5pt,line join=round] ( 72.17, 20.06) --
	( 72.17, 22.31);

\path[draw=drawColor,line width= 0.5pt,line join=round] (100.74, 20.06) --
	(100.74, 22.31);

\path[draw=drawColor,line width= 0.5pt,line join=round] (129.31, 20.06) --
	(129.31, 22.31);

\path[draw=drawColor,line width= 0.5pt,line join=round] (157.88, 20.06) --
	(157.88, 22.31);

\path[draw=drawColor,line width= 0.5pt,line join=round] (186.44, 20.06) --
	(186.44, 22.31);
\end{scope}
\begin{scope}
\path[clip] (  0.00,  0.00) rectangle (209.58,115.63);
\definecolor{drawColor}{gray}{0.30}

\node[text=drawColor,anchor=base,inner sep=0pt, outer sep=0pt, scale=  0.80] at ( 43.61, 12.75) {2};

\node[text=drawColor,anchor=base,inner sep=0pt, outer sep=0pt, scale=  0.80] at ( 72.17, 12.75) {4};

\node[text=drawColor,anchor=base,inner sep=0pt, outer sep=0pt, scale=  0.80] at (100.74, 12.75) {8};

\node[text=drawColor,anchor=base,inner sep=0pt, outer sep=0pt, scale=  0.80] at (129.31, 12.75) {12};

\node[text=drawColor,anchor=base,inner sep=0pt, outer sep=0pt, scale=  0.80] at (157.88, 12.75) {16};

\node[text=drawColor,anchor=base,inner sep=0pt, outer sep=0pt, scale=  0.80] at (186.44, 12.75) {20};
\end{scope}
\begin{scope}
\path[clip] (  0.00,  0.00) rectangle (209.58,115.63);
\definecolor{drawColor}{RGB}{0,0,0}

\node[text=drawColor,anchor=base,inner sep=0pt, outer sep=0pt, scale=  0.90] at (115.03,  2.75) {domain size $d$};
\end{scope}
\begin{scope}
\path[clip] (  0.00,  0.00) rectangle (209.58,115.63);
\definecolor{drawColor}{RGB}{0,0,0}

\node[text=drawColor,rotate= 90.00,anchor=base,inner sep=0pt, outer sep=0pt, scale=  0.90] at (  8.20, 67.97) {$p' \mathbin{/} p$};
\end{scope}
\end{tikzpicture}}
		\caption{Query result quotient $p' \mathbin{/} p$.}
		\label{fig:commutative_approx_plot_quots_k=7_eps=0.01}
	\end{subfigure}
	\caption{(a) Run time of \ac{lve} on the compressed model returned by \ac{acp} and its $\varepsilon$-relaxed variant (\acs{acp} $\pm \varepsilon$), and (b) distribution of the per-query quotient $p' \mathbin{/} p$ between the approximate result $p'$ and the exact result $p$ for input \acp{fg} containing $k = 7$ $\varepsilon$-commutative factors with $\varepsilon = 0.01$.}
	\label{fig:commutative_approx_plot_te_k=7_eps=0.01}
\end{figure}

\begin{figure}[t]
	\centering
	\begin{subfigure}[t]{0.49\linewidth}
		\centering
		\resizebox{\linewidth}{!}{
\begin{tikzpicture}[x=1pt,y=1pt]
\definecolor{fillColor}{RGB}{255,255,255}
\path[use as bounding box,fill=fillColor,fill opacity=0.00] (0,0) rectangle (209.58,115.63);
\begin{scope}
\path[clip] (  0.00,  0.00) rectangle (209.58,115.63);
\definecolor{drawColor}{RGB}{255,255,255}
\definecolor{fillColor}{RGB}{255,255,255}

\path[draw=drawColor,line width= 0.5pt,line join=round,line cap=round,fill=fillColor] (  0.00,  0.00) rectangle (209.58,115.63);
\end{scope}
\begin{scope}
\path[clip] ( 32.24, 22.31) rectangle (203.58,113.63);
\definecolor{fillColor}{RGB}{255,255,255}

\path[fill=fillColor] ( 32.24, 22.31) rectangle (203.58,113.63);
\definecolor{drawColor}{RGB}{247,192,26}

\path[draw=drawColor,line width= 0.5pt,line join=round] ( 40.03, 26.46) --
	( 57.34, 35.02) --
	( 91.95, 49.94) --
	(126.57, 68.95) --
	(161.18, 81.29) --
	(195.79,109.48);
\definecolor{drawColor}{RGB}{78,155,133}

\path[draw=drawColor,line width= 0.5pt,dash pattern=on 4pt off 4pt ,line join=round] ( 40.03, 29.81) --
	( 57.34, 31.78) --
	( 91.95, 33.09) --
	(126.57, 33.79) --
	(161.18, 33.25) --
	(195.79, 35.68);
\definecolor{drawColor}{RGB}{247,192,26}
\definecolor{fillColor}{RGB}{247,192,26}

\path[draw=drawColor,line width= 0.4pt,line join=round,line cap=round,fill=fillColor] ( 40.03, 26.46) circle (  1.53);
\definecolor{fillColor}{RGB}{78,155,133}

\path[fill=fillColor] ( 38.50, 28.28) --
	( 41.57, 28.28) --
	( 41.57, 31.35) --
	( 38.50, 31.35) --
	cycle;
\definecolor{fillColor}{RGB}{247,192,26}

\path[draw=drawColor,line width= 0.4pt,line join=round,line cap=round,fill=fillColor] ( 57.34, 35.02) circle (  1.53);
\definecolor{fillColor}{RGB}{78,155,133}

\path[fill=fillColor] ( 55.81, 30.25) --
	( 58.87, 30.25) --
	( 58.87, 33.32) --
	( 55.81, 33.32) --
	cycle;
\definecolor{fillColor}{RGB}{247,192,26}

\path[draw=drawColor,line width= 0.4pt,line join=round,line cap=round,fill=fillColor] ( 91.95, 49.94) circle (  1.53);
\definecolor{fillColor}{RGB}{78,155,133}

\path[fill=fillColor] ( 90.42, 31.56) --
	( 93.49, 31.56) --
	( 93.49, 34.63) --
	( 90.42, 34.63) --
	cycle;
\definecolor{fillColor}{RGB}{247,192,26}

\path[draw=drawColor,line width= 0.4pt,line join=round,line cap=round,fill=fillColor] (126.57, 68.95) circle (  1.53);
\definecolor{fillColor}{RGB}{78,155,133}

\path[fill=fillColor] (125.03, 32.26) --
	(128.10, 32.26) --
	(128.10, 35.33) --
	(125.03, 35.33) --
	cycle;
\definecolor{fillColor}{RGB}{247,192,26}

\path[draw=drawColor,line width= 0.4pt,line join=round,line cap=round,fill=fillColor] (161.18, 81.29) circle (  1.53);
\definecolor{fillColor}{RGB}{78,155,133}

\path[fill=fillColor] (159.65, 31.71) --
	(162.71, 31.71) --
	(162.71, 34.78) --
	(159.65, 34.78) --
	cycle;
\definecolor{fillColor}{RGB}{247,192,26}

\path[draw=drawColor,line width= 0.4pt,line join=round,line cap=round,fill=fillColor] (195.79,109.48) circle (  1.53);
\definecolor{fillColor}{RGB}{78,155,133}

\path[fill=fillColor] (194.26, 34.15) --
	(197.33, 34.15) --
	(197.33, 37.22) --
	(194.26, 37.22) --
	cycle;
\end{scope}
\begin{scope}
\path[clip] (  0.00,  0.00) rectangle (209.58,115.63);
\definecolor{drawColor}{RGB}{0,0,0}

\path[draw=drawColor,line width= 0.5pt,line join=round] ( 32.24, 22.31) --
	( 32.24,113.63);

\path[draw=drawColor,line width= 0.5pt,line join=round] ( 33.38,111.66) --
	( 32.24,113.63) --
	( 31.11,111.66);
\end{scope}
\begin{scope}
\path[clip] (  0.00,  0.00) rectangle (209.58,115.63);
\definecolor{drawColor}{gray}{0.30}

\node[text=drawColor,anchor=base east,inner sep=0pt, outer sep=0pt, scale=  0.80] at ( 28.19, 19.77) {10};

\node[text=drawColor,anchor=base east,inner sep=0pt, outer sep=0pt, scale=  0.80] at ( 28.19, 64.40) {100};

\node[text=drawColor,anchor=base east,inner sep=0pt, outer sep=0pt, scale=  0.80] at ( 28.19,109.02) {1000};
\end{scope}
\begin{scope}
\path[clip] (  0.00,  0.00) rectangle (209.58,115.63);
\definecolor{drawColor}{gray}{0.20}

\path[draw=drawColor,line width= 0.5pt,line join=round] ( 29.99, 22.53) --
	( 32.24, 22.53);

\path[draw=drawColor,line width= 0.5pt,line join=round] ( 29.99, 67.15) --
	( 32.24, 67.15);

\path[draw=drawColor,line width= 0.5pt,line join=round] ( 29.99,111.77) --
	( 32.24,111.77);
\end{scope}
\begin{scope}
\path[clip] (  0.00,  0.00) rectangle (209.58,115.63);
\definecolor{drawColor}{RGB}{0,0,0}

\path[draw=drawColor,line width= 0.5pt,line join=round] ( 32.24, 22.31) --
	(203.58, 22.31);

\path[draw=drawColor,line width= 0.5pt,line join=round] (201.61, 21.17) --
	(203.58, 22.31) --
	(201.61, 23.45);
\end{scope}
\begin{scope}
\path[clip] (  0.00,  0.00) rectangle (209.58,115.63);
\definecolor{drawColor}{gray}{0.20}

\path[draw=drawColor,line width= 0.5pt,line join=round] ( 40.03, 20.06) --
	( 40.03, 22.31);

\path[draw=drawColor,line width= 0.5pt,line join=round] ( 57.34, 20.06) --
	( 57.34, 22.31);

\path[draw=drawColor,line width= 0.5pt,line join=round] ( 91.95, 20.06) --
	( 91.95, 22.31);

\path[draw=drawColor,line width= 0.5pt,line join=round] (126.57, 20.06) --
	(126.57, 22.31);

\path[draw=drawColor,line width= 0.5pt,line join=round] (161.18, 20.06) --
	(161.18, 22.31);

\path[draw=drawColor,line width= 0.5pt,line join=round] (195.79, 20.06) --
	(195.79, 22.31);
\end{scope}
\begin{scope}
\path[clip] (  0.00,  0.00) rectangle (209.58,115.63);
\definecolor{drawColor}{gray}{0.30}

\node[text=drawColor,anchor=base,inner sep=0pt, outer sep=0pt, scale=  0.80] at ( 40.03, 12.75) {2};

\node[text=drawColor,anchor=base,inner sep=0pt, outer sep=0pt, scale=  0.80] at ( 57.34, 12.75) {4};

\node[text=drawColor,anchor=base,inner sep=0pt, outer sep=0pt, scale=  0.80] at ( 91.95, 12.75) {8};

\node[text=drawColor,anchor=base,inner sep=0pt, outer sep=0pt, scale=  0.80] at (126.57, 12.75) {12};

\node[text=drawColor,anchor=base,inner sep=0pt, outer sep=0pt, scale=  0.80] at (161.18, 12.75) {16};

\node[text=drawColor,anchor=base,inner sep=0pt, outer sep=0pt, scale=  0.80] at (195.79, 12.75) {20};
\end{scope}
\begin{scope}
\path[clip] (  0.00,  0.00) rectangle (209.58,115.63);
\definecolor{drawColor}{RGB}{0,0,0}

\node[text=drawColor,anchor=base,inner sep=0pt, outer sep=0pt, scale=  0.90] at (117.91,  2.75) {domain size $d$};
\end{scope}
\begin{scope}
\path[clip] (  0.00,  0.00) rectangle (209.58,115.63);
\definecolor{drawColor}{RGB}{0,0,0}

\node[text=drawColor,rotate= 90.00,anchor=base,inner sep=0pt, outer sep=0pt, scale=  0.90] at (  8.20, 67.97) {time (ms)};
\end{scope}
\begin{scope}
\path[clip] (  0.00,  0.00) rectangle (209.58,115.63);

\path[] ( 35.62, 97.23) rectangle (104.26,113.60);
\end{scope}
\begin{scope}
\path[clip] (  0.00,  0.00) rectangle (209.58,115.63);
\definecolor{drawColor}{RGB}{247,192,26}

\path[draw=drawColor,line width= 0.5pt,line join=round] ( 36.70,106.50) -- ( 45.38,106.50);
\end{scope}
\begin{scope}
\path[clip] (  0.00,  0.00) rectangle (209.58,115.63);
\definecolor{drawColor}{RGB}{247,192,26}
\definecolor{fillColor}{RGB}{247,192,26}

\path[draw=drawColor,line width= 0.4pt,line join=round,line cap=round,fill=fillColor] ( 41.04,106.50) circle (  1.53);
\end{scope}
\begin{scope}
\path[clip] (  0.00,  0.00) rectangle (209.58,115.63);
\definecolor{drawColor}{RGB}{78,155,133}

\path[draw=drawColor,line width= 0.5pt,dash pattern=on 4pt off 4pt ,line join=round] ( 36.70,100.32) -- ( 45.38,100.32);
\end{scope}
\begin{scope}
\path[clip] (  0.00,  0.00) rectangle (209.58,115.63);
\definecolor{fillColor}{RGB}{78,155,133}

\path[fill=fillColor] ( 39.51, 98.79) --
	( 42.57, 98.79) --
	( 42.57,101.86) --
	( 39.51,101.86) --
	cycle;
\end{scope}
\begin{scope}
\path[clip] (  0.00,  0.00) rectangle (209.58,115.63);
\definecolor{drawColor}{RGB}{0,0,0}

\node[text=drawColor,anchor=base west,inner sep=0pt, outer sep=0pt, scale=  0.70] at ( 50.96,104.09) {LVE (ACP)};
\end{scope}
\begin{scope}
\path[clip] (  0.00,  0.00) rectangle (209.58,115.63);
\definecolor{drawColor}{RGB}{0,0,0}

\node[text=drawColor,anchor=base west,inner sep=0pt, outer sep=0pt, scale=  0.70] at ( 50.96, 97.91) {LVE (ACP $\pm \varepsilon$)};
\end{scope}
\end{tikzpicture}}
		\caption{\acs{lve} run time (ms).}
		\label{fig:commutative_approx_plot_times_k=7_eps=0.1}
	\end{subfigure}
	\begin{subfigure}[t]{0.49\linewidth}
		\centering
		\resizebox{\linewidth}{!}{
\begin{tikzpicture}[x=1pt,y=1pt]
\definecolor{fillColor}{RGB}{255,255,255}
\path[use as bounding box,fill=fillColor,fill opacity=0.00] (0,0) rectangle (209.58,115.63);
\begin{scope}
\path[clip] (  0.00,  0.00) rectangle (209.58,115.63);
\definecolor{drawColor}{RGB}{255,255,255}
\definecolor{fillColor}{RGB}{255,255,255}

\path[draw=drawColor,line width= 0.5pt,line join=round,line cap=round,fill=fillColor] (  0.00,  0.00) rectangle (209.58,115.63);
\end{scope}
\begin{scope}
\path[clip] ( 26.47, 22.31) rectangle (203.58,113.63);
\definecolor{fillColor}{RGB}{255,255,255}

\path[fill=fillColor] ( 26.47, 22.31) rectangle (203.58,113.63);
\definecolor{drawColor}{RGB}{78,155,133}
\definecolor{fillColor}{RGB}{78,155,133}

\path[draw=drawColor,draw opacity=0.20,line width= 0.4pt,line join=round,line cap=round,fill=fillColor,fill opacity=0.20] ( 43.61, 67.09) circle (  0.78);

\path[draw=drawColor,draw opacity=0.20,line width= 0.4pt,line join=round,line cap=round,fill=fillColor,fill opacity=0.20] ( 43.61, 69.34) circle (  0.78);
\definecolor{drawColor}{RGB}{78,155,133}

\path[draw=drawColor,line width= 0.6pt,line join=round] ( 43.61, 68.03) -- ( 43.61, 68.03);

\path[draw=drawColor,line width= 0.6pt,line join=round] ( 43.61, 67.93) -- ( 43.61, 67.92);

\path[draw=drawColor,line width= 0.6pt,fill=fillColor,fill opacity=0.20] ( 32.90, 68.03) --
	( 32.90, 67.93) --
	( 54.32, 67.93) --
	( 54.32, 68.03) --
	( 32.90, 68.03) --
	cycle;

\path[draw=drawColor,line width= 1.1pt] ( 32.90, 67.97) -- ( 54.32, 67.97);
\definecolor{drawColor}{RGB}{78,155,133}

\path[draw=drawColor,draw opacity=0.20,line width= 0.4pt,line join=round,line cap=round,fill=fillColor,fill opacity=0.20] ( 72.17, 68.05) circle (  0.78);
\definecolor{drawColor}{RGB}{78,155,133}

\path[draw=drawColor,line width= 0.6pt,line join=round] ( 72.17, 67.99) -- ( 72.17, 68.02);

\path[draw=drawColor,line width= 0.6pt,line join=round] ( 72.17, 67.96) -- ( 72.17, 67.94);

\path[draw=drawColor,line width= 0.6pt,fill=fillColor,fill opacity=0.20] ( 61.46, 67.99) --
	( 61.46, 67.96) --
	( 82.89, 67.96) --
	( 82.89, 67.99) --
	( 61.46, 67.99) --
	cycle;

\path[draw=drawColor,line width= 1.1pt] ( 61.46, 67.97) -- ( 82.89, 67.97);
\definecolor{drawColor}{RGB}{78,155,133}

\path[draw=drawColor,draw opacity=0.20,line width= 0.4pt,line join=round,line cap=round,fill=fillColor,fill opacity=0.20] (100.74, 68.02) circle (  0.78);

\path[draw=drawColor,draw opacity=0.20,line width= 0.4pt,line join=round,line cap=round,fill=fillColor,fill opacity=0.20] (100.74, 67.77) circle (  0.78);
\definecolor{drawColor}{RGB}{78,155,133}

\path[draw=drawColor,line width= 0.6pt,line join=round] (100.74, 67.97) -- (100.74, 67.97);

\path[draw=drawColor,line width= 0.6pt,line join=round] (100.74, 67.97) -- (100.74, 67.97);

\path[draw=drawColor,line width= 0.6pt,fill=fillColor,fill opacity=0.20] ( 90.03, 67.97) --
	( 90.03, 67.97) --
	(111.45, 67.97) --
	(111.45, 67.97) --
	( 90.03, 67.97) --
	cycle;

\path[draw=drawColor,line width= 1.1pt] ( 90.03, 67.97) -- (111.45, 67.97);
\definecolor{drawColor}{RGB}{78,155,133}

\path[draw=drawColor,draw opacity=0.20,line width= 0.4pt,line join=round,line cap=round,fill=fillColor,fill opacity=0.20] (129.31, 30.24) circle (  0.78);

\path[draw=drawColor,draw opacity=0.20,line width= 0.4pt,line join=round,line cap=round,fill=fillColor,fill opacity=0.20] (129.31, 80.06) circle (  0.78);
\definecolor{drawColor}{RGB}{78,155,133}

\path[draw=drawColor,line width= 0.6pt,line join=round] (129.31, 68.24) -- (129.31, 68.94);

\path[draw=drawColor,line width= 0.6pt,line join=round] (129.31, 66.37) -- (129.31, 65.41);

\path[draw=drawColor,line width= 0.6pt,fill=fillColor,fill opacity=0.20] (118.60, 68.24) --
	(118.60, 66.37) --
	(140.02, 66.37) --
	(140.02, 68.24) --
	(118.60, 68.24) --
	cycle;

\path[draw=drawColor,line width= 1.1pt] (118.60, 67.97) -- (140.02, 67.97);
\definecolor{drawColor}{RGB}{78,155,133}

\path[draw=drawColor,draw opacity=0.20,line width= 0.4pt,line join=round,line cap=round,fill=fillColor,fill opacity=0.20] (157.88, 34.38) circle (  0.78);

\path[draw=drawColor,draw opacity=0.20,line width= 0.4pt,line join=round,line cap=round,fill=fillColor,fill opacity=0.20] (157.88,101.85) circle (  0.78);
\definecolor{drawColor}{RGB}{78,155,133}

\path[draw=drawColor,line width= 0.6pt,line join=round] (157.88, 69.55) -- (157.88, 70.24);

\path[draw=drawColor,line width= 0.6pt,line join=round] (157.88, 64.98) -- (157.88, 60.77);

\path[draw=drawColor,line width= 0.6pt,fill=fillColor,fill opacity=0.20] (147.16, 69.55) --
	(147.16, 64.98) --
	(168.59, 64.98) --
	(168.59, 69.55) --
	(147.16, 69.55) --
	cycle;

\path[draw=drawColor,line width= 1.1pt] (147.16, 67.31) -- (168.59, 67.31);
\definecolor{drawColor}{RGB}{78,155,133}

\path[draw=drawColor,draw opacity=0.20,line width= 0.4pt,line join=round,line cap=round,fill=fillColor,fill opacity=0.20] (186.44, 36.20) circle (  0.78);
\definecolor{drawColor}{RGB}{78,155,133}

\path[draw=drawColor,line width= 0.6pt,line join=round] (186.44, 69.62) -- (186.44, 70.08);

\path[draw=drawColor,line width= 0.6pt,line join=round] (186.44, 63.90) -- (186.44, 58.32);

\path[draw=drawColor,line width= 0.6pt,fill=fillColor,fill opacity=0.20] (175.73, 69.62) --
	(175.73, 63.90) --
	(197.16, 63.90) --
	(197.16, 69.62) --
	(175.73, 69.62) --
	cycle;

\path[draw=drawColor,line width= 1.1pt] (175.73, 66.99) -- (197.16, 66.99);
\end{scope}
\begin{scope}
\path[clip] (  0.00,  0.00) rectangle (209.58,115.63);
\definecolor{drawColor}{RGB}{0,0,0}

\path[draw=drawColor,line width= 0.5pt,line join=round] ( 26.47, 22.31) --
	( 26.47,113.63);

\path[draw=drawColor,line width= 0.5pt,line join=round] ( 27.61,111.66) --
	( 26.47,113.63) --
	( 25.33,111.66);
\end{scope}
\begin{scope}
\path[clip] (  0.00,  0.00) rectangle (209.58,115.63);
\definecolor{drawColor}{gray}{0.30}

\node[text=drawColor,anchor=base east,inner sep=0pt, outer sep=0pt, scale=  0.80] at ( 22.42, 22.11) {0.6};

\node[text=drawColor,anchor=base east,inner sep=0pt, outer sep=0pt, scale=  0.80] at ( 22.42, 43.67) {0.8};

\node[text=drawColor,anchor=base east,inner sep=0pt, outer sep=0pt, scale=  0.80] at ( 22.42, 65.22) {1.0};

\node[text=drawColor,anchor=base east,inner sep=0pt, outer sep=0pt, scale=  0.80] at ( 22.42, 86.77) {1.2};

\node[text=drawColor,anchor=base east,inner sep=0pt, outer sep=0pt, scale=  0.80] at ( 22.42,108.32) {1.4};
\end{scope}
\begin{scope}
\path[clip] (  0.00,  0.00) rectangle (209.58,115.63);
\definecolor{drawColor}{gray}{0.20}

\path[draw=drawColor,line width= 0.5pt,line join=round] ( 24.22, 24.87) --
	( 26.47, 24.87);

\path[draw=drawColor,line width= 0.5pt,line join=round] ( 24.22, 46.42) --
	( 26.47, 46.42);

\path[draw=drawColor,line width= 0.5pt,line join=round] ( 24.22, 67.97) --
	( 26.47, 67.97);

\path[draw=drawColor,line width= 0.5pt,line join=round] ( 24.22, 89.52) --
	( 26.47, 89.52);

\path[draw=drawColor,line width= 0.5pt,line join=round] ( 24.22,111.08) --
	( 26.47,111.08);
\end{scope}
\begin{scope}
\path[clip] (  0.00,  0.00) rectangle (209.58,115.63);
\definecolor{drawColor}{RGB}{0,0,0}

\path[draw=drawColor,line width= 0.5pt,line join=round] ( 26.47, 22.31) --
	(203.58, 22.31);

\path[draw=drawColor,line width= 0.5pt,line join=round] (201.61, 21.17) --
	(203.58, 22.31) --
	(201.61, 23.45);
\end{scope}
\begin{scope}
\path[clip] (  0.00,  0.00) rectangle (209.58,115.63);
\definecolor{drawColor}{gray}{0.20}

\path[draw=drawColor,line width= 0.5pt,line join=round] ( 43.61, 20.06) --
	( 43.61, 22.31);

\path[draw=drawColor,line width= 0.5pt,line join=round] ( 72.17, 20.06) --
	( 72.17, 22.31);

\path[draw=drawColor,line width= 0.5pt,line join=round] (100.74, 20.06) --
	(100.74, 22.31);

\path[draw=drawColor,line width= 0.5pt,line join=round] (129.31, 20.06) --
	(129.31, 22.31);

\path[draw=drawColor,line width= 0.5pt,line join=round] (157.88, 20.06) --
	(157.88, 22.31);

\path[draw=drawColor,line width= 0.5pt,line join=round] (186.44, 20.06) --
	(186.44, 22.31);
\end{scope}
\begin{scope}
\path[clip] (  0.00,  0.00) rectangle (209.58,115.63);
\definecolor{drawColor}{gray}{0.30}

\node[text=drawColor,anchor=base,inner sep=0pt, outer sep=0pt, scale=  0.80] at ( 43.61, 12.75) {2};

\node[text=drawColor,anchor=base,inner sep=0pt, outer sep=0pt, scale=  0.80] at ( 72.17, 12.75) {4};

\node[text=drawColor,anchor=base,inner sep=0pt, outer sep=0pt, scale=  0.80] at (100.74, 12.75) {8};

\node[text=drawColor,anchor=base,inner sep=0pt, outer sep=0pt, scale=  0.80] at (129.31, 12.75) {12};

\node[text=drawColor,anchor=base,inner sep=0pt, outer sep=0pt, scale=  0.80] at (157.88, 12.75) {16};

\node[text=drawColor,anchor=base,inner sep=0pt, outer sep=0pt, scale=  0.80] at (186.44, 12.75) {20};
\end{scope}
\begin{scope}
\path[clip] (  0.00,  0.00) rectangle (209.58,115.63);
\definecolor{drawColor}{RGB}{0,0,0}

\node[text=drawColor,anchor=base,inner sep=0pt, outer sep=0pt, scale=  0.90] at (115.03,  2.75) {domain size $d$};
\end{scope}
\begin{scope}
\path[clip] (  0.00,  0.00) rectangle (209.58,115.63);
\definecolor{drawColor}{RGB}{0,0,0}

\node[text=drawColor,rotate= 90.00,anchor=base,inner sep=0pt, outer sep=0pt, scale=  0.90] at (  8.20, 67.97) {$p' \mathbin{/} p$};
\end{scope}
\end{tikzpicture}}
		\caption{Query result quotient $p' \mathbin{/} p$.}
		\label{fig:commutative_approx_plot_quots_k=7_eps=0.1}
	\end{subfigure}
	\caption{(a) Run time of \ac{lve} on the compressed model returned by \ac{acp} and its $\varepsilon$-relaxed variant (\acs{acp} $\pm \varepsilon$), and (b) distribution of the per-query quotient $p' \mathbin{/} p$ between the approximate result $p'$ and the exact result $p$ for input \acp{fg} containing $k = 7$ $\varepsilon$-commutative factors with $\varepsilon = 0.1$.}
	\label{fig:commutative_approx_plot_te_k=7_eps=0.1}
\end{figure}
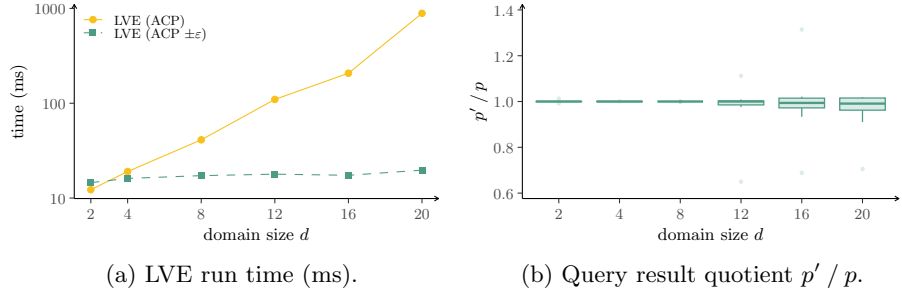

The averaged results in \cref{fig:commutative_approx_plot_avg} aggregate over the number of $\varepsilon$-commutative factors~$k$ and the tolerance~$\varepsilon$.
To disentangle the influence of these two parameters, \crefrange{fig:commutative_approx_plot_te_k=1_eps=0.001}{fig:commutative_approx_plot_te_k=7_eps=0.1} report the \ac{lve} run time and the per-query quotient $p' \mathbin{/} p$ separately for each combination of $k \in \{1, 3, 7\}$ and $\varepsilon \in \{0.001, 0.01, 0.1\}$.
The run time behaviour is qualitatively consistent across all configurations: The run time of \ac{lve} on the $\varepsilon$-relaxed model stays almost constant in the domain size, whereas exact \ac{acp}, which must leave the $\varepsilon$-commutative factors uncompressed, slows down steeply, so the speedup grows exponentially with the domain size~$d$.
The speedup is essentially independent of $\varepsilon$ (given that the $\varepsilon$-deviation in the model is within the chosen $\varepsilon$-value used when running \ac{acp}, which is the case in our experiments), since the structure of the compressed model depends only on the commutativity pattern and not on the magnitude of $\varepsilon$, while it increases with the number of $\varepsilon$-commutative factors~$k$, as each additional compressed factor widens the gap between the two models.
At the same time, the per-query quotient remains tightly concentrated around the optimal value of one in every scenario.
In line with the theoretical bound, increasing $\varepsilon$ from $0.001$ to $0.1$ slightly widens the spread of the quotient, and increasing the number of $\varepsilon$-commutative factors~$k$ admits a few more outliers, since each compressed factor potentially contributes to the overall deviation.
Even for the largest tested $k = 7$ and for the largest tested $\varepsilon = 0.1$, however, the quotient stays well within the theoretical guarantee, confirming that the bound is loose in practice.

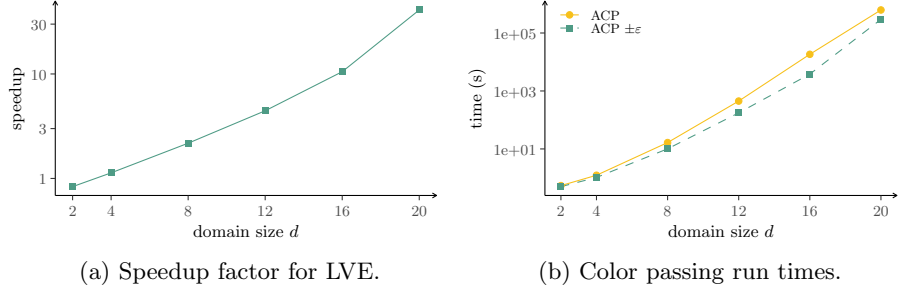
\begin{figure}[t]
	\centering
	\begin{subfigure}[t]{0.49\linewidth}
		\centering
		\resizebox{\linewidth}{!}{
\begin{tikzpicture}[x=1pt,y=1pt]
\definecolor{fillColor}{RGB}{255,255,255}
\path[use as bounding box,fill=fillColor,fill opacity=0.00] (0,0) rectangle (209.58,115.63);
\begin{scope}
\path[clip] (  0.00,  0.00) rectangle (209.58,115.63);
\definecolor{drawColor}{RGB}{255,255,255}
\definecolor{fillColor}{RGB}{255,255,255}

\path[draw=drawColor,line width= 0.5pt,line join=round,line cap=round,fill=fillColor] (  0.00,  0.00) rectangle (209.58,115.63);
\end{scope}
\begin{scope}
\path[clip] ( 24.25, 22.31) rectangle (203.58,113.63);
\definecolor{fillColor}{RGB}{255,255,255}

\path[fill=fillColor] ( 24.25, 22.31) rectangle (203.58,113.63);
\definecolor{drawColor}{RGB}{78,155,133}

\path[draw=drawColor,line width= 0.5pt,line join=round] ( 32.40, 26.46) --
	( 50.51, 32.94) --
	( 86.74, 46.87) --
	(122.97, 62.11) --
	(159.20, 80.47) --
	(195.43,109.48);
\definecolor{fillColor}{RGB}{78,155,133}

\path[fill=fillColor] ( 30.86, 24.93) --
	( 33.93, 24.93) --
	( 33.93, 28.00) --
	( 30.86, 28.00) --
	cycle;

\path[fill=fillColor] ( 48.98, 31.41) --
	( 52.05, 31.41) --
	( 52.05, 34.48) --
	( 48.98, 34.48) --
	cycle;

\path[fill=fillColor] ( 85.21, 45.34) --
	( 88.28, 45.34) --
	( 88.28, 48.40) --
	( 85.21, 48.40) --
	cycle;

\path[fill=fillColor] (121.44, 60.57) --
	(124.51, 60.57) --
	(124.51, 63.64) --
	(121.44, 63.64) --
	cycle;

\path[fill=fillColor] (157.67, 78.94) --
	(160.74, 78.94) --
	(160.74, 82.00) --
	(157.67, 82.00) --
	cycle;

\path[fill=fillColor] (193.90,107.95) --
	(196.97,107.95) --
	(196.97,111.01) --
	(193.90,111.01) --
	cycle;
\end{scope}
\begin{scope}
\path[clip] (  0.00,  0.00) rectangle (209.58,115.63);
\definecolor{drawColor}{RGB}{0,0,0}

\path[draw=drawColor,line width= 0.5pt,line join=round] ( 24.25, 22.31) --
	( 24.25,113.63);

\path[draw=drawColor,line width= 0.5pt,line join=round] ( 25.38,111.66) --
	( 24.25,113.63) --
	( 23.11,111.66);
\end{scope}
\begin{scope}
\path[clip] (  0.00,  0.00) rectangle (209.58,115.63);
\definecolor{drawColor}{gray}{0.30}

\node[text=drawColor,anchor=base east,inner sep=0pt, outer sep=0pt, scale=  0.80] at ( 20.20, 27.55) {1};

\node[text=drawColor,anchor=base east,inner sep=0pt, outer sep=0pt, scale=  0.80] at ( 20.20, 50.98) {3};

\node[text=drawColor,anchor=base east,inner sep=0pt, outer sep=0pt, scale=  0.80] at ( 20.20, 76.65) {10};

\node[text=drawColor,anchor=base east,inner sep=0pt, outer sep=0pt, scale=  0.80] at ( 20.20,100.08) {30};
\end{scope}
\begin{scope}
\path[clip] (  0.00,  0.00) rectangle (209.58,115.63);
\definecolor{drawColor}{gray}{0.20}

\path[draw=drawColor,line width= 0.5pt,line join=round] ( 22.00, 30.31) --
	( 24.25, 30.31);

\path[draw=drawColor,line width= 0.5pt,line join=round] ( 22.00, 53.74) --
	( 24.25, 53.74);

\path[draw=drawColor,line width= 0.5pt,line join=round] ( 22.00, 79.41) --
	( 24.25, 79.41);

\path[draw=drawColor,line width= 0.5pt,line join=round] ( 22.00,102.83) --
	( 24.25,102.83);
\end{scope}
\begin{scope}
\path[clip] (  0.00,  0.00) rectangle (209.58,115.63);
\definecolor{drawColor}{RGB}{0,0,0}

\path[draw=drawColor,line width= 0.5pt,line join=round] ( 24.25, 22.31) --
	(203.58, 22.31);

\path[draw=drawColor,line width= 0.5pt,line join=round] (201.61, 21.17) --
	(203.58, 22.31) --
	(201.61, 23.45);
\end{scope}
\begin{scope}
\path[clip] (  0.00,  0.00) rectangle (209.58,115.63);
\definecolor{drawColor}{gray}{0.20}

\path[draw=drawColor,line width= 0.5pt,line join=round] ( 32.40, 20.06) --
	( 32.40, 22.31);

\path[draw=drawColor,line width= 0.5pt,line join=round] ( 50.51, 20.06) --
	( 50.51, 22.31);

\path[draw=drawColor,line width= 0.5pt,line join=round] ( 86.74, 20.06) --
	( 86.74, 22.31);

\path[draw=drawColor,line width= 0.5pt,line join=round] (122.97, 20.06) --
	(122.97, 22.31);

\path[draw=drawColor,line width= 0.5pt,line join=round] (159.20, 20.06) --
	(159.20, 22.31);

\path[draw=drawColor,line width= 0.5pt,line join=round] (195.43, 20.06) --
	(195.43, 22.31);
\end{scope}
\begin{scope}
\path[clip] (  0.00,  0.00) rectangle (209.58,115.63);
\definecolor{drawColor}{gray}{0.30}

\node[text=drawColor,anchor=base,inner sep=0pt, outer sep=0pt, scale=  0.80] at ( 32.40, 12.75) {2};

\node[text=drawColor,anchor=base,inner sep=0pt, outer sep=0pt, scale=  0.80] at ( 50.51, 12.75) {4};

\node[text=drawColor,anchor=base,inner sep=0pt, outer sep=0pt, scale=  0.80] at ( 86.74, 12.75) {8};

\node[text=drawColor,anchor=base,inner sep=0pt, outer sep=0pt, scale=  0.80] at (122.97, 12.75) {12};

\node[text=drawColor,anchor=base,inner sep=0pt, outer sep=0pt, scale=  0.80] at (159.20, 12.75) {16};

\node[text=drawColor,anchor=base,inner sep=0pt, outer sep=0pt, scale=  0.80] at (195.43, 12.75) {20};
\end{scope}
\begin{scope}
\path[clip] (  0.00,  0.00) rectangle (209.58,115.63);
\definecolor{drawColor}{RGB}{0,0,0}

\node[text=drawColor,anchor=base,inner sep=0pt, outer sep=0pt, scale=  0.90] at (113.91,  2.75) {domain size $d$};
\end{scope}
\begin{scope}
\path[clip] (  0.00,  0.00) rectangle (209.58,115.63);
\definecolor{drawColor}{RGB}{0,0,0}

\node[text=drawColor,rotate= 90.00,anchor=base,inner sep=0pt, outer sep=0pt, scale=  0.90] at (  8.20, 67.97) {speedup};
\end{scope}
\end{tikzpicture}}
		\caption{Speedup factor for \ac{lve}.}
		\label{fig:commutative_approx_plot_speedup_avg}
	\end{subfigure}
	\begin{subfigure}[t]{0.49\linewidth}
		\centering
		\resizebox{\linewidth}{!}{
\begin{tikzpicture}[x=1pt,y=1pt]
\definecolor{fillColor}{RGB}{255,255,255}
\path[use as bounding box,fill=fillColor,fill opacity=0.00] (0,0) rectangle (209.58,115.63);
\begin{scope}
\path[clip] (  0.00,  0.00) rectangle (209.58,115.63);
\definecolor{drawColor}{RGB}{255,255,255}
\definecolor{fillColor}{RGB}{255,255,255}

\path[draw=drawColor,line width= 0.5pt,line join=round,line cap=round,fill=fillColor] (  0.00,  0.00) rectangle (209.58,115.63);
\end{scope}
\begin{scope}
\path[clip] ( 38.02, 22.31) rectangle (203.58,113.63);
\definecolor{fillColor}{RGB}{255,255,255}

\path[fill=fillColor] ( 38.02, 22.31) rectangle (203.58,113.63);
\definecolor{drawColor}{RGB}{247,192,26}

\path[draw=drawColor,line width= 0.5pt,line join=round] ( 45.55, 26.94) --
	( 62.27, 31.81) --
	( 95.72, 47.17) --
	(129.16, 66.67) --
	(162.61, 88.65) --
	(196.06,109.48);
\definecolor{drawColor}{RGB}{78,155,133}

\path[draw=drawColor,line width= 0.5pt,dash pattern=on 4pt off 4pt ,line join=round] ( 45.55, 26.46) --
	( 62.27, 30.78) --
	( 95.72, 44.23) --
	(129.16, 61.13) --
	(162.61, 79.33) --
	(196.06,105.07);
\definecolor{drawColor}{RGB}{247,192,26}
\definecolor{fillColor}{RGB}{247,192,26}

\path[draw=drawColor,line width= 0.4pt,line join=round,line cap=round,fill=fillColor] ( 45.55, 26.94) circle (  1.53);
\definecolor{fillColor}{RGB}{78,155,133}

\path[fill=fillColor] ( 44.01, 24.93) --
	( 47.08, 24.93) --
	( 47.08, 28.00) --
	( 44.01, 28.00) --
	cycle;
\definecolor{fillColor}{RGB}{247,192,26}

\path[draw=drawColor,line width= 0.4pt,line join=round,line cap=round,fill=fillColor] ( 62.27, 31.81) circle (  1.53);
\definecolor{fillColor}{RGB}{78,155,133}

\path[fill=fillColor] ( 60.74, 29.25) --
	( 63.80, 29.25) --
	( 63.80, 32.32) --
	( 60.74, 32.32) --
	cycle;
\definecolor{fillColor}{RGB}{247,192,26}

\path[draw=drawColor,line width= 0.4pt,line join=round,line cap=round,fill=fillColor] ( 95.72, 47.17) circle (  1.53);
\definecolor{fillColor}{RGB}{78,155,133}

\path[fill=fillColor] ( 94.18, 42.70) --
	( 97.25, 42.70) --
	( 97.25, 45.77) --
	( 94.18, 45.77) --
	cycle;
\definecolor{fillColor}{RGB}{247,192,26}

\path[draw=drawColor,line width= 0.4pt,line join=round,line cap=round,fill=fillColor] (129.16, 66.67) circle (  1.53);
\definecolor{fillColor}{RGB}{78,155,133}

\path[fill=fillColor] (127.63, 59.59) --
	(130.70, 59.59) --
	(130.70, 62.66) --
	(127.63, 62.66) --
	cycle;
\definecolor{fillColor}{RGB}{247,192,26}

\path[draw=drawColor,line width= 0.4pt,line join=round,line cap=round,fill=fillColor] (162.61, 88.65) circle (  1.53);
\definecolor{fillColor}{RGB}{78,155,133}

\path[fill=fillColor] (161.08, 77.80) --
	(164.14, 77.80) --
	(164.14, 80.87) --
	(161.08, 80.87) --
	cycle;
\definecolor{fillColor}{RGB}{247,192,26}

\path[draw=drawColor,line width= 0.4pt,line join=round,line cap=round,fill=fillColor] (196.06,109.48) circle (  1.53);
\definecolor{fillColor}{RGB}{78,155,133}

\path[fill=fillColor] (194.52,103.54) --
	(197.59,103.54) --
	(197.59,106.60) --
	(194.52,106.60) --
	cycle;
\end{scope}
\begin{scope}
\path[clip] (  0.00,  0.00) rectangle (209.58,115.63);
\definecolor{drawColor}{RGB}{0,0,0}

\path[draw=drawColor,line width= 0.5pt,line join=round] ( 38.02, 22.31) --
	( 38.02,113.63);

\path[draw=drawColor,line width= 0.5pt,line join=round] ( 39.16,111.66) --
	( 38.02,113.63) --
	( 36.88,111.66);
\end{scope}
\begin{scope}
\path[clip] (  0.00,  0.00) rectangle (209.58,115.63);
\definecolor{drawColor}{gray}{0.30}

\node[text=drawColor,anchor=base east,inner sep=0pt, outer sep=0pt, scale=  0.80] at ( 33.97, 41.43) {1e+01};

\node[text=drawColor,anchor=base east,inner sep=0pt, outer sep=0pt, scale=  0.80] at ( 33.97, 68.67) {1e+03};

\node[text=drawColor,anchor=base east,inner sep=0pt, outer sep=0pt, scale=  0.80] at ( 33.97, 95.90) {1e+05};
\end{scope}
\begin{scope}
\path[clip] (  0.00,  0.00) rectangle (209.58,115.63);
\definecolor{drawColor}{gray}{0.20}

\path[draw=drawColor,line width= 0.5pt,line join=round] ( 35.77, 44.18) --
	( 38.02, 44.18);

\path[draw=drawColor,line width= 0.5pt,line join=round] ( 35.77, 71.42) --
	( 38.02, 71.42);

\path[draw=drawColor,line width= 0.5pt,line join=round] ( 35.77, 98.66) --
	( 38.02, 98.66);
\end{scope}
\begin{scope}
\path[clip] (  0.00,  0.00) rectangle (209.58,115.63);
\definecolor{drawColor}{RGB}{0,0,0}

\path[draw=drawColor,line width= 0.5pt,line join=round] ( 38.02, 22.31) --
	(203.58, 22.31);

\path[draw=drawColor,line width= 0.5pt,line join=round] (201.61, 21.17) --
	(203.58, 22.31) --
	(201.61, 23.45);
\end{scope}
\begin{scope}
\path[clip] (  0.00,  0.00) rectangle (209.58,115.63);
\definecolor{drawColor}{gray}{0.20}

\path[draw=drawColor,line width= 0.5pt,line join=round] ( 45.55, 20.06) --
	( 45.55, 22.31);

\path[draw=drawColor,line width= 0.5pt,line join=round] ( 62.27, 20.06) --
	( 62.27, 22.31);

\path[draw=drawColor,line width= 0.5pt,line join=round] ( 95.72, 20.06) --
	( 95.72, 22.31);

\path[draw=drawColor,line width= 0.5pt,line join=round] (129.16, 20.06) --
	(129.16, 22.31);

\path[draw=drawColor,line width= 0.5pt,line join=round] (162.61, 20.06) --
	(162.61, 22.31);

\path[draw=drawColor,line width= 0.5pt,line join=round] (196.06, 20.06) --
	(196.06, 22.31);
\end{scope}
\begin{scope}
\path[clip] (  0.00,  0.00) rectangle (209.58,115.63);
\definecolor{drawColor}{gray}{0.30}

\node[text=drawColor,anchor=base,inner sep=0pt, outer sep=0pt, scale=  0.80] at ( 45.55, 12.75) {2};

\node[text=drawColor,anchor=base,inner sep=0pt, outer sep=0pt, scale=  0.80] at ( 62.27, 12.75) {4};

\node[text=drawColor,anchor=base,inner sep=0pt, outer sep=0pt, scale=  0.80] at ( 95.72, 12.75) {8};

\node[text=drawColor,anchor=base,inner sep=0pt, outer sep=0pt, scale=  0.80] at (129.16, 12.75) {12};

\node[text=drawColor,anchor=base,inner sep=0pt, outer sep=0pt, scale=  0.80] at (162.61, 12.75) {16};

\node[text=drawColor,anchor=base,inner sep=0pt, outer sep=0pt, scale=  0.80] at (196.06, 12.75) {20};
\end{scope}
\begin{scope}
\path[clip] (  0.00,  0.00) rectangle (209.58,115.63);
\definecolor{drawColor}{RGB}{0,0,0}

\node[text=drawColor,anchor=base,inner sep=0pt, outer sep=0pt, scale=  0.90] at (120.80,  2.75) {domain size $d$};
\end{scope}
\begin{scope}
\path[clip] (  0.00,  0.00) rectangle (209.58,115.63);
\definecolor{drawColor}{RGB}{0,0,0}

\node[text=drawColor,rotate= 90.00,anchor=base,inner sep=0pt, outer sep=0pt, scale=  0.90] at (  8.20, 67.97) {time (s)};
\end{scope}
\begin{scope}
\path[clip] (  0.00,  0.00) rectangle (209.58,115.63);

\path[] ( 44.20, 97.89) rectangle ( 91.45,114.76);
\end{scope}
\begin{scope}
\path[clip] (  0.00,  0.00) rectangle (209.58,115.63);
\definecolor{drawColor}{RGB}{247,192,26}

\path[draw=drawColor,line width= 0.5pt,line join=round] ( 45.28,107.17) -- ( 53.95,107.17);
\end{scope}
\begin{scope}
\path[clip] (  0.00,  0.00) rectangle (209.58,115.63);
\definecolor{drawColor}{RGB}{247,192,26}
\definecolor{fillColor}{RGB}{247,192,26}

\path[draw=drawColor,line width= 0.4pt,line join=round,line cap=round,fill=fillColor] ( 49.62,107.17) circle (  1.53);
\end{scope}
\begin{scope}
\path[clip] (  0.00,  0.00) rectangle (209.58,115.63);
\definecolor{drawColor}{RGB}{78,155,133}

\path[draw=drawColor,line width= 0.5pt,dash pattern=on 4pt off 4pt ,line join=round] ( 45.28,100.99) -- ( 53.95,100.99);
\end{scope}
\begin{scope}
\path[clip] (  0.00,  0.00) rectangle (209.58,115.63);
\definecolor{fillColor}{RGB}{78,155,133}

\path[fill=fillColor] ( 48.08, 99.45) --
	( 51.15, 99.45) --
	( 51.15,102.52) --
	( 48.08,102.52) --
	cycle;
\end{scope}
\begin{scope}
\path[clip] (  0.00,  0.00) rectangle (209.58,115.63);
\definecolor{drawColor}{RGB}{0,0,0}

\node[text=drawColor,anchor=base west,inner sep=0pt, outer sep=0pt, scale=  0.70] at ( 59.54,104.76) {ACP};
\end{scope}
\begin{scope}
\path[clip] (  0.00,  0.00) rectangle (209.58,115.63);
\definecolor{drawColor}{RGB}{0,0,0}

\node[text=drawColor,anchor=base west,inner sep=0pt, outer sep=0pt, scale=  0.70] at ( 59.54, 98.58) {ACP $\pm \varepsilon$};
\end{scope}
\end{tikzpicture}}
		\caption{Color passing run times.}
		\label{fig:commutative_approx_plot_cpruntime_avg}
	\end{subfigure}
	\caption{(a) Speedup for online inference run times (i.e., ratio of the run time of \ac{lve} on the compressed model returned by \ac{acp} and the run time of \ac{lve} on the compressed model returned by the $\varepsilon$-relaxed variant of \ac{acp}). (b) Offline run times of \acs{acp} and its $\varepsilon$-relaxed variant (\acs{acp} $\pm \varepsilon$).}
	\label{fig:commutative_approx_plot_aux_avg}
\end{figure}

A plot showcasing the resulting speedup of online inference (averaged over all choices of $k$ and $\varepsilon$) is given in \cref{fig:commutative_approx_plot_speedup_avg}.
It becomes evident that the speedup grows exponentially with the domain size~$d$ and eventually reaches a factor of more than $30$ for the choice of $d = 20$.
\Cref{fig:commutative_approx_plot_cpruntime_avg} further displays the offline run times of \ac{acp} and its $\varepsilon$-relaxed variant (averaged over all choices of $k$ and $\varepsilon$).
As expected, the offline run time of the $\varepsilon$-relaxed variant does not surpass the offline run time of exact \ac{acp}, demonstrating that the relaxation does not introduce any overhead at all.
In fact, $\varepsilon$-relaxed \ac{acp} even slightly reduces the run time compared to exact \ac{acp}, since $\varepsilon$-relaxed \ac{acp} is able to find $\varepsilon$-commutative factors and hence requires less iterations when searching for subsets of $\varepsilon$-commutative arguments (\ac{acp} needs more iterations as it checks all subsets of arguments until it is able to conclude that no subset of strictly commutative arguments exists).

\begin{figure}[t]
	\centering
	\begin{subfigure}[t]{0.49\linewidth}
		\centering
		\resizebox{\linewidth}{!}{
\begin{tikzpicture}[x=1pt,y=1pt]
\definecolor{fillColor}{RGB}{255,255,255}
\path[use as bounding box,fill=fillColor,fill opacity=0.00] (0,0) rectangle (209.58,126.47);
\begin{scope}
\path[clip] (  0.00,  0.00) rectangle (209.58,126.47);
\definecolor{drawColor}{RGB}{255,255,255}
\definecolor{fillColor}{RGB}{255,255,255}

\path[draw=drawColor,line width= 0.5pt,line join=round,line cap=round,fill=fillColor] (  0.00,  0.00) rectangle (209.58,126.47);
\end{scope}
\begin{scope}
\path[clip] ( 30.47, 30.49) rectangle (203.58,124.47);
\definecolor{fillColor}{RGB}{255,255,255}

\path[fill=fillColor] ( 30.47, 30.49) rectangle (203.58,124.47);
\definecolor{drawColor}{RGB}{247,192,26}

\path[draw=drawColor,line width= 0.5pt,line join=round] ( 38.34,116.13) --
	( 55.82,105.96) --
	( 90.80, 99.00) --
	(125.77, 94.59) --
	(160.74, 94.43) --
	(195.71, 94.34);
\definecolor{drawColor}{RGB}{78,155,133}

\path[draw=drawColor,line width= 0.5pt,dash pattern=on 4pt off 4pt ,line join=round] ( 38.34, 85.33) --
	( 55.82, 59.77) --
	( 90.80, 46.76) --
	(125.77, 42.34) --
	(160.74, 40.50) --
	(195.71, 39.38);
\definecolor{drawColor}{RGB}{247,192,26}
\definecolor{fillColor}{RGB}{247,192,26}

\path[draw=drawColor,line width= 0.4pt,line join=round,line cap=round,fill=fillColor] ( 38.34,116.13) circle (  1.53);
\definecolor{fillColor}{RGB}{78,155,133}

\path[fill=fillColor] ( 36.80, 83.80) --
	( 39.87, 83.80) --
	( 39.87, 86.86) --
	( 36.80, 86.86) --
	cycle;
\definecolor{fillColor}{RGB}{247,192,26}

\path[draw=drawColor,line width= 0.4pt,line join=round,line cap=round,fill=fillColor] ( 55.82,105.96) circle (  1.53);
\definecolor{fillColor}{RGB}{78,155,133}

\path[fill=fillColor] ( 54.29, 58.24) --
	( 57.36, 58.24) --
	( 57.36, 61.30) --
	( 54.29, 61.30) --
	cycle;
\definecolor{fillColor}{RGB}{247,192,26}

\path[draw=drawColor,line width= 0.4pt,line join=round,line cap=round,fill=fillColor] ( 90.80, 99.00) circle (  1.53);
\definecolor{fillColor}{RGB}{78,155,133}

\path[fill=fillColor] ( 89.26, 45.23) --
	( 92.33, 45.23) --
	( 92.33, 48.30) --
	( 89.26, 48.30) --
	cycle;
\definecolor{fillColor}{RGB}{247,192,26}

\path[draw=drawColor,line width= 0.4pt,line join=round,line cap=round,fill=fillColor] (125.77, 94.59) circle (  1.53);
\definecolor{fillColor}{RGB}{78,155,133}

\path[fill=fillColor] (124.23, 40.81) --
	(127.30, 40.81) --
	(127.30, 43.87) --
	(124.23, 43.87) --
	cycle;
\definecolor{fillColor}{RGB}{247,192,26}

\path[draw=drawColor,line width= 0.4pt,line join=round,line cap=round,fill=fillColor] (160.74, 94.43) circle (  1.53);
\definecolor{fillColor}{RGB}{78,155,133}

\path[fill=fillColor] (159.21, 38.97) --
	(162.27, 38.97) --
	(162.27, 42.03) --
	(159.21, 42.03) --
	cycle;
\definecolor{fillColor}{RGB}{247,192,26}

\path[draw=drawColor,line width= 0.4pt,line join=round,line cap=round,fill=fillColor] (195.71, 94.34) circle (  1.53);
\definecolor{fillColor}{RGB}{78,155,133}

\path[fill=fillColor] (194.18, 37.84) --
	(197.25, 37.84) --
	(197.25, 40.91) --
	(194.18, 40.91) --
	cycle;
\end{scope}
\begin{scope}
\path[clip] (  0.00,  0.00) rectangle (209.58,126.47);
\definecolor{drawColor}{RGB}{0,0,0}

\path[draw=drawColor,line width= 0.5pt,line join=round] ( 30.47, 30.49) --
	( 30.47,124.47);

\path[draw=drawColor,line width= 0.5pt,line join=round] ( 31.60,122.50) --
	( 30.47,124.47) --
	( 29.33,122.50);
\end{scope}
\begin{scope}
\path[clip] (  0.00,  0.00) rectangle (209.58,126.47);
\definecolor{drawColor}{gray}{0.30}

\node[text=drawColor,anchor=base east,inner sep=0pt, outer sep=0pt, scale=  0.80] at ( 26.42, 32.01) {0.00};

\node[text=drawColor,anchor=base east,inner sep=0pt, outer sep=0pt, scale=  0.80] at ( 26.42, 52.35) {0.25};

\node[text=drawColor,anchor=base east,inner sep=0pt, outer sep=0pt, scale=  0.80] at ( 26.42, 72.69) {0.50};

\node[text=drawColor,anchor=base east,inner sep=0pt, outer sep=0pt, scale=  0.80] at ( 26.42, 93.04) {0.75};

\node[text=drawColor,anchor=base east,inner sep=0pt, outer sep=0pt, scale=  0.80] at ( 26.42,113.38) {1.00};
\end{scope}
\begin{scope}
\path[clip] (  0.00,  0.00) rectangle (209.58,126.47);
\definecolor{drawColor}{gray}{0.20}

\path[draw=drawColor,line width= 0.5pt,line join=round] ( 28.22, 34.77) --
	( 30.47, 34.77);

\path[draw=drawColor,line width= 0.5pt,line join=round] ( 28.22, 55.11) --
	( 30.47, 55.11);

\path[draw=drawColor,line width= 0.5pt,line join=round] ( 28.22, 75.45) --
	( 30.47, 75.45);

\path[draw=drawColor,line width= 0.5pt,line join=round] ( 28.22, 95.79) --
	( 30.47, 95.79);

\path[draw=drawColor,line width= 0.5pt,line join=round] ( 28.22,116.13) --
	( 30.47,116.13);
\end{scope}
\begin{scope}
\path[clip] (  0.00,  0.00) rectangle (209.58,126.47);
\definecolor{drawColor}{RGB}{0,0,0}

\path[draw=drawColor,line width= 0.5pt,line join=round] ( 30.47, 30.49) --
	(203.58, 30.49);

\path[draw=drawColor,line width= 0.5pt,line join=round] (201.61, 29.36) --
	(203.58, 30.49) --
	(201.61, 31.63);
\end{scope}
\begin{scope}
\path[clip] (  0.00,  0.00) rectangle (209.58,126.47);
\definecolor{drawColor}{gray}{0.20}

\path[draw=drawColor,line width= 0.5pt,line join=round] ( 38.34, 28.24) --
	( 38.34, 30.49);

\path[draw=drawColor,line width= 0.5pt,line join=round] ( 55.82, 28.24) --
	( 55.82, 30.49);

\path[draw=drawColor,line width= 0.5pt,line join=round] ( 90.80, 28.24) --
	( 90.80, 30.49);

\path[draw=drawColor,line width= 0.5pt,line join=round] (125.77, 28.24) --
	(125.77, 30.49);

\path[draw=drawColor,line width= 0.5pt,line join=round] (160.74, 28.24) --
	(160.74, 30.49);

\path[draw=drawColor,line width= 0.5pt,line join=round] (195.71, 28.24) --
	(195.71, 30.49);
\end{scope}
\begin{scope}
\path[clip] (  0.00,  0.00) rectangle (209.58,126.47);
\definecolor{drawColor}{gray}{0.30}

\node[text=drawColor,anchor=base,inner sep=0pt, outer sep=0pt, scale=  0.80] at ( 38.34, 20.93) {2};

\node[text=drawColor,anchor=base,inner sep=0pt, outer sep=0pt, scale=  0.80] at ( 55.82, 20.93) {4};

\node[text=drawColor,anchor=base,inner sep=0pt, outer sep=0pt, scale=  0.80] at ( 90.80, 20.93) {8};

\node[text=drawColor,anchor=base,inner sep=0pt, outer sep=0pt, scale=  0.80] at (125.77, 20.93) {12};

\node[text=drawColor,anchor=base,inner sep=0pt, outer sep=0pt, scale=  0.80] at (160.74, 20.93) {16};

\node[text=drawColor,anchor=base,inner sep=0pt, outer sep=0pt, scale=  0.80] at (195.71, 20.93) {20};
\end{scope}
\begin{scope}
\path[clip] (  0.00,  0.00) rectangle (209.58,126.47);
\definecolor{drawColor}{RGB}{0,0,0}

\node[text=drawColor,anchor=base,inner sep=0pt, outer sep=0pt, scale=  0.90] at (117.02, 10.93) {domain size $d$};
\end{scope}
\begin{scope}
\path[clip] (  0.00,  0.00) rectangle (209.58,126.47);
\definecolor{drawColor}{RGB}{0,0,0}

\node[text=drawColor,rotate= 90.00,anchor=base,inner sep=0pt, outer sep=0pt, scale=  0.90] at (  8.20, 77.48) {$|M'| \mathbin{/} |M|$};
\end{scope}
\begin{scope}
\path[clip] (  0.00,  0.00) rectangle (209.58,126.47);

\path[] ( 49.41,  1.00) rectangle (184.64,  7.18);
\end{scope}
\begin{scope}
\path[clip] (  0.00,  0.00) rectangle (209.58,126.47);
\definecolor{drawColor}{RGB}{247,192,26}

\path[draw=drawColor,line width= 0.5pt,line join=round] ( 54.99,  4.09) -- ( 63.67,  4.09);
\end{scope}
\begin{scope}
\path[clip] (  0.00,  0.00) rectangle (209.58,126.47);
\definecolor{drawColor}{RGB}{247,192,26}
\definecolor{fillColor}{RGB}{247,192,26}

\path[draw=drawColor,line width= 0.4pt,line join=round,line cap=round,fill=fillColor] ( 59.33,  4.09) circle (  1.53);
\end{scope}
\begin{scope}
\path[clip] (  0.00,  0.00) rectangle (209.58,126.47);
\definecolor{drawColor}{RGB}{78,155,133}

\path[draw=drawColor,line width= 0.5pt,dash pattern=on 4pt off 4pt ,line join=round] (117.09,  4.09) -- (125.76,  4.09);
\end{scope}
\begin{scope}
\path[clip] (  0.00,  0.00) rectangle (209.58,126.47);
\definecolor{fillColor}{RGB}{78,155,133}

\path[fill=fillColor] (119.89,  2.56) --
	(122.96,  2.56) --
	(122.96,  5.62) --
	(119.89,  5.62) --
	cycle;
\end{scope}
\begin{scope}
\path[clip] (  0.00,  0.00) rectangle (209.58,126.47);
\definecolor{drawColor}{RGB}{0,0,0}

\node[text=drawColor,anchor=base west,inner sep=0pt, outer sep=0pt, scale=  0.70] at ( 69.25,  1.68) {LVE (ACP)};
\end{scope}
\begin{scope}
\path[clip] (  0.00,  0.00) rectangle (209.58,126.47);
\definecolor{drawColor}{RGB}{0,0,0}

\node[text=drawColor,anchor=base west,inner sep=0pt, outer sep=0pt, scale=  0.70] at (131.34,  1.68) {LVE (ACP $\pm \varepsilon$)};
\end{scope}
\end{tikzpicture}}
		\caption{Compression ratio for $\varepsilon = 0.001$.}
		\label{fig:commutative_approx_plot_comp_k=1_eps=0.001}
	\end{subfigure}
	\begin{subfigure}[t]{0.49\linewidth}
		\centering
		\resizebox{\linewidth}{!}{
\begin{tikzpicture}[x=1pt,y=1pt]
\definecolor{fillColor}{RGB}{255,255,255}
\path[use as bounding box,fill=fillColor,fill opacity=0.00] (0,0) rectangle (209.58,126.47);
\begin{scope}
\path[clip] (  0.00,  0.00) rectangle (209.58,126.47);
\definecolor{drawColor}{RGB}{255,255,255}
\definecolor{fillColor}{RGB}{255,255,255}

\path[draw=drawColor,line width= 0.5pt,line join=round,line cap=round,fill=fillColor] (  0.00,  0.00) rectangle (209.58,126.47);
\end{scope}
\begin{scope}
\path[clip] ( 30.47, 30.49) rectangle (203.58,124.47);
\definecolor{fillColor}{RGB}{255,255,255}

\path[fill=fillColor] ( 30.47, 30.49) rectangle (203.58,124.47);
\definecolor{drawColor}{RGB}{247,192,26}

\path[draw=drawColor,line width= 0.5pt,line join=round] ( 38.34,116.13) --
	( 55.82,105.96) --
	( 90.80, 99.00) --
	(125.77, 94.59) --
	(160.74, 94.43) --
	(195.71, 94.34);
\definecolor{drawColor}{RGB}{78,155,133}

\path[draw=drawColor,line width= 0.5pt,dash pattern=on 4pt off 4pt ,line join=round] ( 38.34, 85.33) --
	( 55.82, 59.77) --
	( 90.80, 46.76) --
	(125.77, 42.34) --
	(160.74, 40.50) --
	(195.71, 39.38);
\definecolor{drawColor}{RGB}{247,192,26}
\definecolor{fillColor}{RGB}{247,192,26}

\path[draw=drawColor,line width= 0.4pt,line join=round,line cap=round,fill=fillColor] ( 38.34,116.13) circle (  1.53);
\definecolor{fillColor}{RGB}{78,155,133}

\path[fill=fillColor] ( 36.80, 83.80) --
	( 39.87, 83.80) --
	( 39.87, 86.86) --
	( 36.80, 86.86) --
	cycle;
\definecolor{fillColor}{RGB}{247,192,26}

\path[draw=drawColor,line width= 0.4pt,line join=round,line cap=round,fill=fillColor] ( 55.82,105.96) circle (  1.53);
\definecolor{fillColor}{RGB}{78,155,133}

\path[fill=fillColor] ( 54.29, 58.24) --
	( 57.36, 58.24) --
	( 57.36, 61.30) --
	( 54.29, 61.30) --
	cycle;
\definecolor{fillColor}{RGB}{247,192,26}

\path[draw=drawColor,line width= 0.4pt,line join=round,line cap=round,fill=fillColor] ( 90.80, 99.00) circle (  1.53);
\definecolor{fillColor}{RGB}{78,155,133}

\path[fill=fillColor] ( 89.26, 45.23) --
	( 92.33, 45.23) --
	( 92.33, 48.30) --
	( 89.26, 48.30) --
	cycle;
\definecolor{fillColor}{RGB}{247,192,26}

\path[draw=drawColor,line width= 0.4pt,line join=round,line cap=round,fill=fillColor] (125.77, 94.59) circle (  1.53);
\definecolor{fillColor}{RGB}{78,155,133}

\path[fill=fillColor] (124.23, 40.81) --
	(127.30, 40.81) --
	(127.30, 43.87) --
	(124.23, 43.87) --
	cycle;
\definecolor{fillColor}{RGB}{247,192,26}

\path[draw=drawColor,line width= 0.4pt,line join=round,line cap=round,fill=fillColor] (160.74, 94.43) circle (  1.53);
\definecolor{fillColor}{RGB}{78,155,133}

\path[fill=fillColor] (159.21, 38.97) --
	(162.27, 38.97) --
	(162.27, 42.03) --
	(159.21, 42.03) --
	cycle;
\definecolor{fillColor}{RGB}{247,192,26}

\path[draw=drawColor,line width= 0.4pt,line join=round,line cap=round,fill=fillColor] (195.71, 94.34) circle (  1.53);
\definecolor{fillColor}{RGB}{78,155,133}

\path[fill=fillColor] (194.18, 37.84) --
	(197.25, 37.84) --
	(197.25, 40.91) --
	(194.18, 40.91) --
	cycle;
\end{scope}
\begin{scope}
\path[clip] (  0.00,  0.00) rectangle (209.58,126.47);
\definecolor{drawColor}{RGB}{0,0,0}

\path[draw=drawColor,line width= 0.5pt,line join=round] ( 30.47, 30.49) --
	( 30.47,124.47);

\path[draw=drawColor,line width= 0.5pt,line join=round] ( 31.60,122.50) --
	( 30.47,124.47) --
	( 29.33,122.50);
\end{scope}
\begin{scope}
\path[clip] (  0.00,  0.00) rectangle (209.58,126.47);
\definecolor{drawColor}{gray}{0.30}

\node[text=drawColor,anchor=base east,inner sep=0pt, outer sep=0pt, scale=  0.80] at ( 26.42, 32.01) {0.00};

\node[text=drawColor,anchor=base east,inner sep=0pt, outer sep=0pt, scale=  0.80] at ( 26.42, 52.35) {0.25};

\node[text=drawColor,anchor=base east,inner sep=0pt, outer sep=0pt, scale=  0.80] at ( 26.42, 72.69) {0.50};

\node[text=drawColor,anchor=base east,inner sep=0pt, outer sep=0pt, scale=  0.80] at ( 26.42, 93.04) {0.75};

\node[text=drawColor,anchor=base east,inner sep=0pt, outer sep=0pt, scale=  0.80] at ( 26.42,113.38) {1.00};
\end{scope}
\begin{scope}
\path[clip] (  0.00,  0.00) rectangle (209.58,126.47);
\definecolor{drawColor}{gray}{0.20}

\path[draw=drawColor,line width= 0.5pt,line join=round] ( 28.22, 34.77) --
	( 30.47, 34.77);

\path[draw=drawColor,line width= 0.5pt,line join=round] ( 28.22, 55.11) --
	( 30.47, 55.11);

\path[draw=drawColor,line width= 0.5pt,line join=round] ( 28.22, 75.45) --
	( 30.47, 75.45);

\path[draw=drawColor,line width= 0.5pt,line join=round] ( 28.22, 95.79) --
	( 30.47, 95.79);

\path[draw=drawColor,line width= 0.5pt,line join=round] ( 28.22,116.13) --
	( 30.47,116.13);
\end{scope}
\begin{scope}
\path[clip] (  0.00,  0.00) rectangle (209.58,126.47);
\definecolor{drawColor}{RGB}{0,0,0}

\path[draw=drawColor,line width= 0.5pt,line join=round] ( 30.47, 30.49) --
	(203.58, 30.49);

\path[draw=drawColor,line width= 0.5pt,line join=round] (201.61, 29.36) --
	(203.58, 30.49) --
	(201.61, 31.63);
\end{scope}
\begin{scope}
\path[clip] (  0.00,  0.00) rectangle (209.58,126.47);
\definecolor{drawColor}{gray}{0.20}

\path[draw=drawColor,line width= 0.5pt,line join=round] ( 38.34, 28.24) --
	( 38.34, 30.49);

\path[draw=drawColor,line width= 0.5pt,line join=round] ( 55.82, 28.24) --
	( 55.82, 30.49);

\path[draw=drawColor,line width= 0.5pt,line join=round] ( 90.80, 28.24) --
	( 90.80, 30.49);

\path[draw=drawColor,line width= 0.5pt,line join=round] (125.77, 28.24) --
	(125.77, 30.49);

\path[draw=drawColor,line width= 0.5pt,line join=round] (160.74, 28.24) --
	(160.74, 30.49);

\path[draw=drawColor,line width= 0.5pt,line join=round] (195.71, 28.24) --
	(195.71, 30.49);
\end{scope}
\begin{scope}
\path[clip] (  0.00,  0.00) rectangle (209.58,126.47);
\definecolor{drawColor}{gray}{0.30}

\node[text=drawColor,anchor=base,inner sep=0pt, outer sep=0pt, scale=  0.80] at ( 38.34, 20.93) {2};

\node[text=drawColor,anchor=base,inner sep=0pt, outer sep=0pt, scale=  0.80] at ( 55.82, 20.93) {4};

\node[text=drawColor,anchor=base,inner sep=0pt, outer sep=0pt, scale=  0.80] at ( 90.80, 20.93) {8};

\node[text=drawColor,anchor=base,inner sep=0pt, outer sep=0pt, scale=  0.80] at (125.77, 20.93) {12};

\node[text=drawColor,anchor=base,inner sep=0pt, outer sep=0pt, scale=  0.80] at (160.74, 20.93) {16};

\node[text=drawColor,anchor=base,inner sep=0pt, outer sep=0pt, scale=  0.80] at (195.71, 20.93) {20};
\end{scope}
\begin{scope}
\path[clip] (  0.00,  0.00) rectangle (209.58,126.47);
\definecolor{drawColor}{RGB}{0,0,0}

\node[text=drawColor,anchor=base,inner sep=0pt, outer sep=0pt, scale=  0.90] at (117.02, 10.93) {domain size $d$};
\end{scope}
\begin{scope}
\path[clip] (  0.00,  0.00) rectangle (209.58,126.47);
\definecolor{drawColor}{RGB}{0,0,0}

\node[text=drawColor,rotate= 90.00,anchor=base,inner sep=0pt, outer sep=0pt, scale=  0.90] at (  8.20, 77.48) {$|M'| \mathbin{/} |M|$};
\end{scope}
\begin{scope}
\path[clip] (  0.00,  0.00) rectangle (209.58,126.47);

\path[] ( 49.41,  1.00) rectangle (184.64,  7.18);
\end{scope}
\begin{scope}
\path[clip] (  0.00,  0.00) rectangle (209.58,126.47);
\definecolor{drawColor}{RGB}{247,192,26}

\path[draw=drawColor,line width= 0.5pt,line join=round] ( 54.99,  4.09) -- ( 63.67,  4.09);
\end{scope}
\begin{scope}
\path[clip] (  0.00,  0.00) rectangle (209.58,126.47);
\definecolor{drawColor}{RGB}{247,192,26}
\definecolor{fillColor}{RGB}{247,192,26}

\path[draw=drawColor,line width= 0.4pt,line join=round,line cap=round,fill=fillColor] ( 59.33,  4.09) circle (  1.53);
\end{scope}
\begin{scope}
\path[clip] (  0.00,  0.00) rectangle (209.58,126.47);
\definecolor{drawColor}{RGB}{78,155,133}

\path[draw=drawColor,line width= 0.5pt,dash pattern=on 4pt off 4pt ,line join=round] (117.09,  4.09) -- (125.76,  4.09);
\end{scope}
\begin{scope}
\path[clip] (  0.00,  0.00) rectangle (209.58,126.47);
\definecolor{fillColor}{RGB}{78,155,133}

\path[fill=fillColor] (119.89,  2.56) --
	(122.96,  2.56) --
	(122.96,  5.62) --
	(119.89,  5.62) --
	cycle;
\end{scope}
\begin{scope}
\path[clip] (  0.00,  0.00) rectangle (209.58,126.47);
\definecolor{drawColor}{RGB}{0,0,0}

\node[text=drawColor,anchor=base west,inner sep=0pt, outer sep=0pt, scale=  0.70] at ( 69.25,  1.68) {LVE (ACP)};
\end{scope}
\begin{scope}
\path[clip] (  0.00,  0.00) rectangle (209.58,126.47);
\definecolor{drawColor}{RGB}{0,0,0}

\node[text=drawColor,anchor=base west,inner sep=0pt, outer sep=0pt, scale=  0.70] at (131.34,  1.68) {LVE (ACP $\pm \varepsilon$)};
\end{scope}
\end{tikzpicture}}
		\caption{Compression ratio for $\varepsilon = 0.1$.}
		\label{fig:commutative_approx_plot_comp_k=1_eps=0.1}
	\end{subfigure}
	\caption{Compression ratio $\abs{M'} \mathbin{/} \abs{M}$ between the compressed model $M'$ and the original model $M$ for input \acp{fg} containing $k = 1$ $\varepsilon$-commutative factors with (a) $\varepsilon = 0.001$ and (b) $\varepsilon = 0.1$.}
	\label{fig:commutative_approx_plot_comp_k=1}
\end{figure}
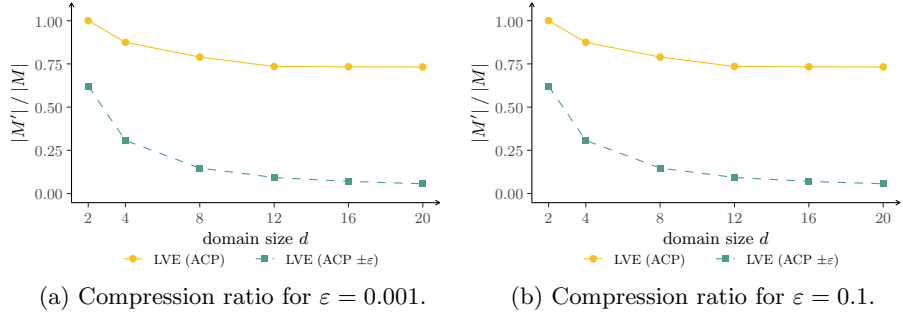

\begin{figure}[t]
	\centering
	\begin{subfigure}[t]{0.49\linewidth}
		\centering
		\resizebox{\linewidth}{!}{
\begin{tikzpicture}[x=1pt,y=1pt]
\definecolor{fillColor}{RGB}{255,255,255}
\path[use as bounding box,fill=fillColor,fill opacity=0.00] (0,0) rectangle (209.58,126.47);
\begin{scope}
\path[clip] (  0.00,  0.00) rectangle (209.58,126.47);
\definecolor{drawColor}{RGB}{255,255,255}
\definecolor{fillColor}{RGB}{255,255,255}

\path[draw=drawColor,line width= 0.5pt,line join=round,line cap=round,fill=fillColor] (  0.00,  0.00) rectangle (209.58,126.47);
\end{scope}
\begin{scope}
\path[clip] ( 30.47, 30.49) rectangle (203.58,124.47);
\definecolor{fillColor}{RGB}{255,255,255}

\path[fill=fillColor] ( 30.47, 30.49) rectangle (203.58,124.47);
\definecolor{drawColor}{RGB}{247,192,26}

\path[draw=drawColor,line width= 0.5pt,line join=round] ( 38.34,116.13) --
	( 55.82,107.79) --
	( 90.80,101.17) --
	(125.77, 96.73) --
	(160.74, 96.51) --
	(195.71, 96.37);
\definecolor{drawColor}{RGB}{78,155,133}

\path[draw=drawColor,line width= 0.5pt,dash pattern=on 4pt off 4pt ,line join=round] ( 38.34, 89.72) --
	( 55.82, 64.60) --
	( 90.80, 49.89) --
	(125.77, 44.53) --
	(160.74, 42.21) --
	(195.71, 40.78);
\definecolor{drawColor}{RGB}{247,192,26}
\definecolor{fillColor}{RGB}{247,192,26}

\path[draw=drawColor,line width= 0.4pt,line join=round,line cap=round,fill=fillColor] ( 38.34,116.13) circle (  1.53);
\definecolor{fillColor}{RGB}{78,155,133}

\path[fill=fillColor] ( 36.80, 88.19) --
	( 39.87, 88.19) --
	( 39.87, 91.26) --
	( 36.80, 91.26) --
	cycle;
\definecolor{fillColor}{RGB}{247,192,26}

\path[draw=drawColor,line width= 0.4pt,line join=round,line cap=round,fill=fillColor] ( 55.82,107.79) circle (  1.53);
\definecolor{fillColor}{RGB}{78,155,133}

\path[fill=fillColor] ( 54.29, 63.07) --
	( 57.36, 63.07) --
	( 57.36, 66.13) --
	( 54.29, 66.13) --
	cycle;
\definecolor{fillColor}{RGB}{247,192,26}

\path[draw=drawColor,line width= 0.4pt,line join=round,line cap=round,fill=fillColor] ( 90.80,101.17) circle (  1.53);
\definecolor{fillColor}{RGB}{78,155,133}

\path[fill=fillColor] ( 89.26, 48.35) --
	( 92.33, 48.35) --
	( 92.33, 51.42) --
	( 89.26, 51.42) --
	cycle;
\definecolor{fillColor}{RGB}{247,192,26}

\path[draw=drawColor,line width= 0.4pt,line join=round,line cap=round,fill=fillColor] (125.77, 96.73) circle (  1.53);
\definecolor{fillColor}{RGB}{78,155,133}

\path[fill=fillColor] (124.23, 42.99) --
	(127.30, 42.99) --
	(127.30, 46.06) --
	(124.23, 46.06) --
	cycle;
\definecolor{fillColor}{RGB}{247,192,26}

\path[draw=drawColor,line width= 0.4pt,line join=round,line cap=round,fill=fillColor] (160.74, 96.51) circle (  1.53);
\definecolor{fillColor}{RGB}{78,155,133}

\path[fill=fillColor] (159.21, 40.68) --
	(162.27, 40.68) --
	(162.27, 43.74) --
	(159.21, 43.74) --
	cycle;
\definecolor{fillColor}{RGB}{247,192,26}

\path[draw=drawColor,line width= 0.4pt,line join=round,line cap=round,fill=fillColor] (195.71, 96.37) circle (  1.53);
\definecolor{fillColor}{RGB}{78,155,133}

\path[fill=fillColor] (194.18, 39.25) --
	(197.25, 39.25) --
	(197.25, 42.32) --
	(194.18, 42.32) --
	cycle;
\end{scope}
\begin{scope}
\path[clip] (  0.00,  0.00) rectangle (209.58,126.47);
\definecolor{drawColor}{RGB}{0,0,0}

\path[draw=drawColor,line width= 0.5pt,line join=round] ( 30.47, 30.49) --
	( 30.47,124.47);

\path[draw=drawColor,line width= 0.5pt,line join=round] ( 31.60,122.50) --
	( 30.47,124.47) --
	( 29.33,122.50);
\end{scope}
\begin{scope}
\path[clip] (  0.00,  0.00) rectangle (209.58,126.47);
\definecolor{drawColor}{gray}{0.30}

\node[text=drawColor,anchor=base east,inner sep=0pt, outer sep=0pt, scale=  0.80] at ( 26.42, 32.01) {0.00};

\node[text=drawColor,anchor=base east,inner sep=0pt, outer sep=0pt, scale=  0.80] at ( 26.42, 52.35) {0.25};

\node[text=drawColor,anchor=base east,inner sep=0pt, outer sep=0pt, scale=  0.80] at ( 26.42, 72.69) {0.50};

\node[text=drawColor,anchor=base east,inner sep=0pt, outer sep=0pt, scale=  0.80] at ( 26.42, 93.04) {0.75};

\node[text=drawColor,anchor=base east,inner sep=0pt, outer sep=0pt, scale=  0.80] at ( 26.42,113.38) {1.00};
\end{scope}
\begin{scope}
\path[clip] (  0.00,  0.00) rectangle (209.58,126.47);
\definecolor{drawColor}{gray}{0.20}

\path[draw=drawColor,line width= 0.5pt,line join=round] ( 28.22, 34.77) --
	( 30.47, 34.77);

\path[draw=drawColor,line width= 0.5pt,line join=round] ( 28.22, 55.11) --
	( 30.47, 55.11);

\path[draw=drawColor,line width= 0.5pt,line join=round] ( 28.22, 75.45) --
	( 30.47, 75.45);

\path[draw=drawColor,line width= 0.5pt,line join=round] ( 28.22, 95.79) --
	( 30.47, 95.79);

\path[draw=drawColor,line width= 0.5pt,line join=round] ( 28.22,116.13) --
	( 30.47,116.13);
\end{scope}
\begin{scope}
\path[clip] (  0.00,  0.00) rectangle (209.58,126.47);
\definecolor{drawColor}{RGB}{0,0,0}

\path[draw=drawColor,line width= 0.5pt,line join=round] ( 30.47, 30.49) --
	(203.58, 30.49);

\path[draw=drawColor,line width= 0.5pt,line join=round] (201.61, 29.36) --
	(203.58, 30.49) --
	(201.61, 31.63);
\end{scope}
\begin{scope}
\path[clip] (  0.00,  0.00) rectangle (209.58,126.47);
\definecolor{drawColor}{gray}{0.20}

\path[draw=drawColor,line width= 0.5pt,line join=round] ( 38.34, 28.24) --
	( 38.34, 30.49);

\path[draw=drawColor,line width= 0.5pt,line join=round] ( 55.82, 28.24) --
	( 55.82, 30.49);

\path[draw=drawColor,line width= 0.5pt,line join=round] ( 90.80, 28.24) --
	( 90.80, 30.49);

\path[draw=drawColor,line width= 0.5pt,line join=round] (125.77, 28.24) --
	(125.77, 30.49);

\path[draw=drawColor,line width= 0.5pt,line join=round] (160.74, 28.24) --
	(160.74, 30.49);

\path[draw=drawColor,line width= 0.5pt,line join=round] (195.71, 28.24) --
	(195.71, 30.49);
\end{scope}
\begin{scope}
\path[clip] (  0.00,  0.00) rectangle (209.58,126.47);
\definecolor{drawColor}{gray}{0.30}

\node[text=drawColor,anchor=base,inner sep=0pt, outer sep=0pt, scale=  0.80] at ( 38.34, 20.93) {2};

\node[text=drawColor,anchor=base,inner sep=0pt, outer sep=0pt, scale=  0.80] at ( 55.82, 20.93) {4};

\node[text=drawColor,anchor=base,inner sep=0pt, outer sep=0pt, scale=  0.80] at ( 90.80, 20.93) {8};

\node[text=drawColor,anchor=base,inner sep=0pt, outer sep=0pt, scale=  0.80] at (125.77, 20.93) {12};

\node[text=drawColor,anchor=base,inner sep=0pt, outer sep=0pt, scale=  0.80] at (160.74, 20.93) {16};

\node[text=drawColor,anchor=base,inner sep=0pt, outer sep=0pt, scale=  0.80] at (195.71, 20.93) {20};
\end{scope}
\begin{scope}
\path[clip] (  0.00,  0.00) rectangle (209.58,126.47);
\definecolor{drawColor}{RGB}{0,0,0}

\node[text=drawColor,anchor=base,inner sep=0pt, outer sep=0pt, scale=  0.90] at (117.02, 10.93) {domain size $d$};
\end{scope}
\begin{scope}
\path[clip] (  0.00,  0.00) rectangle (209.58,126.47);
\definecolor{drawColor}{RGB}{0,0,0}

\node[text=drawColor,rotate= 90.00,anchor=base,inner sep=0pt, outer sep=0pt, scale=  0.90] at (  8.20, 77.48) {$|M'| \mathbin{/} |M|$};
\end{scope}
\begin{scope}
\path[clip] (  0.00,  0.00) rectangle (209.58,126.47);

\path[] ( 49.41,  1.00) rectangle (184.64,  7.18);
\end{scope}
\begin{scope}
\path[clip] (  0.00,  0.00) rectangle (209.58,126.47);
\definecolor{drawColor}{RGB}{247,192,26}

\path[draw=drawColor,line width= 0.5pt,line join=round] ( 54.99,  4.09) -- ( 63.67,  4.09);
\end{scope}
\begin{scope}
\path[clip] (  0.00,  0.00) rectangle (209.58,126.47);
\definecolor{drawColor}{RGB}{247,192,26}
\definecolor{fillColor}{RGB}{247,192,26}

\path[draw=drawColor,line width= 0.4pt,line join=round,line cap=round,fill=fillColor] ( 59.33,  4.09) circle (  1.53);
\end{scope}
\begin{scope}
\path[clip] (  0.00,  0.00) rectangle (209.58,126.47);
\definecolor{drawColor}{RGB}{78,155,133}

\path[draw=drawColor,line width= 0.5pt,dash pattern=on 4pt off 4pt ,line join=round] (117.09,  4.09) -- (125.76,  4.09);
\end{scope}
\begin{scope}
\path[clip] (  0.00,  0.00) rectangle (209.58,126.47);
\definecolor{fillColor}{RGB}{78,155,133}

\path[fill=fillColor] (119.89,  2.56) --
	(122.96,  2.56) --
	(122.96,  5.62) --
	(119.89,  5.62) --
	cycle;
\end{scope}
\begin{scope}
\path[clip] (  0.00,  0.00) rectangle (209.58,126.47);
\definecolor{drawColor}{RGB}{0,0,0}

\node[text=drawColor,anchor=base west,inner sep=0pt, outer sep=0pt, scale=  0.70] at ( 69.25,  1.68) {LVE (ACP)};
\end{scope}
\begin{scope}
\path[clip] (  0.00,  0.00) rectangle (209.58,126.47);
\definecolor{drawColor}{RGB}{0,0,0}

\node[text=drawColor,anchor=base west,inner sep=0pt, outer sep=0pt, scale=  0.70] at (131.34,  1.68) {LVE (ACP $\pm \varepsilon$)};
\end{scope}
\end{tikzpicture}}
		\caption{Compression ratio for $\varepsilon = 0.001$.}
		\label{fig:commutative_approx_plot_comp_k=3_eps=0.001}
	\end{subfigure}
	\begin{subfigure}[t]{0.49\linewidth}
		\centering
		\resizebox{\linewidth}{!}{
\begin{tikzpicture}[x=1pt,y=1pt]
\definecolor{fillColor}{RGB}{255,255,255}
\path[use as bounding box,fill=fillColor,fill opacity=0.00] (0,0) rectangle (209.58,126.47);
\begin{scope}
\path[clip] (  0.00,  0.00) rectangle (209.58,126.47);
\definecolor{drawColor}{RGB}{255,255,255}
\definecolor{fillColor}{RGB}{255,255,255}

\path[draw=drawColor,line width= 0.5pt,line join=round,line cap=round,fill=fillColor] (  0.00,  0.00) rectangle (209.58,126.47);
\end{scope}
\begin{scope}
\path[clip] ( 30.47, 30.49) rectangle (203.58,124.47);
\definecolor{fillColor}{RGB}{255,255,255}

\path[fill=fillColor] ( 30.47, 30.49) rectangle (203.58,124.47);
\definecolor{drawColor}{RGB}{247,192,26}

\path[draw=drawColor,line width= 0.5pt,line join=round] ( 38.34,116.13) --
	( 55.82,107.79) --
	( 90.80,101.17) --
	(125.77, 96.73) --
	(160.74, 96.51) --
	(195.71, 96.37);
\definecolor{drawColor}{RGB}{78,155,133}

\path[draw=drawColor,line width= 0.5pt,dash pattern=on 4pt off 4pt ,line join=round] ( 38.34, 89.72) --
	( 55.82, 64.60) --
	( 90.80, 49.89) --
	(125.77, 44.53) --
	(160.74, 42.21) --
	(195.71, 40.78);
\definecolor{drawColor}{RGB}{247,192,26}
\definecolor{fillColor}{RGB}{247,192,26}

\path[draw=drawColor,line width= 0.4pt,line join=round,line cap=round,fill=fillColor] ( 38.34,116.13) circle (  1.53);
\definecolor{fillColor}{RGB}{78,155,133}

\path[fill=fillColor] ( 36.80, 88.19) --
	( 39.87, 88.19) --
	( 39.87, 91.26) --
	( 36.80, 91.26) --
	cycle;
\definecolor{fillColor}{RGB}{247,192,26}

\path[draw=drawColor,line width= 0.4pt,line join=round,line cap=round,fill=fillColor] ( 55.82,107.79) circle (  1.53);
\definecolor{fillColor}{RGB}{78,155,133}

\path[fill=fillColor] ( 54.29, 63.07) --
	( 57.36, 63.07) --
	( 57.36, 66.13) --
	( 54.29, 66.13) --
	cycle;
\definecolor{fillColor}{RGB}{247,192,26}

\path[draw=drawColor,line width= 0.4pt,line join=round,line cap=round,fill=fillColor] ( 90.80,101.17) circle (  1.53);
\definecolor{fillColor}{RGB}{78,155,133}

\path[fill=fillColor] ( 89.26, 48.35) --
	( 92.33, 48.35) --
	( 92.33, 51.42) --
	( 89.26, 51.42) --
	cycle;
\definecolor{fillColor}{RGB}{247,192,26}

\path[draw=drawColor,line width= 0.4pt,line join=round,line cap=round,fill=fillColor] (125.77, 96.73) circle (  1.53);
\definecolor{fillColor}{RGB}{78,155,133}

\path[fill=fillColor] (124.23, 42.99) --
	(127.30, 42.99) --
	(127.30, 46.06) --
	(124.23, 46.06) --
	cycle;
\definecolor{fillColor}{RGB}{247,192,26}

\path[draw=drawColor,line width= 0.4pt,line join=round,line cap=round,fill=fillColor] (160.74, 96.51) circle (  1.53);
\definecolor{fillColor}{RGB}{78,155,133}

\path[fill=fillColor] (159.21, 40.68) --
	(162.27, 40.68) --
	(162.27, 43.74) --
	(159.21, 43.74) --
	cycle;
\definecolor{fillColor}{RGB}{247,192,26}

\path[draw=drawColor,line width= 0.4pt,line join=round,line cap=round,fill=fillColor] (195.71, 96.37) circle (  1.53);
\definecolor{fillColor}{RGB}{78,155,133}

\path[fill=fillColor] (194.18, 39.25) --
	(197.25, 39.25) --
	(197.25, 42.32) --
	(194.18, 42.32) --
	cycle;
\end{scope}
\begin{scope}
\path[clip] (  0.00,  0.00) rectangle (209.58,126.47);
\definecolor{drawColor}{RGB}{0,0,0}

\path[draw=drawColor,line width= 0.5pt,line join=round] ( 30.47, 30.49) --
	( 30.47,124.47);

\path[draw=drawColor,line width= 0.5pt,line join=round] ( 31.60,122.50) --
	( 30.47,124.47) --
	( 29.33,122.50);
\end{scope}
\begin{scope}
\path[clip] (  0.00,  0.00) rectangle (209.58,126.47);
\definecolor{drawColor}{gray}{0.30}

\node[text=drawColor,anchor=base east,inner sep=0pt, outer sep=0pt, scale=  0.80] at ( 26.42, 32.01) {0.00};

\node[text=drawColor,anchor=base east,inner sep=0pt, outer sep=0pt, scale=  0.80] at ( 26.42, 52.35) {0.25};

\node[text=drawColor,anchor=base east,inner sep=0pt, outer sep=0pt, scale=  0.80] at ( 26.42, 72.69) {0.50};

\node[text=drawColor,anchor=base east,inner sep=0pt, outer sep=0pt, scale=  0.80] at ( 26.42, 93.04) {0.75};

\node[text=drawColor,anchor=base east,inner sep=0pt, outer sep=0pt, scale=  0.80] at ( 26.42,113.38) {1.00};
\end{scope}
\begin{scope}
\path[clip] (  0.00,  0.00) rectangle (209.58,126.47);
\definecolor{drawColor}{gray}{0.20}

\path[draw=drawColor,line width= 0.5pt,line join=round] ( 28.22, 34.77) --
	( 30.47, 34.77);

\path[draw=drawColor,line width= 0.5pt,line join=round] ( 28.22, 55.11) --
	( 30.47, 55.11);

\path[draw=drawColor,line width= 0.5pt,line join=round] ( 28.22, 75.45) --
	( 30.47, 75.45);

\path[draw=drawColor,line width= 0.5pt,line join=round] ( 28.22, 95.79) --
	( 30.47, 95.79);

\path[draw=drawColor,line width= 0.5pt,line join=round] ( 28.22,116.13) --
	( 30.47,116.13);
\end{scope}
\begin{scope}
\path[clip] (  0.00,  0.00) rectangle (209.58,126.47);
\definecolor{drawColor}{RGB}{0,0,0}

\path[draw=drawColor,line width= 0.5pt,line join=round] ( 30.47, 30.49) --
	(203.58, 30.49);

\path[draw=drawColor,line width= 0.5pt,line join=round] (201.61, 29.36) --
	(203.58, 30.49) --
	(201.61, 31.63);
\end{scope}
\begin{scope}
\path[clip] (  0.00,  0.00) rectangle (209.58,126.47);
\definecolor{drawColor}{gray}{0.20}

\path[draw=drawColor,line width= 0.5pt,line join=round] ( 38.34, 28.24) --
	( 38.34, 30.49);

\path[draw=drawColor,line width= 0.5pt,line join=round] ( 55.82, 28.24) --
	( 55.82, 30.49);

\path[draw=drawColor,line width= 0.5pt,line join=round] ( 90.80, 28.24) --
	( 90.80, 30.49);

\path[draw=drawColor,line width= 0.5pt,line join=round] (125.77, 28.24) --
	(125.77, 30.49);

\path[draw=drawColor,line width= 0.5pt,line join=round] (160.74, 28.24) --
	(160.74, 30.49);

\path[draw=drawColor,line width= 0.5pt,line join=round] (195.71, 28.24) --
	(195.71, 30.49);
\end{scope}
\begin{scope}
\path[clip] (  0.00,  0.00) rectangle (209.58,126.47);
\definecolor{drawColor}{gray}{0.30}

\node[text=drawColor,anchor=base,inner sep=0pt, outer sep=0pt, scale=  0.80] at ( 38.34, 20.93) {2};

\node[text=drawColor,anchor=base,inner sep=0pt, outer sep=0pt, scale=  0.80] at ( 55.82, 20.93) {4};

\node[text=drawColor,anchor=base,inner sep=0pt, outer sep=0pt, scale=  0.80] at ( 90.80, 20.93) {8};

\node[text=drawColor,anchor=base,inner sep=0pt, outer sep=0pt, scale=  0.80] at (125.77, 20.93) {12};

\node[text=drawColor,anchor=base,inner sep=0pt, outer sep=0pt, scale=  0.80] at (160.74, 20.93) {16};

\node[text=drawColor,anchor=base,inner sep=0pt, outer sep=0pt, scale=  0.80] at (195.71, 20.93) {20};
\end{scope}
\begin{scope}
\path[clip] (  0.00,  0.00) rectangle (209.58,126.47);
\definecolor{drawColor}{RGB}{0,0,0}

\node[text=drawColor,anchor=base,inner sep=0pt, outer sep=0pt, scale=  0.90] at (117.02, 10.93) {domain size $d$};
\end{scope}
\begin{scope}
\path[clip] (  0.00,  0.00) rectangle (209.58,126.47);
\definecolor{drawColor}{RGB}{0,0,0}

\node[text=drawColor,rotate= 90.00,anchor=base,inner sep=0pt, outer sep=0pt, scale=  0.90] at (  8.20, 77.48) {$|M'| \mathbin{/} |M|$};
\end{scope}
\begin{scope}
\path[clip] (  0.00,  0.00) rectangle (209.58,126.47);

\path[] ( 49.41,  1.00) rectangle (184.64,  7.18);
\end{scope}
\begin{scope}
\path[clip] (  0.00,  0.00) rectangle (209.58,126.47);
\definecolor{drawColor}{RGB}{247,192,26}

\path[draw=drawColor,line width= 0.5pt,line join=round] ( 54.99,  4.09) -- ( 63.67,  4.09);
\end{scope}
\begin{scope}
\path[clip] (  0.00,  0.00) rectangle (209.58,126.47);
\definecolor{drawColor}{RGB}{247,192,26}
\definecolor{fillColor}{RGB}{247,192,26}

\path[draw=drawColor,line width= 0.4pt,line join=round,line cap=round,fill=fillColor] ( 59.33,  4.09) circle (  1.53);
\end{scope}
\begin{scope}
\path[clip] (  0.00,  0.00) rectangle (209.58,126.47);
\definecolor{drawColor}{RGB}{78,155,133}

\path[draw=drawColor,line width= 0.5pt,dash pattern=on 4pt off 4pt ,line join=round] (117.09,  4.09) -- (125.76,  4.09);
\end{scope}
\begin{scope}
\path[clip] (  0.00,  0.00) rectangle (209.58,126.47);
\definecolor{fillColor}{RGB}{78,155,133}

\path[fill=fillColor] (119.89,  2.56) --
	(122.96,  2.56) --
	(122.96,  5.62) --
	(119.89,  5.62) --
	cycle;
\end{scope}
\begin{scope}
\path[clip] (  0.00,  0.00) rectangle (209.58,126.47);
\definecolor{drawColor}{RGB}{0,0,0}

\node[text=drawColor,anchor=base west,inner sep=0pt, outer sep=0pt, scale=  0.70] at ( 69.25,  1.68) {LVE (ACP)};
\end{scope}
\begin{scope}
\path[clip] (  0.00,  0.00) rectangle (209.58,126.47);
\definecolor{drawColor}{RGB}{0,0,0}

\node[text=drawColor,anchor=base west,inner sep=0pt, outer sep=0pt, scale=  0.70] at (131.34,  1.68) {LVE (ACP $\pm \varepsilon$)};
\end{scope}
\end{tikzpicture}}
		\caption{Compression ratio for $\varepsilon = 0.1$.}
		\label{fig:commutative_approx_plot_comp_k=3_eps=0.1}
	\end{subfigure}
	\caption{Compression ratio $\abs{M'} \mathbin{/} \abs{M}$ between the compressed model $M'$ and the original model $M$ for input \acp{fg} containing $k = 3$ $\varepsilon$-commutative factors with (a) $\varepsilon = 0.001$ and (b) $\varepsilon = 0.1$.}
	\label{fig:commutative_approx_plot_comp_k=3}
\end{figure}
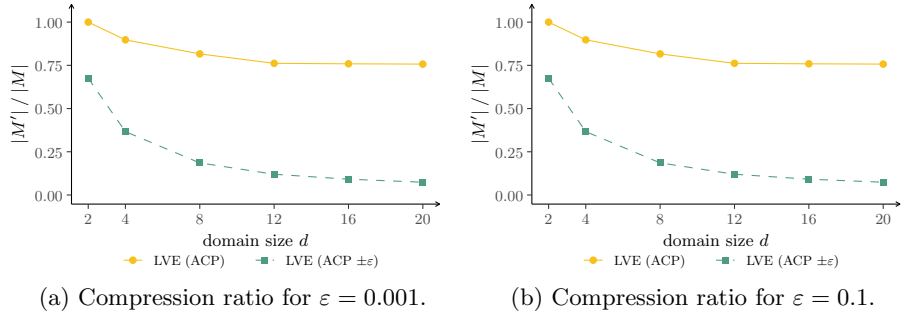

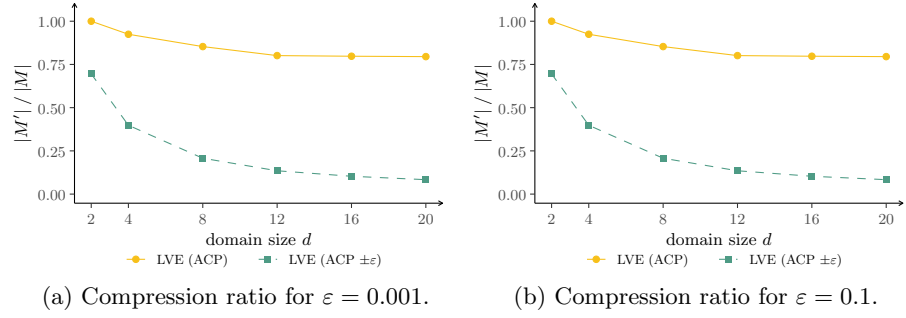
\begin{figure}[t]
	\centering
	\begin{subfigure}[t]{0.49\linewidth}
		\centering
		\resizebox{\linewidth}{!}{
\begin{tikzpicture}[x=1pt,y=1pt]
\definecolor{fillColor}{RGB}{255,255,255}
\path[use as bounding box,fill=fillColor,fill opacity=0.00] (0,0) rectangle (209.58,126.47);
\begin{scope}
\path[clip] (  0.00,  0.00) rectangle (209.58,126.47);
\definecolor{drawColor}{RGB}{255,255,255}
\definecolor{fillColor}{RGB}{255,255,255}

\path[draw=drawColor,line width= 0.5pt,line join=round,line cap=round,fill=fillColor] (  0.00,  0.00) rectangle (209.58,126.47);
\end{scope}
\begin{scope}
\path[clip] ( 30.47, 30.49) rectangle (203.58,124.47);
\definecolor{fillColor}{RGB}{255,255,255}

\path[fill=fillColor] ( 30.47, 30.49) rectangle (203.58,124.47);
\definecolor{drawColor}{RGB}{247,192,26}

\path[draw=drawColor,line width= 0.5pt,line join=round] ( 38.34,116.13) --
	( 55.82,109.99) --
	( 90.80,104.19) --
	(125.77, 99.95) --
	(160.74, 99.65) --
	(195.71, 99.47);
\definecolor{drawColor}{RGB}{78,155,133}

\path[draw=drawColor,line width= 0.5pt,dash pattern=on 4pt off 4pt ,line join=round] ( 38.34, 91.35) --
	( 55.82, 67.16) --
	( 90.80, 51.65) --
	(125.77, 45.75) --
	(160.74, 43.16) --
	(195.71, 41.56);
\definecolor{drawColor}{RGB}{247,192,26}
\definecolor{fillColor}{RGB}{247,192,26}

\path[draw=drawColor,line width= 0.4pt,line join=round,line cap=round,fill=fillColor] ( 38.34,116.13) circle (  1.53);
\definecolor{fillColor}{RGB}{78,155,133}

\path[fill=fillColor] ( 36.80, 89.81) --
	( 39.87, 89.81) --
	( 39.87, 92.88) --
	( 36.80, 92.88) --
	cycle;
\definecolor{fillColor}{RGB}{247,192,26}

\path[draw=drawColor,line width= 0.4pt,line join=round,line cap=round,fill=fillColor] ( 55.82,109.99) circle (  1.53);
\definecolor{fillColor}{RGB}{78,155,133}

\path[fill=fillColor] ( 54.29, 65.63) --
	( 57.36, 65.63) --
	( 57.36, 68.69) --
	( 54.29, 68.69) --
	cycle;
\definecolor{fillColor}{RGB}{247,192,26}

\path[draw=drawColor,line width= 0.4pt,line join=round,line cap=round,fill=fillColor] ( 90.80,104.19) circle (  1.53);
\definecolor{fillColor}{RGB}{78,155,133}

\path[fill=fillColor] ( 89.26, 50.11) --
	( 92.33, 50.11) --
	( 92.33, 53.18) --
	( 89.26, 53.18) --
	cycle;
\definecolor{fillColor}{RGB}{247,192,26}

\path[draw=drawColor,line width= 0.4pt,line join=round,line cap=round,fill=fillColor] (125.77, 99.95) circle (  1.53);
\definecolor{fillColor}{RGB}{78,155,133}

\path[fill=fillColor] (124.23, 44.21) --
	(127.30, 44.21) --
	(127.30, 47.28) --
	(124.23, 47.28) --
	cycle;
\definecolor{fillColor}{RGB}{247,192,26}

\path[draw=drawColor,line width= 0.4pt,line join=round,line cap=round,fill=fillColor] (160.74, 99.65) circle (  1.53);
\definecolor{fillColor}{RGB}{78,155,133}

\path[fill=fillColor] (159.21, 41.62) --
	(162.27, 41.62) --
	(162.27, 44.69) --
	(159.21, 44.69) --
	cycle;
\definecolor{fillColor}{RGB}{247,192,26}

\path[draw=drawColor,line width= 0.4pt,line join=round,line cap=round,fill=fillColor] (195.71, 99.47) circle (  1.53);
\definecolor{fillColor}{RGB}{78,155,133}

\path[fill=fillColor] (194.18, 40.02) --
	(197.25, 40.02) --
	(197.25, 43.09) --
	(194.18, 43.09) --
	cycle;
\end{scope}
\begin{scope}
\path[clip] (  0.00,  0.00) rectangle (209.58,126.47);
\definecolor{drawColor}{RGB}{0,0,0}

\path[draw=drawColor,line width= 0.5pt,line join=round] ( 30.47, 30.49) --
	( 30.47,124.47);

\path[draw=drawColor,line width= 0.5pt,line join=round] ( 31.60,122.50) --
	( 30.47,124.47) --
	( 29.33,122.50);
\end{scope}
\begin{scope}
\path[clip] (  0.00,  0.00) rectangle (209.58,126.47);
\definecolor{drawColor}{gray}{0.30}

\node[text=drawColor,anchor=base east,inner sep=0pt, outer sep=0pt, scale=  0.80] at ( 26.42, 32.01) {0.00};

\node[text=drawColor,anchor=base east,inner sep=0pt, outer sep=0pt, scale=  0.80] at ( 26.42, 52.35) {0.25};

\node[text=drawColor,anchor=base east,inner sep=0pt, outer sep=0pt, scale=  0.80] at ( 26.42, 72.69) {0.50};

\node[text=drawColor,anchor=base east,inner sep=0pt, outer sep=0pt, scale=  0.80] at ( 26.42, 93.04) {0.75};

\node[text=drawColor,anchor=base east,inner sep=0pt, outer sep=0pt, scale=  0.80] at ( 26.42,113.38) {1.00};
\end{scope}
\begin{scope}
\path[clip] (  0.00,  0.00) rectangle (209.58,126.47);
\definecolor{drawColor}{gray}{0.20}

\path[draw=drawColor,line width= 0.5pt,line join=round] ( 28.22, 34.77) --
	( 30.47, 34.77);

\path[draw=drawColor,line width= 0.5pt,line join=round] ( 28.22, 55.11) --
	( 30.47, 55.11);

\path[draw=drawColor,line width= 0.5pt,line join=round] ( 28.22, 75.45) --
	( 30.47, 75.45);

\path[draw=drawColor,line width= 0.5pt,line join=round] ( 28.22, 95.79) --
	( 30.47, 95.79);

\path[draw=drawColor,line width= 0.5pt,line join=round] ( 28.22,116.13) --
	( 30.47,116.13);
\end{scope}
\begin{scope}
\path[clip] (  0.00,  0.00) rectangle (209.58,126.47);
\definecolor{drawColor}{RGB}{0,0,0}

\path[draw=drawColor,line width= 0.5pt,line join=round] ( 30.47, 30.49) --
	(203.58, 30.49);

\path[draw=drawColor,line width= 0.5pt,line join=round] (201.61, 29.36) --
	(203.58, 30.49) --
	(201.61, 31.63);
\end{scope}
\begin{scope}
\path[clip] (  0.00,  0.00) rectangle (209.58,126.47);
\definecolor{drawColor}{gray}{0.20}

\path[draw=drawColor,line width= 0.5pt,line join=round] ( 38.34, 28.24) --
	( 38.34, 30.49);

\path[draw=drawColor,line width= 0.5pt,line join=round] ( 55.82, 28.24) --
	( 55.82, 30.49);

\path[draw=drawColor,line width= 0.5pt,line join=round] ( 90.80, 28.24) --
	( 90.80, 30.49);

\path[draw=drawColor,line width= 0.5pt,line join=round] (125.77, 28.24) --
	(125.77, 30.49);

\path[draw=drawColor,line width= 0.5pt,line join=round] (160.74, 28.24) --
	(160.74, 30.49);

\path[draw=drawColor,line width= 0.5pt,line join=round] (195.71, 28.24) --
	(195.71, 30.49);
\end{scope}
\begin{scope}
\path[clip] (  0.00,  0.00) rectangle (209.58,126.47);
\definecolor{drawColor}{gray}{0.30}

\node[text=drawColor,anchor=base,inner sep=0pt, outer sep=0pt, scale=  0.80] at ( 38.34, 20.93) {2};

\node[text=drawColor,anchor=base,inner sep=0pt, outer sep=0pt, scale=  0.80] at ( 55.82, 20.93) {4};

\node[text=drawColor,anchor=base,inner sep=0pt, outer sep=0pt, scale=  0.80] at ( 90.80, 20.93) {8};

\node[text=drawColor,anchor=base,inner sep=0pt, outer sep=0pt, scale=  0.80] at (125.77, 20.93) {12};

\node[text=drawColor,anchor=base,inner sep=0pt, outer sep=0pt, scale=  0.80] at (160.74, 20.93) {16};

\node[text=drawColor,anchor=base,inner sep=0pt, outer sep=0pt, scale=  0.80] at (195.71, 20.93) {20};
\end{scope}
\begin{scope}
\path[clip] (  0.00,  0.00) rectangle (209.58,126.47);
\definecolor{drawColor}{RGB}{0,0,0}

\node[text=drawColor,anchor=base,inner sep=0pt, outer sep=0pt, scale=  0.90] at (117.02, 10.93) {domain size $d$};
\end{scope}
\begin{scope}
\path[clip] (  0.00,  0.00) rectangle (209.58,126.47);
\definecolor{drawColor}{RGB}{0,0,0}

\node[text=drawColor,rotate= 90.00,anchor=base,inner sep=0pt, outer sep=0pt, scale=  0.90] at (  8.20, 77.48) {$|M'| \mathbin{/} |M|$};
\end{scope}
\begin{scope}
\path[clip] (  0.00,  0.00) rectangle (209.58,126.47);

\path[] ( 49.41,  1.00) rectangle (184.64,  7.18);
\end{scope}
\begin{scope}
\path[clip] (  0.00,  0.00) rectangle (209.58,126.47);
\definecolor{drawColor}{RGB}{247,192,26}

\path[draw=drawColor,line width= 0.5pt,line join=round] ( 54.99,  4.09) -- ( 63.67,  4.09);
\end{scope}
\begin{scope}
\path[clip] (  0.00,  0.00) rectangle (209.58,126.47);
\definecolor{drawColor}{RGB}{247,192,26}
\definecolor{fillColor}{RGB}{247,192,26}

\path[draw=drawColor,line width= 0.4pt,line join=round,line cap=round,fill=fillColor] ( 59.33,  4.09) circle (  1.53);
\end{scope}
\begin{scope}
\path[clip] (  0.00,  0.00) rectangle (209.58,126.47);
\definecolor{drawColor}{RGB}{78,155,133}

\path[draw=drawColor,line width= 0.5pt,dash pattern=on 4pt off 4pt ,line join=round] (117.09,  4.09) -- (125.76,  4.09);
\end{scope}
\begin{scope}
\path[clip] (  0.00,  0.00) rectangle (209.58,126.47);
\definecolor{fillColor}{RGB}{78,155,133}

\path[fill=fillColor] (119.89,  2.56) --
	(122.96,  2.56) --
	(122.96,  5.62) --
	(119.89,  5.62) --
	cycle;
\end{scope}
\begin{scope}
\path[clip] (  0.00,  0.00) rectangle (209.58,126.47);
\definecolor{drawColor}{RGB}{0,0,0}

\node[text=drawColor,anchor=base west,inner sep=0pt, outer sep=0pt, scale=  0.70] at ( 69.25,  1.68) {LVE (ACP)};
\end{scope}
\begin{scope}
\path[clip] (  0.00,  0.00) rectangle (209.58,126.47);
\definecolor{drawColor}{RGB}{0,0,0}

\node[text=drawColor,anchor=base west,inner sep=0pt, outer sep=0pt, scale=  0.70] at (131.34,  1.68) {LVE (ACP $\pm \varepsilon$)};
\end{scope}
\end{tikzpicture}}
		\caption{Compression ratio for $\varepsilon = 0.001$.}
		\label{fig:commutative_approx_plot_comp_k=7_eps=0.001}
	\end{subfigure}
	\begin{subfigure}[t]{0.49\linewidth}
		\centering
		\resizebox{\linewidth}{!}{
\begin{tikzpicture}[x=1pt,y=1pt]
\definecolor{fillColor}{RGB}{255,255,255}
\path[use as bounding box,fill=fillColor,fill opacity=0.00] (0,0) rectangle (209.58,126.47);
\begin{scope}
\path[clip] (  0.00,  0.00) rectangle (209.58,126.47);
\definecolor{drawColor}{RGB}{255,255,255}
\definecolor{fillColor}{RGB}{255,255,255}

\path[draw=drawColor,line width= 0.5pt,line join=round,line cap=round,fill=fillColor] (  0.00,  0.00) rectangle (209.58,126.47);
\end{scope}
\begin{scope}
\path[clip] ( 30.47, 30.49) rectangle (203.58,124.47);
\definecolor{fillColor}{RGB}{255,255,255}

\path[fill=fillColor] ( 30.47, 30.49) rectangle (203.58,124.47);
\definecolor{drawColor}{RGB}{247,192,26}

\path[draw=drawColor,line width= 0.5pt,line join=round] ( 38.34,116.13) --
	( 55.82,109.99) --
	( 90.80,104.19) --
	(125.77, 99.95) --
	(160.74, 99.65) --
	(195.71, 99.47);
\definecolor{drawColor}{RGB}{78,155,133}

\path[draw=drawColor,line width= 0.5pt,dash pattern=on 4pt off 4pt ,line join=round] ( 38.34, 91.35) --
	( 55.82, 67.16) --
	( 90.80, 51.65) --
	(125.77, 45.75) --
	(160.74, 43.16) --
	(195.71, 41.56);
\definecolor{drawColor}{RGB}{247,192,26}
\definecolor{fillColor}{RGB}{247,192,26}

\path[draw=drawColor,line width= 0.4pt,line join=round,line cap=round,fill=fillColor] ( 38.34,116.13) circle (  1.53);
\definecolor{fillColor}{RGB}{78,155,133}

\path[fill=fillColor] ( 36.80, 89.81) --
	( 39.87, 89.81) --
	( 39.87, 92.88) --
	( 36.80, 92.88) --
	cycle;
\definecolor{fillColor}{RGB}{247,192,26}

\path[draw=drawColor,line width= 0.4pt,line join=round,line cap=round,fill=fillColor] ( 55.82,109.99) circle (  1.53);
\definecolor{fillColor}{RGB}{78,155,133}

\path[fill=fillColor] ( 54.29, 65.63) --
	( 57.36, 65.63) --
	( 57.36, 68.69) --
	( 54.29, 68.69) --
	cycle;
\definecolor{fillColor}{RGB}{247,192,26}

\path[draw=drawColor,line width= 0.4pt,line join=round,line cap=round,fill=fillColor] ( 90.80,104.19) circle (  1.53);
\definecolor{fillColor}{RGB}{78,155,133}

\path[fill=fillColor] ( 89.26, 50.11) --
	( 92.33, 50.11) --
	( 92.33, 53.18) --
	( 89.26, 53.18) --
	cycle;
\definecolor{fillColor}{RGB}{247,192,26}

\path[draw=drawColor,line width= 0.4pt,line join=round,line cap=round,fill=fillColor] (125.77, 99.95) circle (  1.53);
\definecolor{fillColor}{RGB}{78,155,133}

\path[fill=fillColor] (124.23, 44.21) --
	(127.30, 44.21) --
	(127.30, 47.28) --
	(124.23, 47.28) --
	cycle;
\definecolor{fillColor}{RGB}{247,192,26}

\path[draw=drawColor,line width= 0.4pt,line join=round,line cap=round,fill=fillColor] (160.74, 99.65) circle (  1.53);
\definecolor{fillColor}{RGB}{78,155,133}

\path[fill=fillColor] (159.21, 41.62) --
	(162.27, 41.62) --
	(162.27, 44.69) --
	(159.21, 44.69) --
	cycle;
\definecolor{fillColor}{RGB}{247,192,26}

\path[draw=drawColor,line width= 0.4pt,line join=round,line cap=round,fill=fillColor] (195.71, 99.47) circle (  1.53);
\definecolor{fillColor}{RGB}{78,155,133}

\path[fill=fillColor] (194.18, 40.02) --
	(197.25, 40.02) --
	(197.25, 43.09) --
	(194.18, 43.09) --
	cycle;
\end{scope}
\begin{scope}
\path[clip] (  0.00,  0.00) rectangle (209.58,126.47);
\definecolor{drawColor}{RGB}{0,0,0}

\path[draw=drawColor,line width= 0.5pt,line join=round] ( 30.47, 30.49) --
	( 30.47,124.47);

\path[draw=drawColor,line width= 0.5pt,line join=round] ( 31.60,122.50) --
	( 30.47,124.47) --
	( 29.33,122.50);
\end{scope}
\begin{scope}
\path[clip] (  0.00,  0.00) rectangle (209.58,126.47);
\definecolor{drawColor}{gray}{0.30}

\node[text=drawColor,anchor=base east,inner sep=0pt, outer sep=0pt, scale=  0.80] at ( 26.42, 32.01) {0.00};

\node[text=drawColor,anchor=base east,inner sep=0pt, outer sep=0pt, scale=  0.80] at ( 26.42, 52.35) {0.25};

\node[text=drawColor,anchor=base east,inner sep=0pt, outer sep=0pt, scale=  0.80] at ( 26.42, 72.69) {0.50};

\node[text=drawColor,anchor=base east,inner sep=0pt, outer sep=0pt, scale=  0.80] at ( 26.42, 93.04) {0.75};

\node[text=drawColor,anchor=base east,inner sep=0pt, outer sep=0pt, scale=  0.80] at ( 26.42,113.38) {1.00};
\end{scope}
\begin{scope}
\path[clip] (  0.00,  0.00) rectangle (209.58,126.47);
\definecolor{drawColor}{gray}{0.20}

\path[draw=drawColor,line width= 0.5pt,line join=round] ( 28.22, 34.77) --
	( 30.47, 34.77);

\path[draw=drawColor,line width= 0.5pt,line join=round] ( 28.22, 55.11) --
	( 30.47, 55.11);

\path[draw=drawColor,line width= 0.5pt,line join=round] ( 28.22, 75.45) --
	( 30.47, 75.45);

\path[draw=drawColor,line width= 0.5pt,line join=round] ( 28.22, 95.79) --
	( 30.47, 95.79);

\path[draw=drawColor,line width= 0.5pt,line join=round] ( 28.22,116.13) --
	( 30.47,116.13);
\end{scope}
\begin{scope}
\path[clip] (  0.00,  0.00) rectangle (209.58,126.47);
\definecolor{drawColor}{RGB}{0,0,0}

\path[draw=drawColor,line width= 0.5pt,line join=round] ( 30.47, 30.49) --
	(203.58, 30.49);

\path[draw=drawColor,line width= 0.5pt,line join=round] (201.61, 29.36) --
	(203.58, 30.49) --
	(201.61, 31.63);
\end{scope}
\begin{scope}
\path[clip] (  0.00,  0.00) rectangle (209.58,126.47);
\definecolor{drawColor}{gray}{0.20}

\path[draw=drawColor,line width= 0.5pt,line join=round] ( 38.34, 28.24) --
	( 38.34, 30.49);

\path[draw=drawColor,line width= 0.5pt,line join=round] ( 55.82, 28.24) --
	( 55.82, 30.49);

\path[draw=drawColor,line width= 0.5pt,line join=round] ( 90.80, 28.24) --
	( 90.80, 30.49);

\path[draw=drawColor,line width= 0.5pt,line join=round] (125.77, 28.24) --
	(125.77, 30.49);

\path[draw=drawColor,line width= 0.5pt,line join=round] (160.74, 28.24) --
	(160.74, 30.49);

\path[draw=drawColor,line width= 0.5pt,line join=round] (195.71, 28.24) --
	(195.71, 30.49);
\end{scope}
\begin{scope}
\path[clip] (  0.00,  0.00) rectangle (209.58,126.47);
\definecolor{drawColor}{gray}{0.30}

\node[text=drawColor,anchor=base,inner sep=0pt, outer sep=0pt, scale=  0.80] at ( 38.34, 20.93) {2};

\node[text=drawColor,anchor=base,inner sep=0pt, outer sep=0pt, scale=  0.80] at ( 55.82, 20.93) {4};

\node[text=drawColor,anchor=base,inner sep=0pt, outer sep=0pt, scale=  0.80] at ( 90.80, 20.93) {8};

\node[text=drawColor,anchor=base,inner sep=0pt, outer sep=0pt, scale=  0.80] at (125.77, 20.93) {12};

\node[text=drawColor,anchor=base,inner sep=0pt, outer sep=0pt, scale=  0.80] at (160.74, 20.93) {16};

\node[text=drawColor,anchor=base,inner sep=0pt, outer sep=0pt, scale=  0.80] at (195.71, 20.93) {20};
\end{scope}
\begin{scope}
\path[clip] (  0.00,  0.00) rectangle (209.58,126.47);
\definecolor{drawColor}{RGB}{0,0,0}

\node[text=drawColor,anchor=base,inner sep=0pt, outer sep=0pt, scale=  0.90] at (117.02, 10.93) {domain size $d$};
\end{scope}
\begin{scope}
\path[clip] (  0.00,  0.00) rectangle (209.58,126.47);
\definecolor{drawColor}{RGB}{0,0,0}

\node[text=drawColor,rotate= 90.00,anchor=base,inner sep=0pt, outer sep=0pt, scale=  0.90] at (  8.20, 77.48) {$|M'| \mathbin{/} |M|$};
\end{scope}
\begin{scope}
\path[clip] (  0.00,  0.00) rectangle (209.58,126.47);

\path[] ( 49.41,  1.00) rectangle (184.64,  7.18);
\end{scope}
\begin{scope}
\path[clip] (  0.00,  0.00) rectangle (209.58,126.47);
\definecolor{drawColor}{RGB}{247,192,26}

\path[draw=drawColor,line width= 0.5pt,line join=round] ( 54.99,  4.09) -- ( 63.67,  4.09);
\end{scope}
\begin{scope}
\path[clip] (  0.00,  0.00) rectangle (209.58,126.47);
\definecolor{drawColor}{RGB}{247,192,26}
\definecolor{fillColor}{RGB}{247,192,26}

\path[draw=drawColor,line width= 0.4pt,line join=round,line cap=round,fill=fillColor] ( 59.33,  4.09) circle (  1.53);
\end{scope}
\begin{scope}
\path[clip] (  0.00,  0.00) rectangle (209.58,126.47);
\definecolor{drawColor}{RGB}{78,155,133}

\path[draw=drawColor,line width= 0.5pt,dash pattern=on 4pt off 4pt ,line join=round] (117.09,  4.09) -- (125.76,  4.09);
\end{scope}
\begin{scope}
\path[clip] (  0.00,  0.00) rectangle (209.58,126.47);
\definecolor{fillColor}{RGB}{78,155,133}

\path[fill=fillColor] (119.89,  2.56) --
	(122.96,  2.56) --
	(122.96,  5.62) --
	(119.89,  5.62) --
	cycle;
\end{scope}
\begin{scope}
\path[clip] (  0.00,  0.00) rectangle (209.58,126.47);
\definecolor{drawColor}{RGB}{0,0,0}

\node[text=drawColor,anchor=base west,inner sep=0pt, outer sep=0pt, scale=  0.70] at ( 69.25,  1.68) {LVE (ACP)};
\end{scope}
\begin{scope}
\path[clip] (  0.00,  0.00) rectangle (209.58,126.47);
\definecolor{drawColor}{RGB}{0,0,0}

\node[text=drawColor,anchor=base west,inner sep=0pt, outer sep=0pt, scale=  0.70] at (131.34,  1.68) {LVE (ACP $\pm \varepsilon$)};
\end{scope}
\end{tikzpicture}}
		\caption{Compression ratio for $\varepsilon = 0.1$.}
		\label{fig:commutative_approx_plot_comp_k=7_eps=0.1}
	\end{subfigure}
	\caption{Compression ratio $\abs{M'} \mathbin{/} \abs{M}$ between the compressed model $M'$ and the original model $M$ for input \acp{fg} containing $k = 7$ $\varepsilon$-commutative factors with (a) $\varepsilon = 0.001$ and (b) $\varepsilon = 0.1$.}
	\label{fig:commutative_approx_plot_comp_k=7}
\end{figure}

Finally, we examine the structural reduction achieved by the compression stemming from exact \ac{acp} and its $\varepsilon$-relaxed variant.
\Crefrange{fig:commutative_approx_plot_comp_k=1}{fig:commutative_approx_plot_comp_k=7} show the compression ratio $\abs{M'} \mathbin{/} \abs{M}$ between the compressed model $M'$ and the original model $M$ for $k \in \{1, 3, 7\}$ at the smallest and largest tolerance $\varepsilon \in \{0.001, 0.1\}$ (that is, a compression ratio of one means that no compression is achieved at all, whereas a smaller ratio indicates more compression).
The $\varepsilon$-relaxed variant attains a substantially smaller compression ratio than exact \ac{acp}, because it is able to compress the $\varepsilon$-commutative factors that exact \ac{acp} cannot compress.
For both variants, the compression ratio decreases as the domain size grows, illustrating that larger models offer more indistinguishability to be exploited.

\begin{figure}[t]
	\centering
	\begin{subfigure}[t]{0.49\linewidth}
		\centering
		\resizebox{\linewidth}{!}{
\begin{tikzpicture}[x=1pt,y=1pt]
\definecolor{fillColor}{RGB}{255,255,255}
\path[use as bounding box,fill=fillColor,fill opacity=0.00] (0,0) rectangle (209.58,126.47);
\begin{scope}
\path[clip] (  0.00,  0.00) rectangle (209.58,126.47);
\definecolor{drawColor}{RGB}{255,255,255}
\definecolor{fillColor}{RGB}{255,255,255}

\path[draw=drawColor,line width= 0.5pt,line join=round,line cap=round,fill=fillColor] (  0.00,  0.00) rectangle (209.58,126.47);
\end{scope}
\begin{scope}
\path[clip] ( 30.47, 30.49) rectangle (203.58,124.47);
\definecolor{fillColor}{RGB}{255,255,255}

\path[fill=fillColor] ( 30.47, 30.49) rectangle (203.58,124.47);
\definecolor{drawColor}{RGB}{247,192,26}

\path[draw=drawColor,line width= 0.5pt,line join=round] ( 38.34,116.13) --
	( 55.82,116.13) --
	( 90.80,116.13) --
	(125.77,116.13) --
	(160.74,116.13) --
	(195.71,116.13);
\definecolor{drawColor}{RGB}{78,155,133}

\path[draw=drawColor,line width= 0.5pt,dash pattern=on 4pt off 4pt ,line join=round] ( 38.34, 87.20) --
	( 55.82, 65.50) --
	( 90.80, 51.61) --
	(125.77, 46.37) --
	(160.74, 43.61) --
	(195.71, 41.92);
\definecolor{drawColor}{RGB}{247,192,26}
\definecolor{fillColor}{RGB}{247,192,26}

\path[draw=drawColor,line width= 0.4pt,line join=round,line cap=round,fill=fillColor] ( 38.34,116.13) circle (  1.53);
\definecolor{fillColor}{RGB}{78,155,133}

\path[fill=fillColor] ( 36.80, 85.67) --
	( 39.87, 85.67) --
	( 39.87, 88.74) --
	( 36.80, 88.74) --
	cycle;
\definecolor{fillColor}{RGB}{247,192,26}

\path[draw=drawColor,line width= 0.4pt,line join=round,line cap=round,fill=fillColor] ( 55.82,116.13) circle (  1.53);
\definecolor{fillColor}{RGB}{78,155,133}

\path[fill=fillColor] ( 54.29, 63.97) --
	( 57.36, 63.97) --
	( 57.36, 67.04) --
	( 54.29, 67.04) --
	cycle;
\definecolor{fillColor}{RGB}{247,192,26}

\path[draw=drawColor,line width= 0.4pt,line join=round,line cap=round,fill=fillColor] ( 90.80,116.13) circle (  1.53);
\definecolor{fillColor}{RGB}{78,155,133}

\path[fill=fillColor] ( 89.26, 50.07) --
	( 92.33, 50.07) --
	( 92.33, 53.14) --
	( 89.26, 53.14) --
	cycle;
\definecolor{fillColor}{RGB}{247,192,26}

\path[draw=drawColor,line width= 0.4pt,line join=round,line cap=round,fill=fillColor] (125.77,116.13) circle (  1.53);
\definecolor{fillColor}{RGB}{78,155,133}

\path[fill=fillColor] (124.23, 44.83) --
	(127.30, 44.83) --
	(127.30, 47.90) --
	(124.23, 47.90) --
	cycle;
\definecolor{fillColor}{RGB}{247,192,26}

\path[draw=drawColor,line width= 0.4pt,line join=round,line cap=round,fill=fillColor] (160.74,116.13) circle (  1.53);
\definecolor{fillColor}{RGB}{78,155,133}

\path[fill=fillColor] (159.21, 42.08) --
	(162.27, 42.08) --
	(162.27, 45.15) --
	(159.21, 45.15) --
	cycle;
\definecolor{fillColor}{RGB}{247,192,26}

\path[draw=drawColor,line width= 0.4pt,line join=round,line cap=round,fill=fillColor] (195.71,116.13) circle (  1.53);
\definecolor{fillColor}{RGB}{78,155,133}

\path[fill=fillColor] (194.18, 40.38) --
	(197.25, 40.38) --
	(197.25, 43.45) --
	(194.18, 43.45) --
	cycle;
\end{scope}
\begin{scope}
\path[clip] (  0.00,  0.00) rectangle (209.58,126.47);
\definecolor{drawColor}{RGB}{0,0,0}

\path[draw=drawColor,line width= 0.5pt,line join=round] ( 30.47, 30.49) --
	( 30.47,124.47);

\path[draw=drawColor,line width= 0.5pt,line join=round] ( 31.60,122.50) --
	( 30.47,124.47) --
	( 29.33,122.50);
\end{scope}
\begin{scope}
\path[clip] (  0.00,  0.00) rectangle (209.58,126.47);
\definecolor{drawColor}{gray}{0.30}

\node[text=drawColor,anchor=base east,inner sep=0pt, outer sep=0pt, scale=  0.80] at ( 26.42, 32.01) {0.00};

\node[text=drawColor,anchor=base east,inner sep=0pt, outer sep=0pt, scale=  0.80] at ( 26.42, 52.35) {0.25};

\node[text=drawColor,anchor=base east,inner sep=0pt, outer sep=0pt, scale=  0.80] at ( 26.42, 72.69) {0.50};

\node[text=drawColor,anchor=base east,inner sep=0pt, outer sep=0pt, scale=  0.80] at ( 26.42, 93.04) {0.75};

\node[text=drawColor,anchor=base east,inner sep=0pt, outer sep=0pt, scale=  0.80] at ( 26.42,113.38) {1.00};
\end{scope}
\begin{scope}
\path[clip] (  0.00,  0.00) rectangle (209.58,126.47);
\definecolor{drawColor}{gray}{0.20}

\path[draw=drawColor,line width= 0.5pt,line join=round] ( 28.22, 34.77) --
	( 30.47, 34.77);

\path[draw=drawColor,line width= 0.5pt,line join=round] ( 28.22, 55.11) --
	( 30.47, 55.11);

\path[draw=drawColor,line width= 0.5pt,line join=round] ( 28.22, 75.45) --
	( 30.47, 75.45);

\path[draw=drawColor,line width= 0.5pt,line join=round] ( 28.22, 95.79) --
	( 30.47, 95.79);

\path[draw=drawColor,line width= 0.5pt,line join=round] ( 28.22,116.13) --
	( 30.47,116.13);
\end{scope}
\begin{scope}
\path[clip] (  0.00,  0.00) rectangle (209.58,126.47);
\definecolor{drawColor}{RGB}{0,0,0}

\path[draw=drawColor,line width= 0.5pt,line join=round] ( 30.47, 30.49) --
	(203.58, 30.49);

\path[draw=drawColor,line width= 0.5pt,line join=round] (201.61, 29.36) --
	(203.58, 30.49) --
	(201.61, 31.63);
\end{scope}
\begin{scope}
\path[clip] (  0.00,  0.00) rectangle (209.58,126.47);
\definecolor{drawColor}{gray}{0.20}

\path[draw=drawColor,line width= 0.5pt,line join=round] ( 38.34, 28.24) --
	( 38.34, 30.49);

\path[draw=drawColor,line width= 0.5pt,line join=round] ( 55.82, 28.24) --
	( 55.82, 30.49);

\path[draw=drawColor,line width= 0.5pt,line join=round] ( 90.80, 28.24) --
	( 90.80, 30.49);

\path[draw=drawColor,line width= 0.5pt,line join=round] (125.77, 28.24) --
	(125.77, 30.49);

\path[draw=drawColor,line width= 0.5pt,line join=round] (160.74, 28.24) --
	(160.74, 30.49);

\path[draw=drawColor,line width= 0.5pt,line join=round] (195.71, 28.24) --
	(195.71, 30.49);
\end{scope}
\begin{scope}
\path[clip] (  0.00,  0.00) rectangle (209.58,126.47);
\definecolor{drawColor}{gray}{0.30}

\node[text=drawColor,anchor=base,inner sep=0pt, outer sep=0pt, scale=  0.80] at ( 38.34, 20.93) {2};

\node[text=drawColor,anchor=base,inner sep=0pt, outer sep=0pt, scale=  0.80] at ( 55.82, 20.93) {4};

\node[text=drawColor,anchor=base,inner sep=0pt, outer sep=0pt, scale=  0.80] at ( 90.80, 20.93) {8};

\node[text=drawColor,anchor=base,inner sep=0pt, outer sep=0pt, scale=  0.80] at (125.77, 20.93) {12};

\node[text=drawColor,anchor=base,inner sep=0pt, outer sep=0pt, scale=  0.80] at (160.74, 20.93) {16};

\node[text=drawColor,anchor=base,inner sep=0pt, outer sep=0pt, scale=  0.80] at (195.71, 20.93) {20};
\end{scope}
\begin{scope}
\path[clip] (  0.00,  0.00) rectangle (209.58,126.47);
\definecolor{drawColor}{RGB}{0,0,0}

\node[text=drawColor,anchor=base,inner sep=0pt, outer sep=0pt, scale=  0.90] at (117.02, 10.93) {domain size $d$};
\end{scope}
\begin{scope}
\path[clip] (  0.00,  0.00) rectangle (209.58,126.47);
\definecolor{drawColor}{RGB}{0,0,0}

\node[text=drawColor,rotate= 90.00,anchor=base,inner sep=0pt, outer sep=0pt, scale=  0.90] at (  8.20, 77.48) {$|M'| \mathbin{/} |M|$};
\end{scope}
\begin{scope}
\path[clip] (  0.00,  0.00) rectangle (209.58,126.47);

\path[] ( 49.41,  1.00) rectangle (184.64,  7.18);
\end{scope}
\begin{scope}
\path[clip] (  0.00,  0.00) rectangle (209.58,126.47);
\definecolor{drawColor}{RGB}{247,192,26}

\path[draw=drawColor,line width= 0.5pt,line join=round] ( 54.99,  4.09) -- ( 63.67,  4.09);
\end{scope}
\begin{scope}
\path[clip] (  0.00,  0.00) rectangle (209.58,126.47);
\definecolor{drawColor}{RGB}{247,192,26}
\definecolor{fillColor}{RGB}{247,192,26}

\path[draw=drawColor,line width= 0.4pt,line join=round,line cap=round,fill=fillColor] ( 59.33,  4.09) circle (  1.53);
\end{scope}
\begin{scope}
\path[clip] (  0.00,  0.00) rectangle (209.58,126.47);
\definecolor{drawColor}{RGB}{78,155,133}

\path[draw=drawColor,line width= 0.5pt,dash pattern=on 4pt off 4pt ,line join=round] (117.09,  4.09) -- (125.76,  4.09);
\end{scope}
\begin{scope}
\path[clip] (  0.00,  0.00) rectangle (209.58,126.47);
\definecolor{fillColor}{RGB}{78,155,133}

\path[fill=fillColor] (119.89,  2.56) --
	(122.96,  2.56) --
	(122.96,  5.62) --
	(119.89,  5.62) --
	cycle;
\end{scope}
\begin{scope}
\path[clip] (  0.00,  0.00) rectangle (209.58,126.47);
\definecolor{drawColor}{RGB}{0,0,0}

\node[text=drawColor,anchor=base west,inner sep=0pt, outer sep=0pt, scale=  0.70] at ( 69.25,  1.68) {LVE (ACP)};
\end{scope}
\begin{scope}
\path[clip] (  0.00,  0.00) rectangle (209.58,126.47);
\definecolor{drawColor}{RGB}{0,0,0}

\node[text=drawColor,anchor=base west,inner sep=0pt, outer sep=0pt, scale=  0.70] at (131.34,  1.68) {LVE (ACP $\pm \varepsilon$)};
\end{scope}
\end{tikzpicture}}
		\caption{\emph{employee} instance class.}
		\label{fig:commutative_approx_plot_comp_instance_employee}
	\end{subfigure}
	\begin{subfigure}[t]{0.49\linewidth}
		\centering
		\resizebox{\linewidth}{!}{
\begin{tikzpicture}[x=1pt,y=1pt]
\definecolor{fillColor}{RGB}{255,255,255}
\path[use as bounding box,fill=fillColor,fill opacity=0.00] (0,0) rectangle (209.58,126.47);
\begin{scope}
\path[clip] (  0.00,  0.00) rectangle (209.58,126.47);
\definecolor{drawColor}{RGB}{255,255,255}
\definecolor{fillColor}{RGB}{255,255,255}

\path[draw=drawColor,line width= 0.5pt,line join=round,line cap=round,fill=fillColor] (  0.00,  0.00) rectangle (209.58,126.47);
\end{scope}
\begin{scope}
\path[clip] ( 30.47, 30.49) rectangle (203.58,124.47);
\definecolor{fillColor}{RGB}{255,255,255}

\path[fill=fillColor] ( 30.47, 30.49) rectangle (203.58,124.47);
\definecolor{drawColor}{RGB}{247,192,26}

\path[draw=drawColor,line width= 0.5pt,line join=round] ( 38.34,116.13) --
	( 55.82, 99.69) --
	( 90.80, 86.77) --
	(125.77, 78.05) --
	(160.74, 77.60) --
	(195.71, 77.32);
\definecolor{drawColor}{RGB}{78,155,133}

\path[draw=drawColor,line width= 0.5pt,dash pattern=on 4pt off 4pt ,line join=round] ( 38.34, 90.40) --
	( 55.82, 62.18) --
	( 90.80, 47.26) --
	(125.77, 42.04) --
	(160.74, 40.30) --
	(195.71, 39.23);
\definecolor{drawColor}{RGB}{247,192,26}
\definecolor{fillColor}{RGB}{247,192,26}

\path[draw=drawColor,line width= 0.4pt,line join=round,line cap=round,fill=fillColor] ( 38.34,116.13) circle (  1.53);
\definecolor{fillColor}{RGB}{78,155,133}

\path[fill=fillColor] ( 36.80, 88.87) --
	( 39.87, 88.87) --
	( 39.87, 91.93) --
	( 36.80, 91.93) --
	cycle;
\definecolor{fillColor}{RGB}{247,192,26}

\path[draw=drawColor,line width= 0.4pt,line join=round,line cap=round,fill=fillColor] ( 55.82, 99.69) circle (  1.53);
\definecolor{fillColor}{RGB}{78,155,133}

\path[fill=fillColor] ( 54.29, 60.65) --
	( 57.36, 60.65) --
	( 57.36, 63.72) --
	( 54.29, 63.72) --
	cycle;
\definecolor{fillColor}{RGB}{247,192,26}

\path[draw=drawColor,line width= 0.4pt,line join=round,line cap=round,fill=fillColor] ( 90.80, 86.77) circle (  1.53);
\definecolor{fillColor}{RGB}{78,155,133}

\path[fill=fillColor] ( 89.26, 45.72) --
	( 92.33, 45.72) --
	( 92.33, 48.79) --
	( 89.26, 48.79) --
	cycle;
\definecolor{fillColor}{RGB}{247,192,26}

\path[draw=drawColor,line width= 0.4pt,line join=round,line cap=round,fill=fillColor] (125.77, 78.05) circle (  1.53);
\definecolor{fillColor}{RGB}{78,155,133}

\path[fill=fillColor] (124.23, 40.51) --
	(127.30, 40.51) --
	(127.30, 43.58) --
	(124.23, 43.58) --
	cycle;
\definecolor{fillColor}{RGB}{247,192,26}

\path[draw=drawColor,line width= 0.4pt,line join=round,line cap=round,fill=fillColor] (160.74, 77.60) circle (  1.53);
\definecolor{fillColor}{RGB}{78,155,133}

\path[fill=fillColor] (159.21, 38.76) --
	(162.27, 38.76) --
	(162.27, 41.83) --
	(159.21, 41.83) --
	cycle;
\definecolor{fillColor}{RGB}{247,192,26}

\path[draw=drawColor,line width= 0.4pt,line join=round,line cap=round,fill=fillColor] (195.71, 77.32) circle (  1.53);
\definecolor{fillColor}{RGB}{78,155,133}

\path[fill=fillColor] (194.18, 37.69) --
	(197.25, 37.69) --
	(197.25, 40.76) --
	(194.18, 40.76) --
	cycle;
\end{scope}
\begin{scope}
\path[clip] (  0.00,  0.00) rectangle (209.58,126.47);
\definecolor{drawColor}{RGB}{0,0,0}

\path[draw=drawColor,line width= 0.5pt,line join=round] ( 30.47, 30.49) --
	( 30.47,124.47);

\path[draw=drawColor,line width= 0.5pt,line join=round] ( 31.60,122.50) --
	( 30.47,124.47) --
	( 29.33,122.50);
\end{scope}
\begin{scope}
\path[clip] (  0.00,  0.00) rectangle (209.58,126.47);
\definecolor{drawColor}{gray}{0.30}

\node[text=drawColor,anchor=base east,inner sep=0pt, outer sep=0pt, scale=  0.80] at ( 26.42, 32.01) {0.00};

\node[text=drawColor,anchor=base east,inner sep=0pt, outer sep=0pt, scale=  0.80] at ( 26.42, 52.35) {0.25};

\node[text=drawColor,anchor=base east,inner sep=0pt, outer sep=0pt, scale=  0.80] at ( 26.42, 72.69) {0.50};

\node[text=drawColor,anchor=base east,inner sep=0pt, outer sep=0pt, scale=  0.80] at ( 26.42, 93.04) {0.75};

\node[text=drawColor,anchor=base east,inner sep=0pt, outer sep=0pt, scale=  0.80] at ( 26.42,113.38) {1.00};
\end{scope}
\begin{scope}
\path[clip] (  0.00,  0.00) rectangle (209.58,126.47);
\definecolor{drawColor}{gray}{0.20}

\path[draw=drawColor,line width= 0.5pt,line join=round] ( 28.22, 34.77) --
	( 30.47, 34.77);

\path[draw=drawColor,line width= 0.5pt,line join=round] ( 28.22, 55.11) --
	( 30.47, 55.11);

\path[draw=drawColor,line width= 0.5pt,line join=round] ( 28.22, 75.45) --
	( 30.47, 75.45);

\path[draw=drawColor,line width= 0.5pt,line join=round] ( 28.22, 95.79) --
	( 30.47, 95.79);

\path[draw=drawColor,line width= 0.5pt,line join=round] ( 28.22,116.13) --
	( 30.47,116.13);
\end{scope}
\begin{scope}
\path[clip] (  0.00,  0.00) rectangle (209.58,126.47);
\definecolor{drawColor}{RGB}{0,0,0}

\path[draw=drawColor,line width= 0.5pt,line join=round] ( 30.47, 30.49) --
	(203.58, 30.49);

\path[draw=drawColor,line width= 0.5pt,line join=round] (201.61, 29.36) --
	(203.58, 30.49) --
	(201.61, 31.63);
\end{scope}
\begin{scope}
\path[clip] (  0.00,  0.00) rectangle (209.58,126.47);
\definecolor{drawColor}{gray}{0.20}

\path[draw=drawColor,line width= 0.5pt,line join=round] ( 38.34, 28.24) --
	( 38.34, 30.49);

\path[draw=drawColor,line width= 0.5pt,line join=round] ( 55.82, 28.24) --
	( 55.82, 30.49);

\path[draw=drawColor,line width= 0.5pt,line join=round] ( 90.80, 28.24) --
	( 90.80, 30.49);

\path[draw=drawColor,line width= 0.5pt,line join=round] (125.77, 28.24) --
	(125.77, 30.49);

\path[draw=drawColor,line width= 0.5pt,line join=round] (160.74, 28.24) --
	(160.74, 30.49);

\path[draw=drawColor,line width= 0.5pt,line join=round] (195.71, 28.24) --
	(195.71, 30.49);
\end{scope}
\begin{scope}
\path[clip] (  0.00,  0.00) rectangle (209.58,126.47);
\definecolor{drawColor}{gray}{0.30}

\node[text=drawColor,anchor=base,inner sep=0pt, outer sep=0pt, scale=  0.80] at ( 38.34, 20.93) {2};

\node[text=drawColor,anchor=base,inner sep=0pt, outer sep=0pt, scale=  0.80] at ( 55.82, 20.93) {4};

\node[text=drawColor,anchor=base,inner sep=0pt, outer sep=0pt, scale=  0.80] at ( 90.80, 20.93) {8};

\node[text=drawColor,anchor=base,inner sep=0pt, outer sep=0pt, scale=  0.80] at (125.77, 20.93) {12};

\node[text=drawColor,anchor=base,inner sep=0pt, outer sep=0pt, scale=  0.80] at (160.74, 20.93) {16};

\node[text=drawColor,anchor=base,inner sep=0pt, outer sep=0pt, scale=  0.80] at (195.71, 20.93) {20};
\end{scope}
\begin{scope}
\path[clip] (  0.00,  0.00) rectangle (209.58,126.47);
\definecolor{drawColor}{RGB}{0,0,0}

\node[text=drawColor,anchor=base,inner sep=0pt, outer sep=0pt, scale=  0.90] at (117.02, 10.93) {domain size $d$};
\end{scope}
\begin{scope}
\path[clip] (  0.00,  0.00) rectangle (209.58,126.47);
\definecolor{drawColor}{RGB}{0,0,0}

\node[text=drawColor,rotate= 90.00,anchor=base,inner sep=0pt, outer sep=0pt, scale=  0.90] at (  8.20, 77.48) {$|M'| \mathbin{/} |M|$};
\end{scope}
\begin{scope}
\path[clip] (  0.00,  0.00) rectangle (209.58,126.47);

\path[] ( 49.41,  1.00) rectangle (184.64,  7.18);
\end{scope}
\begin{scope}
\path[clip] (  0.00,  0.00) rectangle (209.58,126.47);
\definecolor{drawColor}{RGB}{247,192,26}

\path[draw=drawColor,line width= 0.5pt,line join=round] ( 54.99,  4.09) -- ( 63.67,  4.09);
\end{scope}
\begin{scope}
\path[clip] (  0.00,  0.00) rectangle (209.58,126.47);
\definecolor{drawColor}{RGB}{247,192,26}
\definecolor{fillColor}{RGB}{247,192,26}

\path[draw=drawColor,line width= 0.4pt,line join=round,line cap=round,fill=fillColor] ( 59.33,  4.09) circle (  1.53);
\end{scope}
\begin{scope}
\path[clip] (  0.00,  0.00) rectangle (209.58,126.47);
\definecolor{drawColor}{RGB}{78,155,133}

\path[draw=drawColor,line width= 0.5pt,dash pattern=on 4pt off 4pt ,line join=round] (117.09,  4.09) -- (125.76,  4.09);
\end{scope}
\begin{scope}
\path[clip] (  0.00,  0.00) rectangle (209.58,126.47);
\definecolor{fillColor}{RGB}{78,155,133}

\path[fill=fillColor] (119.89,  2.56) --
	(122.96,  2.56) --
	(122.96,  5.62) --
	(119.89,  5.62) --
	cycle;
\end{scope}
\begin{scope}
\path[clip] (  0.00,  0.00) rectangle (209.58,126.47);
\definecolor{drawColor}{RGB}{0,0,0}

\node[text=drawColor,anchor=base west,inner sep=0pt, outer sep=0pt, scale=  0.70] at ( 69.25,  1.68) {LVE (ACP)};
\end{scope}
\begin{scope}
\path[clip] (  0.00,  0.00) rectangle (209.58,126.47);
\definecolor{drawColor}{RGB}{0,0,0}

\node[text=drawColor,anchor=base west,inner sep=0pt, outer sep=0pt, scale=  0.70] at (131.34,  1.68) {LVE (ACP $\pm \varepsilon$)};
\end{scope}
\end{tikzpicture}}
		\caption{\emph{epidemic} instance class.}
		\label{fig:commutative_approx_plot_comp_instance_epidemic}
	\end{subfigure}
	\caption{Compression ratio $\abs{M'} \mathbin{/} \abs{M}$ between the compressed model $M'$ and the original model $M$ broken down per instance class, averaged over the number of commutative factors~$k$ and the tolerance~$\varepsilon$.}
	\label{fig:commutative_approx_plot_comp_instance}
\end{figure}
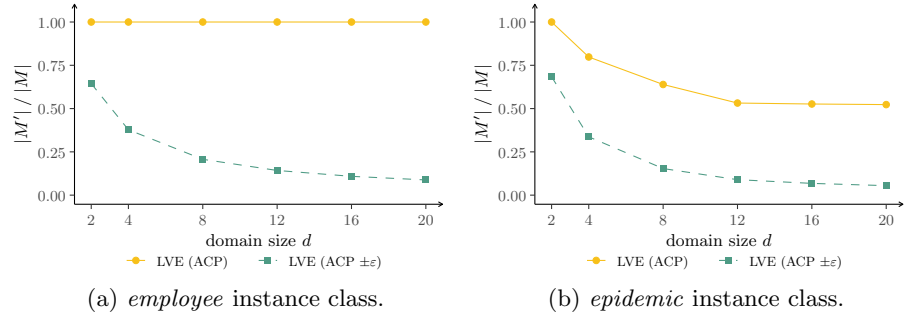

Interestingly, as shown in \cref{fig:commutative_approx_plot_comp_instance}, \ac{acp} is not able to obtain any compression at all for some of the input \acp{fg}: The input instances are generated from two different classes of \acp{fg}, \emph{employee} and \emph{epidemic}, which have different structural properties, and exact \ac{acp} is only able to achieve compression for the \emph{employee} instance class, while it fails to achieve any compression for the \emph{epidemic} instance class.
This particular behaviour highlights that even if there is just a single $\varepsilon$-commutative factor in a given \ac{fg}, exact \ac{acp} may fail to find any compression at all (although there are other parts in the input \ac{fg} that could be compressed and that do not involve commutativity), which is due to the propagation of information throughout \ac{acp}'s colour passing procedure.
\end{document}